%% file: main.tex
\documentclass[11pt,a4paper,oneside,openright]{book}
\pdfoutput=1
 
\usepackage{geometry}
\usepackage{amsfonts} 
\usepackage{amsmath}
\usepackage[T1]{fontenc}
\usepackage{txfonts}
\usepackage{microtype}
\usepackage[usenames,dvipsnames]{color}
\usepackage[english]{babel}
\makeatletter 
\adddialect\l@ENGLISH\l@english
\makeatother
\usepackage{multirow}
\usepackage[svgnames,table]{xcolor}
\definecolor{redSap}{rgb}{0.50980392156,0.14117647058,0.2}
\definecolor{lightgraySap}{rgb}{0.84313725490,0.82745098039,0.78039215687}
\definecolor{blueSap}{rgb}{0,0.40392156862,0.47058823529}
\usepackage[pdftex]{graphicx}
\usepackage{url}
\usepackage[numbers,square,sort&compress]{natbib}
\usepackage{lettrine}
\usepackage{setspace}
\usepackage{acronym}
\usepackage{fancyhdr,textcase}
\usepackage[notquote]{hanging} 
\usepackage[titles]{tocloft}
\addto\captionsenglish{
  \renewcommand{\contentsname}%
    {Table of Contents}%
}

\usepackage{minitoc}

\usepackage[caption=false,font=footnotesize]{subfig}
\usepackage{booktabs}

\usepackage{caption}
\usepackage{listings}
\include{settings/listings-settings}

\include{settings/heading-settings}

\include{settings/theorem-settings}

\usepackage[colorlinks=true,linkcolor=blueSap,citecolor=blueSap]{hyperref}

\include{settings/algorithms-settings}

\newcommand{\vect}[1]{\mathbf{#1}}

\DeclareMathSymbol{\R}{\mathalpha}{AMSb}{"52}

\newcommand{\norm}[2][2]{\left\lVert #2 \right\rVert_{#1}}

\DeclareMathOperator*{\argmin}{arg\,min}

\newcommand{\boldbeta}{\ensuremath{\boldsymbol{\beta}}}

\newcommand\blfootnote[1]{%
  \begingroup
  \renewcommand\thefootnote{}\footnote{#1}%
  \addtocounter{footnote}{-1}%
  \endgroup
}

\begin{document}
\dominitoc

\input{settings/fancyhdr-settings.tex}

\pagestyle{empty}
\frontmatter
\include{frontmatter}

\pagestyle{starred}
\onehalfspacing

\include{chapters/abstract}

\begin{spacing}{1.2}
\tableofcontents
\end{spacing}

\singlespacing
\chapter*{List of Acronyms}
\markboth{List of Acronyms}{List of Acronyms}
\addcontentsline{toc}{chapter}{\bf List of Acronyms} \mtcaddchapter
\newcommand{\acrospace}{\hspace{1.5em}}
\begin{acronym}[GD/SGD]
\acro{AMC}{\acrospace Automatic Music Classification}
\acro{ANN}{\acrospace Artificial Neural Network}
\acro{ATC}{\acrospace Adapt-Then-Combine}
\acro{BP}{\acrospace Back-Propagation}
\acro{BRLS}{\acrospace Blockwise Recursive Least Square (see also RLS)}
\acro{CTA}{\acrospace Combine-Then-Adapt (see also ATC)}
\acro{DA}{\acrospace Diffusion Adaptation}
\acro{DAC}{\acrospace Decentralized Average Consensus}
\acro{DAI}{\acrospace Distributed Artificial Intelligence}
\acro{DF}{\acrospace Diffusion Filtering}
\acro{DGD}{\acrospace Distributed Gradient Descent (see also DA)}
\acro{DL}{\acrospace Distributed Learning}
\acro{DSO}{\acrospace Distributed Sum Optimization}
\acro{EDM}{\acrospace Euclidean Distance Matrix}
\acro{ESN}{\acrospace Echo State Network}
\acro{ESP}{\acrospace Echo State Property}
\acro{FL}{\acrospace Functional Link}
\acro{GD/SGD}{\acrospace (Stochastic) Gradient Descent}
\acro{HP}{\acrospace Horizontally Partitioned (see also VP)}
\acro{KAF}{\acrospace Kernel Adaptive Filtering}
\acro{KRR}{\acrospace Kernel Ridge Regression}
\acro{LASSO}{\acrospace Least Angle Shrinkage and Selection Operator}
\acro{LIP}{\acrospace Linear-In-the-Parameters}
\acro{LMS}{\acrospace Least Mean Square}
\acro{LRR}{\acrospace Linear Ridge Regression (see also KRR)}
\acro{MEB}{\acrospace Minimum Enclosing Ball}
\acro{MFCC}{\acrospace Mel Frequency Cepstral Coefficient}
\acro{MIR}{\acrospace Music Information Retrieval}
\acro{ML}{\acrospace Machine Learning}
\acro{MLP}{\acrospace Multilayer Perceptron}
\acro{MR}{\acrospace Manifold Regularization}
\acro{MSD}{\acrospace Mean-Squared Deviation}
\acro{MSE}{\acrospace Mean-Squared Error}
\acro{NEXT}{\acrospace In-Network Nonconvex Optimization}
\acro{NRMSE}{\acrospace Normalized Root Mean-Squared Error (see also MSE)}
\acro{P2P}{\acrospace Peer-to-Peer}
\acro{PSD}{\acrospace Positive Semi-Definite}
\acro{QP}{\acrospace Quadratic Programming}
\acro{RBF}{\acrospace Radial Basis Function}
\acro{RKHS}{\acrospace Reproducing Kernel Hilbert Space}
\acro{RLS}{\acrospace Recursive Least Square}
\acro{RNN}{\acrospace Recurrent Neural Network}
\acro{RVFL}{\acrospace Random Vector Functional-Link (see also FL)}
\acro{SAF}{\acrospace Spline Adaptive Filter}
\acro{SL}{\acrospace}{Supervised Learning}
\acro{SSL}{\acrospace Semi-Supervised Learning}
\acro{SV}{\acrospace Support Vector (see also SVM)}
\acro{S$^3$VM}{\acrospace Semi-Supervised Support Vector Machine}
\acro{SVM}{\acrospace Support Vector Machine}
\acro{VP}{\acrospace Vertically Partitioned (see also HP)}
\acro{WLS}{\acrospace Weighted Least Square}
\acro{WSN}{\acrospace Wireless Sensor Network}
\end{acronym}
\onehalfspacing

\makeatletter
     \renewcommand*\l@figure{\@dottedtocline{1}{1em}{2.9em}}
\makeatother

\listoffigures
\addcontentsline{toc}{chapter}{\bf List of Figures} \mtcaddchapter 

\listoftables
\addcontentsline{toc}{chapter}{\bf List of Tables} \mtcaddchapter 

\listofalgorithms
\addcontentsline{toc}{chapter}{\bf List of Algorithms} \mtcaddchapter 

\mainmatter

\pagestyle{standard}
\renewcommand{\chaptermark}[1]{\markboth{#1}{}}

\include{chapters/chapter1-introduction}

\part{Background Material}
\include{chapters/chapter2-single_agent_sl}
\include{chapters/chapter3-multi_agent_sl}

\part{Distributed Training Algorithms for RVFL Networks}
\include{chapters/chapter4-distributed_rvfl}
\include{chapters/chapter5-dist_rvfl_sequential}
\include{chapters/chapter6-dist_rvfl_vertical_partitioning}

\part{Distributed Semi-Supervised Learning}
\include{chapters/chapter7-dist_ssl}
\include{chapters/chapter8-dist_ssl_svm}

\part{Distributed Learning from Time-Varying Data}
\include{chapters/chapter9-dist_esn}
\include{chapters/chapter10-dist_saf}

\part{Conclusions and Future Works}
\include{chapters/chapter11-conclusions}

\pagestyle{starred}

\appendix
\addcontentsline{toc}{part}{\bf Appendices}
\part*{Appendices}
\include{chapters/appendixA-graph_theory}
\include{chapters/appendixB-software}

\include{chapters/acknowledgments}

\renewcommand{\bibname}{References}
\bibliographystyle{IEEEtranS}
\begin{footnotesize}
\bibliography{biblio}
\end{footnotesize}
\addcontentsline{toc}{chapter}{\bf References}\mtcaddchapter 

\end{document}

%% file: settings/listings-settings.tex
 \definecolor{MATLABgreen}{RGB}{28,172,0} 
 \definecolor{MATLABlilas}{RGB}{170,55,241}
 
\DeclareCaptionFont{white}{\color{white}}
\DeclareCaptionFormat{listing}{\colorbox[cmyk]{0.35, 0.25, 0.25, 0.3}{\parbox{\textwidth}{\hspace{15pt}#1#2#3}}}

%% file: settings/theorem-settings.tex
\usepackage[framed,amsmath,thmmarks]{ntheorem}
\usepackage{framed}
\usepackage{tikz}
\usetikzlibrary{shadows}

\theoremclass{Theorem}
\theoremstyle{break}

\tikzstyle{thmbox} = [rectangle, rounded corners, draw=black,
  fill=redSap!14, inner sep=10pt]

\newshadedtheorem{definition}{Definition}
\newshadedtheorem{theorem}{Theorem}
\newshadedtheorem{lemma}{Lemma}

\theoremsymbol{\openbox}
\newshadedtheorem{proof}{Proof}

%% file: settings/algorithms-settings.tex
\usepackage[chapter]{algorithm}
\usepackage{algpseudocode}

\let\Algorithm\algorithm
\renewcommand\algorithm[1][]{\Algorithm[#1]\setstretch{1.2}}

\algrenewcommand\algorithmicrequire{\textbf{Inputs:}}
\algrenewcommand\algorithmicensure{\textbf{Output:}}

\floatstyle{plain}
\newfloat{myalgo}{thbp}{mya}

\newenvironment{AlgorithmCustomWidth}[1][h]%
{\begin{myalgo}[#1]
\centering
\begin{minipage}{12.5cm}
\begin{algorithm}[H]}%
{\end{algorithm}
\end{minipage}
\end{myalgo}}

%% file: settings/fancyhdr-settings.tex
\fancypagestyle{standard}{%
	\fancyhf{}
	\fancyhead[L]{\chaptername~\thechapter}
	\fancyhead[R]{\scshape\MakeTextUppercase{\leftmark}}
	\fancyfoot[C]{\thepage}
	\renewcommand{\headrulewidth}{0.4pt}
}

\fancypagestyle{starred}{%
	\fancyhf{}
	\fancyhead[L]{}
	\fancyhead[R]{\scshape\MakeTextUppercase{\leftmark}}
	\fancyfoot[C]{\thepage}
	\renewcommand{\headrulewidth}{0.4pt}
}

%% file: frontmatter.tex

\newcommand{\HRule}{\rule{\linewidth}{0.4mm}}

\thispagestyle{empty}
\frontmatter
\begin{titlepage}
 \begin{center}
 	 \vspace*{2em}
     \includegraphics[scale=0.7]{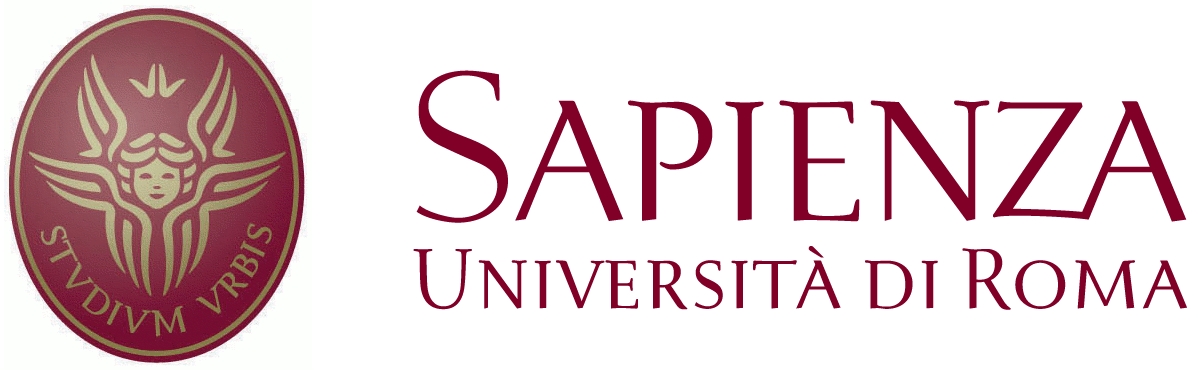}\\
     \vspace{2em}
     {\large \textsc{Sapienza University of Rome \\ Department of Information Engineering, \\Electronics and Telecommunications}}\\
     \vspace{4em}
     {\Large \textsc{PhD Thesis}}\\
     \vspace{2em}
     \HRule \\[0.4cm]
     {\huge \textsc{\textbf{Distributed Supervised Learning}}}\\
     \vspace{1em}
     {\huge \textsc{\textbf{using Neural Networks}}}
     \HRule \\[0.4cm]
     \vspace{2em}
     {\normalsize \textsc{Dissertation submitted in partial fulfillment of the \\ requirements for the degree of Doctor of Philosophy in \\ Information and Communication Engineering \\ XXVIII cycle}}\\     
 \end{center}

\vskip 1.3cm
  \begin{center}
    \begin{tabular}{c c c c c c c c}
      Supervisor & & & & & & & Candidate \\[0.2cm]
      \large{Prof. Aurelio Uncini} & & & & & & & \large{Simone Scardapane}\\[0.4cm]
    \end{tabular}
  \end{center}

\vskip 0.7cm
\begin{center}
  \normalsize Rome, Italy \\ April 2016
\end{center}

\vfill

\end{titlepage}

\clearpage

%% file: chapters/abstract.tex
\chapter*{Abstract}
\addcontentsline{toc}{chapter}{\bf Abstract} \mtcaddchapter 

\begin{quotation}
\noindent Distributed learning is the problem of inferring a function in the case where training data is distributed among multiple geographically separated sources. Particularly, the focus is on designing learning strategies with low computational requirements, in which communication is restricted only to neighboring agents, with no reliance on a centralized authority. In this thesis, we analyze multiple distributed protocols for a large number of neural network architectures. The first part of the thesis is devoted to a definition of the problem, followed by an extensive overview of the state-of-the-art. Next, we introduce different strategies for a relatively simple class of single layer neural networks, where a linear output layer is preceded by a nonlinear layer, whose weights are stochastically assigned in the beginning of the learning process. We consider both batch and sequential learning, with horizontally and vertically partitioned data. In the third part, we consider instead the more complex problem of semi-supervised distributed learning, where each agent is provided with an additional set of unlabeled training samples. We propose two different algorithms based on diffusion processes for linear support vector machines and kernel ridge regression. Subsequently, the fourth part extends the discussion to learning with time-varying data (e.g. time-series) using recurrent neural networks. We consider two different families of networks, namely echo state networks (extending the algorithms introduced in the second part), and spline adaptive filters. Overall, the algorithms presented throughout the thesis cover a wide range of possible practical applications, and lead the way to numerous future extensions, which are briefly summarized in the conclusive chapter.
\end{quotation}

\clearpage

%% file: chapters/chapter1-introduction.tex
\chapter{Introduction}
\label{chap:introduction}

\lettrine[lines=2]{S}{upervised} learning (SL) is the task of automatically inferring a mathematical function, starting from a finite set of examples \citep{Hastie2009}. Together with unsupervised learning and reinforcement learning, it is one of the three main subfields of machine learning (ML). Its roots as a scientific discipline can be traced back to the introduction of the first fully SL algorithms, namely the perceptron rule around $1960$ \citep{Rosenblatt1958}, and the $k$-nearest neighbors ($k$-NN) in $1967$ \citep{Cover1967}.\footnote{While this is a generally accepted convention, one may easily choose earlier works to denote a starting point, such as the work by R. A. Fisher on linear discriminant analysis in $1936$ \cite{Fisher1936}.} The perceptron, in particular, became the basis for a wider family of models, which are known today as artificial neural networks (ANNs). ANNs model the unknown desired relation using the interconnection of several building blocks, denoted as artificial neurons, which are loosely inspired to the biological neuron. Over the last decades, hundreds of variants of ANNs, and associated learning algorithms, have been proposed. Their development was sparked by a few fundamental innovations, including the Widrow-Hoff algorithm in $1960$ \citep{Widrow1990}, the popularization of the back-propagation (BP) rule in $1986$ \cite{Rumelhart1986} (and its later extension for dynamical systems \citep{Werbos1990}), the support vector machine (SVM) in $1992$ \citep{Boser1992}, and additional recent developments on `deep' architectures from $2006$ onwards \citep{Schmidhuber2014}.

As a fundamentally \textit{data-driven} technology, SL has been changed greatly by the impact of the so-called `big data' revolution \citep{Wu2014}. Big data is a general terminology, which is used to refer to any application where data cannot be processed using `conventional' means. As such, big data is not defined axiomatically, but only through its possible characteristics. These include, among others, its volume and speed of arrival. Each of these aspects has influenced SL theory and algorithms \citep{Wu2014}, although in many cases solutions were developed prior to the emergence of the big data paradigm itself. As an example, handling large volumes of data is known in the SL community as the large-scale learning problem \citep{Bottou2008}. This has brought forth multiple developments in parallel solutions for training SL models \citep{Garcia-Pedrajas2012}, particularly with the use of \textit{commodity computing} frameworks such as MapReduce \citep{Chu2007}. Similarly, learning with continuously arriving streaming data is at the center of the subfield of online SL \citep{zinkevich2009slow}.

In this thesis, we focus on another characteristic of several real-world big data applications, namely, their \textit{distributed} nature \citep{Wu2014}. In fact, an ever-increasing number of authors is starting to consider this last aspect as a defining property of big data in many real-world scenarios, which complements the more standard characteristics (e.g. volume, velocity, etc.). As an example, Wu \textit{et al.} \citep{Wu2014} state that ``\textit{autonomous data sources with distributed and decentralized controls are a main characteristic of Big Data applications}''. In particular, we consider the case where training data is distributed among a network of interconnected agents, a setting denoted as distributed learning (DL). If we assume that the agents can communicate with one (or more) coordinating nodes, then it is possible to apply a number of parallel SL algorithm, such as those described above. In this thesis, however, we focus on a more general setting, in which nodes can communicate exclusively with a set of neighboring agents, but none of them is allowed to coordinate in any way the training process. This is a rather general formalization, which subsumes multiple applicative domains, including learning on wireless sensor networks (WSNs) \citep{predd2006distributed,barbarossa2013distributed}, peer-to-peer (P2P) networks \citep{datta2006distributed}, robotic swarms, smart grids, distributed databases \citep{lazarevic2002boosting}, and several others.

More specifically, we develop distributed SL algorithms for multiple classes of ANN models. We consider first the standard SL setting, and then multiple extensions of it, including online learning \citep{zinkevich2009slow}, semi-supervised learning (SSL) \citep{Chapelle2006}, and learning with time-varying signals \citep{Uncini2015}. Due to the generality of our setting, we assume that computational constraints may be present at each agent, such as a sensor in a WSN. Thus, we focus mostly on relatively simple classes of ANNs, where both training and prediction can be performed with possibly low computational capabilities. In particular, in the first and last part of this thesis, we will be concerned with two-layered ANNs, where the weights of the first layer are stochastically assigned from a predefined probability distribution. These include random vector functional-link (RVFLs) \citep{Pao1994,igelnik1995stochastic}, and echo state networks (ESNs) \citep{lukovsevivcius2009reservoir}. Despite this simplification, these models are capable of high accuracies in most real-world settings. Additional motivations and an historical perspective on this are provided in the respective chapters. The rest of the thesis deals with linear SVMs, kernel ridge regression, and single neurons with flexible activation functions.

Another important point to note is that the algorithms developed here do not require the exchange of examples between the nodes, but only of a finite subset of parameters of the ANN itself (and a few auxiliary variables).\footnote{With the exception of Chapter \ref{chap:dist_ssl}, where nodes are allowed to compute a small subset of similarities between their training samples. As is shown in the chapter, privacy can still be kept with the use of privacy-preserving protocols for computing Euclidean distances \citep{liu2006random}.} Thus, they are able to scale easily to large, and possibly time-varying, networks, while keeping a fixed communication overhead. Constraining the exchange of data points is a reasonable assumption in big data scenarios, where datasets are generally large. However, it might be desirable even in other contexts, e.g. whenever privacy concerns are present \citep{verykios2004state}. A prototypical example in this case is that of distributed \textit{medical} databases, where sensible information on each patient must go through strict controls on its diffusion \citep{Torii2011}.

\section*{Structure of the Thesis}
\addcontentsline{toc}{section}{Structure of the Thesis}
A schematic categorization of the algorithms presented in the thesis is given in Fig. \ref{chap1:fig:thesis_organization}. A group corresponding to a specific part of the thesis is shown with a green rectangle, while the algorithms (together with their corresponding chapters) are given with light blue rectangles.
\begin{figure}
\centering
\includegraphics[]{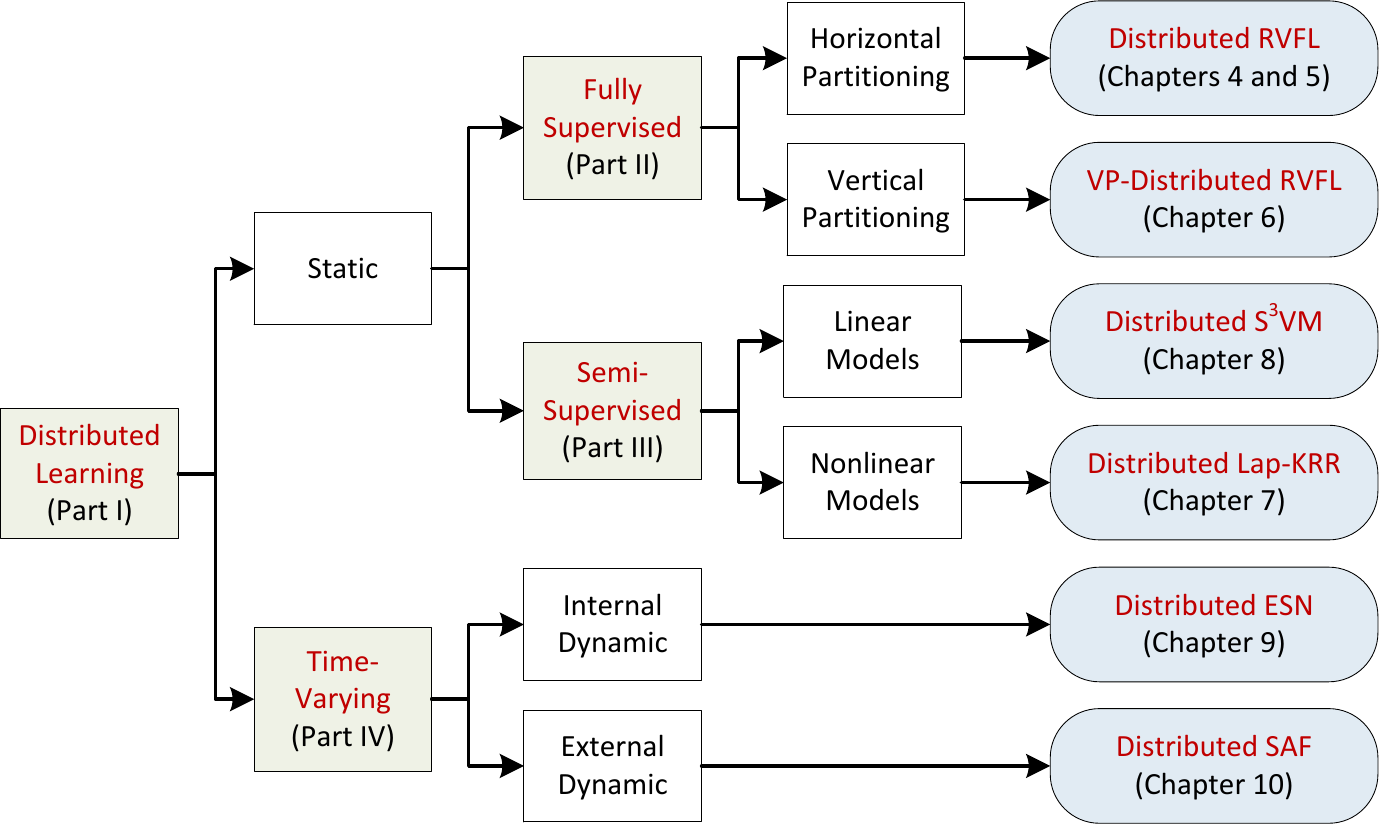}
\caption{Schematic organization of the algorithms presented in the thesis.}
\label{chap1:fig:thesis_organization}
\end{figure}
\newline\newline
\noindent \textbf{Part I} is devoted to introducing the required background material.

\begin{itemize}
\item[] \textbf{Chapter \ref{chap:single_agent_sl}} describes the basic tools of SL theory in the centralized case. It starts by stating formally the SL problem, and then moves on to introduce the ANN models to be used successively.
\item[] \textbf{Chapter \ref{chap:multi_agent_sl}} provides a formal definition of the distributed (multi-agent) SL problem. Additionally, we provide an in-depth overview of previous works dealing with distributed SL with ANN models. The overview combines works coming from multiple research fields, and tries to give an unified discussion by taking a model-based approach.
\end{itemize}

\noindent \textbf{Part II} introduces algorithms for training RVFL networks in the DL setting.

\begin{itemize}
\item[] In \textbf{Chapter \ref{chap:dist_rvfl}}, we develop two fully distributed training algorithms for them. Strength and weaknesses of both approaches are analyzed and compared to the pre-existing literature.
\item[] \textbf{Chapter \ref{chap:dist_rvfl_sequential}} extends one algorithm presented in the previous chapter to the sequential setting, where new data is arriving continuously at every node. Additionally, we present an application to the problem of distributed music classification, and we analyze how different strategies for computing a distributed average can influence the convergence time of the algorithm. 
\item[] \textbf{Chapter \ref{chap:dist_rvfl_vertical_partitioning}} presents a second extension of Chapter \ref{chap:dist_rvfl}, to the situation where each example is partitioned across multiple nodes. Technically, this is known as `vertical partitioning' in the data mining literature.
\end{itemize}

\noindent In \textbf{Part III} we consider distributed SL with the presence of additional \textit{unlabeled} data at every node, thus extending the standard theory of SSL \citep{Chapelle2006}. This part focuses on kernel models. To the best of our knowledge, these are the first algorithms for general purpose distributed SSL.

\begin{itemize}
\item[] In \textbf{Chapter \ref{chap:dist_ssl}} we provide a distributed protocol for a kernel-based algorithm belonging to the manifold regularization (MR) framework \citep{belkin2006manifold}. To this end, we also derive a novel algorithm for decentralized Euclidean distance matrix (EDM) completion, inspired to the theory of diffusion adaptation (DA) \citep{sayed2014adaptive}. 
\item[] In \textbf{Chapter \ref{chap:dist_ssl_next}} we propose two distributed algorithms for a family of semi-supervised linear SVMs, derived from the transductive literature. The first algorithm is again inspired to the DA theory, while the second builds on more recent developments in the field of distributed non-convex optimization.
\end{itemize}

\noindent \textbf{Part IV} considers the more general setting of DL in a \textit{time-varying} scenario.

\begin{itemize}
\item[] In \textbf{Chapter \ref{chap:dist_esn}}, we exploit a well-known recurrent extension of the RVFL network, called ESN \citep{lukovsevivcius2009reservoir}. We leverage on this to provide an extension of one algorithm presented in Chapter \ref{chap:dist_rvfl}. The algorithm is then tested on four large-scale prediction problems. We also present an extension for training ESNs with a sparse output layer.
\item[] Then, \textbf{Chapter \ref{chap:dist_saf}} considers learning from time-varying signals with the use of particular neurons with flexible nonlinear activation functions, called spline adaptive filters (SAF) \citep{Scarpiniti2013}. Again, the theory of DA is used to derive a fully distributed training protocol for SAFs, with local interactions between neighboring nodes. It requires only a small, fixed overhead with respect to a linear counterpart.
\end{itemize}

\noindent Finally, \textbf{Chapter \ref{chap:conclusions}} summarizes the main contributions of this thesis, along with the possible further developments. The thesis is complemented by two appendices. In \textbf{Appendix \ref{app:graph_theory}}, we provide a general overview of algebraic graph theory (which is used to model networks of agents), and of the decentralized average consensus (DAC) protocol \citep{olfati2007consensus,Xiao2007}. DAC is a flexible routine to compute global averages over a network, which is used extensively throughout the thesis, including Chapters \ref{chap:dist_rvfl}-\ref{chap:dist_rvfl_vertical_partitioning} and Chapter \ref{chap:dist_esn}. \textbf{Appendix \ref{app:software}} details the open-source software implementations which can be used to replicate the algorithms presented here.

\section*{Research contributions}
\addcontentsline{toc}{section}{Research Contributions}

Part of this thesis is adapted from material published (or currently under review) on several journals and conferences. Table \ref{chap1:tab:research_contributions} shows a global overview, while an introductory footnote on each chapter provides more information whenever required.

\begin{center}
\begin{table*}[h]
	\caption{Schematic overview of the research contributions related to the thesis}
	{\hfill{}
		\setlength{\tabcolsep}{4pt}
		\renewcommand{\arraystretch}{1.5}
		\begin{footnotesize}
			\begin{tabular}{@{}m{0.1\columnwidth}m{0.15\columnwidth}m{0.65\columnwidth}@{}}
				\toprule
				\multirow{3}{*}[-0.5em]{Part II} & \textbf{Chapter 4} & Published on \textit{Information Sciences} \cite{Scardapane2015} \\
				& \multicolumn{1}{l}{\textbf{Chapter 5}} & Presented at the \textit{2015 International Joint Conference on Neural Networks} \cite{scardapane2015distmusic}; one section is published as a book chapter in \cite{fierimonte2015a} \\
				& \textbf{Chapter 6} & Presented at the \textit{2015 INNS Conference on Big Data} \cite{scardapane2015hetfeatures} \\
				\midrule
				\multirow{2}{*}[-0.5em]{Part III} & \textbf{Chapter 7} & Conditionally accepted at \textit{IEEE Transactions on Neural Networks and Learning Systems} \\
				& \textbf{Chapter 8} & Published in \textit{Neural Networks} \cite{Scardapane2015esn}; final section is in final editorial review at \textit{IEEE Computational Intelligence Magazine} \\
				\midrule
				\multirow{2}{*}[-0.5em]{Part IV} & \textbf{Chapter 9} & Published in \textit{Neural Networks} \cite{scardapane2016distributedsemi} \\
				& \textbf{Chapter 10} & Submitted for presentation at the \textit{2016 European Signal Processing Conference} \\
				\bottomrule
			\end{tabular}
		\end{footnotesize}}
		\hfill{}\vspace{0.3em}
		\label{chap1:tab:research_contributions}
	\end{table*}
\end{center}

\vspace{-3em}
\section*{Notation}
\addcontentsline{toc}{section}{Notation}

Throughout the thesis, vectors are denoted by boldface lowercase letters, e.g. $\vect{a}$, while matrices are denoted by boldface uppercase letters, e.g. $\vect{A}$. All vectors are assumed to be column vectors, with $\vect{a}^T$ denoting the transpose of $\vect{a}$. The notation $A_{ij}$ denotes the $(i,j)$th entry of matrix $\vect{A}$, and similarly for vectors. $\norm[p]{\vect{a}}$ is used for the $L_p$-norm of a generic vector $\vect{a}$. For $p=2$ this is the standard Euclidean norm, while for $p=1$ we have $\norm[1]{\vect{a}} = \sum_i a_i$. The notation $a[n]$ is used to denote dependence with respect to a time-instant, both for time-varying signals (in which case $n$ refers to a time-instant) and for elements in an iterative procedure (in which case $n$ is the iteration's index). The spectral radius of a generic matrix $\vect{A}$ is $\rho(\vect{A}) = \underset{i}{\max}\left\{ | \lambda_i(\vect{A} ) | \right\}$, where $\lambda_i(\vect{A})$ is the $i$th eigenvector of $\vect{A}$. Finally, we use $\vect{A} \succeq 0$ to denote a positive semi-definite (PSD) matrix, i.e. a matrix for which $\vect{x}^T\vect{A}\vect{x} \ge 0$ for any vector $\vect{x}$ of suitable dimensionality.

%% file: chapters/chapter2-single_agent_sl.tex
\chapter{Centralized Supervised Learning}
\label{chap:single_agent_sl}

\minitoc
\vspace{15pt}

\lettrine[lines=2]{T}{his} chapter is devoted to the exposition of the basic concepts of SL in the centralized (single-agent) case. It starts with the formalization of the SL problem, using standard tools from regularization theory, in Section \ref{sec:sl_definition}. Next, we introduce the ANN models (and associated learning algorithms) that are used successively, going from the simplest one (i.e., a linear regression) to a more complex multilayer perceptron (MLP). The exposition focuses on a few fundamental concepts, without going into the details on consistency, stability, and so on. The interested reader is referred to any introductory book on the subject for a fuller treatment, e.g. \citep{Hastie2009}.

\section{General definitions}
\label{sec:sl_definition}

SL is concerned on automatically extracting a mathematical relation between an \textit{input space} $\mathcal{X}$, and an \textit{output space} $\mathcal{Y}$. Throughout the thesis, we assume that the input is a $d$-dimensional vector of real numbers, i.e. $\mathcal{X} \subseteq \R^d$. The input $\vect{x}$ is also called \textit{example} or \textit{pattern}, while a single element $x_i$ of $\vect{x}$ is called a \textit{feature}. For ease of notation, we also assume that the output is a single scalar number, such that $\mathcal{Y} \subseteq \R$. However, everything that follows can be extended straightforwardly to the case of a multi-dimensional output vector. It is worth noting here that many representations can be transformed to a vector of real numbers through suitable pre-processing procedures, including categorical variables, complex inputs, texts, sequences, and so on. Hence, restricting ourselves to this case is a reasonable assumption in most real-world applications. Possible choices for the output space are discussed at the end of this section.

Generally speaking, in a stationary environment, it is assumed that the relation between $\mathcal{X}$ and $\mathcal{Y}$ can be described in its entirety with a joint probability distribution $p(\vect{x} \in \mathcal{X}, y \in \mathcal{Y})$. This probabilistic point of view takes into account the fact that the entries in the input $\vect{x}$ may not identify univocally a single output $y$, that noise may be present in the measurements, and so on. The only information we are given is in the form of a \textit{training} dataset of $N$ samples of the relation:
\begin{definition}[Dataset]
A dataset $D$ of size $N$ is a collection of $N$ samples of the unknown relation, in the form $D = \left\{ \vect{x}_i, y_i \right\}_{i=1}^N$. The set of all datasets of size $N$ is denoted as $\mathcal{D}(N)$.
\end{definition}
Informally, the task of SL it to infer from $D$ a function $f(\cdot)$ such that $f(\vect{x}) \approx y$ for any unseen pair $(\vect{x},y)$ sampled from $p(\vect{x}, y)$.\footnote{The emphasis on \textit{predicting} an output instead of explaining the underlying process distinguishes the ML field from a large part of previous statistics literature, see e.g. \citep{Shmueli2011}.} This process is denoted as \textit{training}. To make this definition more formal, let us assume that the unknown function belongs to a functional space $\mathcal{H}$. We refer to each element $f(\cdot) \in \mathcal{H}$ as an \textit{hypothesis} or, more commonly, as a \textit{model}. Consequently, we call $\mathcal{H}$ the hypothesis (or model) space. Additionally, a non-negative \textit{loss} function $l(y,f(\vect{x})): \mathcal{Y} \times \mathcal{Y} \rightarrow \R^+$ is used to determine the error incurred in estimating $f(\vect{x})$ instead of the true $y$ for any possible pair $(\vect{x},y)$. Using these elements, we are ready to define the (ideal) SL problem.
\begin{definition}[Ideal SL Problem]
Given an hypothesis space $\mathcal{H}$ and a loss function $l(\cdot, \cdot)$, the ideal solution to the SL problem is the function $f(\cdot)$ minimizing the following \textit{expected risk functional}:
\begin{equation}
I_{\text{exp}}[f] = \int l(y, f(\vect{x})) dp(\vect{x},y), \; f \in \mathcal{H} \;.
\label{chap2:eq:exp_risk_functional}
\end{equation}
\end{definition}
The function $f^*$ minimizing Eq. \eqref{chap2:eq:exp_risk_functional} is called Bayes estimator, while $I[f^*]$ is called Bayes risk. Since the probability distribution is unknown, Eq. \eqref{chap2:eq:exp_risk_functional} can be approximated using a generic dataset $D$:
\begin{equation}
I_{\text{emp}}[f] = \sum_{i=1}^N l(y_i, f(\vect{x}_i)) \;.
\label{chap2:eq:emp_risk_functional}
\end{equation}
Eq. \eqref{chap2:eq:emp_risk_functional} is known as the empirical risk functional. It is relatively easy to show that minimizing Eq. \eqref{chap2:eq:emp_risk_functional} instead of Eq. \eqref{chap2:eq:exp_risk_functional} may lead to a risk of \textit{overfitting}, i.e., a function which is not able to generalize efficiently to unseen data. A common solution is to include in the optimization process an additional ``regularizing'' term, imposing reasonable assumptions on the unknown function, such as smoothness, sparsity, and so on. This gives rise to the regularized SL problem.
\begin{definition}[Regularized SL problem]
Given a dataset $D \in \mathcal{D}(N)$, an hypothesis space $\mathcal{H}$, a loss function $l(\cdot, \cdot)$, a regularization functional $\phi[f]: \mathcal{H} \rightarrow \R$, and a scalar coefficient $\lambda > 0$, the regularized SL problem is defined as the minimization of the following functional:
\begin{equation}
I_{\text{reg}}[f] = \sum_{i=1}^N l(y_i, f(\vect{x}_i)) + \lambda\phi[f] \;.
\label{chap2:eq:reg_risk_functional}
\end{equation}
\end{definition}
Problem in Eq. \eqref{chap2:eq:reg_risk_functional} can be justified, and analyzed, from a wide variety of viewpoints, including the theory of linear inverse problems, statistical learning theory and Bayes' theory \citep{Evgeniou2000,Cucker2002}. Throughout the rest of this thesis, we will be concerned with solving it for different choices of its elements. In particular, we consider models belonging to the class of ANNs. These are briefly summarized in the rest of this chapter, going from the simplest one, a linear neuron trained via least-square regression, to the more complex MLP trained using SGD and the BP rule.

Before this, however, it is necessary to spend a few words on the possible choices for the output space $\mathcal{Y}$. We distinguish two different cases. In a \textit{regression} task, the output can take any real value in a proper subset of $\R$. Conversely, in a \textit{binary classification} task, the output can take only two values, which are customarily denoted as $-1$ and $+1$. More generally, in multi-class classification, the output can assume any value in the set $\left\{1,2,\ldots,M\right\}$, where $M$ is the total number of classes. This problem can be addressed by a proper transformation of the output (if the model allows for a multi-dimensional output), or by converting it to a set of binary classification problems, using well-known strategies \citep{rifkin2004defense}.

\section{ANN models for SL}

\subsection{Linear neuron}
\label{sec:linear_neuron}

The simplest ANN model for SL is given by the linear neuron, which performs a linear combination of its input vector:
\begin{equation}
f(\vect{x}) = \boldbeta^T\vect{x} + b \;,
\label{chap2:eq:linear_neuron}
\end{equation}
where $\boldbeta \in \R^d$ and $b \in \R$. For ease of notation, in the following we drop the bias term $b$, since it can always be included by considering an additional constant unitary input. A standard choice in this case is minimizing the squared loss $l(y,f(\vect{x})) = \left( y - f(\vect{x}) \right)^2$, subject to an $L_2$ regularization term on the weights. This gives rise to the well-known linear ridge regression (LRR) problem.
\begin{definition}[Linear ridge regression]
Given a dataset $D \in \mathcal{D}(N)$, the LRR problem is defined as the solution to the following optimization problem:
\begin{equation}
\min_{\boldbeta \in \R^d} \Biggl\{ \frac{1}{2}\norm{\vect{y} - \vect{X}\boldbeta}^2 + \frac{\lambda}{2}\norm{\boldbeta}^2 \Biggr\} \;,
\label{chap2:eq:lls}
\end{equation}
where $\vect{X} = \left[ \vect{x}_1^T \ldots \vect{x}_N^T \right]^T$ and $\vect{y} = \left[ y_1 \ldots y_N \right]^T$.
\end{definition}
Assuming that $\left( \vect{X}^T\vect{X} + \lambda \vect{I} \right)$ is invertible (which is always true for sufficiently large $\lambda$), where $\vect{I}$ is the identity matrix of suitable dimensionality, the solution of the LRR problem can be expressed in closed form as:
\begin{equation}
\boldbeta^* = \left( \vect{X}^T\vect{X} + \lambda\vect{I} \right)^{-1}\vect{X}^T\vect{y} \;.
\label{chap2:eq:lls_solution}
\end{equation}
Work on unregularized LRR dates as far back as Gauss and Legendre \citep{stigler1981gauss}, and it poses a cornerstone on which most of this thesis is built. It is interesting to note that the effect of the regularization term amounts in adding a fixed scalar value on the diagonal of $\vect{X}^T\vect{X}$, which is a common heuristic in linear algebra to ensure both the existence of an inverse matrix, and stability in its computation.

Three additional points are worth mentioning here, as they will be used in subsequent chapters. First of all, whenever $N < d$, it is possible to reformulate Eq. \eqref{chap2:eq:lls_solution} in order to obtain a computationally cheaper expression:
\begin{equation}
\boldbeta^* = \vect{X}^T\left( \vect{X}\vect{X}^T + \lambda\vect{I} \right)^{-1}\vect{y} \;.
\label{chap2:eq:lls_solution_alternative}
\end{equation}
Secondly, it is possible to modify the standard LRR problem in order to obtain a \textit{sparse} solution, meaning that only a subset of the entries of the optimal weight vector $\boldbeta^*$ are non-zero. This is achieved by substituting the $L_2$ norm in Eq. \eqref{chap2:eq:lls} with the $L_1$ norm $\norm[1]{\boldbeta}$, which provides a convex approximation to the $L_0$ norm. The resulting algorithm is known as the least absolute shrinkage and selection operator (LASSO) problem \citep{Society2007}. It provides an efficient feature selection strategy, as well as being central to multiple developments in sparse signal processing, including compressed sensing \citep{Candes2008}. While the optimization problem of LASSO cannot be solved in closed form anymore, efficient algorithms are available for its solution.

A third aspect that we briefly consider is the use of the LRR problem in binary classification tasks. In this case, in the testing phase, the obtained linear model in Eq. \eqref{chap2:eq:linear_neuron} is generally binarized using a predefined threshold, making it similar to the original perceptron \citep{Rosenblatt1958}. The squared loss acts as a convex proxy (or, more technically, as a surrogate loss) of the more accurate misclassification error. Other choices for binary classification might be more accurate, including the hinge loss commonly used in SVMs (introduced in Section \ref{sec:kernel_methods}), or the logistic loss \citep{Hastie2009}.

\subsection{Fixed nonlinear projection}
\label{sec:fixed_hidden_layer}

\begin{figure}
\centering
\includegraphics[scale=0.7]{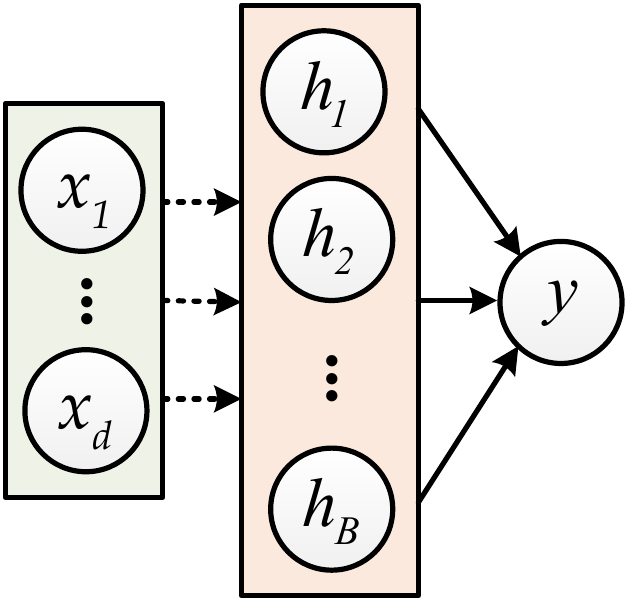}
\caption[Architecture of an ANN with one fixed hidden layer and a linear output layer.]{Architecture of an ANN with one fixed hidden layer and a linear output layer. Fixed connections are shown with a dashed line, while trainable connections are shown with a solid line.}
\label{chap2:fig:RVFL}
\end{figure}

Linear models, as described in the previous section, have been widely investigated in the literature due to their simplicity, particularly in terms of training efficiency. Clearly, their usefulness is limited to cases where the assumption of linearity in the underlying process is reasonable. One possibility of maintaining the general theory of linear models, while at the same time obtaining nonlinear modeling capability, is to add an additional \textit{fixed} layer of nonlinearities in front of the linear neuron. This is shown schematically in Fig. \ref{chap2:fig:RVFL}, where fixed and adaptable connections are shown with dashed and solid lines, respectively. In the context of binary classification, the usefulness of such transformations is known since the seminal work of Cover \citep{cover1965geometrical}.

Mathematically, we consider a model of the form:
\begin{equation}
f(\vect{x}) = \sum_{i=1}^B \beta_i h_i(\vect{x}) = \boldbeta^T\vect{h}(\vect{x}) \;,
\label{chap2:eq:rvfl}
\end{equation}
where $\boldbeta \in \R^B$ and we defined $\vect{h}(\vect{x}) = \left[ h_1(\vect{x}_1), \ldots, h_B(\vect{x}) \right]^T$. Clearly, Eq. \eqref{chap2:eq:rvfl} is equivalent to a linear model over the transformed vector $\vect{h}(\vect{x})$, hence it can be trained by considering the linear methods described in the previous section. Due to their characteristics, these models are widespread in SL, including functional link (FL) networks \cite{Pao1992}, kernel methods (introduced in the next section), radial basis function networks (once the centers are chosen) \citep{park1991universal}, wavelet expansions, and others \citep{gorban1998approximation}. One particular class of FL networks, namely RVFL networks, is introduced in Chapter \ref{chap:dist_rvfl} and further analyzed in Chapters \ref{chap:dist_rvfl_sequential} and \ref{chap:dist_rvfl_vertical_partitioning}.

\subsection{Kernel methods}
\label{sec:kernel_methods}

Although the methods considered in the previous section possess good nonlinear modeling capabilities, they may require an extremely large hidden layer (i.e., large $B$), possibly even infinite. An alternative approach, based on the idea of kernel functions, has been popularized by the introduction of the SVM \citep{Boser1992}.\footnote{The notion of kernel itself was known long before the introduction of the SVM, particularly in statistics and functional analysis, see for example \citep[Section 2.3.3]{Hofmann2008}.} The starting observation is that, for a wide range of feature mappings $\vect{h}(\cdot)$, there exists a function $\mathcal{K}(\cdot, \cdot): \R^d \times \R^d \rightarrow \R$, such that:
\begin{equation}
\mathcal{K}(\vect{x}, \vect{x}') = \vect{h}^T(\vect{x})\vect{h}(\vect{x}') \; \forall \vect{x}, \vect{x}' \in \R^d \;.
\label{chap2:eq:kernel_trick}
\end{equation}
The function $\mathcal{K}(\cdot, \cdot)$ is called a \textit{kernel} function, while Eq. \eqref{chap2:eq:kernel_trick} is known informally as the kernel trick. It allows transforming  any dot product in the transformed space to a function evaluation over the original space. To understand its importance, we first need to introduce a particular class of model spaces.

\begin{definition}[Reproducing Kernel Hilbert Space]
A Reproducing Kernel Hilbert Space (RKHS) $\mathcal{H}$ defined over $\mathcal{X}$ is an Hilbert space of functions such that any evaluation functional defined as:
\begin{equation}
\mathcal{F}_{\vect{x}}[f] = f(\vect{x}) \; \forall f \in \mathcal{H} \;,
\end{equation}
is linear and bounded.
\end{definition}

It can be shown that any RKHS has an associated kernel function. More importantly, solving a regularized SL problem over an RKHS has a fundamental property.

\begin{theorem}[Representer's Theorem]
Consider the regularized SL problem in Eq. \eqref{chap2:eq:reg_risk_functional}. Suppose that $\mathcal{H}$ is an RKHS, and $\phi[f] = \Phi(\norm[\mathcal{H}]{f})$, where $\norm[\mathcal{H}]{f}$ is the norm in the RKHS, and $\Phi(\cdot)$ is a monotonically increasing function. Then, any $f^* \in \mathcal{H}$ minimizing it admits a representation of the form:
\begin{equation}
f^*(\vect{x}) = \sum_{i=1}^N \alpha_i\mathcal{K}(\vect{x}, \vect{x}_i) \;,
\label{chap2:eq:representers_theorem}
\end{equation}
where $\alpha_i, i=1,\ldots,N \in \R$.
\end{theorem}

\begin{proof}
See \citep{scholkopf2001generalized}.
\end{proof}

The representer's theorem shows that an optimization problem over a possibly infinite dimensional RKHS is equivalent to an optimization problem over the finite dimensional set of linear coefficients $\alpha_i, \; i =1, \ldots, N$. SL methods working on RKHSs are known as kernel methods, and we conclude this section by introducing two of them. First, by employing the standard squared loss as error function, and $\Phi(\norm[\mathcal{H}]{f}) = \norm[\mathcal{H}]{f}^2$, we obtain a kernel extension of LRR, which we denote as KRR. Similarly to LRR, the coefficients $\boldsymbol{\alpha}$ of the kernel expansion can be computed in closed form as \citep{Evgeniou2000}:
\begin{equation}
\boldsymbol{\alpha}^* = \left( \vect{K} + \lambda \vect{I} \right)^{-1}\vect{y} \;,
\label{chap2:eq:kernel_ridge_regression}
\end{equation}
where $K_{ij} = \mathcal{K}(\vect{x}_i, \vect{x}_j)$, and $\vect{K}$ is called the kernel matrix. KRR is used in Chapter \ref{chap:dist_ssl} to derive a distributed algorithm for SSL.

In the binary case, an alternative algorithm is given by the SVM, which considers the same squared norm, but substitutes the squared loss with the hinge loss $l(y, f(\vect{x})) = \max\left( 0, 1 - yf(\vect{x}) \right)$ \cite{Steinwart}. In this case, the optimization problem does not allow for a closed-form solution anymore, since it results in a quadratic programming (QP) problem. However, the resulting optimal weight vector is sparse, and the patterns $\vect{x}_i$ corresponding to its non-zero elements are called support vectors (SV). Similar formulations can be obtained for regression, such as the $\nu$-SVM and the $\varepsilon$-SVM \cite{Steinwart}. Due to the sparseness property of SVs, SVMs have been used extensively in the distributed scenario, as detailed more in depth in the next chapter.

\subsection{Multiple adaptable hidden layers}
\label{sec:mlp}

\begin{figure}
\centering
\includegraphics[scale=0.7]{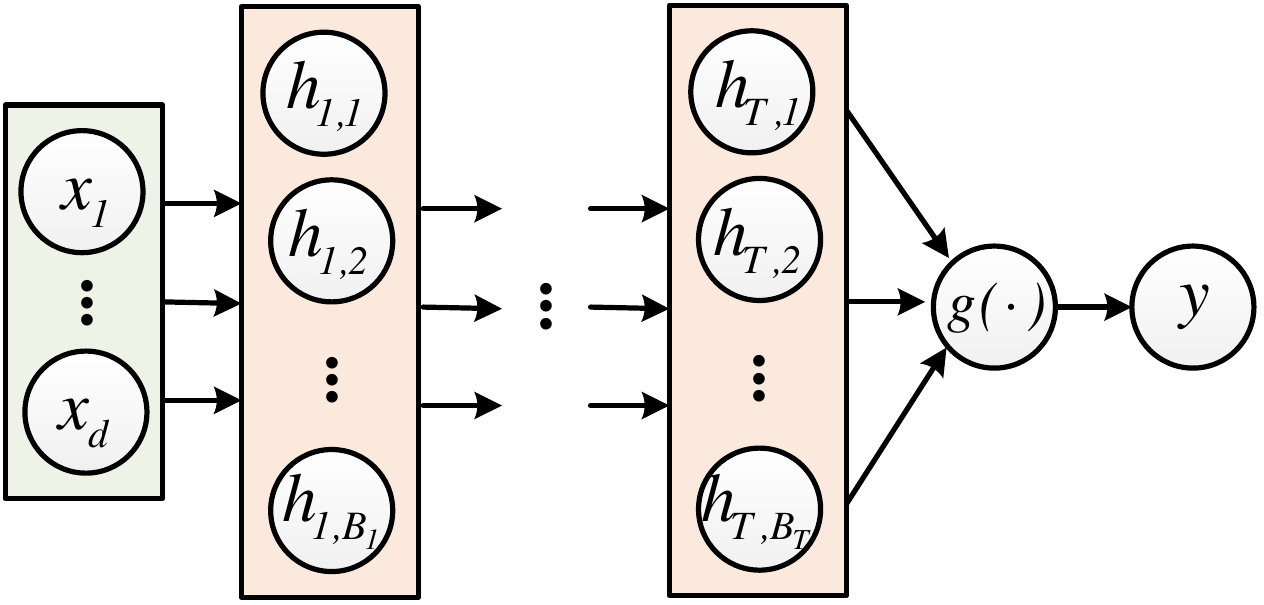}
\caption[Architecture of an MLP with $T$ hidden layers and a linear output layer with a single output neuron.]{Architecture of an MLP with $T$ hidden layers and a linear output layer with a single output neuron. All connections are adaptable.}
\label{chap2:fig:MLP}
\end{figure}

The ANN models discussed in Sections \ref{sec:fixed_hidden_layer} and \ref{sec:kernel_methods} are built on a \textit{single} layer of nonlinearities, followed by an adaptable linear layer. While this is enough in many practical situations (and indeed these methods generally possess universal approximation capabilities), more complex real-world applications may require the presence of multiple layers of adaptable nonlinearities, e.g. in the case of classification of multimedia signals \citep{Schmidhuber2014}. An ANN with these characteristics is called an MLP, and it is shown schematically in Fig. \ref{chap2:fig:MLP}. In this case, the input vector $\vect{x}$ is propagated through $T$ hidden layers. The activation of the $i$th neuron in the $j$th layer, with $i \in 1, \ldots, B_i$, and $j \in 1, \ldots, T$, is given by:
\begin{equation}
h_{i,j}(\vect{x}) = \sum_{t = 1}^{B_{j-1}} w_{t,j,i} h_{t,j-1}(\vect{x}) \;,
\label{chap2:eq:neuron_activation_mlp}
\end{equation}
where $h_{i,j}(\cdot)$ is the scalar activation function of the neuron, and we define axiomatically $B_0$ as $B_0 = d$ and $h_{t,0}$ as $h_{t,0} = x_i$ ($t = 1, \ldots, B_0$). In the one-dimensional output case, the output of the MLP is then given by:
\begin{equation}
y = g\left(\sum_{t=1}^{B_T}  w_{t,T+1,1} h_{t,T}(\vect{x})\right) \;.
\label{chap2:eq:mlp_output}
\end{equation}
Generally speaking, adapting the full set of weights $\left\{w_{t,j,i}\right\}$ results in a non-convex optimization problem, differently from the previous, simpler architectures \citep{Hastie2009}. This is commonly solved with the use of stochastic gradient descent (SGD), or Quasi-Newton optimization methods, where the error at the output layer can be analytically computed, while it is computed recursively (by back-propagating the outer error \citep{Rumelhart1986}) for the hidden layers.

As we stated in Chapter \ref{chap:introduction}, due to the generality of our distributed setting, in this thesis we focus on the simpler methods described previously, as they provide cheaper algorithms for training and prediction. However, we mention some works on DL for MLPs in the next chapter. Additionally, extending the algorithms presented subsequently to MLPs is a natural future research line, as we discuss in Chapter \ref{chap:conclusions}.

%% file: chapters/chapter3-multi_agent_sl.tex
\chapter[Distributed Learning: Formulation and State-of-the-art]{Distributed Learning: \\Formulation and State-of-the-art}
\chaptermark{Distributed Learning Overview}
\label{chap:multi_agent_sl}

\minitoc
\vspace{15pt}

\lettrine[lines=2]{T}{his} chapter is devoted to an analysis of the problem of DL using ANN models. We provide a categorization of DL algorithms, in terms of required network topology, communication capabilities, and data exchange. After this, the biggest part of the chapter is devoted to an overview of previous work on DL using ANN models. For readability, the exposition follows the same structure as the previous chapter, i.e. it moves from the simplest ANN model, corresponding to a linear regression, to the more complex MLP. For each model, we describe relative strengths and weaknesses when applied in a distributed scenario. These comments will also serve as motivating remarks for the algorithms introduced in the next chapters. The review aggregates works coming from multiple interdisciplinary fields, including signal processing, machine learning, distributed databases, and several others. When possible, we group works coming from the same research field, in order to provide coherent pointers to the respective literature.

\section{Formulation of the Problem}
\label{sec:dist_learning_formulation}

In the previous chapter, it was assumed that the training dataset $S$ is available on a centralized location for processing. In many contexts, however, this assumption does not hold. As a motivating example for the following, consider the case of a distributed music classification task on a P2P network. Each peer in the network has access to a personal set of labeled songs, e.g., every user has categorized a certain number of its own songs with respect to a predefined set of musical genres. Clearly, solving efficiently this task requires leveraging over \textit{all} local datasets, since we can safely assume that no single dataset alone is sufficient for obtaining adequate performance. Practically, this means that the peers in the network must implement a suitable training protocol for converging to an optimal solution to this DL task. Other examples of DL abounds, and a few of them will be mentioned successively.

\begin{figure}
\centering
\includegraphics[width=0.7\columnwidth]{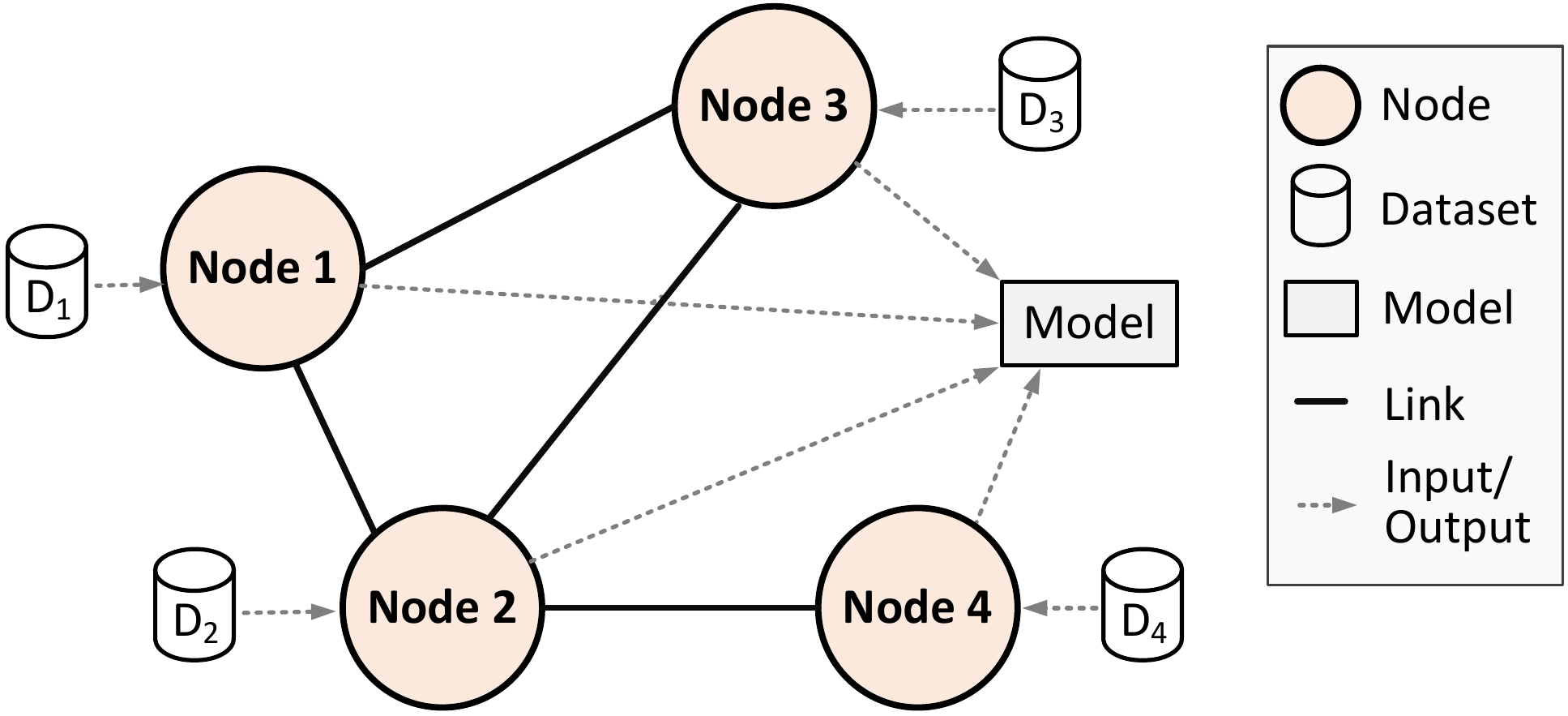}
\caption[Schematic depiction of DL in a network of $4$ agents.]{DL in a network of agents: training data is distributed throughout the nodes, and all of them must converge to the optimal parameters of a single model. For readability, we assume undirected connections between agents.}
\label{chap3:fig:distributed_learning_setting}
\end{figure}

More formally, we consider the setting described schematically in Fig. \ref{chap3:fig:distributed_learning_setting}. We have $L$ agents (or nodes), each of which has access to a local training dataset $D_k \in  \mathcal{D}(N_k)$, such that $\bigcup_{k=1}^L D_k = D$ and $\sum_{k=1}^L N_k = N$.\footnote{In the data mining literature, this is known as `horizontal partitioning' \citep{peteiro2013survey}. Chapter \ref{chap:dist_rvfl_vertical_partitioning} considers the complementary case of `vertical partitioning', where the features of every pattern $\vect{x}$ are distributed throughout the network.} The connectivity of the agents can be described entirely by a matrix $\vect{C} \in \R^{L \times L}$, as detailed in Appendix \ref{app:graph_theory}. Given these elements, we are now ready to provide a formal definition of the DL problem.
\begin{definition}[Distributed learning]
Given $L$ datasets $D_k \in \mathcal{D}(N_k), \; k = 1,\ldots, L$ distributed over a network, an hypothesis space $\mathcal{H}$, a loss function $l(\cdot, \cdot)$, a regularization functional $\phi[f]: \mathcal{H} \rightarrow \R$, and a scalar coefficient $\lambda > 0$, the distributed learning problem is defined as the minimization of the following (joint) functional:
\begin{equation}
I_{\text{dist}}[f] = \sum_{k=1}^L \sum_{\left(\vect{x}_i, y_i\right) \in S_k} l(y_i, f(\vect{x}_i)) + \lambda\phi[f] \;.
\label{chap3:eq:dist_risk_functional}
\end{equation}
\end{definition}

We distinguish between \textit{batch} DL algorithms and \textit{sequential} DL algorithms. In the latter case, each dataset $D_k$ is assumed to be observed in a set of successive batches $D_{k,1}, \ldots, D_{k,T}$, such that $D_k = \bigcup_{i = 1}^T D_{k,i}$. In the extreme case where each batch is composed of a single element, the resulting formulation is closely linked to the distributed adaptive filtering problem \citep{sayed2014adaptive}. New batches may arrive synchronously or asynchronously at every agent, as detailed next. The objective in this case is to produce a sequence of estimates $f_{k,1}, \ldots, f_{k,T}$ converging as rapidly as possible to the global solution of Eq. \eqref{chap3:fig:distributed_learning_setting} computed over the overall dataset.

\section{Categorization of DL algorithms}
\label{sec:categorization_dl_algorithms}

Despite the generality of the DL setting, existing algorithms can be categorized with respect to a few broad characteristics, which are briefly summarized next.

\begin{description}
\item[Coordination] Generally speaking, no node is allowed to coordinate specific aspects of the training process, and we assume that there is no shared memory capability. This lack of centralization is in fact the major difference with respect to prior work on parallel SL \citep{Garcia-Pedrajas2012}.\footnote{Clearly, there are also important overlaps between parallel SL algorithms and the DL problem considered here, such as the Cascade SVM detailed in Section \ref{chap3:sec:dist_svm}.} Still, some DL algorithms may require the presence of a given subset of dynamically chosen nodes aggregating results from their local neighborhood, such as the clusterheads in a WSN or the super-peers in a P2P network \cite{ang2009communication}.
\item[Connectivity] The minimum assumption in DL is that the overall network is connected, i.e. each node can be reached from any other node in a finite number of steps. DL algorithms differentiate themselves on whether they require specific additional properties on the connectivity (e.g. undirected connections). Additionally, some algorithms may assume that the connectivity graph is time-varying.
\item[Communication] Distributed training protocols can be categorized based on the communication infrastructure that is required. In particular, messages can be exchanged via one-hop or multi-hop connectivity. In multi-hop communication (e.g. ip-based protocols), messages can be routed from any node to any other node, while in single-hop communication, nodes can exchange messages only with their neighbors. At the extreme, each node is allowed to communicate with a single other node at every time slot, as in gossip algorithms \citep{boyd2006randomized}. It is easy to understand that multi-hop protocols are not able to efficiently scale to large networks, while they make the design of the algorithm simpler. Similarly, multi-hop communication may not be feasible in particularly unstructured scenarios (e.g. ad-hoc WSNs). This distinction is blurred in some contexts, as it is possible to design broadcast protocols starting from one-hop communication.
\item[Privacy] In our context, a privacy violation refers to the need of exchanging local training patterns to other nodes in the network. Algorithms that are designed to preserve privacy are important for two main aspects. First, datasets are generally large, particularly in big data scenarios, and their communication can easily become the main bottleneck in a practical implementation. Secondly, in some context privacy has to be preserved due to the sensitivity of the data, especially in medical applications \citep{Torii2011}.
\item[Primitives] Algorithms can be categorized according to the specific mathematical primitives that are requested on the network. Some algorithms do not require operations in addition to the one-hop exchange. Others may require the possibility of computing vector-sums over the network, Hamiltonian cycles, or even more complex operations. These primitives can then be implemented differently depending on the specific technology of the network, e.g. a sum implemented via a DAC protocol in a WSN \citep{olfati2007consensus}.
\item[Synchronization] Lastly, the algorithms differentiate themselves on whether synchronization among the different agents is required, e.g. in case of successive optimization steps. Most of the literature makes this assumption, as the resulting protocols are easier to analyze and implement. However, designing asynchronous strategies can lead to enormous speed-ups in terms of computational costs and training time.
\end{description}

\section{Relation to other research fields}
\label{chap3:sec:relation_other_research_fields}

Before continuing on to the state-of-the-art, we spend a few words on the relationships between the DL problem and other research fields. 

First of all, the problem in Eq. \eqref{chap3:eq:dist_risk_functional} is strictly related to a well-known problem in the distributed optimization field, known as distributed sum optimization (DSO). We briefly introduce it here, as its implications will be used extensively in the subsequent sections. Suppose that the $k$th agent must minimize a generic function $J_k(\vect{w})$ parameterized by the vector $\vect{w}$. In SL, the function can represent a specific form of the loss functional in Eq. \eqref{chap2:eq:reg_risk_functional}, minimized over the local dataset $S_k$, and where the vector $\vect{w}$ embodies the parameters of the learning model $h \in \mathcal{H}$. DSO is the problem of minimizing the global joint cost function given by:
\begin{equation}
J(\vect{w}) = \sum_{k=1}^L J_k(\vect{w}) \,.
\label{chap3:eq:dso_cost_function}
\end{equation}
Note the relation between Eq. \eqref{chap3:eq:dso_cost_function} and Eq. \eqref{chap3:eq:dist_risk_functional}. For a single-agent minimizing a differentiable cost function, the most representative algorithm for minimizing it is given by the gradient descent (GD) procedure. In this case, denote by $\vect{w}_k[n]$ the estimate of the single node $k$ at the $n$th time instant. GD computes the minimum of $J_k(\vect{w})$ by iteratively updating the estimate as:
\begin{align}
\vect{w}_k[n+1] & = \vect{w}_k[n] - \eta_k \nabla_{\vect{w}} J(\vect{w}_k[n]) \,,
\label{chap3:eq:gd_step}
\end{align}
where $\eta_k$ is the local step-size at time $k$, whose sequence should be sufficiently small in order to guarantee convergence to the global optimum. Much work on DSO is sparked by the additivity property of the gradient update in the previous equation. In particular, a GD step for the joint cost function in Eq. \eqref{chap3:eq:dso_cost_function} can be computed by summing the gradient contributions from each local node.

Starting from this observation and the seminal work of Tsitsiklis \textit{et al.} \citep{tsitsiklis1986distributed}, a large number of approaches for DSO have been developed. These include \citep{nedic2009distributed,jakovetic2014fast} for convex unconstrained problems, \citep{sundhar2012new,boyd2011distributed,duchi2012dual} for convex constrained problems, and \citep{bianchi2013convergence} for the extension to non-convex problems. Beside GD, representative approaches include subgradient descents \citep{nedic2009distributed}, dual averaging \citep{duchi2012dual}, ADMM \citep{boyd2011distributed}, and others. Many of these algorithms can be (and have been) applied seamlessly to the setting of DL. For simplicity, in the following we mention only works that have been directly applied or conceived for the DL setting.

In signal processing, instead, the problem of distributed \textit{parametric} inference has a long history \cite{predd2006distributed}, and it was revived recently thanks to the interest in large, unstructured WSN networks. Novel approaches in this context are discussed in Section \ref{chap4:sec:diffusion_filtering} (linear distributed filtering) and Section \ref{chap4:sec:distributed_kernel} (distributed kernel-based filtering).

More in general, distributed AI (DAI) and distributed problem solving have always been two major themes in the AI community, particularly due to the diffusion of parallel and concurrent programming paradigms \citep{chaib1992trends}. From a philosophical perspective, this is due also to the realization that ``\textit{a system may be so complicated and contain so much knowledge that it is better to break it down into different cooperative entities in order to obtain more efficiency}'' \citep{chaib1992trends}. Recently, DAI has received renowned attention in the context of multi-agent systems theory.

Finally, distributed learning has also received attention from the data mining \citep{park2002distributed} and P2P fields \citep{datta2006distributed}. However, before $2004$, almost no work was done in this context using ANN models.

\section{State-of-the-art}


\subsection{Distributed linear regression}
\label{sec:dist_linear_neuron}

We start our analysis of the state-of-the-art in DL from the LRR algorithm introduced in Section \ref{sec:linear_neuron}. It is easy to show that this training algorithm is extremely suitable for a distributed implementation. In fact, denote as $\vect{X}_k$ and $\vect{y}_k$ the input matrix and output vector computed with respect to the $k$th local dataset. Eq. \eqref{chap2:eq:lls_solution} can be rewritten as:
\begin{equation}
\boldbeta^* = \left( \sum_{k=1}^L \left( \vect{X}_k^T \vect{X}_k  \right) + \lambda \vect{I} \right)^{-1}\sum_{k=1}^L \left( \vect{X}_k^T\vect{y}_k\right) \;.
\label{chap3:eq:dist_lrr}
\end{equation}
Thus, distributed LRR can be implemented straightforwardly with two sums over the network, the first one on the $d \times d$ matrices $\vect{X}_k^T \vect{X}_k$, the second one on the $d$-dimensional vectors $\vect{X}_k^T\vect{y}_k$. Generally speaking, sums can be considered as primitive on most networks, even the most unstructured ones, e.g. with the use of the DAC protocol introduced in Appendix \ref{appA:sec:consensus}. 

Due to this, the basic idea underlying Eq. \eqref{chap3:eq:dist_lrr} has been discussed multiple times in the literature. In the following, we consider three representative examples. Karr \textit{et al.} \citep{karr2005secure} were among the first to exploit it, with the additional use of a secure summation protocol for ensuring data privacy. In the same paper, secure summation is also used to compute diagnostic statistics, in order to confirm the validity of the linear model.

A similar idea is derived in Xiao \textit{et al.} \citep{xiao2005scheme,xiao2006space}, where it is applied to a generalization of LRR denoted as `weighted least-square' (WLS). In WLS, we assume that each output is corrupted by Gaussian noise with mean zero and covariance matrix $\boldsymbol{\Sigma}$. In the centralized case, in the absence of regularization, the solution of the WLS problem is then given by:
\begin{equation}
\boldbeta^*_{\text{WLS}} = \left( \vect{X}^T \boldsymbol{\Sigma}^{-1} \vect{X} \right)^{-1}\vect{X}^T\boldsymbol{\Sigma}^{-1}\vect{y} \,.
\label{chap4:eq:wls}
\end{equation}
In \citep{xiao2005scheme}, this is solved for a single example at every node using two DAC steps as detailed above. In \citep{xiao2006space}, this is extended to the case of multiple examples arriving asynchronously. In this case, the WLS solution is obtained by interleaving temporal updates with respect to the newly arrived data, with spatial updates corresponding to a single DAC iteration.

Similar concepts are also explored in Bhaduri and Kargupta \citep{bhaduri2008scalable}. As in the previous case, new data is arriving continuously at every node. Differently than before, however, the LRR solution is recomputed with a global sum only when the error over the training set exceeds a predefined threshold, to reduce communication.

The LRR problem has also been solved in a distributed fashion with the use of distributed optimization techniques, including the subgradient algorithm \citep{ram2009incremental} and the ADMM procedure \citep{boyd2011distributed}. These techniques can also be adapted to handle different loss functions for the linear neuron, including the Huber loss \citep{johansson2007simple} and the logistic loss \citep{jakovetic2014fast}.

An alternative approach is followed in \citep{mcwilliams2014loco} for the case where the overall output vector $\vect{y}$ is globally known. In this case, each agent projects its local matrix $\vect{X}_k$ to a lower-dimensional space with the use of random projections. Next, these projections are broadcasted to the rest of the network. By concatenating the resulting matrices, each agent can independently solve the global LRR problem with bounded error.

If we allow the nodes to exchange data points, Balcan \textit{et al.} \citep{balcan2012distributed} and Daum{\'e} III \textit{et al.} \citep{daume2012efficient} independently derive bounds on the number of patterns that must be exchanged between the agents for obtaining a desired level of accuracy in the context of probably approximately correct (PAC) theory. As an example, \citep[Section 7.3]{balcan2012distributed} shows that, if data is separated by a margin $\gamma$, the model is trained using the perceptron rule, and the nodes communicate in a round robin fashion, learning the model requires $\mathcal{O}\left(\frac{1}{\gamma^2}\right)$ rounds. Making additional assumptions on the distribution of the data allows to reduce this bound \citep{balcan2012distributed}.

\subsection{Diffusion filtering and adaptation}
\label{chap4:sec:diffusion_filtering}

Next, we consider the problem of solving the LRR problem in Eq. \eqref{chap2:eq:lls} with continuously arriving data. Additionally, we suppose that the nodes have stringent computational requirements, so that solving multiple times Eq. \eqref{chap2:eq:lls_solution} is infeasible. This setting is closely linked to the problem of adaptive filtering in signal processing \citep{Uncini2015}, where the input typically represents a buffer of the last observed samples of an unknown linear process. Two widespread solutions in the centralized case are the least mean square (LMS) algorithm, which is strictly related to the GD minimization procedure, and the recursive least square (RLS) algorithm, which recursively computes Eq. \eqref{chap2:eq:lls_solution}\footnote{The RLS is formally introduced in Chapter \ref{chap:dist_rvfl_sequential}.} \citep{Uncini2015}.

In the context of distributed filtering, these algorithms were initially extended using incremental gradient updates \citep{sayed2006distributed,lopes2007incremental}, were information on the update steps is propagated on a Hamiltonian cycle over the network. This includes incremental LMS \citep{lopes2007incremental} and incremental RLS \citep{sayed2006distributed}. These methods, however, have a major drawback, in that computing such a cycle is an NP-hard problem. Due to this, an alternative formulation, denoted as diffusion filtering (DF), was popularized in Lopes and Sayed \citep{lopes2008diffusion} for the LMS and in Cattivelli \textit{et al.} \citep{cattivelli2008diffusion} for the RLS. In a DF, local updates are interpolated with `diffusion' steps, where the estimates are locally weighted with information coming from the neighbors. Multiple extensions over this basic scheme have been introduced in the following years, including DF with adaptive combiners \citep{takahashi2010diffusion}, total least-square \citep{arablouei2013diffusion}, sparse models \citep{di2013sparse}, asynchronous networks \citep{zhao2015asynchronous1}, and so on.

The popularity of the DF field has led to its application to the wider problem of DSO, under the name of diffusion adaptation (DA) or distributed gradient descent (DGD) \citep{sayed2014adaptive}. Since DA will be used extensively in Chapters \ref{chap:dist_ssl} and \ref{chap:dist_saf}, we briefly detail it here. DA works by interleaving local gradient descents as in Eq. \eqref{chap3:eq:gd_step} with averaging steps given by:
\begin{align}
\vect{w}_k[n+1] & = C_{kk}\vect{w}_k[n+1] + \sum_{t \in \mathcal{N}_k} C_{kt} \vect{w}_t[n+1] \,,
\end{align}
where the weights $C_{kt}$ have the same meaning as the connectivity matrix of the DAC protocol (see Appendix \ref{appA:sec:consensus}). In fact, the previous equation can be understood as a single DAC step. An example of a diffusion step is shown in Fig. \ref{chap3:fig:diffusion_step_example}. In particular, this strategy is known as adapt-then-combine (ATC), while an equivalent combine-then-adapt (CTA) formulation can be obtained by interchanging the two updates. Different choices of $J_k(\vect{w})$ give rise to different algorithms, including the diffusion LMS and RLS mentioned before. For a recent exposition on the theory of DA, its convergence properties and applications to stochastic optimization, see the monograph by Sayed \citep{sayed2014adaptive}, where the author also mentions the application to a diffusion logistic regression in Section V-C.

\begin{figure}
\centering
\includegraphics[width=0.45\columnwidth]{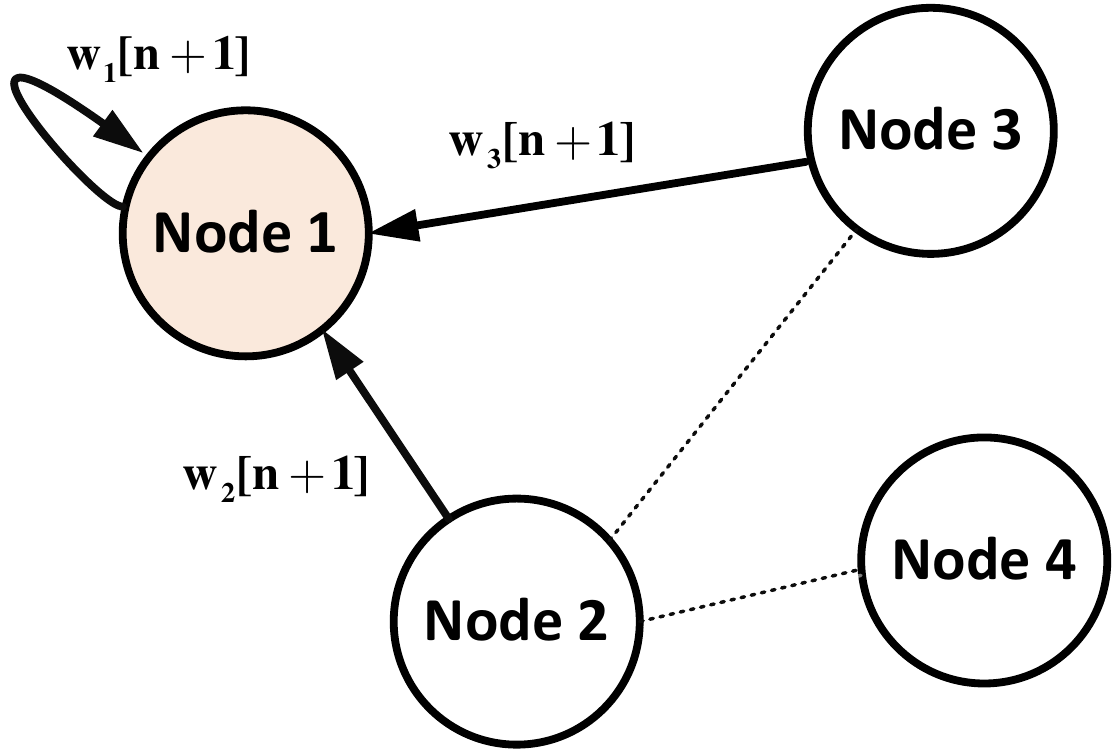}
\caption[Example of a diffusion step for the first node in the $4$-nodes network of Fig. \ref{chap3:fig:distributed_learning_setting}.]{Example of a diffusion step for the first node in the $4$-nodes network of Fig. \ref{chap3:fig:distributed_learning_setting}. Links that are deactivated are shown with dashed lines. Note that node $4$ is not directly connected to node $1$, thus its estimate will only reach it in an indirect way through node $2$.}
\label{chap3:fig:diffusion_step_example}
\end{figure}

Before concluding this section, we mention that distributed linear filters without the use of DF theory have also been proposed in the literature. As an example, Schizas \textit{et al.} \citep{schizas2009distributed} present a distributed LMS, where ADMM is used to enforce consensus on a set of `bridging' sensors. A similar formulation for the RLS is derived in Mateos \textit{et al.} \citep{mateos2009distributed}. An alternative RLS algorithm, which bypasses the need for bridge sensors, is analyzed instead in Mateos and Giannakis \citep{mateos2012distributed}.

\subsection{Distributed sparse linear regression}

Distributed training of a sparse linear method has also been investigated extensively in the literature. Mateos \textit{et al.} \citep{mateos2010distributed} reformulate the problem of LASSO in a separable form, and then solve it by enforcing consensus constraints with the use of the ADMM procedure. They present three different versions, which differ in the amount of computational resources required by the single node. Particularly, in the orthonormal design case, it is shown that the local update step can be computed by an elementary thresholding operation. Mota \textit{et al.} \citep{mota2012distributed} solve in a similar way a closely related problem, denoted as basis pursuit. In \citep[Section V]{mateos2010distributed}, the authors discuss also a distributed cross-validation procedure for selecting an optimal $\lambda$ in a decentralized fashion.

An alternative formulation is presented in Chen and Sayed \citep{chen2012diffusion}, where the $L_1$ norm is approximated with the twice-differentiable regularization term given by:
\begin{equation}
\norm[1]{\boldbeta} \approx \sum_{i=1}^d \sqrt{\beta_i^2 + \varepsilon^2} \,,
\end{equation}
where $\varepsilon$ is a small number. The problem is solved with DA (see Section \ref{chap4:sec:diffusion_filtering}).

A third approach, based on the method of iterative thresholding, is instead presented in Ravazzi \textit{et al.} \citep{ravazzi2015distributed}, for both the LASSO problem and the optimally sparse LRR problem with an $L_0$ regularization term. Results are similar to \citep{mateos2012distributed}, but the algorithm requires significantly less computations at every node.

Much work has been done also in the case of sequential distributed LASSO problems. Liu \textit{et al.} \citep{liu2012diffusion} extend the standard diffusion LMS with the inclusion of $L_0$ and $L_1$ penalties, showing significant improvements with respect to the standard formulation when the underlying vector is sparse. A similar formulation is derived in Di Lorenzo and Sayed \citep{di2013sparse}, with two important differences. First, they consider two different sets of combination coefficients, allowing for a faster rate of convergence. Secondly, they consider an adaptive procedure for selecting an optimal $\lambda$ coefficient.

In the case of sparse RLS, Liu \textit{et al.} \citep{liu2014distributed} present an algorithm framed on the principle of maximum likelihood, with the use of expectation maximization and thresholding operators. An alternative, more demanding formulation, is presented in Barbarossa \textit{et al.} \citep[Section IV-A4]{barbarossa2013distributed}, where the optimization problem is solved with the use of the ADMM procedure.

\subsection{Distributed linear models with a fixed nonlinear projection layer}
\label{sec:dist_fixed_hidden_layer}

We now consider DL with nonlinear ANN models, starting from the linear neuron with a fixed nonlinear projection layer introduced in Section \ref{sec:fixed_hidden_layer}. We already remarked that, in the centralized case, this family of models offers a good compromise between speed of training and nonlinear modeling capabilities. In the distributed scenario, however, their use has been limited to a few cases, which are briefly summarized next. This remark, in fact, offers a substantial motivation for the algorithms introduced in the next chapters.

Hershberger and Kargupta \citep{hershberger2001distributed} analyze a vertically partitioned scenario, where the input is projected using a set of wavelet basis functions. Particularly, they consider dilated and translated instances of the scaling function, denoted as ``box'' functions. The coefficients of the wavelet representation are then transmitted to a fusion center, which is in charge of computing the global LRR solution.

Sun \textit{et al.} \citep{sun2011elm} perform DL in a P2P network with an ensemble of extreme learning machine (ELM) networks. In ELM, the parameters of the nonlinear functions in the hidden layer are stochastically assigned at the beginning of the learning procedure (see Section \ref{chap4:sec:historical_perspective_rvfl_networks}). In \citep{sun2011elm}, an ensemble of ELM functions are handled by a set of `super-peers', using an efficient data structure in order to minimize data exchange over the network. An alternative approach for training an ELM network is presented in Samet and Miri \citep{samet2012privacy}, both for horizontally and vertically partitioned data, which makes use of secure protocols for computing vector products and the SVD decomposition. A third approach is presented in Huang and Li \citep{huang2015distributed}, where the output layer is trained with the use of diffusion LMS and diffusion RLS (see Section \ref{chap4:sec:diffusion_filtering}).

\subsection{Kernel filtering on sensor networks}
\label{chap4:sec:distributed_kernel}

We now begin our analysis of distributed kernel methods, starting with the KRR algorithm described in Section \ref{sec:kernel_methods}. On first glance, this algorithm is not particularly suited for a distributed implementation, as the kernel model in Eq. \eqref{chap2:eq:representers_theorem} requires knowledge of all the local datasets. This is particularly daunting for the implementation of incremental gradient (and subgradient) methods, as is already discussed in Predd \textit{et al.} \citep{predd2006distributed}: ``\textit{In consequence, all the data will ultimately propagate to all the sensors, since exchanging [the examples] is necessary to compute
[the gradient] and hence to share [the model] (assuming that the sensors are preprogrammed with the kernel).}''. Despite this apparent limitation, much work has been done in this context, particularly for non-parametric inference on WSNs.

Possibly the first investigation in this sense was done in Simi\'c \citep{simic2003learning}. In a WSN, in many cases, we can assume that the input $\vect{x}$ represents the geographical coordinates of the sensor itself, e.g. in the case of sensors measuring a specific field. In this case, if we use a translational kernel, i.e. a kernel that depends only on Euclidean distances, we have that $\mathcal{K}(\vect{x}_1, \vect{x}_2) = 0$ for any two sensors which are sufficiently far away. Thus, the resulting kernel matrix $\vect{K}$ is sparse. In \citep{simic2003learning}, each node solves its local KRR model, sending its optimal coefficients to a fusion center, which combines them by taking into consideration the previous observation. A similar procedure without the need for a fusion center is presented in Guestrin \textit{et al.} \citep{guestrin2004distributed}, where the problem is solved with a distributed Gaussian elimination procedure. The approach in \citep{guestrin2004distributed} has strong convergence guarantees and can be used even in cases where the matrix $\vect{K}$ is not sparse, albeit loosing most of its attractiveness in terms of communication efficiency.

A second approach is proposed for a more general case in Rabbat and Nowak \citep{rabbat2005quantized}, and applied to the KRR algorithm in \citep{predd2006distributed}. The method is based on incrementally passing the subgradient updates over the network, thus it requires the presence of a Hamiltonian cycle. Additionally, as we discussed before, the data must be propagated throughout the network. Hence, this approach is feasible only in specific cases, e.g. when the RKHS admits a lower dimensional parameterization.

A third approach is investigated in Predd \textit{et al.} \citep{predd2009collaborative}. The overall DL setting is represented as a bipartite graph. Each node in the first part corresponds to an agent, while each node in the second part corresponds to an example. An edge between the two parts means that a node has access to a given pattern. A relaxed version of the overall optimization problem is solved, by imposing that the agents reach a consensus only on the patterns that are shared among them. Due to this, it is possible to avoid sending the complete datasets, while communication is restricted to a set of state messages. Some extensions, particularly to asynchronous updates, are discussed in P\'erez-Cruz and Kulkarni \citep{perez2010robust}. All the three approaches introduced up to now are summarized in \citep{predd2006distributed}.

In the centralized case, solving the KRR optimization problem in a sequential setting has received considerable attention in the field of kernel adaptive filtering (KAF) \citep{Liu2010}, giving rise to multiple kernel-based extensions of the linear adaptive filters discussed in Section \ref{chap4:sec:diffusion_filtering}. The fact that the resulting model grows linearly with the number of processed patterns is also one of the main drawbacks of KAFs, where it is known as the `growth' problem. Much work has been done to curtail it \citep{Liu2010}, and in a limited part it has been extended to the decentralized case. In particular, Honeine \textit{et al.} \citep{honeine2008distributed} propose a criterion for discarding examples based on a previously introduced concept of `coherence'. Given a set of patterns $\vect{x}_1, \ldots, \vect{x}_m$, the coherence with respect to a new pattern $\vect{x}$ is defined as:
\begin{equation}
\min_{\gamma_1, \ldots, \gamma_m} \norm[\mathcal{H}]{\mathcal{K}(\vect{x}, \cdot) - \sum_{i=1}^m \gamma_i \mathcal{K}(\vect{x}_i, \cdot)} \,.
\end{equation}
In \citep{honeine2008distributed}, the pattern $\vect{x}$ is discarded if the coherence is greater than a certain threshold. The authors discuss efficient implementations of this idea. In \citep{honeine2008prediction}, similar ideas are developed for \textit{removing} elements that have already been processed.

Finally, we mention the distributed algorithm presented in Chen \textit{et al.} \citep{chen2010non}, where the previous setting is extended by considering non-negativity constraints on the model's coefficients. This is particularly important in applications imposing non-negativity constraints on the parameters to estimate.

\subsection{Distributed support vector machines}
\label{chap3:sec:dist_svm}

As we discussed in Section \ref{sec:kernel_methods}, an alternative widespread kernel method is the SVM. Intuitively, this algorithm is preferable to KRR for a distributed implementation, due to the sparseness property of the resulting output vector. In fact, the SVs embed all the information which is required from a classification point of view, providing for a theoretically efficient way of compressing information to be sent throughout the network. However, this idea is hindered by a practical problem; namely, the SVs of a reduced dataset may not correspond to the SVs of the entire dataset. More formally, denote by $\text{SV}(S)$ the set of SVs obtained by solving the QP problem with dataset $S$. Given two partitions $S_1$, $S_2$ such that $S_1 \cup S_2 = S$, we have:
\begin{equation}
\text{SV}(S_1) \cup \text{SV}(S_2) \neq \text{SV}(S) \,.
\label{chap3:eq:sv_subdivision}
\end{equation}
Nonetheless, for a proper subdivision of the dataset, we may expect that the two terms in the previous equation may still share a good amount of SVs.  Initial work on distributed SVM was fueled by an algorithm exploiting this idea, the cascade SVM \citep{graf2004parallel}, originally developed for parallelizing the solution to the global QP problem. In a cascade SVM, the network is organized in a set of successive layers. Nodes in the first layer receive parts of the input dataset, and propagate forward their SVs. Nodes in the next layers receive the set of SVs from their ancestors, merge them, and solve again the QP problem, up to a final node which outputs a final set of SVs. As stated by the authors, ``\textit{Often a single pass through this Cascade produces satisfactory accuracy, but if the global optimum has to be reached, the result of the last layer is fed back into the first layer}'' \citep{graf2004parallel}. This is shown schematically in Fig. \ref{chap3:fig:cascadesvm}.

\begin{figure}
\centering
\includegraphics[width=0.55\columnwidth]{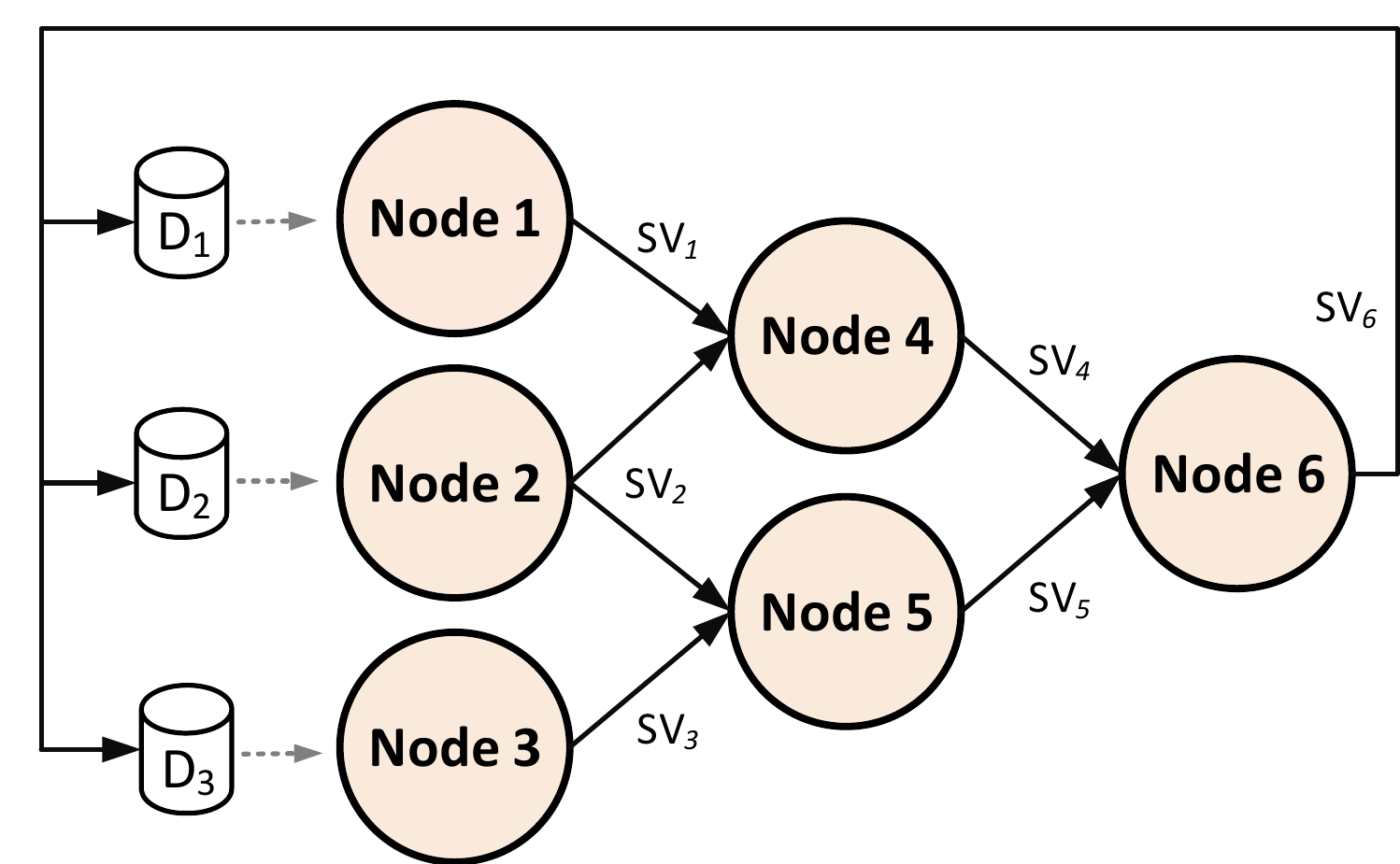}
\caption[Example of cascade SVM in a network with $6$ nodes.]{Example of cascade SVM in a network with $6$ nodes. For readability, $\text{SV}_k$ denotes the output SVs of the $k$th node.}
\label{chap3:fig:cascadesvm}
\end{figure}

The first work we are aware of to explore training an SVM in a fully distributed setting, without constraining the network topology as in the cascade SVM, is the PhD thesis by Pedersen \citep{pedersen2005using}, which explores informally multiple data exchange protocols in a distributed Java implementation. Another early work was the Distributed Semiparametric Support Vector Machine (DSSVM) presented in \citep{navia2006distributed}. In the DSSVM, every local node selects a number of centroids from the training data. These centroids are then shared with the other nodes, and their corresponding weights are updated locally based on an iterated reweighted least squares (IRWLS) procedure. Privacy can be preserved by adding noise to the elements of the training data when selecting the centroids. The DSSVM may be suboptimal depending on the particular choice of centroids. Moreover, it requires incremental passing of the SVs, or centroids, between the nodes, which in turn requires the computation of a Hamiltonian cycle between them. An alternative Distributed Parallel SVM (DPSVM) is presented in \citep{lu2008distributed}. Differently from the DSSVM, the DPSVM does not depend on the particular choice of centroids, and it is guaranteed to reach the global optimal solution of the centralized SVM in a finite number of steps. Moreover, it considers general strongly connected networks, with only exchanges of SVs between neighboring nodes. Still, the need of exchanging the set of SVs reduces the capability of the algorithm to scale to very large networks. A third approach is presented in \citep{forero2010consensus}, where the problem is recast as multiple convex subproblems at every node, and solved with the use of the ADMM procedure. Most of the attempts described up to now are summarized in the overview by Wang and Zhou \citep{wang2012distributed}.

To conclude this section, we shortly present a set of alternative formulations for distributed SVM implementations that were proposed in the last years. Flouri \textit{et al.} \citep{flouri2006training} consider the possibility of exchanging SVs cyclically, through a set of `clusterheads' distributed over the network. In \citep{flouri2008distributed,flouri2009optimal}, the algorithm is refined in order to consider only one-hop communication. Additionally, the authors propose the idea of exchanging only the SVs which are on the borders of the convex hulls for each class, in order to accelerate the convergence speed.

Ang \textit{et al.} \citep{ang2008cascade} combines the idea of the cascade SVM and the Reduced SVM (similar to the Semiparametric SVM in \citep{navia2006distributed}) in order to reduce the communication cost on a P2P network. In \citep{ang2009communication}, the cost is further reduced by considering a bagging procedure at each node. 

Hensel and Duta \citep{hensel2009gadget} investigate a gradient descent procedure, where the overall gradient update is computed with a `Push-Sum' protocol, a special gossip procedure for computing approximate sums over a network, allowing the algorithm to scale linearly with respect to the size of the agents' network. Differently, in Wang \textit{et al.} \citep{wang2010scalable}, gradient descent is used locally to update the local models, which are then fused at the prediction phase with a DAC procedure.

Lodi \textit{et al.} \citep{lodi2010single} explicitly consider the problem of multi-class SVM, exploiting the equivalence between a multi-class extension of SVM and the Minimum Enclosing Ball (MEB) problem (see \citep[Section 3.1]{lodi2010single} for the definition of MEB). Each node computes its own MEB, and forwards the results to a central processor, which is in charge of computing the global solution. The authors mention that this last step can be parallelized by reformulating the algorithms in the Cascade SVM \citep{graf2004parallel} or the DPSVM \citep{lu2008distributed}, substituting the SVs with the solution to the MEB problems.

Finally, several authors have considered the use of distributed optimization routines for solving the distributed linear SVM problem (i.e. with a kernel corresponding to dot products in the input space, see Chapter \ref{chap:dist_ssl_next}). Among these, we may cite the random projection algorithm \citep{lee2013distributed}, dual coordinate ascent \citep{jaggi2014communication}, and the box-constrained QP \citep{lee2015distributed}. 

\subsection{Distributed multilayer perceptrons}
\label{sec:distributed_mlp}

We conclude this chapter with a brief overview on distributed MLPs. Remember from Section \ref{sec:mlp} that MLPs are generally trained with SGD. As we stated previously in Section \ref{chap3:sec:relation_other_research_fields}, gradient descent is relatively easy to implement in a distributed fashion, due to the additivity of the gradient. In fact, the literature on online learning and prediction has considered multiple distributed implementations of SGD, starting from the work of Zinkevich \textit{et al.} \citep{zinkevich2009slow,zinkevich2010parallelized}, including variants with asynchronous updates \citep{recht2011hogwild}, and without the need for a parameters' server \citep{dekel2012optimal}.

However, these ideas have been rarely applied to the distributed training of MLPs, except in a handful of cases. As examples, Georgopoulos and Hasler \citep{georgopoulos2014distributed} train it by summing the gradient updates directly with a DAC procedure. Similarly, Schlitter \citep{schlitter2008protocol} and Samet and Miri \citep{samet2012privacy} investigate the use of secure summation protocols for ensuring privacy during the updates. This scarcity of results has a strong motivation. In fact, relatively large MLPs may possess millions of free parameters, resulting in millions of gradient computations to be exchanged throughout each node, making it impractical (e.g. \citep{seide2014parallelizability}). This problem has started to being addressed in different ways. The first is model parallelism, where the MLPs itself is split over multiple machines, such as in DistBelief \citep{dean2012large}. The other is quantization, where the gradient updates are heavily quantized in order to reduce the communication cost \citep{seide20141}. Additionally, the problem involved in training the MLP is non-convex, making it more complex (both theoretically and practically) to apply the aforementioned ideas.

An alternative approach is to construct an \textit{ensemble} of MLPs, one for each node, in order to avoid the gradient exchanges. Lazarevic and Obradovic \citep{lazarevic2002boosting} were among the first to consider this idea, with a distributed version of the standard AdaBoost algorithm. A similar approach is presented in Zhang and Zhong \citep{zhang2013privacy}.

%% file: chapters/chapter4-distributed_rvfl.tex
\chapter{Distributed Learning for RVFL Networks}
\label{chap:dist_rvfl}

\minitoc
\vspace{15pt}

\blfootnote{The content of this chapter, except Sections 4.1.1 and 4.1.2, is adapted from the material published in \citep{Scardapane2015}.}

\lettrine[lines=2]{T}{his} chapter introduces two distributed algorithms for RVFL networks, which are a special case of the fixed hidden layer ANN models presented in Section \ref{sec:fixed_hidden_layer}. As we said in Section \ref{sec:fixed_hidden_layer}, the use of these models is widespread in the centralized case, due to their good trade-off of algorithmic simplicity and nonlinear modeling capabilities. At the same time, as detailed in Section \ref{sec:dist_fixed_hidden_layer}, their use in the DL setting has been relatively limited, which is the main motivation for this chapter. After introducing the RVFL network, we describe the two distributed strategies, based on the DAC protocol and the ADMM optimization algorithm. Next, we evaluate them on multiple real-world scenarios.

\section{Basic concepts of RVFL networks}

RVFL networks are a particular class of ANN models with a fixed hidden layer, as depicted in Section \ref{sec:fixed_hidden_layer}. Mathematically, their most common variation is given by \citep{Pao1994}:
\begin{equation}
f(\vect{x}) = \sum_{m=1}^B \beta_m h_m(\vect{x}; \vect{w}_m) = \boldsymbol{\beta}^T\vect{h}(\vect{x}; \vect{w}_1, \ldots, \vect{w}_B) \;,
\label{eq:flnn_model}
\end{equation}
\noindent where the $m$th transformation is parametrized by the vector $\vect{w}_m$.\footnote{The original derivation in \citep{Pao1994} had additional connections from the input layer to the output layer, however, this is a trivial extension with respect to our formulation.} The parameters $\vect{w}_1, \ldots, \vect{w}_B$ are chosen in the beginning of the learning process, in particular, they are extracted randomly from a predefined probability distribution. Conceptually, this is similar to the well-known technique of random projections \citep{mahoney2011randomized}, which is a common procedure in statistics for dimensionality reduction. Differently from it, however, in RVFL networks the stochastic transformation of the input vector is not required to preserve distances and, more importantly, can increase the dimensionality. In the following, dependence of the hidden functions with respect to the stochastic parameters is omitted for readability. If we define the hidden matrix $\vect{H} \in \R^{N \times B}$ as:
\begin{equation}
	\vect{H} = 
	 \begin{pmatrix}
		h_1(\vect{x}_1) &  \cdots & h_B(\vect{x}_1) \\
		\vdots  & \ddots & \vdots  \\
		h_1(\vect{x}_N)  & \cdots & h_B(\vect{x}_N)
	 \end{pmatrix} \,,
\label{eq:hidden}
\end{equation}
it is straightforward to show that training of an RVFL network can be implemented efficiently with the use of the LRR algorithm described in Eq. \eqref{chap2:eq:lls} and Eq. \eqref{chap2:eq:lls_solution}, by substituting the input matrix $\vect{X}$ with the hidden matrix $\vect{H}$. The resulting output weights are given by:
\begin{equation}
\boldbeta^* = \left( \vect{H}^T\vect{H} + \lambda \vect{I} \right)^{-1}\vect{H}^T \vect{y} \,.
\label{eq:rvfl_opt}
\end{equation}
Throughout this chapter (and subsequent ones), we will use sigmoid activation functions given by:
\begin{equation}
h(\vect{x}; \vect{w}, b) = \frac{1}{1 + \exp\left\{ - \left(\vect{w}^T\vect{x} + b\right) \right\}} \,.
\label{eq:sigmoid}
\end{equation} 
The derivation in this section extends trivially also to the situation of $M > 1$ outputs. In this case, $\boldbeta$ becomes a $B \times M$ matrix and the output vector $\vect{y}$ becomes an $N \times M$ matrix, where the $i$th row corresponds to the $M$-dimensional output $\vect{y}_i^T$ of the training set. Additionally, we replace the $L_2$-norm on vectors in Eq. \eqref{chap2:eq:lls} with a suitable matrix norm.

\subsection{An historical perspective on RVFL networks}
\label{chap4:sec:historical_perspective_rvfl_networks}

This kind of random-weights ANNs have a long history in the field of SL. The original perceptron, indeed, considered a fixed layer of binary projections, which was loosely inspired to the biological vision \citep{Rosenblatt1958}. In 1992, Schmidt \textit{et al.} \citep{schmidt1992feedforward} investigated a model equivalent to the RVFL network, however, this was not ``\textit{presented as an alternative learning method}'', but ``\textit{only to analyse the functional behavior of the networks with respect to learning}''. The RVFL network itself was presented by Pao and his coworkers \citep{Pao1994} as a variant of the more general functional-link network \citep{Pao1992}. In \citep{igelnik1995stochastic}, it was shown to be an universal approximator for smooth functions, provided that the weights were extracted in a proper range and $B$ was large enough. In particular, the rate of convergence to zero of the approximation error is $\mathcal{O}(\frac{C}{\sqrt{B}})$, with the constant $C$ independent of $B$. Further analyses of approximation with random bases were obtained successively in \citep{Rahimi2008,Rahimi2009,Gorban2015a}. Recently, similar models were popularized under the name extreme learning machine (ELM) \citep{huang2006extreme}, raising multiple controversies due to the lack of proper acknowledgment of previous material, particularly RVFL networks \citep{wang2008comments}. Historically, the RVFL network is also connected to the radial basis function (RBF) network investigated by Broomhead and Lowe in 1988 \citep{lowe2multi}, to the statistical test for neglected nonlinearities presented by White \citep{white1989additional}, and to the later QuickNet family of networks by the same author \citep{white2006approximate}. Another similar algorithm has been proposed in 2013 as the `no-prop' algorithm \citep{widrow2013no}.\footnote{Which, ironically, has been criticized for its similarity to the ELM network \citep{widrow2013letter}.}

\subsection{On the effectiveness of random-weights ANNs}

Despite the stochastic assignment of weights in the first layer, RVFL networks are known to provide excellent performance in many real-world scenario, giving a good trade-off between accuracy and training simplicity. This was shown clearly in a 2014 analysis by Fernandez-Delgato \textit{et al.} \citep{Fernandez-Delgado2014}. Over $179$ classifiers, a kernel-based variation of RVFL networks with RBF functions was shown to be among the top-three performing algorithms over $121$ different datasets. In the words of B. Widrow \citep{widrow2013letter}: ``\textit{we [...] have independently discovered that it is not necessary to train the hidden layers of a multi-layer neural network. Training the output layer will be sufficient for many applications.}''. Clearly, randomly selecting bases is at most a naive approach, which can easily be outperformed by proper adaptation of the hidden layer. Worse, in some cases this choice can introduce a large variance in the results, as stated by Principe and Chen \citep{principe2015universal}: ``\textit{[random-weights models] still suffer from design choices, translated in free parameters, which are difficult to set optimally with the current mathematical framework, so practically they involve many trials and cross validation to find a good projection space, on top of the selection of the number of hidden [processing elements] and the nonlinear functions.}''. Although we use RVFL networks for their efficiency in the DL setting, these limitations should be kept in mind.

\section{Distributed training strategies for RVFL networks}
\label{chap4:sec:dist_training_strategies_rvfl_networks}

Let us now consider the problem of training an RVFL network in the DL setting. By combining the DL problem in Eq. \eqref{chap3:eq:dist_risk_functional} with the RVFL least-square optimization criterion, the global optimization problem of the distributed RVFL can be stated as:
\begin{equation}
\boldsymbol{\beta}^* = \argmin_{\boldsymbol{\beta} \in \R^B}  \frac{1}{2} \left( \sum_{k=1}^L \norm{\vect{H}_{k} \boldsymbol{\beta} - \vect{y}_{k}}^2 \right) + \frac{\lambda}{2}\norm{\boldsymbol{\beta}}^2 \,,
\label{eq:dd_rvfl_opt}
\end{equation}
\noindent where $\vect{H}_{k}$ and $\vect{y}_{k}$ are the hidden matrix and output vector computed over the local dataset $S_k$. Remember from Section \ref{sec:dist_linear_neuron} that the optimal weight vector in this case can be expressed as:
\begin{align}
\boldbeta^* = & \left( \sum_{k=1}^L \left( \vect{H}_k^T \vect{H}_k  \right) + \lambda \vect{I} \right)^{-1}\sum_{k=1}^L \left( \vect{H}_k^T\vect{y}_k\right) =  \nonumber \\
 = & \left( \vect{H}^T \vect{H} + \lambda\vect{I} \right)^{-1}\vect{H}^T\vect{y} \;.
\end{align}
This can be implemented in a fully distributed fashion by executing two sequential DAC steps\footnote{The DAC protocol is introduced in Appendix \ref{appA:sec:consensus}.}: the first on the matrices $\vect{H}_k^T \vect{H}_k$, and the second on $\vect{H}_k^T\vect{y}_k$. However, since the matrices $\vect{H}_k^T \vect{H}_k$ have size $B \times B$, this approach is feasible only for small hidden expansions, i.e. small $B$. Otherwise, the free exchange of these matrices over the network can become a computational bottleneck or, worse, be infeasible. For this reason, we do not consider this idea further in this chapter, and we focus on computationally cheaper algorithms which are able to scale better with large hidden layers. Two strategies to this end are introduced next.

\subsection{Consensus-based distributed training}
\label{sec:consrvfl}

The first strategy that we investigate for training an RVFL network in a fully decentralized way is simple, yet it results in a highly efficient training algorithm.  It is composed of three steps:
\begin{enumerate}
\item \textbf{Initialization}: Parameters $\vect{w}_1, \dots, \vect{w}_B$ of the activation functions are agreed between nodes. For example, one node can draw these parameters from a uniform distribution and broadcast them to the rest of the network. This can be achieved in a decentralized way using a basic leader election strategy \cite{awerbuch1987optimal}. Alternatively, they can be generated during the design of the distributed system (i.e. hardcoded in the network's design), so that they are already available when the system is actually started.
\item \textbf{Local training}: Each node solves its local training problem, considering only its own training dataset $S_k$. Solution is given by Eq. \eqref{eq:rvfl_opt}, obtaining a local set of output weights $\boldbeta^*_k$, ${k = 1\dots L}$.
\item \textbf{Averaging}: Local parameters vectors are averaged using a DAC strategy. After running DAC, the final weight vector at every node is given by:
\begin{equation}
\boldbeta^*_{\text{CONS}} = \frac{1}{L} \sum_{k=1}^L \boldbeta^*_k \,.
\label{eq:consensus_beta_final}
\end{equation} 
\end{enumerate}
\noindent Despite its simplicity, consensus-based RVFL (denoted as CONS-RVFL) results in an interesting algorithm. It is easy to implement, even on low-cost hardware \cite{olfati2007consensus}; it requires low training times (i.e., local training and a short set of consensus iterations); moreover, our results show that it achieves a very low error, in many cases comparable to that of the centralized problem. From a theoretical standpoint, this algorithm can be seen as an ensemble of multiple linear predictors defined over the feature space induced by the mapping $\vect{h}(\cdot)$, i.e. it is similar to a bagged ensemble of linear predictors \cite{breiman1996bagging}. The overall algorithm is summarized in Algorithm \ref{alg:dac_dist_rvfl}.

\begin{AlgorithmCustomWidth}[h]
    \caption{CONS-RVFL: Consensus-based training for RVFL networks ($k$th node).}
    \label{alg:dac_dist_rvfl}
  \begin{algorithmic}[1]
    \Require{Training set $S_k$, number of nodes $L$ (global), regularization factor $\lambda$ (global)}
    \Ensure{Optimal weight vector $\boldbeta_k^*$}
    \State Select parameters $\vect{w}_1, \dots, \vect{w}_B$, in agreement with the other $L-1$ nodes.
    \State Compute $\vect{H}_k$ and $\vect{y}_k$ from $S_k$.
    \State Compute $\boldbeta_k^*$ via Eq. \eqref{eq:rvfl_opt}.
    \State $\boldbeta^* \leftarrow \text{DAC}(\boldbeta_1^*, \ldots, \boldbeta_L^*)$. \Comment Run in parallel, see Appendix \ref{app:graph_theory}.
    \State \Return{$\boldbeta^*$}
  \end{algorithmic}
\end{AlgorithmCustomWidth}

\subsection{ADMM-based distributed training}
\label{sec:admmrvfl}

Another strategy for training in a decentralized way a RVFL network is to optimize directly the global problem in Eq. \eqref{eq:dd_rvfl_opt} in a distributed fashion. Although potentially more demanding in computational time, this would ensure convergence to the global optimum. We can obtain a fully decentralized solution to problem in Eq. \eqref{eq:dd_rvfl_opt} using the well-known ADMM. Most of the following derivation will follow \citep[Section 8.2]{boyd2011distributed}.
\subsubsection{Derivation of the training algorithm}
First, we reformulate the problem in the so-called `global consensus' form, by introducing local variables $\boldbeta_{k}$ for every node, and forcing them to be equal at convergence. Hence, we rephrase the optimization problem as:
\begin{equation}
\boldsymbol{\beta}^*_{\text{ADMM}} = \begin{aligned}
&\underset{\vect{z}, \boldbeta_{1}, \dots, \boldbeta_{L} \in \R^B}{\text{minimize}}
& & \frac{1}{2} \left( \sum_{k=1}^L \norm{\vect{H}_{k} \boldsymbol{\beta}_{k} - \vect{y}_{k}}^2 \right) + \frac{\lambda}{2}\norm{\vect{z}}^2 \\
&\;\;\;\;\text{subject to}
& & \boldsymbol{\beta}_{k} = \vect{z}, \; k = 1\ldots L \,.
\end{aligned}
\label{eq:dd_rvfl_local}
\end{equation}
\noindent Then, we construct the augmented Lagrangian:
\begin{equation}
\begin{aligned}
\mathcal{L}   & = \frac{1}{2} \left( \sum_{k=1}^L \norm{\vect{H}_{k} \boldsymbol{\beta}_{k} - \vect{y}_{k}}^2 \right) + \frac{\lambda}{2}\norm{\vect{z}}^2 + \\ 
              & + \sum_{k=1}^L \vect{t}^T_k (\boldbeta_{k} - \vect{z}) + \frac{\gamma}{2} \sum_{k=1}^L \norm{\boldbeta_k - \vect{z}}^2 \,,
\end{aligned}
\label{eq:dd_rvfl_local_augmented_lagrangian}
\end{equation}
\noindent where ${\mathcal{L}=\mathcal{L}\left(\vect{z}, \boldbeta_{1}, \dots, \boldbeta_{L}, \vect{t}_1,\dots,\vect{t}_L\right)}$, the vectors $\vect{t}_k$, ${k=1\dots L}$, are the Lagrange multipliers, $\gamma > 0$ is a penalty parameter, and the last term is introduced to ensure differentiability and convergence \cite{boyd2011distributed}. ADMM solves problems of this form using an iterative procedure, where at each step we optimize separately for $\boldbeta_k$, $\vect{z}$, and we update the Lagrangian multipliers using a steepest-descent approach:
\begin{align}
\boldbeta_k[n+1] & =  \underset{\boldbeta_k \in \R^B}{\arg\min} \; \mathcal{L}\left(\vect{z}[n], \boldbeta_{1}, \dots, \boldbeta_{L}, \vect{t}_1[n],\dots,\vect{t}_L[n]\right) \,, \label{eq:admm_step1}\\
\vect{z}[n+1] & = \underset{\vect{z} \in \R^B}{\arg\min} \; \mathcal{L}\left(\vect{z}, \boldbeta_{1}[n+1], \dots, \boldbeta_{L}[n+1], \vect{t}_1[n],\dots,\vect{t}_L[n]\right) \,, \label{eq:admm_step2}\\
\vect{t}_k[n+1]&  =  \vect{t}_k[n] + \gamma\left( \boldbeta_k[n+1] - \vect{z}[n+1] \right) \,.\label{eq:admm_step3}
\end{align}
\noindent In our case, the updates for $\boldbeta_k[n+1]$ and $\vect{z}[n+1]$ can be computed in a closed form:
\begin{align}
\boldbeta_k[n+1] & =  \left( \vect{H}_k^T \vect{H}_k + \gamma \vect{I} \right)^{-1} \left( \vect{H}^T_k \vect{y}_k - \vect{t}_k[n] + \gamma \vect{z}[n] \right) \,, \label{eq:admm_1} \\
\vect{z}[n+1] & =  \frac{\gamma \hat{\boldsymbol{\beta}} + \hat{\vect{t}}}{\lambda/L + \gamma} \,,
\label{eq:admm_2}
\end{align}
\noindent where we introduced the averages $\hat{\boldsymbol{\beta}} = \frac{1}{L}\sum_{k=1}^L \boldsymbol{\beta}_k[n+1]$ and $ \hat{\vect{t}} = \frac{1}{L} \sum_{k=1}^L \vect{t}_k[n]$. These averages can be computed in a decentralized fashion using a DAC step. We refer to \cite{boyd2011distributed} for a proof of the asymptotic convergence of ADMM.

\subsubsection*{Remark 1}
In cases where, on a node, $N_k \ll B$, we can exploit the matrix inversion lemma to obtain a more convenient matrix inversion step \cite{mateos2010distributed}:
\begin{equation}
\left( \vect{H}_k^T \vect{H}_k + \gamma \vect{I} \right)^{-1} = \gamma^{-1}\left[ \vect{I} - \vect{H}_k^T\left( \gamma \vect{I} + \vect{H}_k\vect{H}_k^T \right)^{-1} \vect{H}_k \right] \,.
\label{eq:lemma_2}
\end{equation} 
\noindent Moreover, with respect to the training complexity, we note that the matrix inversion and the term $\vect{H}^T_k \vect{y}_k$ in Eq. \eqref{eq:admm_1} can be precomputed at the beginning and stored into memory. More advanced speedups can also be obtained with the use of Cholesky decompositions. Hence, time complexity is mostly related to the DAC step required in Eq. \eqref{eq:admm_2}. Roughly speaking, if we allow ADMM to run for $T$ iterations (see next subsection), the ADMM-based strategy is approximately $T$ times slower than the consensus-based one.

\subsubsection{Stopping criterion}
\label{chap4:sec:stopping_criterion}

Convergence of the algorithm at the $k$th node can be tracked by computing the `primal residual' $\vect{r}_k[n]$ and `dual residual' $\vect{s}[n]$, which are defined as:
\begin{align}
\vect{r}_k[n] & = \boldbeta_k[n] - \vect{z}[n] \label{eq:primal_residual} \,, \\
\vect{s}[n] & = - \gamma \left( \vect{z}[n] - \vect{z}[n-1] \right) \,. \label{eq:dual_residual}
\end{align}
\noindent A possible stopping criterion is that both residuals should be less (in norm) than two thresholds:
\begin{align}
\norm{\vect{r}_k[n]} & < \epsilon_{\text{primal}} \label{eq:primal_stopping_criterion}  \,, \\
\norm{\vect{s}[n]} & < \epsilon_{\text{dual}} \label{eq:dual_stopping_criterion} \,.
\end{align}
\noindent A way of choosing the thresholds is given by \cite{boyd2011distributed}:
\begin{align}
\epsilon_{\text{primal}} & = \sqrt{L}\,\epsilon_{\text{abs}} + \epsilon_{\text{rel}} \max \big\{ \norm{\boldbeta_k[n]}, \norm{\vect{z}[n]} \big\} \,, \label{eq:primal_epsilon} \\
\epsilon_{\text{dual}} & = \sqrt{L}\,\epsilon_{\text{abs}} + \epsilon_{\text{rel}} \norm{\vect{t}_k[n]} \label{eq:dual_epsilon} \,,
\end{align}
\noindent where $\epsilon_{\text{abs}}$ and $\epsilon_{\text{rel}}$ are user-specified absolute and relative tolerances, respectively. Alternatively, as in the previous case, the algorithm can be stopped after a maximum number of iterations is reached. The pseudocode for the overall algorithm, denoted as ADMM-RVFL, at a single node is given in Algorithm \ref{tab:admmcode}.

\begin{AlgorithmCustomWidth}[h]
	\caption{ADMM-RVFL: ADMM-based training for RVFL networks ($k$th node).}
	\label{tab:admmcode}
	\begin{algorithmic}[1]
		\Require Training set $S_k$, number of nodes $L$ (global), regularization factors $\lambda, \gamma$ (global), maximum number of iterations $T$ (global)
		\Ensure Optimal vector $\boldbeta_k^*$
		\State Select parameters $\vect{w}_1, \dots, \vect{w}_B$, in agreement with the other $L-1$ nodes.
		\State Compute $\vect{H}_k$ and $\vect{y}_k$ from $S_k$.
		\State Initialize $\vect{t}_k[0] = \vect{0}$, $\vect{z}[0] = \vect{0}$.
		\For{$n$ from $0$ to $T$}
		\State Compute $\boldbeta_k[n+1]$ according to Eq. \eqref{eq:admm_1}.
		\State Compute averages $\hat{\boldsymbol{\beta}}$ and $\hat{\vect{t}}$ using DAC.
		\State Compute $\vect{z}[n+1]$ according to Eq. \eqref{eq:admm_2}.
		\State Update $\vect{t}_k[n]$ according to Eq. \eqref{eq:admm_step3}.
		\State Check termination with residuals.
		\EndFor
		\State \textbf{return} $\vect{z}[n]$
	\end{algorithmic}
\end{AlgorithmCustomWidth}

\section{Experimental Setup}

\subsection{Description of the Datasets}
\label{chap4:sec:description_datasets}

We tested our algorithms on four publicly available datasets, whose characteristics are summarized in Table~\ref{chap4:tab:datasets}. 

\begin{center}
\begin{table*}[h]
	\caption[General description of the datasets for testing CONS-RVFL and ADMM-RVFL]{General description of the datasets}
	{\hfill{}
		\setlength{\tabcolsep}{4pt}
		\renewcommand{\arraystretch}{1.3}
		\begin{footnotesize}
			\begin{tabular}{lrrll}  
				\toprule
				\textbf{Dataset name} & \textbf{Features} & \textbf{Instances} & \textbf{Desired output} & \textbf{Task type}\\
				\midrule
				G50C & 50 & 550 & Gaussian of origin & Classification (2 classes) \\ 
				Garageband &  44 &  1856 & Genre recognition & Classification (9 classes) \\
				Skills &  18 & 3338 & User's level & Regression \\
				Sylva & 216  & 14394 & Forest Type & Classification (2 classes) \\
				\bottomrule
			\end{tabular}
		\end{footnotesize}}
		\hfill{}\vspace{0.3em}
		\label{chap4:tab:datasets}
	\end{table*}
\end{center}

\vspace{-2.5em} \noindent We have chosen them to represent different applicative domains of our algorithms, and to provide enough diversity in terms of size, number of features, and imbalance of the classes:

\begin{itemize}
\item \textit{Garageband} is a music classification problem \cite{mierswa2005automatic}, where the task is to discern among $9$ different genres. As we stated in the previous chapter, in the distributed case we can assume that the songs are present over different computers, and we can use our strategies as a way of leveraging over the entire dataset without a centralized controller.
\item \textit{Skills} is a regression dataset taken from the UCI repository \cite{thompson2013video}. The task is to assess the skill level of a video game user, based on a set of recordings of its actions in the video game itself. This is useful for letting the game adapt to the user's characteristics. In this case, data is distributed by definition throughout the different players in the network. By employing our strategy several computers, each playing their own version of the game, can learn to adapt better by exploiting collective data.
\item \textit{Sylva} is a binary classification task for distinguishing classes of trees (Ponderosa pine vs. everything else).\footnote{\url{http://www.causality.inf.ethz.ch/al_data/SYLVA.html}} It is an interesting dataset since it has a large imbalance between the positive and negative examples (approximately $15:1$), and a large subset of the features are not informative from a classification point of view. In the distributed case, we can imagine that data is collected by different sensors.
\item \textit{G50C}, differently from the others, is an artificial dataset \cite{melacci2011laplacian}, whose main interest is given by the fact that the optimal (Bayes) error rate is designed to be equal exactly to $5\%$.
\end{itemize} 
In all cases, input variables are normalized between $0$ and $1$, and missing values are replaced with the average computed over the rest of the dataset. Multi-class classification is handled with the standard $M$ bit encoding for the output, associating to an input $\vect{x}_i$ a single output vector $\vect{y}_i$ of $M$ bits, where if its elements are ${y_{ij} = 1}$ and ${y_{ik} = 0}$, ${k\neq j}$, then the corresponding pattern is of class $j$.  We can retrieve the actual class from the $M$-dimensional RVFL output as:
\begin{equation}
\text{Class of } \vect{x} = \underset{j = 1\dots M}{\arg\max} \; f_j(\vect{x}) \,,
\end{equation}
\noindent where $f_j(\vect{x})$ is the $j$th element of the $M$-dimensional output $f(\vect{x})$. For all the models, testing accuracy and training times is computed by executing a $5$-fold cross-validation over the available data. This $5$-fold procedure is then repeated $15$ times by varying the topology of the agents and the initial weights of the RVFL net. Final misclassification error and training time is then collected for all the $15 \times 5 = 75$ repetitions, and the average values and standard deviations are computed.

\subsection{Algorithms and Software Implementation}
\label{chap4:sec:algo_implementation}

We compare the following algorithms:

\begin{itemize}
\item \textbf{Centralized RVFL} (C-RVFL): this is a RVFL trained with all the available training data. It is equivalent to a centralized node collecting all the data, and it can be used as a baseline for the other approaches.
\item \textbf{Local RVFL} (L-RVFL): in this case, training data is distributed evenly across the nodes. Every node trains a standard RVFL with its own local dataset, but no communication is performed. Testing error is averaged throughout the nodes.
\item \textbf{Consensus-based RVFL} (CONS-RVFL): as before, data is evenly distributed in the network, and the consensus strategy explained in Section \ref{sec:consrvfl} is executed. We set a maximum of $300$ iterations and $\delta = 10^{-3}$.
\item \textbf{ADMM-based RVFL} (ADMM-RVFL): similar to before, but we employ the ADMM-based strategy described in Section \ref{sec:admmrvfl}. In this case, we set a maximum of $300$ iterations, ${\epsilon_{\text{rel}} = \epsilon_{\text{abs}} = 10^{-3}}$ and ${\gamma = 1}$.
\end{itemize}
In all cases, we use sigmoid hidden functions given by Eq. \eqref{eq:sigmoid}, where parameters $\vect{w}$ and $b$ in \eqref{eq:sigmoid} are extracted randomly from an uniform distribution over the interval ${\left[ -1, +1 \right]}$. To compute the optimal number of hidden nodes and the regularization parameter $\lambda$, we execute an inner $3$-fold cross-validation on the training data only for C-RVFL. In particular, we search the uniform interval ${\left\{ 50, 100, 150, \dots, 1000 \right\}}$ for the number of hidden nodes, and the exponential interval $2^j$, ${j \in \left\{ -10, -9, \dots, 9, 10 \right\}}$ for $\lambda$. The step size of $50$ in the hidden nodes interval was found to provide a good compromise between final accuracy and the computational cost of the grid-search procedure. These parameters are then shared with the three remaining models. We experimented with a separate fine-tuning for each model, but no improvement in performance was found. Optimal parameters averaged over the runs are shown in Table~\ref{chap4:tab:gridsearch}.

\begin{center}
	\begin{table}[h]
		\caption[Optimal parameters found by the grid-search procedure for CONS-RVFL and ADMM-RVFL]{Optimal parameters found by the grid-search procedure}
		{\centering\hfill{}
			\setlength{\tabcolsep}{4pt}
			\renewcommand{\arraystretch}{1.3}
			\begin{footnotesize}
				\begin{tabular}{lcl}   
					\toprule
					\textbf{Dataset} & \textbf{Hidden nodes} & $\lambda$   \\
					\midrule
					G50C & $500$ & $2^3$ \\
					Garageband & $200$ & $2^{-3}$ \\
					Skills & $400$ & $2^{-2}$ \\
					Sylva & $450$ & $2^{-5}$ \\
					\bottomrule
				\end{tabular}
			\end{footnotesize}
		}
		\hfill{}\vspace{0.3em}
		\label{chap4:tab:gridsearch}
	\end{table}
\end{center}

\vspace{-2.5em} \noindent
We have implemented CONS-RVFL and ADMM-RVFL in the open-source Lynx MATLAB toolbox (see Appendix \ref{app:software}). Throughout this thesis, we are not concerned with the analysis of communication overhead over a realistic channel; hence, we employ a serial version of the code where the network is simulated artificially. However, in the aforementioned toolbox we also provide a fully parallel version, able to work on a cluster architecture, in order to test the accuracy of the system in a more realistic setting.

\section{Results and Discussion}
\label{sec:expresults}

\subsection{Accuracy and Training Times}

The first set of experiments is to show that both algorithms that we propose are able to approximate very closely the centralized solution, irrespective of the number of nodes in the network. The topology of the network in these experiments is constructed according to the so-called `Erd\H{o}s$-$R\'enyi model' \cite{newman2010networks}, i.e., once we have selected a number $L$ of nodes, we randomly construct an adjacency matrix such that every edge has a probability $p$ of appearing, with $p$ specified a-priori. For the moment, we set $p = 0.2$; an example of such a network for ${L = 8}$ is shown in Fig.~\ref{chap4:fig:networkexample}. 

\begin{figure}[h]
\centering
\includegraphics[width=0.5\linewidth]{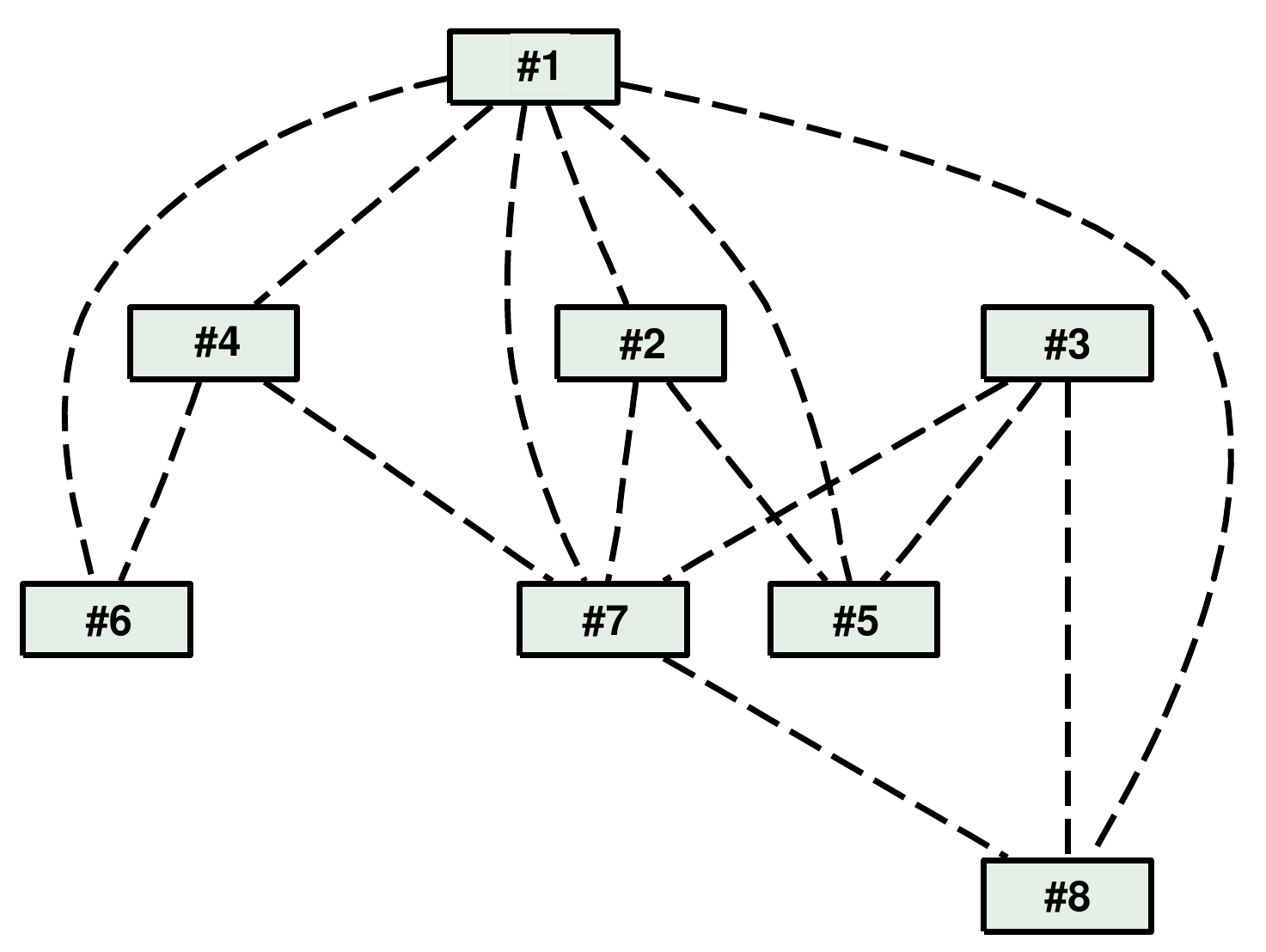}
\caption[Example of network (with $8$ nodes) considered in the experimental sections throughout the thesis.]{Example of network used in the experiments with $8$ nodes. Connectivity is generated at random, with a $20\%$ probability for each link of being present.}
\label{chap4:fig:networkexample}
\end{figure}

To test the accuracy of the algorithms, we vary $L$ from $5$ to $50$ by steps of $5$. Results are presented in Fig.~\ref{chap4:fig:errors} \subref{chap4:fig:g50cerror}-\subref{chap4:fig:sylvaerror}. For the three classification datasets, we show the averaged misclassification error. For the Skills dataset, instead, we show the Normalized Root Mean-Squared Error (NRMSE), defined for a test set $\mathcal{T}$ as:
\begin{equation}
\text{NRMSE}(\mathcal{T}) = \sqrt{\frac{\sum_{(\vect{x}_i, y_i) \in \mathcal{T}} \left[f(\vect{x}_i)-y_i\right]^2}{|\mathcal{T}| \hat{\sigma}_y}} \,,
\label{eq:nrmse}
\end{equation}
\noindent where $|\mathcal{T}|$ denotes the cardinality of the set $\mathcal{T}$ and $\hat{\sigma}_y$ is an empirical estimate of the variance of the output samples $y_i$, ${i=1, \dots, |\mathcal{T}|}$. For every fold, the misclassification error of L-RVFL is obtained by averaging the error over the $L$ different nodes. While this is a common practice, it can introduce a small bias with respect to the other curves. Nonetheless, we stress that it does not influence the following discussion, which mostly focuses on the comparison of the other three algorithms.

\begin{figure}
\centering
\subfloat[Dataset G50C]{\includegraphics[scale=0.75]{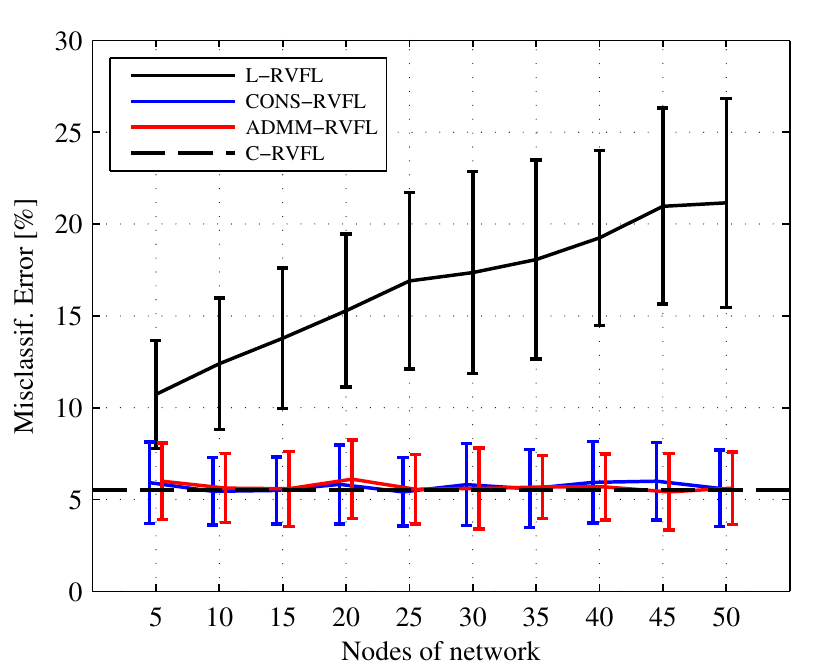}%
\label{chap4:fig:g50cerror}}
\hfil
\subfloat[Dataset Garageband]{\includegraphics[scale=0.75]{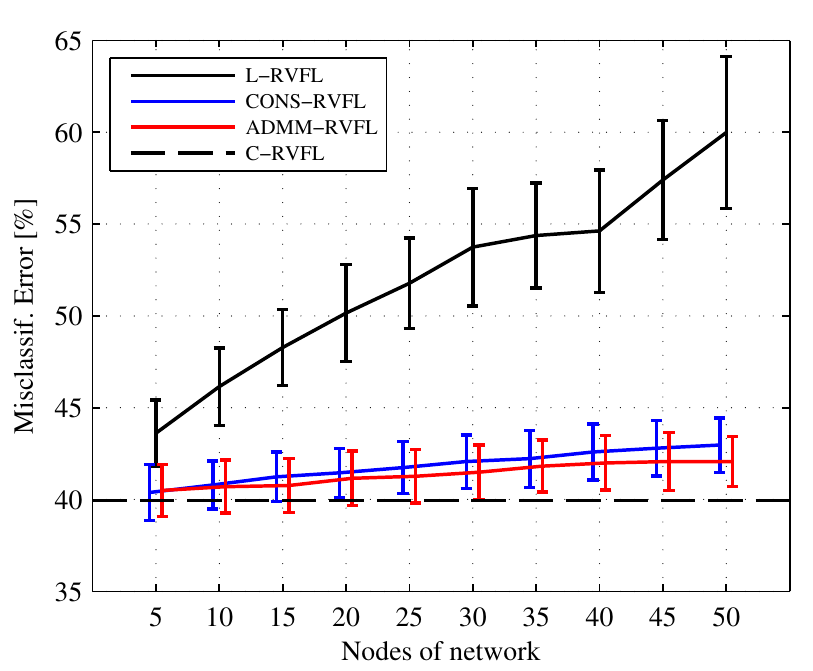}%
\label{chap4:fig:garagebanderror}}
\vfill
\subfloat[Dataset Skills]{\includegraphics[scale=0.75]{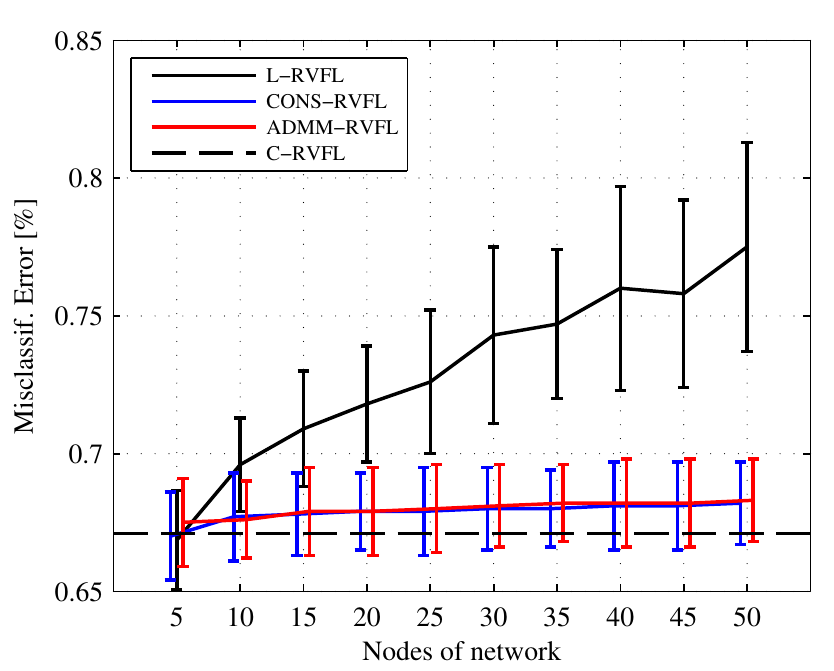}%
\label{chap4:fig:skillserror}}
\hfil
\subfloat[Dataset Sylva]{\includegraphics[scale=0.75]{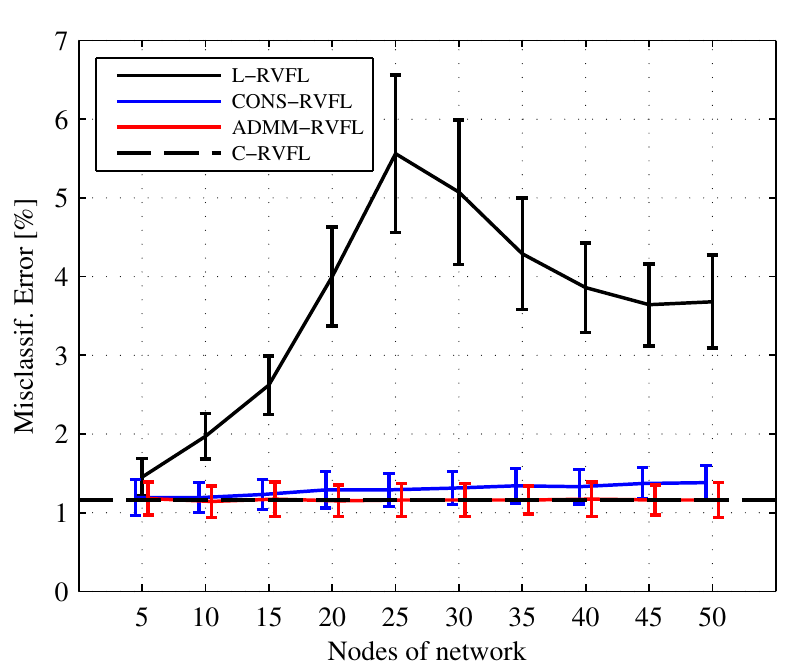}%
\label{chap4:fig:sylvaerror}}
\caption[Average error and standard deviation of CONS-RVFL and ADMM-RVFL on four datasets, when varying the number of nodes in the network from $5$ to $50$.]{Average error and standard deviation of the models on the four datasets, when varying the number of nodes in the network from $5$ to $50$. For G50C, Garageband, and Sylva we show the misclassification error, while for Skills we show the NRMSE. Lines for CONS-RVFL and ADMM-RVFL are slightly separated for better readability. Vertical bars represent the standard deviation from the average result.}
\label{chap4:fig:errors}
\end{figure}

\begin{figure}
\centering
\subfloat[CONS-RVFL]{\includegraphics[scale=0.75]{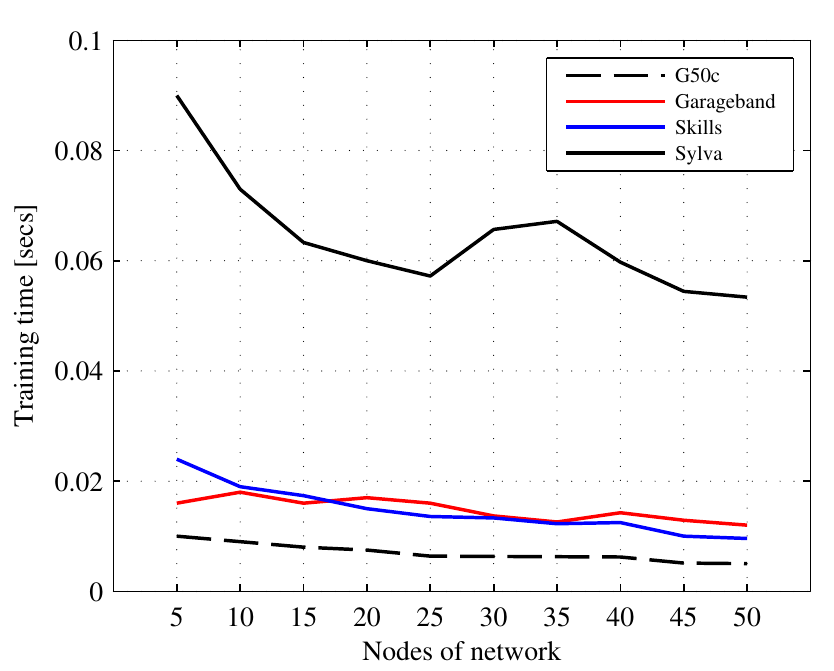}%
\label{chap4:fig:time_consensus}}
\subfloat[ADMM-RVFL]{\includegraphics[scale=0.75]{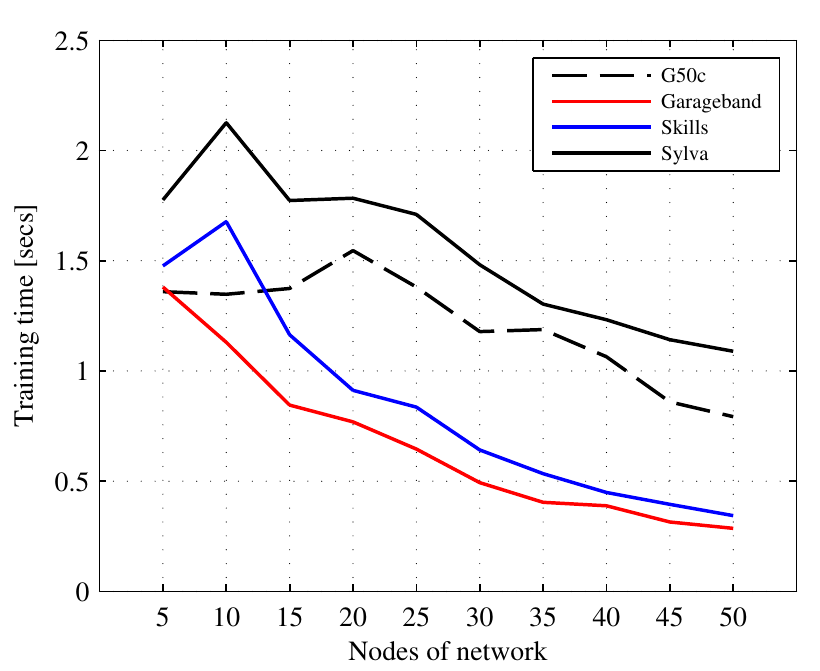}%
\label{chap4:fig:time_admm}}
\caption[Average training time for CONS-RVFL and ADMM-RVFL on a single node.]{Average training time for (a) CONS-RVFL and (b) ADMM-RVFL on a single node.}
\end{figure}

The first thing to observe is that L-RVFL has a steady decrease in performance in all situations, ranging from a small decrease in the Skills dataset with $5$ nodes, to more than $20\%$ of classification accuracy when considering networks of $50$ agents in the G50C and Garageband datasets. Despite being obvious, because of the decrease of available data at each node, it is an experimental confirmation of the importance of leveraging over \textit{all} possible data in terms of accuracy. It is also interesting to note that the gap between L-RVFL and C-RVFL does not always increase monotonically with respect to the size of the network, as shown by Fig.~\ref{chap4:fig:errors}-\subref{chap4:fig:sylvaerror}. A possible explanation of this fact is that, by keeping fixed the $\lambda$ parameter, the effect of the regularization factor in Eq. \eqref{eq:dd_rvfl_opt} is proportionally higher when decreasing the amount of training data.

The second important aspect is that CONS-RVFL and ADMM-RVFL are \textit{both} able to match very closely the performance of C-RVFL, irrespective of the network's size. In particular, they have the same performance on the G50C and Skills datasets, whilst a small gap is present in the Garageband and Sylva cases, although it is not significant.

Next, let us analyze the training times of the distributed algorithm, shown in Fig.~\ref{chap4:fig:time_consensus} for CONS-RVFL and Fig.~\ref{chap4:fig:time_admm} for ADMM-RVFL, respectively. In particular, we show the average training time spent at a single node. Generally speaking, CONS-RVFL is approximately one order of magnitude faster than ADMM-RVFL, which requires multiple iterations of consensus. In both cases, the average training time spent at a single node is monotonically decreasing with respect to the overall number of nodes. Hence, the computational time of the matrix inversion is predominant compared to the overhead introduced by the DAC and ADMM procedures.
\subsection{Effect of Network Topology}
\label{sec:networktopology}

Now that we have ascertained the convergence properties of both algorithms, we analyze an interesting aspect: how does the topology of the network influences the convergence time? Clearly, as long as the network stays connected, the accuracy is not influenced. However, the time required for the consensus to achieve convergence is dependent on how the nodes are interconnected. At the extreme, in a fully connected network, two iterations are always sufficient to achieve convergence at any desired level of accuracy. More in general, the time will be roughly proportional to the average distance between any two nodes. To test this, we compute the iterations needed to reach consensus for several topologies of networks composed of $50$ nodes:

\begin{itemize}
\item \textbf{Random network}: this is the network constructed according to the Erd\H{o}s$-$R\'enyi model described in the previous subsection. We experiment with $p=0.2$ and $p=0.5$, and denote the corresponding graphs as $R(0.2)$ and $R(0.5)$ respectively.
\item \textbf{Linear network}: in this network the nodes are ordered, and each node in the sequence is connected to its most immediate $K$ successors, with $K$ specified a-priori, except the last $K-1$ nodes, which are connected only to the remaining ones. We experiment with $K=1$ and $K=4$, and denote the networks as $K(1)$ and $K(4)$ respectively.
\item \textbf{Small world}: this is a network constructed according to the well-known `Watts-Strogazt' mechanism \cite{watts1998collective}. First, a cyclic topology is constructed, i.e., nodes are ordered in a circular sequence, and every node is connected to $K$ nodes to its left and $K$ to its right. Then, every link is `rewired' with probability set by a parameter $\alpha \in \left[0,1\right]$. In our case, we have $K=6$ and $\alpha=0.15$, and denote the resulting topology as $SW$.
\item \textbf{Scale-free}: this is another topology that tries to reflect realistic networks, in this case exhibiting a power law with respect to the degree distribution. We construct it according to the `Barab\'asi-Albert' model of preferential attachment \cite{albert2002statistical}, and denote the resulting topology as $SF$.
\end{itemize}

\noindent Results are presented in Fig.~\ref{chap4:fig:iterations} \subref{chap4:fig:g50citerations}-\subref{chap4:fig:sylvaiterations}. 

\begin{figure*}[h]
\centering
\subfloat[Dataset G50C]{\includegraphics[scale=0.9]{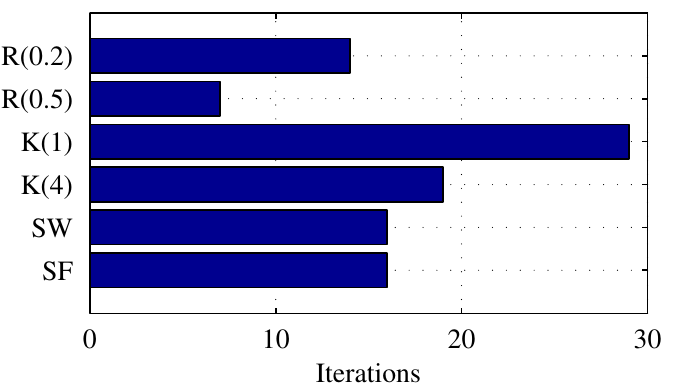}%
\label{chap4:fig:g50citerations}}
\hfil
\subfloat[Dataset Garageband]{\includegraphics[scale=0.9]{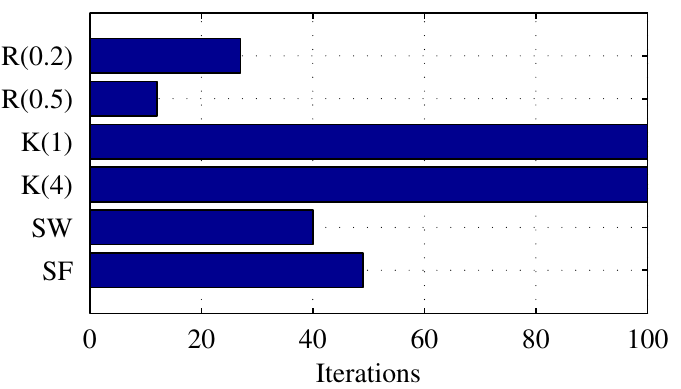}%
\label{chap4:fig:garagebanditerations}}
\vfill
\subfloat[Dataset Skills]{\includegraphics[scale=0.9]{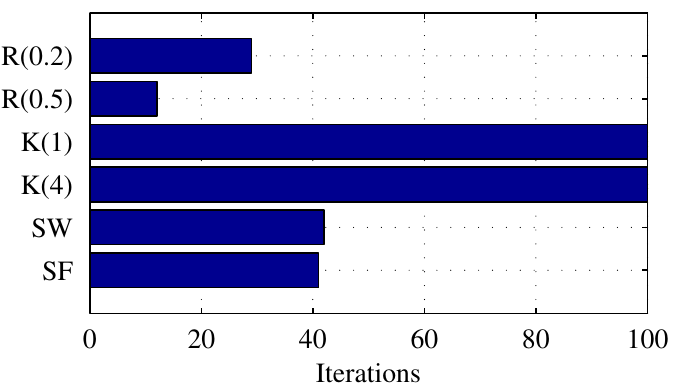}%
\label{chap4:fig:skillsiterations}}
\hfil
\subfloat[Dataset Sylva]{\includegraphics[scale=0.9]{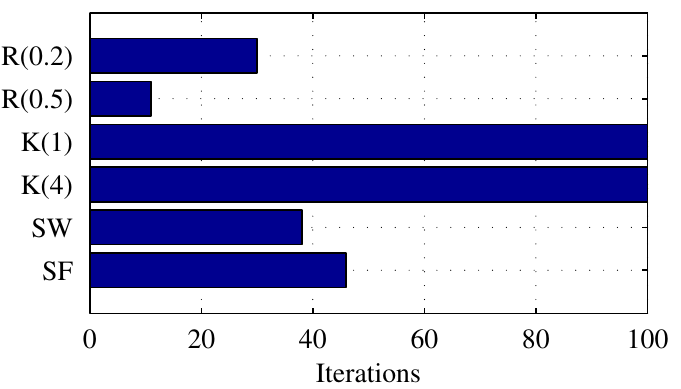}%
\label{chap4:fig:sylvaiterations}}
\caption[Consensus iterations needed to reach convergence in CONS-RVFL when varying the network topology.]{Consensus iterations needed to reach convergence when varying the network topology. For the explanation of the topologies see Sect.~\ref{sec:networktopology}. The number of iterations is truncated at $100$ for better readability.}
\label{chap4:fig:iterations}
\end{figure*}

We see that the algorithm has very similar results on all four datasets. In particular, as we expected, consensus is extremely slow in reaching agreement when considering linear topologies, where information takes several iterations to reach one end of the graph from the other. At the other extreme, it takes a very limited number of iterations in the case of a highly connected graph, as in the case of $R(0.5)$. In between, we can see that consensus is extremely robust to a change in topology and its performance is not affected when considering small-world or scale-free graphs.

\subsection{Early Stopping for ADMM}
Next, we explore a peculiar difference between CONS-RVFL and ADMM-RVFL. In the case of CONS-RVFL, no agreement is reached between the different nodes until the consensus procedure is completed. Differently from it, an intermediate solution is available at every iteration in ADMM-RVFL, given by the vector $\vect{z}[n]$. This allows for the use of an early stopping procedure, i.e., the possibility of stopping the optimization process before actual convergence, by fixing in advance a predefined (small) number of iterations. In fact, several experimental findings support the idea that ADMM can achieve a reasonable degree of accuracy in the initial stages of optimization \cite{boyd2011distributed}. To test this, we experiment early stopping for ADMM-RVFL at ${\left\{5, 10, 15, 25, 50, 100, 200\right\}}$ iterations for the three classification datasets. In Fig.~\ref{chap4:fig:error_admm} we plot the relative decrease in performance with respect to L-RVFL. We can see that $10$-$15$ iterations are generally enough to reach a good performance, while the remaining iterations are proportionally less useful. As a concrete example, misclassification error of ADMM-RVFL for G50C is $16.55\%$ after only $5$ iterations, $12.44\%$ after $10$, $7.89\%$ after $25$, while the remaining $175$ iterations are used to decrease the error only by an additional $2$ percentage points.
\begin{figure}[t]
\centering
\includegraphics[]{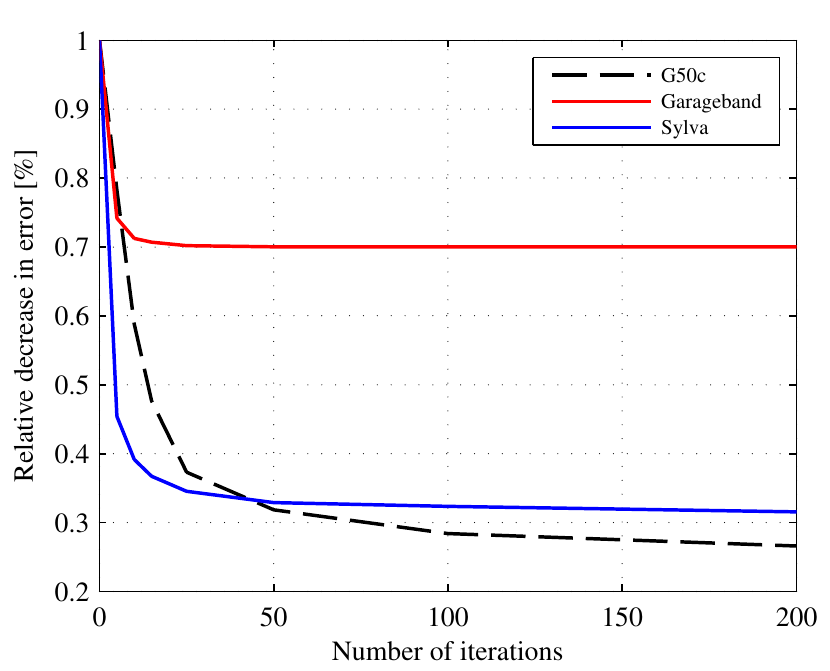}
\caption{Relative decrease in error of ADMM-RVFL with respect to L-RVFL, when using an early stopping procedure at different iterations.}
\label{chap4:fig:error_admm}
\end{figure}

\subsection{Experiment on Large-Scale Data}

As a final experimental validation, we analyze the behavior of CONS-RVFL and ADMM-RVFL on a realistic large-scale dataset, the well-known CIFAR-10 image classification database \cite{krizhevsky2009learning}. It is composed of $50000$ images labeled in $10$ different classes, along with a standard testing set of additional $10000$ images. Each image is composed of exactly $32 \times 32$ pixels, and each pixel is further represented by $3$ integer values in the interval $\left[1, \; 255\right]$, one for each color channel in the RGB color space. Classes are equally distributed between the training patterns, i.e., every class is represented by exactly $5000$ images. Since we are mostly interested into the relative difference in performance between the algorithms, and not in achieving the lowest possible classification error, we preprocess the images using the relatively simple procedure detailed in \cite{coates2011analysis}. In particular, we extract $1600$ significant patches from the original images, and represent each image using their similarity with respect to each of the patches. We refer to \cite{coates2011analysis} for more details on the overall workflow. In this experiment, we use $B = 3000$ and $\lambda = 10$, a setting which was found to work consistently on all situations. Moreover, we use the $R(0.2)$ graph explained before, but we experiment with lower number of nodes in the network, which we vary from $2$ to $12$ by steps of $2$. All the other parameters are set as in the previous experiments. Although the test set is fixed in this case, we repeat each experiment $15$ times to average out the effect of randomness in the RVFL and connectivity initializations.

\begin{figure}[t]
\centering
\subfloat[Misclassification error]{\includegraphics[scale=0.75]{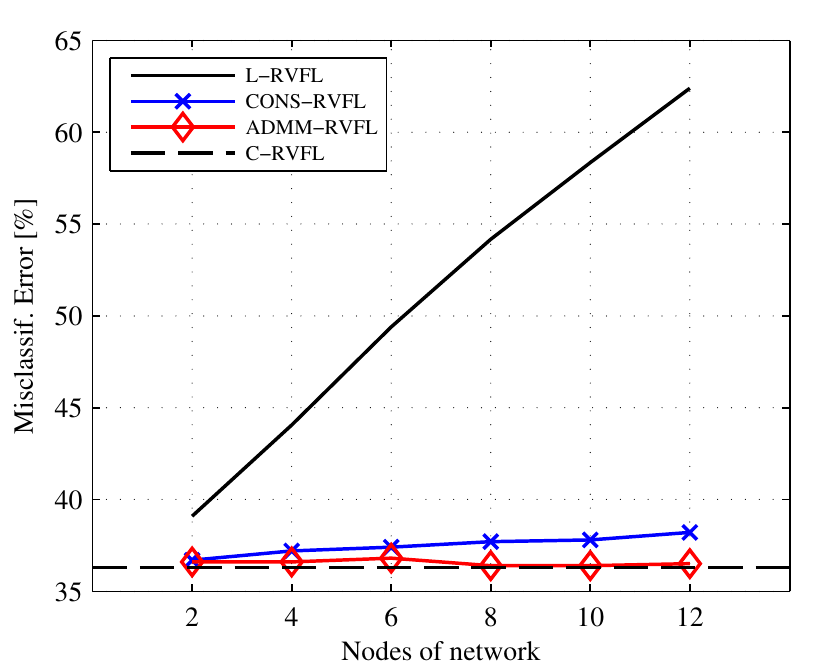}%
\label{chap4:fig:cifar10_error}}
\subfloat[Training time]{\includegraphics[scale=0.75]{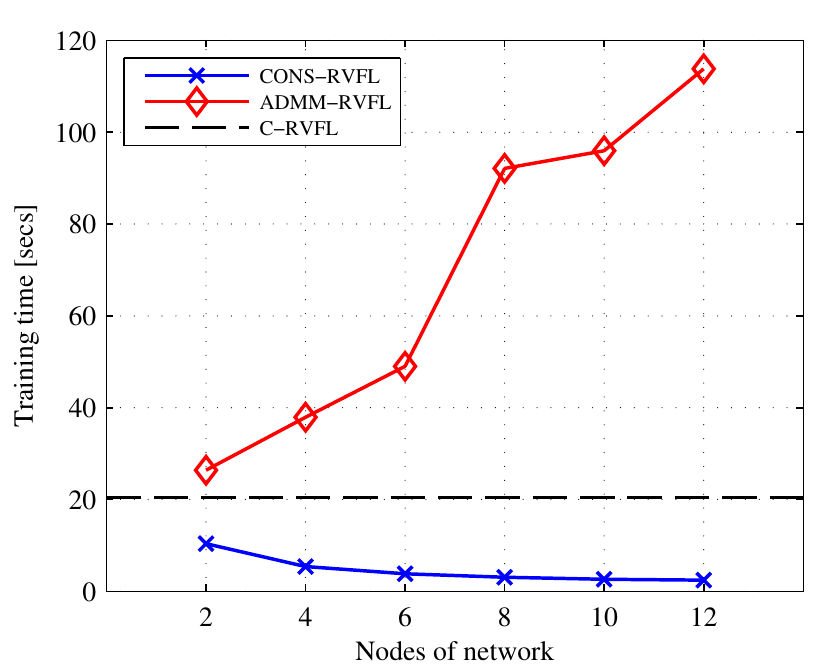}%
\label{chap4:fig:cifar10_time}}
\caption{Average misclassification error and training time of CONS-RVFL and ADMM-RVFL on the CIFAR-10 dataset, when varying the nodes of the network from $2$ to $12$.}
\end{figure}

The average misclassification error of the four models is shown in Fig. \ref{chap4:fig:cifar10_error}. In this case, the effect of splitting data is extremely pronounced, and the average misclassification error of L-RVFL goes from $39 \%$ with $2$ nodes, up to $62.4 \%$ with $12$ nodes. Both CONS-RVFL and ADMM-RVFL are able to track very efficiently the centralized solution, although there is a small gap in performance between the two (of approximately $1 \%$), when distributing over more than $8$ nodes. This is more than counter-balanced, however, by considering the advantage of CONS-RVFL with respect to the required training time. To show this, we present the average training time (averaged over the nodes) in Fig. \ref{chap4:fig:cifar10_time}. In this case, due to the large expansion block, time required to perform the multiple consensus iterations in ADMM-RVFL prevails over the rest, and the average training time tends to increase when increasing the size of the network. This is not true of CONS-RVFL, however, which obtains an extremely low training time with respect to C-RVFL, up to an order of magnitude for sufficiently large networks. Hence, we can say that CONS-RVFL can also be an efficient way of computing an approximate solution to a standard RVFL, with good accuracy, by distributing the computation over multiple machines.

%% file: chapters/chapter5-dist_rvfl_sequential.tex
\chapter[Extending Distributed RVFL Networks to a Sequential Scenario]{Extending Distributed RVFL Networks \\ to a Sequential Scenario}
\chaptermark{Sequential Distributed RVFL Networks}
\label{chap:dist_rvfl_sequential}

\minitoc
\vspace{15pt}

\blfootnote{The content of this chapter is adapted from the material published in \citep{scardapane2015distmusic} and \citep{fierimonte2015a}.}

\lettrine[lines=2]{T}{he} algorithms presented in the previous chapter have been designed for working in a batch setting. In this chapter, we extend CONS-RVFL to the distributed training of RVFL networks in the case where data is arriving sequentially at every node. Particularly, we combine the DAC-based strategies with local updates based on the blockwise RLS (BRLS) training algorithm. Next, we present a case study for distributed music classification in Section \ref{chap5:sec:dist_music_classification}. Finally, we compare the impact of using advanced choices for the connectivity matrix $\vect{C}$ of the DAC protocol in Section \ref{chap5:sec:dac_comparison}.

\section{Derivation of the algorithm}

Remember from Section \ref{sec:categorization_dl_algorithms} that in a sequential setting, the local dataset $S_k$ is not processed as a whole, but it is presented in a series of batches (or chunks) $S_{k,1}, \ldots, S_{k,T}$ such that:
\begin{equation}
\bigcup_{i=1}^T S_{k,T} = S_k \;\; k = 1, \ldots, L \;.
\label{eq:chunck}
\end{equation}

\noindent This encompasses situations where training data arrives in a streaming fashion, or the case where the dataset $S_k$ is too large for the matrix inversion in Eq. \eqref{eq:rvfl_opt} to be practical. In this section, we assume that new batches arrive synchronously at every node. In the single-agent case, an RVFL network can be trained efficiently in the sequential setting by the use of the BRLS algorithm \cite{Uncini2015}. Denote by  $\boldbeta[n]$ the estimate of its optimal weight vector after having observed the first $n$ chunks, and by $\vect{H}_{n+1}$ and $\vect{y}_{n+1}$ the matrices collecting the hidden nodes values and outputs of the $(n+1)$th chunk $S_{n+1}$. BRLS recursively computes Eq. \eqref{eq:rvfl_opt} by the following two-step update:
\begin{align}
\vect{P}[n+1] & =  \vect{P}[n] -  \vect{P}[n]\vect{H}_{n+1}^T\vect{M}_{n+1}^{-1} \vect{H}_{n+1}\vect{P}[n] \,, \label{eq:p_update_rls}\\
\boldsymbol{\beta}[n+1] & = \boldsymbol{\beta}[n] + \vect{P}[n+1]\vect{H}_{n+1}^T\left( \vect{y}_{n+1} - \vect{H}_{n+1}\boldsymbol{\beta}[n] \right) \,,
\label{eq:brls}
\end{align}
\noindent where we have defined:
\begin{equation}
	\vect{M}_{n+1} = \vect{I} + \vect{H}_{n+1}\vect{P}[n]\vect{H}_{n+1}^T \,.
	\label{eq:m_matrix}
\end{equation}
\noindent The matrix $\vect{P}$ in Eq. \eqref{eq:brls} and Eq. \eqref{eq:m_matrix} can be initialized as $\vect{P}[0] = \lambda^{-1} \vect{I}$, while the weights $\boldbeta[0]$ as the zero vector. For a derivation of the algorithm, based on the Sherman-Morrison formula, and an analysis of its convergence properties we refer the interested reader to \cite{Uncini2015}. The BRLS gives rise to a straightforward extension of the DAC-based training algorithm presented in Section \ref{sec:consrvfl} for the DL setting, consisting in interleaving local update steps with global averaging over the output weight vector. Practically, we consider the following algorithm:

\begin{enumerate}
	\item \textbf{Initialization}: the nodes agree on parameters $\vect{w}_1, \ldots, \vect{w}_B$ in Eq. \eqref{eq:flnn_model}. The same considerations made in Section \ref{sec:consrvfl} apply here. Moreover, all the nodes initialize their own local estimate of the $\vect{P}$ matrix in Eq. \eqref{eq:p_update_rls} and Eq. \eqref{eq:m_matrix} as $\vect{P}_k[0] = \lambda^{-1} \vect{I}$, and their estimate of the output weight vector as $\boldbeta_k[0] = \vect{0}$.
	\item At every iteration $n+1$, each node $k$ receives a new batch $S_{k,n+1}$. The following steps are performed:
	\begin{enumerate}
		\item \textbf{Local update}: every node computes (locally) its estimate $\boldbeta_k[n+1]$ using Eqs. \eqref{eq:p_update_rls}-\eqref{eq:brls} and local data $S_{k,n+1}$.
		\item \textbf{Global average}: the nodes agree on a single parameter vector by averaging their local estimates with a DAC protocol. The final weight vector at iteration $n+1$ is then given by:
		\begin{equation}
		\boldbeta[n+1] = \frac{1}{L}\sum_{k=1}^L \boldbeta_k[n+1] \,.
		\end{equation}
	\end{enumerate}
\end{enumerate}

\noindent The overall algorithm, denoted as S-CONS-RVFL, is summarized in Algorithm \ref{alg:dist_brls_rvlf}.

\begin{AlgorithmCustomWidth}[h]
    \caption{S-CONS-RVFL: Extension of CONS-RVFL to the sequential setting ($k$th node).}
    \label{alg:dist_brls_rvlf}
  \begin{algorithmic}[1]
    \Require{Number of nodes $L$ (global), regularization factor $\lambda$ (global)}
    \Ensure{Optimal weight vector $\boldbeta_k^*$}
    \State Select parameters $\vect{w}_1, \dots, \vect{w}_B$, in agreement with the other $L-1$ nodes.
    \State $\vect{P}_k[0] = \lambda^{-1} \vect{I}$.
    \State $\boldbeta_k[0] = \vect{0}$.
    \For{$n = 1, \ldots, T$}
    \State Receive batch $S_{k,n}$.
    \State Update $\boldbeta_k[n+1]$ using Eqs. \eqref{eq:p_update_rls}-\eqref{eq:brls}.
    \State $\boldbeta_k[n+1] \leftarrow \text{DAC}(\boldbeta_k[n+1], \ldots, \boldbeta_L[n+1])$. \Comment Run in parallel, see Appendix \ref{app:graph_theory}.
    \EndFor
    \State \textbf{return} $\boldbeta_k[T]$
  \end{algorithmic}
\end{AlgorithmCustomWidth}

\section{Experiments on Distributed Music Classification}
\label{chap5:sec:dist_music_classification}

\subsection{The Distributed Music Classification Problem}

As an experimental setting, we consider the problem of distributed automatic music classification (AMC). AMC is the task of automatically assigning a song to one (or more) classes, depending on its audio content \cite{scardapane2013music}. It is a fundamental task in many music information retrieval (MIR) systems, whose broader scope is to efficiently retrieve songs from a vast database depending on the user's requirements \cite{fu2011survey}. Examples of labels that can be assigned to a song include its musical genre, artist \cite{ellis2007classifying}, induced mood \cite{fu2011survey} and leading instrument. Classically, the interest in music classification is two-fold. First, being able to correctly assess the aforementioned characteristics can increase the efficiency of a generic MIR system (see survey \cite{fu2011survey} and references therein). Secondly, due to its properties, music classification can be considered as a fundamental benchmark for supervised learning algorithms \cite{scardapane2013music}: apart from the intrinsic partial subjectivity of assigning labels, datasets tend to be relatively large, and a wide variety of features can be used to describe each song. These features can also be supplemented by meta-informations and social tags.

More formally, we suppose that the input $\vect{x} \in \R^d$ to the model is given by a suitable $d$-dimensional representation of a song. Examples of features that can be used in this sense include temporal features such as the zero-crossing count, compact statistics in the frequency and cepstral domain \cite{fu2011survey}, higher-order descriptors (e.g. timbre \cite{ellis2007classifying}), meta-information on the track (e.g., author), and social tags extracted from the web. The output is instead given by one of $M$ predefined classes, where each class represents a particular categorization of the song, such as its musical genre. In the distributed AMC setting, these songs are distributed over a network, as is common in distributed AMC on peer-to-peer (P2P) systems, and over wireless sensor networks \cite{ravindran2004low}.

\subsection{Experiment Setup}
\label{chap5:sec:exp-setup}
We use four freely available AMC benchmarks. A schematic description of their characteristics is given in Table \ref{chap5:tab:datasets}.

\begin{center}
\begin{table*}[h]
\caption{General description of the datasets for testing the sequential S-CONS-RVFL algorithm.}
{\hfill{}
	\setlength{\tabcolsep}{4pt}
	\renewcommand{\arraystretch}{1.3}
	\begin{footnotesize}
		\begin{tabular}{lcclcc}  
			\toprule
			\textbf{Dataset name} & \textbf{Features} & \textbf{Instances} & \textbf{Task} & \textbf{Classes}& \textbf{Reference}\\
			\midrule
			Garageband &  49 &  1856 & Genre recognition & 9 & \cite{mierswa2005automatic} \\
			\midrule
			Latin Music Database (LMD) & 30 & 3160 & Genre recognition & 10 & \cite{silla2007automatic} \\
			\midrule
			Artist20 & 30 & 1413 & Artist recognition & 20 & \cite{ellis2007classifying} \\
			\midrule
			YearPredictionMSD & 90 & 200000 & Decade identification & 2 & \cite{bertin2011million} \\
			\bottomrule
		\end{tabular}
	\end{footnotesize}}
	\hfill{}\vspace{0.6em}
	\label{chap5:tab:datasets}
\end{table*}
\end{center}

\vspace{-3.5em} \noindent Below we provide more information on each of them.
\begin{itemize}
	\item \textit{Garageband} \cite{mierswa2005automatic} is a genre classification dataset, considering $1856$ songs and $9$ different genres (alternative, blues, electronic, folkcountry, funksoulrnb, jazz, pop, raphiphop and rock). The input is given by $49$ features extracted according to the procedure detailed in \cite{mierswa2005automatic}. It is the same as the one used in the previous chapter.
	\item \textit{LMD} is another genre classification task, of higher difficulty \cite{silla2007automatic}. In this case, we have $3160$ different songs categorized in $10$ Latin American genres (tango, bolero, batchata, salsa, merengue, ax, forr, sertaneja, gacha and pagode). The input is a $30$-dimensional feature vector, extracted from the middle $30$ seconds of every song using the Marsyas software.\footnote{\url{http://marsyas.info/}} Features are computed both in the frequency domain (e.g. the spectral centroid) and in the cepstral domain, i.e. Mel Frequency Cepstral Coefficients (MFCC).
	\item \textit{Artist20} is an artist recognition task comprising $1413$ songs distributed between $20$ different artists \cite{ellis2007classifying}. The $30$-dimensional input vector comprises both MFCC and chroma features (see \cite{ellis2007classifying} for additional details).
	\item \textit{YearPredictionMSD} is a year recognition task derived from the subset of the million song dataset \cite{bertin2011million} available on the UCI machine learning repository.\footnote{\url{https://archive.ics.uci.edu/ml/}} It is a dataset of $500000$ songs categorized by year. In our experiment, we consider a simplified version comprising only the initial $200000$ songs, and the following binary classification output: a song is of class (a) if it was written previously than $2000$, and of class (b) otherwise. This is a meaningful task due to the unbalance of the original dataset with respect to the decade $2001-2010$.
\end{itemize}
In all cases, input features were normalized between $-1$ and $+1$ before the experiments. Testing accuracy is computed over a $10$-fold cross-validation of the data, and every experiment is repeated $50$ times to average out randomness effects due to the initialization of the parameters. Additionally, to increase the dataset size, we artificially replicate twice the training data for all datasets, excluding YearDatasetMSD.

We consider networks of $8$ nodes, whose topology is constructed according to the `Erd\H{o}s$-$R\'enyi model' (see Section \ref{chap4:sec:algo_implementation}). In particular, every pair of nodes in the network has a $20\%$ probability of being connected, with the only constraint that the overall network is connected. Training data is distributed evenly across the nodes, and chunks are constructed such that every batch is composed of approximately $20$ examples ($100$ for the YearPredictionMSD dataset). We compare the following algorithms:
\begin{itemize}
	\item \textbf{Sequential CONS-RVFL} (S-CONS-RVFL): this is trained according to the consensus-based sequential algorithm. For the DAC procedure, we set the maximum number of iterations to $300$, and $\delta = 10^{-4}$.
	\item \textbf{Centralized RVFL} (C-RVFL): this is a RVFL trained by first collecting all the local chunks and aggregating them in a single batch. It can be considered as an upper bound on the performance of S-CONS-RVFL.
	\item \textbf{Local RVFL} (L-RVFL): in this case, nodes update their estimate using their local batch, but no communication is performed. Final misclassification error is averaged across the nodes. This can be considered as a worst-case baseline for the performance of any distributed algorithm for RVFL networks.
\end{itemize}
In all cases, we use sigmoid hidden functions given by Eq. \eqref{eq:sigmoid}. Optimal parameters for C-RVFL are found by executing an inner $3$-fold cross-validation on the training data. In particular, we search the uniform interval ${\left\{ 50, 100, 150, \dots, 1000 \right\}}$ for the number of hidden nodes, and the exponential interval $2^j$, ${j \in \left\{ -10, -9, \dots, 9, 10 \right\}}$ for $\lambda$. These parameters are then shared with L-RVFL and S-CONS-RVFL. Resulting parameters from the grid search procedure are listed in Table \ref{chap5:tab:gridsearch}.

\begin{center}
	\begin{table}[h]
		\caption[Optimal parameters found by the grid-search procedure for S-CONS-RVFL]{Optimal parameters found by the grid-search procedure.}
		{\centering\hfill{}
			\setlength{\tabcolsep}{4pt}
			\renewcommand{\arraystretch}{1.3}
			\begin{footnotesize}
				\begin{tabular}{lcl}   
					\toprule
					\textbf{Dataset} & \textbf{Hidden nodes} & $\lambda$   \\
					\midrule
					Garageband & $300$ & $2^{-3}$ \\
					\midrule
					LMD & $400$ & $2^{-2}$ \\
					\midrule
					Artist20 & $200$ & $2^{-4}$ \\
					\midrule
					YearPredictionMSD & $300$ & $1$ \\
					\bottomrule
				\end{tabular}
			\end{footnotesize}
		}
		\hfill{}\vspace{0.6em}
		\label{chap5:tab:gridsearch}
	\end{table}
\end{center}

\vspace{-2.5em} \noindent 
We note that C-RVFL can be considered as a benchmark for audio classification using shallow neural networks. In fact, in \cite{scardapane2013music} it is shown that it outperforms a standard MLP trained using SGD.
\subsection{Results and Discussion}

We start our discussion of the results by analyzing the final misclassification error and training time for the three models, reported in Table \ref{chap5:tab:finalerrorandtime}. Results of the proposed algorithm, S-CONS-RVFL, are highlighted in bold.

\begin{center}
	\begin{table}[h]
		\caption[Final misclassification error and training time for the sequential S-CONS-RVFL algorithm, together with one standard deviation.]{Final misclassification error and training time for the three models, together with standard deviation. The proposed algorithm is highlighted in bold. Training time for S-CONS-RVFL and L-RVFL is averaged over the nodes.}
		{\centering\hfill{}
			\setlength{\tabcolsep}{4pt}
			\renewcommand{\arraystretch}{1.3}
			\begin{footnotesize}
				\begin{tabular}{llcc}
					\toprule
					\textbf{Dataset} & \textbf{Algorithm} & \textbf{Error} & \textbf{Time} [secs]  \\
					\midrule
					\multirow{3}{*}{Garageband $\,\,$}  & C-RVFL & $0.40 \pm 0.02$ & $0.24 \pm 0.09$ \\
					& L-RVFL & $0.45  \pm 0.03$ & $0.13 \pm 0.03$ \\
					& \textbf{S-CONS-RVFL} & $\vect{0.40  \pm 0.02}$ & $\vect{0.15 \pm 0.04}$ \\
					\midrule
					\multirow{3}{*}{LMD}  & C-RVFL & $0.25 \pm 0.02$ & $0.70 \pm 0.17$ \\
					& L-RVFL & $0.31 \pm 0.03$ & $0.46 \pm 0.08$ \\
					& \textbf{S-CONS-RVFL} & $\vect{0.26 \pm 0.02}$ & $\vect{0.49 \pm 0.10}$ \\
					\midrule
					\multirow{3}{*}{Artist20}  & C-RVFL & $0.37 \pm 0.04$ & $0.13 \pm 0.07$ \\
					& L-RVFL & $0.47 \pm 0.04$ & $0.06 \pm 0.01$ \\
					& \textbf{S-CONS-RVFL} & $\vect{0.37 \pm 0.04}$ & $\vect{0.09 \pm 0.02}$ \\
					\midrule
					\multirow{3}{*}{YearPredictionMSD}  & C-RVFL & $0.27 \pm 0.01$ & $8.66 \pm 0.93$ \\
					& L-RVFL & $0.27 \pm 0.01$ & $2.35 \pm 0.48$ \\
					& \textbf{S-CONS-RVFL} & $\vect{0.27 \pm 0.01}$ & $\vect{2.46 \pm 0.62}$ \\
					\bottomrule
				\end{tabular}
			\end{footnotesize}
		}
		\hfill{}\vspace{0.6em}
		\label{chap5:tab:finalerrorandtime}
	\end{table}
\end{center}

\vspace{-2.5em} \noindent
Whenever we consider medium-sized datasets, the performance of L-RVFL is strictly worse than the performance of C-RVFL (similarly to the previous chapter), ranging from an additional $5\%$ misclassification error for Garageband and LMD, up to an additional $10\%$ for Artist20. The most important fact highlighted in Table \ref{chap5:tab:finalerrorandtime}, however, is that S-CONS-RVFL is able to efficiently match the performance of C-RVFL in all situations, except for a small decrease in the LMD dataset. From a computational perspective, this performance is achieved with a very small overhead in terms of training time with respect to L-RVFL in all cases (as evidenced by the fourth column in Table \ref{chap5:tab:finalerrorandtime}).

In a sequential setting, the evolution of the testing error after every batch is equally as important as the final accuracy obtained. We report it in Fig. \ref{chap5:fig:error}\subref{chap5:fig:garageband}-\subref{chap5:fig:msd} for the four datasets. 

\begin{figure*}[t]
	\centering
	\subfloat[Dataset Garageband]{\includegraphics[scale=0.75]{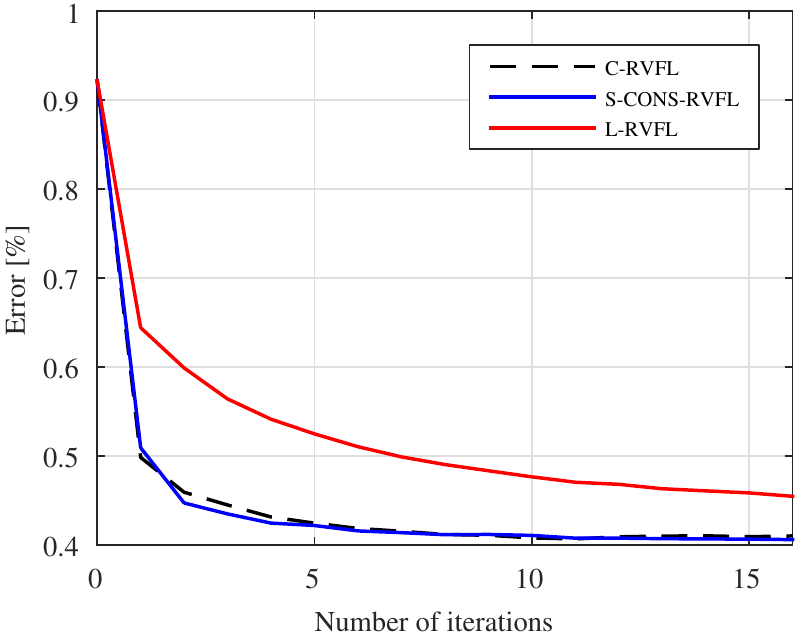}%
		\label{chap5:fig:garageband}}
	\hfil
	\subfloat[Dataset LMD]{\includegraphics[scale=0.75]{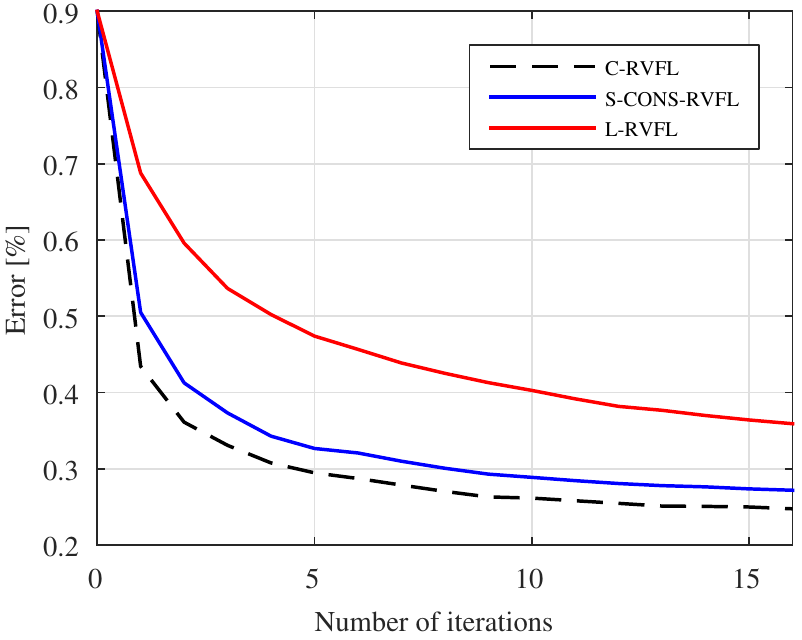}%
		\label{chap5:fig:lmd}}
	\vfill
	\subfloat[Dataset Artist20]{\includegraphics[scale=0.75]{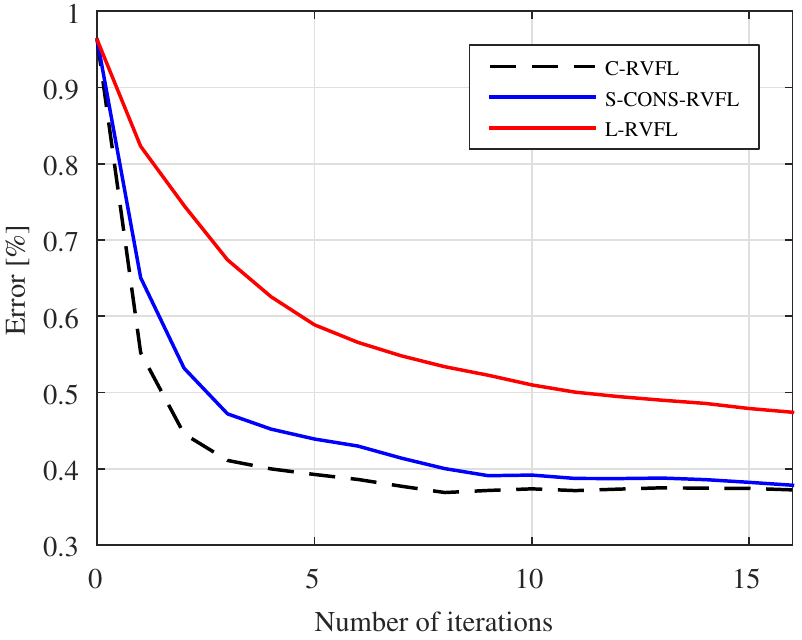}%
		\label{chap5:fig:artist20}}
	\hfil
	\subfloat[Dataset YearPredictionMSD]{\includegraphics[scale=0.75]{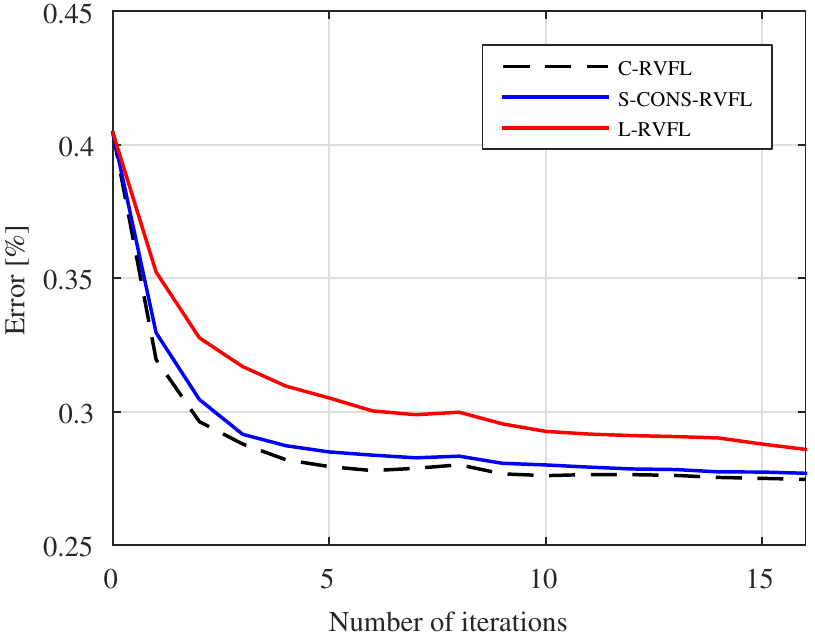}%
		\label{chap5:fig:msd}}
	\vfill
	\caption[Evolution of the testing error for the sequential S-CONS-RVFL after every iteration.]{Evolution of the testing error after every iteration. Performance of L-RVFL is averaged across the nodes.}
	\label{chap5:fig:error}
\end{figure*}

Performance of C-RVFL, L-RVFL and S-CONS-RVFL are shown with dashed black, solid red and solid blue lines respectively. Moreover, performance of L-RVFL is averaged across the nodes. Once again, we see that S-CONS-RVFL is able to track very efficiently the accuracy obtained by C-RVFL. The performance is practically equivalent in the Garageband and YearPredictionMSD datasets (Fig. \ref{chap5:fig:error}\subref{chap5:fig:garageband} and Fig. \ref{chap5:fig:error}\subref{chap5:fig:msd}), while convergence speed is slightly slower in the LMD and Artist20 case (Fig. \ref{chap5:fig:error}\subref{chap5:fig:lmd} and Fig. \ref{chap5:fig:error}\subref{chap5:fig:artist20}), although by a small amount. This gap depends on the fact that S-CONS-RVFL remains an approximation of C-RVFL. In particular, in the current version of S-CONS-RVFL no information is exchanged with respect to the state matrices $\vect{P}_k[n]$, which would be infeasible for large $B$ (see also the similar observation for the batch case in Section \ref{chap4:sec:dist_training_strategies_rvfl_networks}).

Next, we investigate the behavior of S-CONS-RVFL when varying the size of the network. In fact, due to its parallel nature, we expect that, the higher the number of nodes, the lower the training time (apart from communication bottlenecks, depending on the real channel of the network). The following experiments show that the increase in time required by the DAC procedure for bigger networks is more than compensated by the gain in time obtained by processing a lower number of samples per node. To this end, we consider the training time required by CONS-RVFL when varying the number of nodes of the network from $2$ to $14$ by steps of $2$, keeping the same topology model as before. Results of this experiment are presented in Fig. \ref{chap5:fig:training_time}\subref{chap5:fig:time-a} for datasets Garageband and Artist20, and in Fig. \ref{chap5:fig:training_time}\subref{chap5:fig:time-b} for datasets LMD and YearPredictionMSD. 

\begin{figure*}[t]
	\centering
	\subfloat[Datasets Garageband and Artist20]{\includegraphics[scale=0.75]{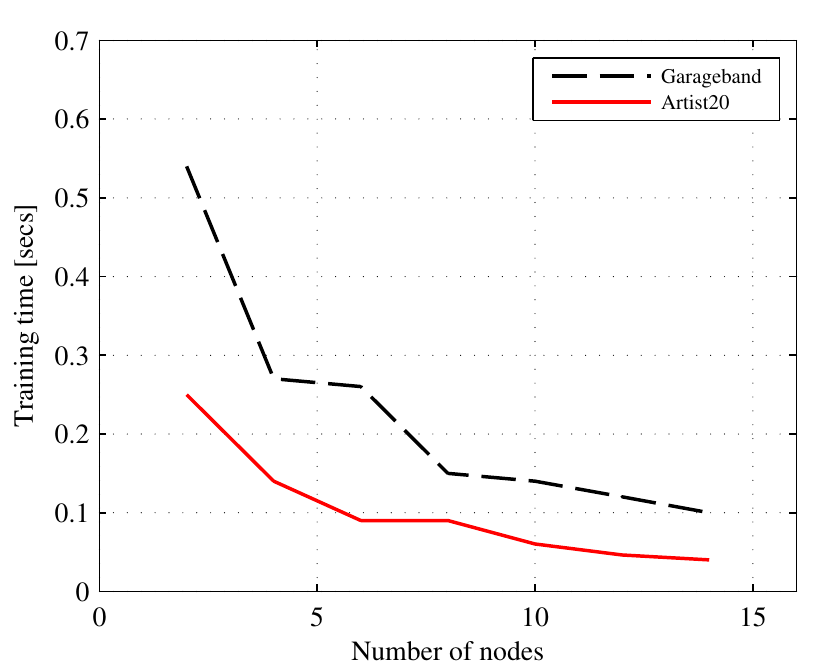}%
		\label{chap5:fig:time-a}}
	\hfil
	\subfloat[Dataset LMD and YearPredictionMSD]{\includegraphics[scale=0.75]{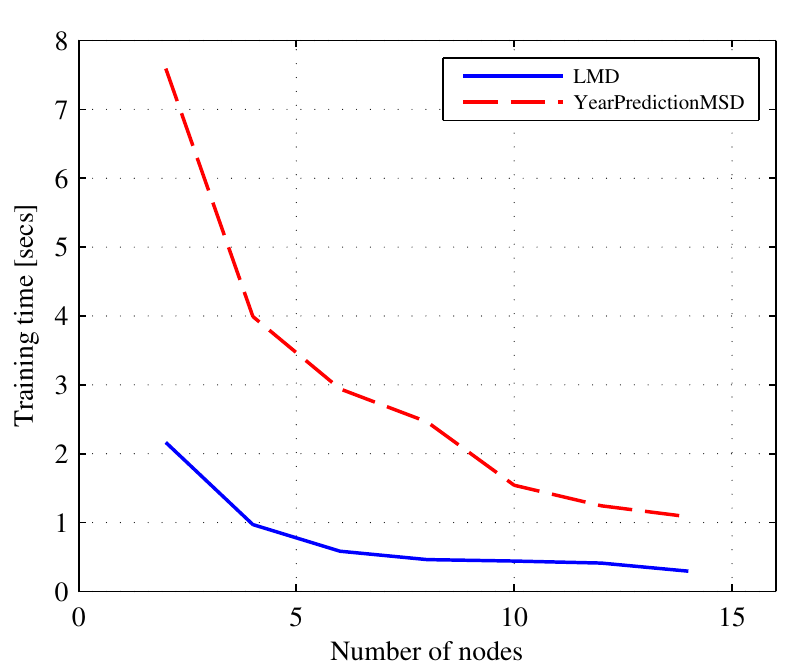}%
		\label{chap5:fig:time-b}}
	\vfill
	\caption{Training time required by the sequential S-CONS-RVFL, for varying sizes of the network, from $2$ to $14$ by steps of $2$.}
	\label{chap5:fig:training_time}
\end{figure*} 

The decrease in training time is extremely pronounced for Garageband, with a five-fold decrease going from $2$ to $14$ nodes, and for YearPredictionMSD, with a seven-fold decrease. This result is especially important, showing that S-CONS-RVFL can be efficiently used in large-scale situations. It is also consistent with the analysis of the batch CONS-RVFL in the previous chapter. Similarly, the number of consensus iterations needed to reach the desired accuracy is shown in Fig. \ref{chap5:fig:consensus_iterations}. 

\begin{figure}[t]
	\centering
	\includegraphics[scale=0.9]{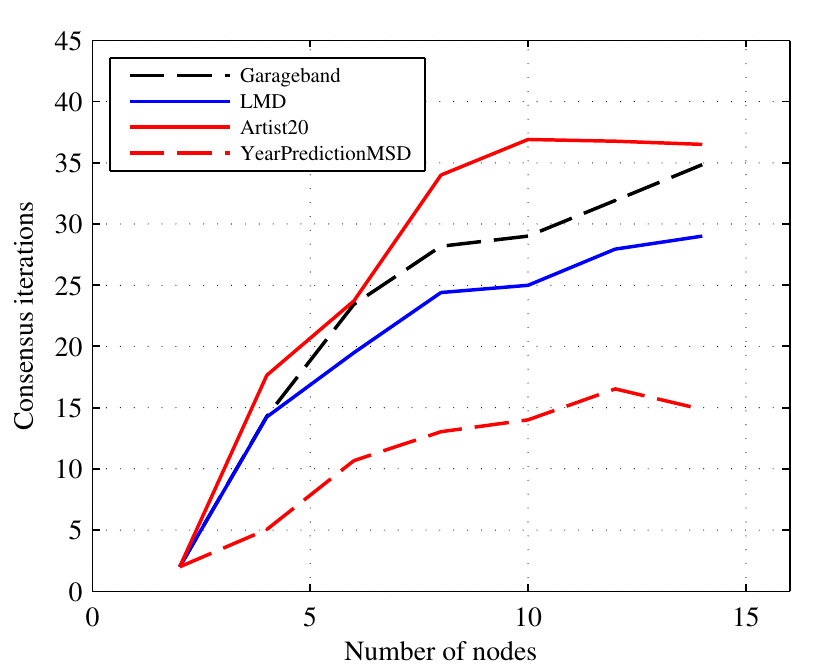}
	\caption{Number of consensus iterations required to reach convergence in the S-CONS-RVFL, when varying the number of nodes in the network from $2$ to $14$.}
	\label{chap5:fig:consensus_iterations}
\end{figure}

Although the required number of iterations grows approximately linearly with respect to the size of the network, a low number of iterations is generally enough to reach convergence to a very good accuracy. In fact, no experiment in this section required more than $35$ iterations in total. Additionally, the consensus procedure is extremely robust to a change in the network topology, as shown in the previous chapter. The same considerations apply here.

\section{Comparison of DAC strategies}
\label{chap5:sec:dac_comparison}

Up to this point, we have considered only the `max-degree' strategy for the DAC protocol, detailed in Appendix \ref{app:graph_theory}. However, it is known that the performance of the DAC protocol, and by consequence the performance of any distributed training algorithm based on its application, can improve significantly with proper choices of the mixing parameters \citep{xiao2004fast}. At the same time, a thorough investigation of multiple strategies for choosing the weights is missing in the literature. In this section, we compare four of them in the context of S-CONS-RVFL, ranging from choosing a fixed value for every coefficient, to more complex choices satisfying strong optimality conditions. Our experimental results show that the performance of the DAC protocol, and by consequence the performance of any distributed training algorithm based on its application, can improve significantly with proper choices of the mixing parameters.

\subsection{Description of the strategies}

Different strategies for the DAC protocol corresponds to different choices of the weights matrix. Clearly, the choice of a particular weight matrix depends on the available information at every node about the network topology, and on their specific computational requirements. Apart from the max-degree strategy, we consider three additional strategies, which are briefly detailed next.

\subsubsection{Metropolis-Hastings} The Metropolis-Hastings weights matrix is defined as:
\begin{equation}
C_{ij}=
\begin{cases} 
	\frac{1}{\max\lbrace d_i,d_j\rbrace +1} &  j \in \mathcal{N}_i \\ 
	1 - \sum_{j\in\mathcal{N}_i} \frac{1}{\max\lbrace d_i,d_j\rbrace+1} & i=j \\ 
	0 & \text{otherwise} 
\end{cases} \;.
\end{equation}
\noindent Differently from the max-degree strategy, the Metropolis-Hastings strategy does not require the knowledge of global information (the maximum degree) about the network topology, but requires that each node knows the degrees of all its neighbors.

\subsubsection{Minimum Asymptotic} The third matrix strategy considered here corresponds to the optimal strategy introduced in \cite{xiao2004fast}, wherein the weights matrix is constructed to minimize the asymptotic convergence factor $\rho(\vect{C} - \frac{\vect{1}\vect{1}^{\mathrm{T}}}{L})$, where $\rho(\cdot)$ denotes the \mbox{spectral} radius operator. This is achieved by solving the constrained optimization \mbox{problem}:
\begin{equation}
	\begin{aligned}
		& \mathrm{minimize}\quad\rho(\vect{C}- \frac{\vect{1}\vect{1}^{\mathrm{T}}}{L}) \\ 
		& \mathrm{subject \ to}\quad \vect{C}\in \mathcal{C}, \quad  \vect{1}^{\mathrm{T}}\vect{C}=\vect{1}^{\mathrm{T}},\quad \vect{C}\vect{1}=\vect{1}
	\end{aligned}
\label{eq:min_asymptotic}
\end{equation}
\noindent where $\mathcal{C}$ is the set of possible weight matrices. Problem \eqref{eq:min_asymptotic} is non-convex, but it can be shown to be equivalent to a semidefinite programming (SDP) problem \cite{xiao2004fast}, solvable using efficient ad-hoc algorithms.

\subsubsection{Laplacian Heuristic} The fourth and last matrix considered here is an heuristic approach based on constant edge weights matrix \cite{xiao2004fast}:
\begin{equation}
\vect{C} = \vect{I} - \alpha \vect{L} \,,
\label{eq:const_edge}
\end{equation}
where $\alpha \in \R$ is a user-defined parameter, and $\vect{L}$ is the Laplacian matrix associated to the network (see Appendix \ref{app:graph_theory}). For weights matrices in the form of \eqref{eq:const_edge}, the asymptotic convergence factor satisfies:
\begin{equation}
	\begin{aligned}
		\rho(\vect{C} - \frac{\vect{1}\vect{1}^{\mathrm{T}}}{L}) & = \max\lbrace\lambda_2(\vect{C}),-		\lambda_n(\vect{C})\rbrace \\
		& = \max\lbrace 1-\alpha\lambda_{n-1}(\vect{L}),\alpha\lambda_1(\vect{L})-1\rbrace \,,
	\end{aligned}
\label{eq:spec_radius}
\end{equation}
where $\lambda_i(\vect{C})$ denotes the $i$-th eigenvalue associated to $\vect{C}$. The value of $\alpha$ that minimizes \eqref{eq:spec_radius} is given by:
\begin{equation}
\alpha^*=\frac{2}{\lambda_1(\vect{L})+\lambda_{L-1}(\vect{L})} \,.
\end{equation}

\subsection{Experimental Results}

We compare the performance of the $4$ different strategies illustrated in the previous section in terms of number of iterations required to converge to the average and speed of convergence. In order to avoid that a particular network topology compromises the statistical significance of the experiments, we perform $25$ rounds of simulation. In each round, we generate a random topology for an $8$-nodes network, according to the Erd\H{o}s-R\'enyi model with $p=0.5$. We consider $2$  datasets: G50C (detailed in Section \ref{chap4:sec:description_datasets}); and CCPP, a regression dataset with $4$ features and $9568$ examples taken from the UCI repository.\footnote{\url{https://archive.ics.uci.edu/ml/datasets/Combined+Cycle+Power+Plant}} At each round, datasets are subdivided in batches following the procedure detailed in the previous experimental section. Since in real applications the value of the average is not available to the nodes, in order to evaluate the number of iterations, we consider that all the nodes reached consensus when $\Vert\boldbeta_k(t)-\boldbeta_k(t-1)\Vert^2\leq 10^{-6}$ for any possible value of $k$. In Fig. \ref{chap5:fig:iterations} we show the average number of iterations required by the DAC protocol, averaged over the rounds. 

\begin{figure*}[h]
\centering
\subfloat[Dataset: G50C]{\includegraphics[scale=0.8]{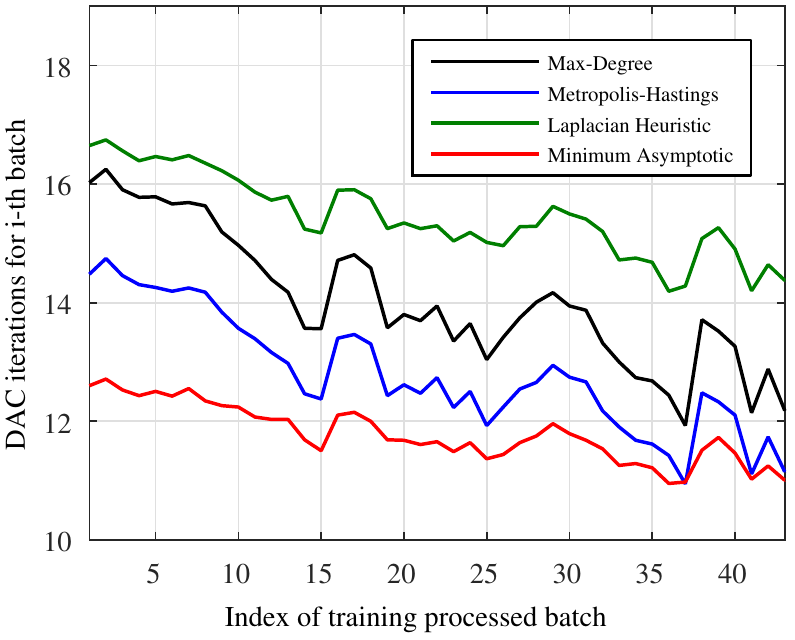} \label{chap5:fig:g50c_iter}}
\hfill
\subfloat[Dataset: CCPP]{\includegraphics[scale=0.8]{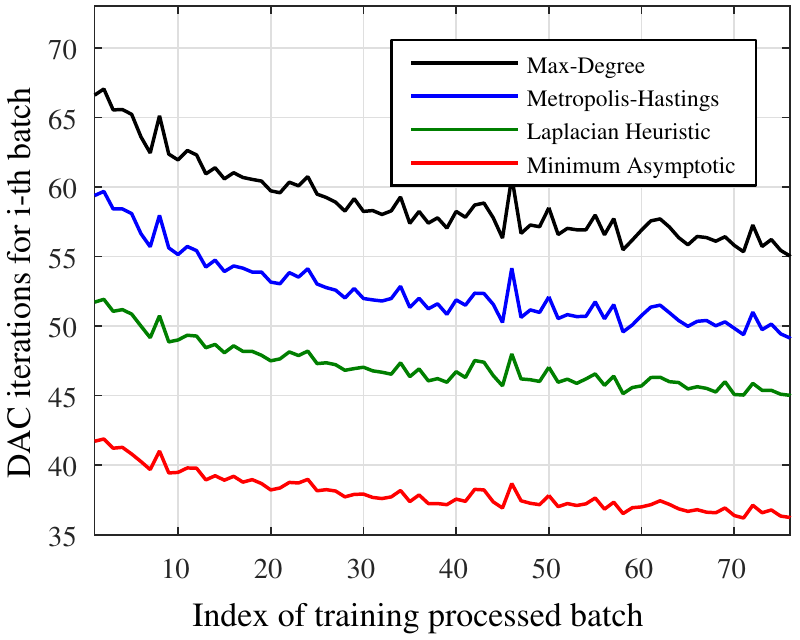} \label{chap5:fig:CCPP_iter}}
\vfill
\caption[Evolution of the DAC iterations required by four different strategies, when processing successive amounts of training batches.]{Evolution of the DAC iterations required by the considered strategies to converge to the average, when processing successive amounts of training batches.}
\label{chap5:fig:iterations}
\end{figure*}

The $x$-axis in Fig. \ref{chap5:fig:iterations} shows the index of the processed batch. As expected, the number of DAC iterations shows a decreasing trend as the number of processed training batches grows, since the nodes are slowly converging to a single RVFL model. The main result in Fig. \ref{chap5:fig:iterations}, however, is that a suitable choice of the mixing strategy can significantly improve the convergence time (and hence the training time) required by the algorithm. In particular, the optimal strategy defined by Eq. \eqref{eq:min_asymptotic} achieves the best performance, with a reduction of the required number of iterations up to $35\%$ and $28\%$ when compared with max-degree and Metropolis-Hasting strategies respectively. On the other side, the strategy based on constant edge matrix in Eq. \eqref{eq:const_edge} shows different behaviors for the $2$ datasets, probably due to the heuristic nature of this strategy.

The second experiment, whose results are shown in Fig. \ref{chap5:fig:disagreement}, is to show the speed of convergence for the considered strategies. This is made by evaluating the trend of the relative network disagreement:

\begin{equation}
RND(t)=\frac{1}{N}\sum_{i=1}^N\frac{\Vert\boldbeta_i(t)-\hat{\boldbeta}\Vert^2}{\Vert\boldbeta_i(0)-\hat{\boldbeta}\Vert^2} \,,
\label{eq:network_disagreement}
\end{equation}

as the number of DAC iterations increases. 
\begin{figure*}[h]
\centering
\subfloat[Dataset: G50C]{\includegraphics[scale=0.8]{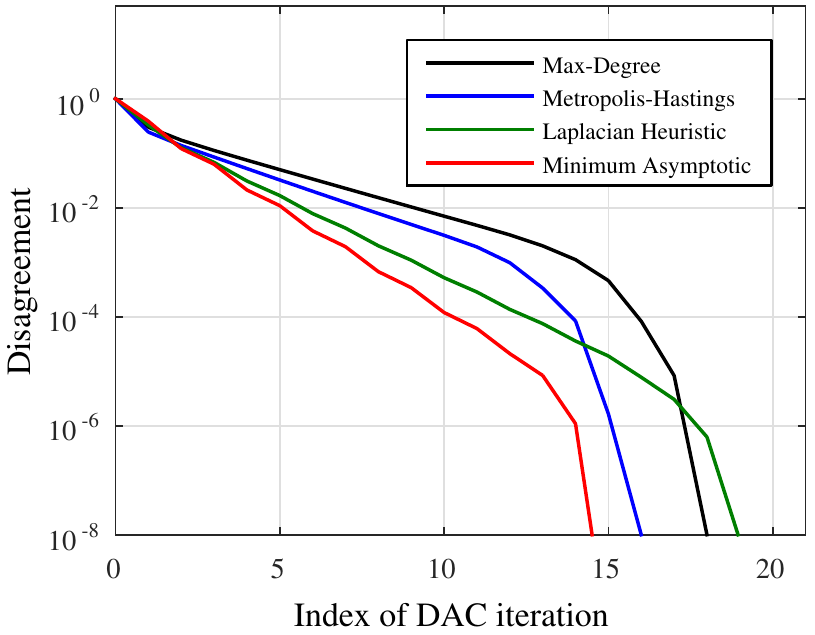} \label{chap5:fig:g50c_disagr}}
\hfill
\subfloat[Dataset: CCPP]{\includegraphics[scale=0.8]{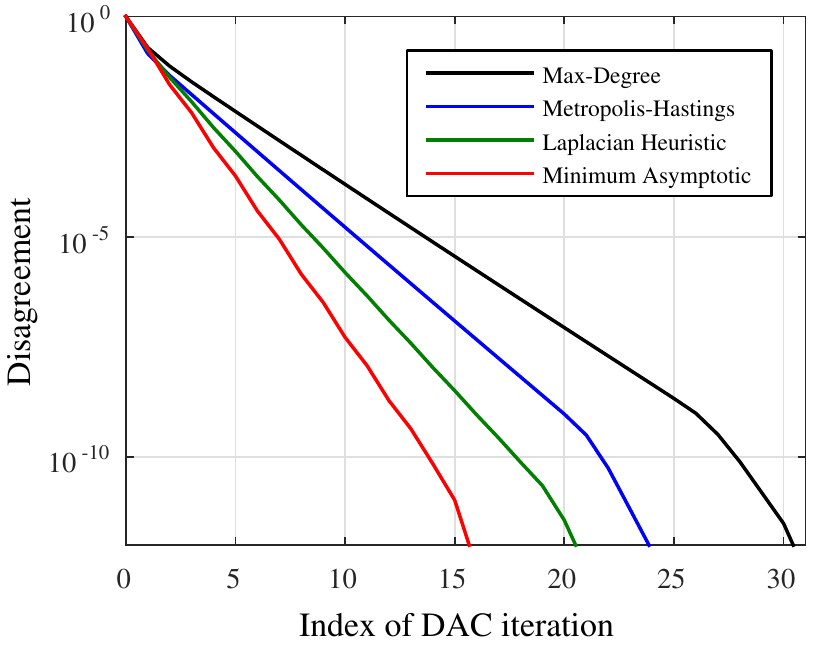} \label{chap5:fig:CCPP_disagr}}
\vfill
\caption[Evolution of the relative network disagreement for four different DAC strategies as the number of DAC iterations increases.]{Evolution of the relative network disagreement for the considered strategies as the number of DAC iterations increases. The $y$-axis is shown with a logarithmic scale.}
\label{chap5:fig:disagreement}
\end{figure*}
The value of $\hat{\boldbeta}$ in Eq. \eqref{eq:network_disagreement} is the true average. The $y$-axis in Fig. \ref{chap5:fig:disagreement} is shown with a logarithmic scale. Results show that the ``optimal'' strategy has the fastest speed of convergence, as expected, while it is interesting to notice how, when compared to max-degree and Metropolis-Hastings weights, the heuristic strategy achieves a rapid decrease in disagreement in the initial iterations, while its speed become slower in the end (this is noticeable in Fig. \ref{chap5:fig:g50c_disagr}). This may help to explain the lower performance of this strategy in Fig. \ref{chap5:fig:g50c_iter}.

Overall, this set of experimental results show how an appropriate choice of the weights matrix can lead to considerable improvements both in the number of iterations required by the protocol to converge to the average, and in the speed of convergence. In particular, when compared to other strategies, an ``optimal'' choice of the weights matrix can save up to $30\%$ in time.

%% file: chapters/chapter6-dist_rvfl_vertical_partitioning.tex
\chapter[Distributed RVFL Networks with Vertically Partitioned Data]{Distributed RVFL Networks with \\ Vertically Partitioned Data}
\chaptermark{Vertically-partitioned RVFL Networks}
\label{chap:dist_rvfl_vertical_partitioning}

\minitoc
\vspace{15pt}

\blfootnote{The content of this chapter is adapted from the material published in \citep{scardapane2015hetfeatures}.}

\lettrine[lines=2]{T}{his} chapter presents an extension of the ADMM-RVFL algorithm presented in Section \ref{sec:admmrvfl} to the case of vertically partitioned (VP) data. In the VP scenario, the features of every pattern are partitioned over the nodes. A prototypical example of this is found in the field of distributed databases \cite{lazarevic2002boosting}, where several organizations possess only a partial view on the overall dataset (e.g., global health records distributed over multiple medical databases). In the centralized case, this is also known as the problem of learning from heterogeneous sources, and it is typically solved with the use of ensemble procedures \cite{li2012flame}. However, as we show in our experimental results, in the VP setting naive ensembles over a network tend to achieve highly sub-optimal results, with respect to a fully centralized solution.

\section{Derivation of the algorithm}
\label{chap6:sec:description_datasets}

We suppose that the $k$th agent has access to a subset $\vect{x}_k$ of features, such that:
\[
\vect{x} = \left[ \vect{x}_1 \vline \ldots \vline \vect{x}_L \right]
\]
The main problem for the distributed training of an RVFL network in this setting is that the computation of any functional link in Eq. \eqref{eq:flnn_model} requires knowledge of the full sample. However, as we stated in the previous chapters, we would like to avoid exchange of data patterns, due to both size and privacy concerns. To this end, we approximate model in Eq. \eqref{eq:flnn_model} by considering local expansion blocks:
\begin{equation}
f(\vect{x}) = \sum_{k=1}^L \left( \boldsymbol{\beta}_k \right)^T \vect{h}_k(\vect{x}_k) \,.
\label{eq:rvfl_model_vp_distributed}
\end{equation} 
In this way, each term $\vect{h}_k(\vect{x}_k)$ can be computed locally. Input vectors and expansion blocks may have different lengths at every node, depending on the application and on the local computational requirements. This is shown pictorially in Fig. \ref{chap6:fig:Schema}.

\begin{figure}[t]
\centering
\includegraphics[width=0.65\columnwidth, keepaspectratio]{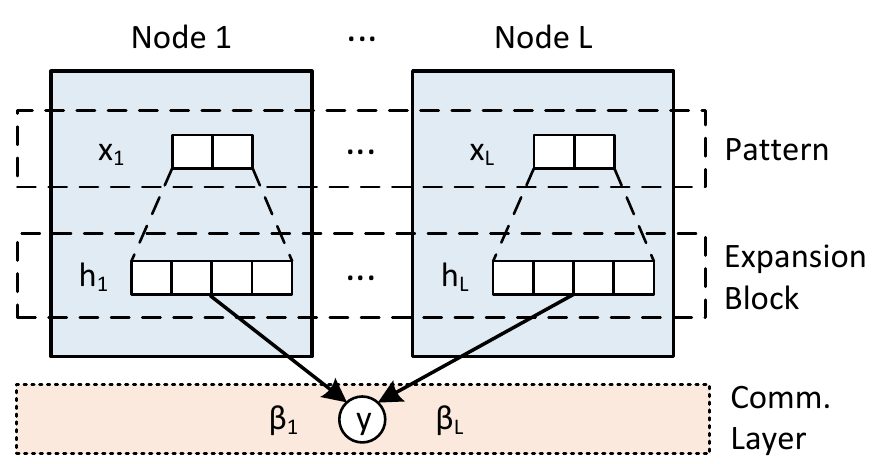}
\caption[Schematic description of the proposed algorithm for training an RVFL with vertically partitioned data.]{Schematic description of the proposed algorithm. Each node has access to a subset of the global pattern. This local feature vector is projected to a local expansion block, and the overall output is computed by a linear combination of the local expansions, through a suitable communication layer.}
\label{chap6:fig:Schema}
\end{figure}
 
The overall optimization problem becomes:
\begin{equation}
\argmin_{\boldsymbol{\beta}}  \frac{1}{2}\norm{\sum_{k=1}^L \vect{H}_k\boldsymbol{\beta}_k - \vect{y}}^2 + \frac{\lambda}{2}\sum_{k=1}^L \norm{\boldsymbol{\beta}_k}^2 \,,
\label{eq:rvfl_dist_opt}
\end{equation}
where $\vect{H}_k$ denotes the hidden matrix computed at the $k$-th node, such that $\vect{H} = \left[ \vect{H}_1 \vline \ldots \vline \vect{H}_L \right]$. The ADMM optimization algorithm can be adapted to this setting, as shown in \cite[Section 8.3]{boyd2011distributed}. To this end, we consider the equivalent optimization problem:
\begin{equation}
\begin{aligned}
&\underset{\boldbeta}{\text{minimize}}
& & \frac{1}{2} \norm{\sum_{k=1}^L \vect{z}_i - \vect{y}}^2 + \frac{\lambda}{2} \sum_{k=1}^L \norm{\boldbeta_k}^2 \\
&\;\;\;\;\text{subject to}
& & \vect{H}_k \boldbeta_k - \vect{z}_k = \vect{0}, \; k = 1\ldots L \,.
\end{aligned}
\label{eq:vp_rvfl_reformulation}
\end{equation}
where we introduced local variables $\vect{z}_k = \vect{H}_k \boldbeta_k$. The augmented Lagrangian of this problem is given by:
\begin{align}
\mathcal{L}(\boldbeta_k, \vect{z}_k, \vect{t}_k) & = \frac{1}{2}\norm{\sum_{k=1}^L \vect{z}_k - \vect{y}}^2 + \frac{\lambda}{2}\sum_{k=1}^L \norm{\boldsymbol{\beta}_k}^2 + \nonumber \\
			& + \sum_{k=1}^L \vect{t}_k^T \left( \vect{H}_k \boldbeta_k - \vect{z}_i \right) + \frac{\rho}{2} \sum_{k=1}^L \norm{\vect{H}_k \boldbeta_k - \vect{z}_i}^2 \,,
\label{eq:augm_lagrangian}
\end{align}
where $\vect{t}_k$ are the Lagrange multipliers, $\rho \in \R^+$ is a regularization factor, and the last term is added to ensure convergence. The solution to problem \eqref{eq:rvfl_dist_opt} can computed by iterating the updates in Eqs. \eqref{eq:admm_step1}-\eqref{eq:admm_step3}. Following the derivation in \cite[Section 8.3]{boyd2011distributed}, and computing the gradient terms, the final updates can be expressed as:
\begin{eqnarray}
\boldbeta_k[n+1] & = & \left( \frac{\lambda}{\rho} \vect{I} + \vect{H}_k^T \vect{H}_k \right)^{-1}\vect{H}_k^T \left( \vect{H}_k\boldbeta_k[n] + \overline{\vect{z}}[n] - \overline{\vect{H}\boldbeta}[n] - \vect{t}[n] \right) \,,\label{eq:admm_1_final}\\
\overline{\vect{z}}[n+1] & = & \frac{1}{L + \rho} \left( \vect{y} + \overline{\vect{H}\boldbeta}[n+1] + \vect{t}[n] \right) \,,\label{eq:admm_2_final}\\
\vect{t}[n+1] & = & \vect{t}[n] + \overline{\vect{H}\boldbeta}[n+1] - \overline{\vect{z}}[n+1] \,,\label{eq:admm_3_final}
\end{eqnarray}
where we defined the averages $\overline{\vect{H}\boldbeta}[n] = \frac{1}{L}\sum_{k=1}^L \vect{H}_k\boldbeta_k[n]$, and $\overline{\vect{z}}[n] = \sum_{k=1}^L \vect{z}_k[n]$. Additionally, the variables $\vect{t}_k$ can be shown to be equal between every node \cite{boyd2011distributed}, so we removed the subscript. Convergence of the algorithm can be tracked locally by computing the residual:
\begin{equation}
\vect{r}_k[n] = \vect{H}_k^T\boldbeta_k[n] - \vect{z}_k[n] \,.
\label{eq:residual}
\end{equation}
It can be shown that, for the iterations defined by Eqs. \eqref{eq:admm_1_final}-\eqref{eq:admm_3_final}, $\norm{\vect{r}_k[n]} \rightarrow 0$ as $n\rightarrow +\infty$, with the solution converging asymptotically to the solution of problem in Eq. \eqref{eq:rvfl_dist_opt}. The overall algorithm, denoted as VP-ADMM-RVFL, is summarized in Algorithm \ref{alg:admm_rvfl_vertical_partitioning}.

\begin{AlgorithmCustomWidth}[h]
    \caption{VP-ADMM-RVFL: Extension of ADMM-RVFL to vertically partitioned data ($k$th node).}
    \label{alg:admm_rvfl_vertical_partitioning}
  \begin{algorithmic}[1]
    \Require{Training set $S_k$, number of nodes $L$ (global), regularization factors $\lambda, \gamma$ (global), maximum number of iterations $T$ (global)}
    \Ensure{Optimal weight vector $\boldbeta^*$}
    \State Select parameters $\vect{w}_1, \dots, \vect{w}_B$, in agreement with the other $L-1$ nodes.
    \State Compute $\vect{H}_k$ and $\vect{y}_k$ from $S_k$.
    \State Initialize $\vect{t}[0] = \vect{0}$, $\overline{\vect{z}}[0] = \vect{0}$.
    \For{$n$ from $0$ to $T$}
    \State Compute $\boldbeta_k[n+1]$ according to Eq. \eqref{eq:admm_1_final}.
    \State Compute averages $\overline{\vect{H}\boldbeta}[n]$ and $\overline{\vect{z}}[n]$ with DAC.
    \State Compute $\overline{\vect{z}}[n+1]$ according to Eq. \eqref{eq:admm_2_final}.
    \State Update $\vect{t}[n]$ according to Eq. \eqref{eq:admm_3_final}.
    \State Check termination with residuals.
    \EndFor
    \State \textbf{return} $\overline{\vect{z}}[n+1]$
  \end{algorithmic}
\end{AlgorithmCustomWidth}

After training, every node has access to its own local mapping $\vect{h}_k(\cdot)$, and to its subset of coefficients $\boldbeta_k$. Differently from the horizontally partitioned (HP) scenario, when the agents require a new prediction, the overall output defined by Eq. \eqref{eq:rvfl_model_vp_distributed} has to be computed in a decentralized fashion. Once again, this part will depend on the actual communication layer available to the agents. As an example, it is possible to run the DAC protocol over the values $\left( \boldsymbol{\beta}_k \right)^T \vect{h}_k(\vect{x}_k)$, such that every node obtain a suitable approximation of $\frac{1}{L}f(\vect{x})$. For smaller networks, it is possible to compute an Hamiltonian cycle between the nodes \cite{predd2006distributed}. Once the cycle is known to the agents, they can compute Eq. \eqref{eq:rvfl_model_vp_distributed} by forward propagating the partial sums up to the final node of the cycle, and then back-propagating the result. Clearly, many other choices are possible, depending on the network.

\section{Experimental setup}

In this section, we present an experimental validation of the proposed algorithm on three classification tasks: Garageband, G50C and Sylva (detailed in Sections \ref{chap4:sec:description_datasets} and \ref{chap5:sec:exp-setup}). Optimal parameters for the RVFL network are taken from the corresponding sections. In our first set of experiments, we consider networks of $8$ agents, whose connectivity is randomly generated such that every pair of nodes has a $60\%$ probability of being connected, with the only global requirement that the overall network is connected. The input features are equally partitioned through the nodes, i.e., every node has access to roughly $d/8$ features, where $d$ is the dimensionality of the dataset. We compare the following algorithms:
\begin{description}
\item[\textbf{Centralized RVFL}] (C-RVFL): this corresponds to the case where a fusion center is available, collecting all local datasets and solving directly Eq. \eqref{eq:rvfl_opt}. Settings for this model are the optimal ones.
\item[\textbf{Local RVFL}] (L-RVFL): this is a naive implementation, where each node trains a local model with its own dataset, and no communication is performed. Accuracy of the models is then averaged throughout the $L$ nodes. As a general settings, we employ the same regularization coefficient for every node as C-RVFL, and $B_k = \lceil B/8 \rceil$ expansions in every agent.
\item[\textbf{Ensemble RVFL}] (ENS-RVFL): this corresponds to a basic distributed ensemble. As for L-RVFL, during the training phase every node trains a local model with its own dataset. In the testing phase, the nodes agree on a single class prediction by taking a majority vote over their local predictions. Parameters are the same as for L-RVFL.
\item[\textbf{Distributed RVFL}] (VP-ADMM-RVFL): this is trained using the distributed protocol introduced in the previous section. Settings are the same as L-RVFL, while for the ADMM we set $\rho = 0.1$ and a maximum number of $200$ iterations.
\end{description}
To compute the misclassification rate, we perform a $3$-fold cross-validation on the overall dataset, and repeat the procedure $15$ times.

\section{Results and discussion}
Results of the experiments are presented in Table \ref{chap6:tab:results}. 

\begin{center}
\begin{table*}[h]
\caption[Misclassification error and training time for VP-ADMM-RVFL.]{Misclassification error and training time for the four algorithms. Results are averaged over the $8$ different nodes of the network. Standard deviation is provided between brackets.}
{\hfill{}
\renewcommand{\arraystretch}{1.3}
\begin{footnotesize}
\begin{tabular}{llcc}  
	\toprule
	\textbf{Dataset} & \textbf{Algorithm} & \textbf{Misclassification error [\%]} & \textbf{Training time [secs.]} \\
	\midrule
	\multirow{4}{*}{Garageband} & C-RVFL & $41.32$ $(\pm 1.24)$ & $0.03$ $(\pm 0.02)$ \\
								& L-RVFL & $82.79$ $(\pm 3.82)$ & $0.01$ $(\pm 0.01)$ \\
								& ENS-RVFL & $61.01$ $(\pm 1.97)$ & $0.01$ $(\pm 0.01)$ \\
								& VP-ADMM-RVFL & $41.34$ $(\pm 1.34)$ & $2.35$ $(\pm 0.58)$ \\
	\midrule
	\multirow{4}{*}{Sylva} & C-RVFL & $1.18$ $(\pm 0.13)$ & $0.44$ $(\pm 0.06)$ \\
						   & L-RVFL & $49.80$ $(\pm 36.35)$ & $0.05$ $(\pm 0.02)$ \\
						   & ENS-RVFL & $6.04$ $(\pm 0.12)$ & $0.06$ $(\pm 0.02)$ \\
						   & VP-ADMM-RVFL & $1.22$ $(\pm 0.15)$ & $1.94$ $(\pm 0.40)$ \\
	\midrule
	\multirow{4}{*}{G50C} & C-RVFL & $5.80$ $(\pm 1.19)$ & $0.05$ $(\pm 0.02)$ \\
						  & L-RVFL & $49.51$ $(\pm 6.37)$ & $0.01$ $(\pm 0.01)$ \\
						  & ENS-RVFL & $10.98$ $(\pm 2.32)$ & $0.01$ $(\pm 0.01)$ \\
						  & VP-ADMM-RVFL & $5.80$ $(\pm 1.37)$ &  $0.38$ $(\pm 0.16)$ \\
	\bottomrule
\end{tabular}
\end{footnotesize}}
\hfill{}\vspace{0.6em}
\label{chap6:tab:results}
\end{table*}
\end{center}

\vspace{-2.5em}\noindent
It can be seen that, despite we approximate the global expansion block of C-RVFL using $L$ distinct local expansions, this has minimal or no impact on the global solution. In fact, VP-ADMM-RVFL is able to achieve performance comparable to C-RVFL in all three datasets, while the ensemble approach is performing relatively poorly: it has a $20\%$, $5\%$ and $5\%$ increase in error respectively in each dataset. This shows that the relatively common approach of averaging over the local models may be highly sub-optimal in practical situations. 

As a reference, in Table \ref{chap6:tab:results} we also provide the average training time spent at every node. However, we note that in our experiments the network was simulated in a serial architecture, removing all communication costs. Clearly, a practical analysis of this point would require knowledge of the communication layer, which goes beyond the scope of the thesis. Still, we can see from the fourth column of Table \ref{chap6:tab:results} that the proposed algorithm requires an acceptable computational time for performing the $200$ iterations, since the matrix inversions in Eq. \eqref{eq:admm_1_final} can be pre-computed at the beginning of the training process. Additionally, we add that the training time of VP-ADMM-RVFL can be greatly reduced in practice by the implementation of an efficient stopping criterion as in the previous chapter.

Finally, we show the evolution of the misclassification error for VP-ADMM-RVFL and ENS-RVFL when varying the size of the network from $L=4$ to $L=12$. Results of this experiment are given in Fig. \ref{chap6:fig:error} \subref{chap6:fig:garagebanderror}-\subref{chap6:fig:g50cerror}. Settings are kept fixed with respect to the previous experiment, while the features are equally partitioned as before (hence, for smaller networks each node has access to a larger subset of features). Performance of C-RVFL is given as a comparison with a dashed black line. As expected, we see that, although the behavior of ENS-RVFL strongly depends on the number of nodes in the network, VP-ADMM-RVFL is resilient to such change, always approximating very well the centralized performance. It is also interesting to note that the behavior of ENS-RVFL is not always monotonically increasing, as is shown in Fig. \ref{chap6:fig:error}-\subref{chap6:fig:g50cerror}, possibly due to its ensembling characteristics and to the artificial nature of the G50C dataset.

\begin{figure*}[h]
\centering
\subfloat[Dataset Garageband]{\includegraphics[scale=0.8]{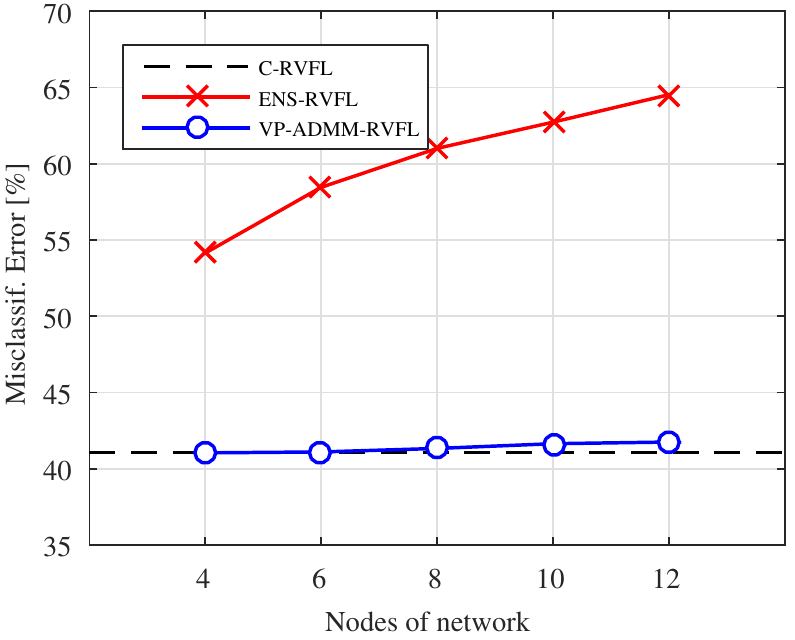}%
\label{chap6:fig:garagebanderror}}
\hfil
\subfloat[Dataset Sylva]{\includegraphics[scale=0.8]{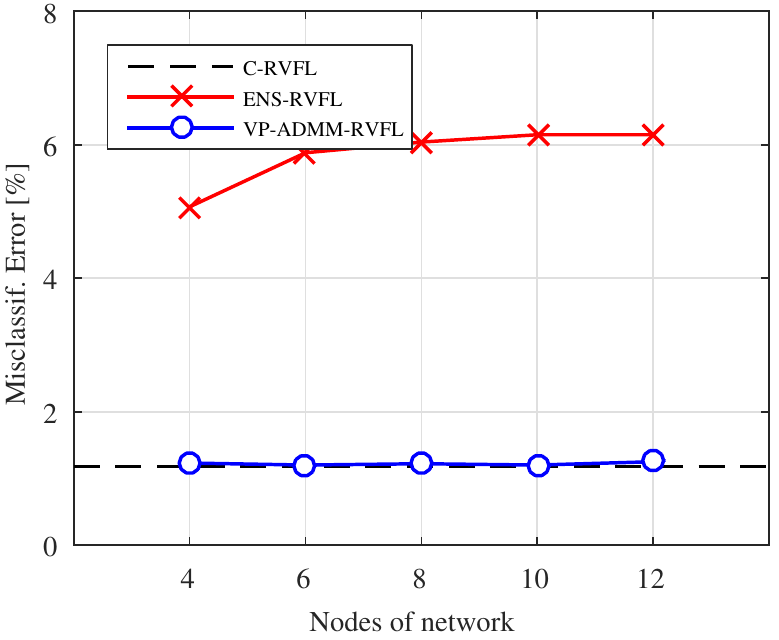}%
\label{chap6:fig:sylvaerror}}
\vfill
\subfloat[Dataset G50C]{\includegraphics[scale=0.8]{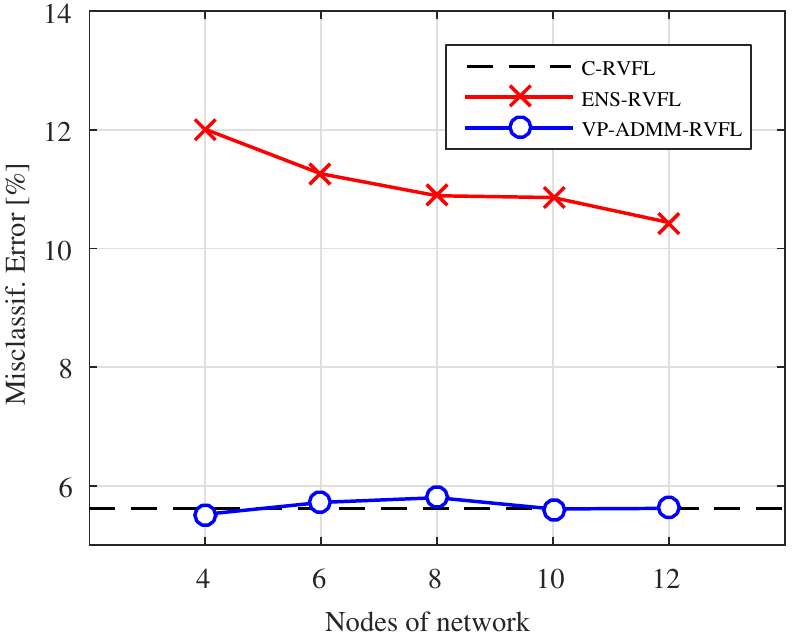}%
\label{chap6:fig:g50cerror}}
\hfil
\caption{Evolution of the error for VP-ADMM-RVFL and ENS-RVFL when varying the size of the network from $L=4$ to $L=12$.}
\label{chap6:fig:error}
\end{figure*}

%% file: chapters/chapter7-dist_ssl.tex
\chapter[Decentralized Semi-supervised Learning via Privacy-Preserving Matrix Completion]{Decentralized Semi-supervised Learning \\ via Privacy-Preserving Matrix Completion}
\chaptermark{Decentralized SSL via Matrix Completion}
\label{chap:dist_ssl}

\minitoc
\vspace{15pt}

\blfootnote{The content of this chapter has been (conditionally) accepted for publication at IEEE Transactions on Neural Networks and Learning Systems.}

\newcommand{\boldalpha}{\ensuremath{\boldsymbol{\alpha}}}

\section{Introduction}

\lettrine[lines=2]{A}{s} we saw in the previous chapters, many centralized SL algorithms have been extended successfully to the distributed setting. However, many crucial sub-areas of machine learning remain to be extended to the fully distributed scenario. Among these, the DL setting could benefit strongly from the availability of distributed protocols for semi-supervised learning (SSL) \citep{Chapelle2006}. In SSL, it is assumed that the labeled training data is supplemented by some additional unlabeled data, which has to be suitably exploited in order to improve the test accuracy. State-of-the-art research on SSL is concerned on the single-agent (centralized) case, e.g. with the use of manifold regularization (MR) \cite{belkin2006manifold,melacci2011laplacian}, transductive learning \cite{chapelle2008optimization}, and several others. To the best of our knowledge, the case of SSL over multiple agents has been addressed only in very specific settings, such as localization over WSNs \cite{chen2011semi}, while no algorithm is available for the general case. However, we argue that such an algorithm would be well suited for a wide range of applications. As an example, consider the case of medical diagnosis, with labeled and unlabeled data distributed over multiple clinical databases. Other examples include distributed text classification over peer-to-peer networks, distributed music classification (which we considered in Chapter \ref{chap:dist_rvfl_sequential}), and so on. In all of them, labeled data at every agent is costly to obtain, while unlabeled data is plentiful. The overall setting is summarized in Fig. \ref{chap7:fig:setting}, where each agent in a network receives two training datasets, one composed of labeled patterns and one composed of unlabeled patterns.

\begin{figure}[h]
\centering
\includegraphics[width=0.7\columnwidth]{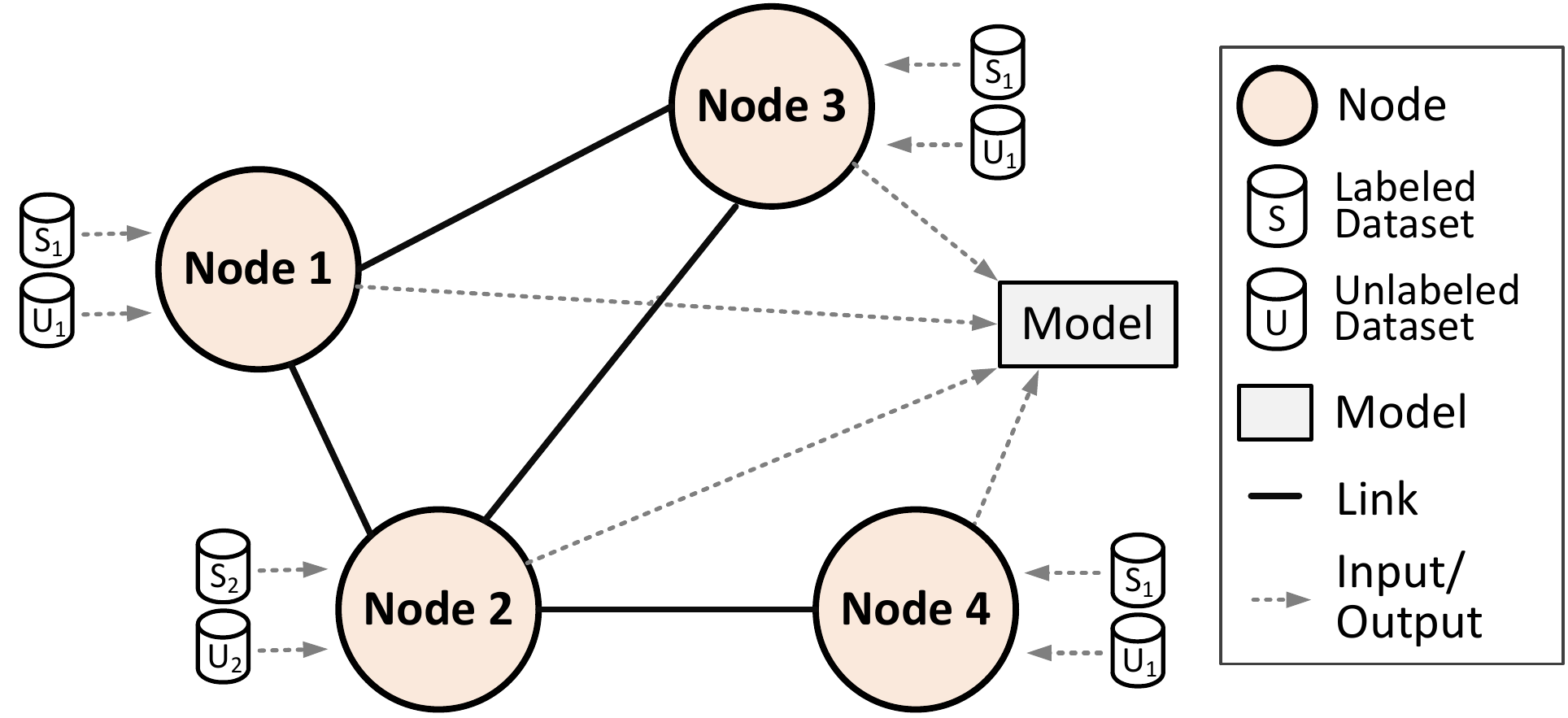} 
\caption[Depiction of distributed SSL over a network of agents.]{Depiction of SSL over a network of agents. Each agent receives a labeled training dataset, together with an unlabeled one. The task is for all the nodes to converge to a single model, by exploiting all their local datasets.}
\label{chap7:fig:setting}
\end{figure}

In this chapter, we propose the first fully distributed algorithm for SSL over networks, satisfying the above requirements. In particular, we extend an algorithm belonging to the MR family, namely laplacian kernel ridge regression (LapKRR) \cite{belkin2006manifold}. MR algorithms, originated in the seminal works of \cite{belkin2004semi} and \cite{belkin2006manifold}, are based on the assumption that data often lie in a low-dimensional manifold $\mathcal{M}$ embedded in the higher-dimensional input space. When the structure of the manifold is unknown, it can be approximated well by a weighted graph where the vertexes are represented by the data points and the weights of the edges represent a measure of similarity between the points. In the MR framework, the classification function is obtained by solving an extension of the classical regularized optimization problem, with an additional regularization term, which incorporates information about the function's smoothness on the manifold.

The algorithm presented in this chapter starts from the observation that, in the MR optimization problem, information is mostly encoded in a matrix $\vect{D}$  of pairwise distances between patterns. In fact, both the additional regularization term, and the kernel matrix (for any translation-invariant kernel function) can be computed using the information about the distance between points. In the distributed setting, each agent can compute this matrix relatively only to its own training data, while information about the distance between points belonging to different agents are unknown. Obtaining this information would allow a very simple protocol for solving the overall optimization problem. As a consequence, we subdivide the training algorithm in two steps: a distributed protocol for computing $\vect{D}$, followed by a distributed strategy for solving the optimization problem.

For the former step, in the initial phase of the algorithm, we allow a small exchange of data patterns between agents. In this phase, privacy can be preserved with the inclusion of any privacy-preserving protocol for the computation of distances \cite{verykios2004state}. For completeness, we describe the strategies that are used in our experiments in Section \ref{subsec:privacy}. As a second step, we recover the rest of the global distance matrix $\vect{D}$ by building on previous works on Euclidean distance matrix (EDM) completion \cite{candes2009exact,mishra2011low}. To this end, we consider two strategies. The first one is a simple modification of the state-of-the-art algorithm presented in \cite{ling2012decentralized,lin2015decentralized}, which is based on a column-wise partitioning of $\vect{D}$ over the agents. In this chapter, we modify it to take into account the specific nature of Euclidean distance matrices, by the incorporation of non-negativity and symmetry constraints. As a second strategy, we propose a novel algorithm for EDM completion, which is inspired to the framework of diffusion adaptation (DA) (see Section \ref{chap4:sec:diffusion_filtering}). The algorithm works by interleaving gradient descent steps with local interpolation of a suitable low-rank factorization of $\vect{D}$. While the first algorithm has a lower computational cost, we found that this comes at the cost of a worse performance, particularly when the sampling set of the matrix to complete is small. On the opposite, our algorithm exploits the particular structure of EDMs, at the cost of a possibly greater computational demanding. We discuss in more detail the advantages and disadvantages of the two approaches in Section \ref{chap7:sec:EDMdistr} and in the experimental section.

As we stated before, once the matrix $\vect{D}$ is known, solving the rest of the optimization problem is trivial. In this chapter we focus on the LapKRR algorithm, and we show that its distributed version can be solved using a single operation of sum over the network. Our experimental results show that, in most cases, the performance of the novel diffusion adaptation-based algorithm for distributed EDM completion overcome those of the state-of-the-art column-wise partitioning strategy. Secondly, experiments show that the distributed LapKRR is competitive with a centralized LapKRR model trained on the overall dataset.

The rest of the chapter is structured as follows: in Section \ref{chap7:sec:prelim} we introduce the theoretical tools upon which our algorithm is based. In particular, we detail the problem of SSL in the framework of MR in Section \ref{subsec:ssl}, some notions of EDM completion in Section \ref{subsec:EDMcomp}, and two strategies for privacy-preserving similarity computation in Section \ref{subsec:privacy}. In Section \ref{chap7:sec:EDMdistr} we propose our algorithm to complete an EDM in a decentralized fashion. Then, Section \ref{chap7:sec:algo} details the proposed framework for distributed LapKRR. In Section \ref{chap7:sec:results} we present the results for both the distributed EDM completion and distributed LapKRR.

\section{Preliminaries}
\label{chap7:sec:prelim}

In this section, we introduce some concepts that are used in the development of our algorithm. We start by describing the basic setting of SSL in Section \ref{subsec:ssl}. Then, we introduce the matrix completion problem and its application to the EDMs in Section \ref{subsec:EDMcomp}. As the last point, in Section \ref{subsec:privacy} we report some results on privacy-preserving similarity computation.
\subsection{Semi-supervised learning}
\label{subsec:ssl}

In the SSL setting, we are provided with a set of $l$ input/output labeled data $S = \left\{ (\vect{x}_1, y_1), \ldots, (\vect{x}_{l}, y_{l}) \right\}$ and an additional set of $u$ unlabeled data $U = \left\{ \vect{x}_{l+1}, \ldots, \vect{x}_{l+u} \right\}$ \cite{Chapelle2006}. As before, in the following inputs are assumed to be $d$-dimensional real vectors $\vect{x} \in \mathcal{X} \subseteq \R^d$, while outputs are assumed to be scalars $y \in \mathcal{Y} \subseteq \R$. The discussion can be extended straightforwardly to the case of a multi-dimensional output. In this chapter, we consider one particular class of SSL algorithms belonging to the family of MR \cite{belkin2006manifold}. Practically, MR learning algorithms are based on three assumptions.
\begin{itemize}
\item[-] Smoothness assumption: if two points $\vect{x}_1, \vect{x}_2\in \mathcal{X}$ are close in the intrinsic geometry of their marginal distribution, then their conditional distributions $p(y \mid \vect{x}_1)$ and $p(y\mid \vect{x}_2)$ are similar.
\item[-] Cluster assumption: the decision boundary should lie in a low-density region of the input space $\mathcal{X}$.
\item[-] Manifold assumption: the marginal distribution $p(\vect{x})$ is supported on a low-dimensional manifold $\mathcal{M}$ embedded in $\mathcal{X}$.
\end{itemize}
We now define the SLL problem formally.

\begin{definition}[SSL problem with manifold regularization]
Let $\mathcal{H}_K$ be a Reproducing Kernel Hilbert Space defined by the kernel function $\mathcal{K}:\mathcal{X}\times \mathcal{X} \rightarrow\mathbb{R}$ with norm $\Vert f\Vert_K^2$, the approximation function for the SSL problem is estimated by solving:
\begin{equation}
\label{chap7:eq:ssl_problem}
f^* = \underset{f\in\mathcal{H}_K}{\text{argmin}}\sum_{i=1}^l l(y_i,f(\vect{x}_i))+\gamma_A\Vert f\Vert_K^2+\gamma_I \Vert f\Vert_I^2\;,
\end{equation}
where $l(\cdot, \cdot)$ is a suitable loss function, $\Vert f\Vert_I^2$ is a penalty term that penalizes the structure of $f$ with respect to the manifold and $\gamma_A,\gamma_I\geq0$ are the regularization parameters.
\end{definition}

Usually, the structure of the manifold $\mathcal{M}$ is unknown and it must be estimated from both labeled and unlabeled data. In particular, we can define an adjacency matrix $\vect{W} \in \R^{l+u\times l+u}$, where each entry $W_{ij}$ is a measure of similarity between patterns $\vect{x}_i$ and $\vect{x}_j$ (see \cite{belkin2006manifold} for possible ways of constructing this matrix). Using this, the regularization term $\Vert f\Vert_I^2$ can be rewritten as \cite{belkin2006manifold}:
\begin{equation}
\Vert f\Vert_I^2 = f^\text{T}\vect{\vect{L}}f\;,
\end{equation}
where $\vect{\vect{L}} \in \R^{l+u\times l+u}$ is the data adjacency graph Laplacian (see Appendix \ref{app:graph_theory}). Practically, the overall manifold $\mathcal{M}$ is approximated with an adjacency graph, which can be computed from both labeled and unlabeled data. In order to obtain better performances, usually a normalized Laplacian $\hat{\vect{\vect{L}}} = \vect{G}^{-1/2}\vect{\vect{L}}\vect{G}^{-1/2}$, or an iterated version $\hat{\vect{\vect{L}}}^q,\;q\geq 0$, is used \cite{belkin2006manifold}.
An extension of the classical Representer Theorem proves that the function $f^*$ is in the form of:
\begin{equation}
\label{chap7:eq:kernel_exp}
f^*(\vect{x}) = \sum_{i=1}^{N}\alpha_i \mathcal{K}(\vect{x},\vect{x}_i)\;,
\end{equation}
where $N = l+u$ and $\alpha_i$ are weight parameters. As we stated in the introduction, for simplicity in this chapter we focus on a particular algorithm belonging to this framework, denoted as LapKRR. This is obtained by substituting Eq. (\ref{chap7:eq:kernel_exp}) into problem (\ref{chap7:eq:ssl_problem}) and setting a squared loss function:
\begin{equation}
l(y_i,f(\vect{x}_i))=\Vert y_i -f(\vect{x}_i)\Vert_2^2\;.
\end{equation}
Considering the dual optimization problem, by the optimality conditions the final parameters vector $\boldalpha^* = \left[\alpha_1,\ldots,\alpha_N\right]^\text{T}$ is easily obtained as:
\begin{equation}
\boldalpha^* = \left(\vect{JK}+\gamma_A\vect{I}+\gamma_I\vect{\vect{L}}\vect{K}\right)^{-1}\hat{\vect{y}}\;,
\end{equation}
where $\hat{\vect{y}}$ is an $N$-dimensional vector with components:
\begin{equation}
\hat{y}_{i} = \begin{cases} y_{i} &\mbox{if } i \in \left\{1, \ldots, l \right\} \\ 0 &\mbox{if } i \in \left\{l+1, \ldots, l+u\right\} \end{cases}\;,
\end{equation}
$\vect{J}$ is an $N\times N$ diagonal matrix with elements:
\begin{equation}
J_{ii} = \begin{cases} 1 &\mbox{if } i \in \left\{1, \ldots, l \right\} \\ 0 &\mbox{if } i \in \left\{l+1, \ldots, l+u\right\} \end{cases}\;,
\end{equation}
and finally $\vect{K}$ is the $N \times N$ kernel matrix defined by $\{ K_{ij} = \mathcal{K}\left(\vect{x}_i,\vect{x}_j\right) \}$.

\subsection{(Euclidean) matrix completion}
\label{subsec:EDMcomp}

The second notion that will be used in the proposed algorithm is the EDM completion problem \cite{alfakih1999solving}. A matrix completion problem is defined as the problem of recovering the missing entries of a matrix only from a set of known entries \cite{candes2009exact}. This problem has many practical applications, i.e. sensors localization, covariance estimation and customer recommendations, and it was largely investigated in the literature. 

In this chapter, we focus on completion of the square matrix $\vect{D} \in \R^{N \times N}$ containing the pairwise distances among the training patterns, i.e.:
\begin{equation}
D_{ij} = \norm{\vect{x}_i - \vect{x}_j}^2 \; \forall i,j = 1, \ldots, N \,.
\label{chap7:eq:distance}
\end{equation}
$\vect{D}$ is called an Euclidean Distance Matrix (EDM). Clearly, Eq. \eqref{chap7:eq:distance} implies that $\vect{D}$ is symmetric and $D_{ii} = 0$ for all the elements on the main diagonal. It is possible to show that the rank $r$ of $\vect{D}$ is upper bounded by $d + 2$, meaning that $\vect{D}$ is low-rank whenever $d \ll N$, which is common in all practical applications.

In the following, we suppose to have observed only a subset of entries of $\vect{D}$, in the form of a matrix $\hat{\vect{D}}$. More formally, there exists a matrix with binary entries $\boldsymbol{\Omega} \in \left[ 0, 1 \right]^{N \times N}$ such that:
\begin{equation}
\hat{\vect{D}} = 
\begin{cases}
		\hat{D}_{ij} = D_{ij} & \text{ if } \Omega_{ij} = 1 \\
		\hat{D}_{ij} = 0 & \text{ otherwise}
\end{cases}\,.
\end{equation}
We wish to recover the original matrix $\vect{D}$ from $\hat{\vect{D}}$, i.e. we want to solve the following optimization problem:
\begin{equation}
\underset{\vect{D} \in \text{EDM}(N)}{\min} \norm[\text{F}]{\boldsymbol{\Omega} \circ \left( \hat{\vect{D}} - \vect{D} \right)}^2 \,,
\label{chap7:eq:edm_problem}
\end{equation}
where $\circ$ denotes the Hadamard product between two matrices, $\text{EDM(N)}$ is the set of all EDMs of size $N$, and $\norm[F]{\vect{A}}$ is the Frobenius norm of matrix $\vect{A}$. It is possible to reformulate problem in Eq. \eqref{chap7:eq:edm_problem} as a semidefinite problem by considering the Schoenberg mapping between EDMs and positive semidefinite matrices \cite{alfakih1999solving}:
\begin{equation}
\begin{aligned}
&\underset{\vect{D}}{\min}
& & \norm[\text{F}]{\boldsymbol{\Omega} \circ \left[ \hat{\vect{D}} - \kappa(\vect{D}) \right]}^2 \\
&\;\text{s. t.}
& & \vect{D} \succeq 0
\end{aligned} \,,
\label{chap7:eq:edm_completion_reformulation}
\end{equation}
where $\vect{D} \succeq 0$ means that $\vect{D}$ is positive semidefinite and:
\begin{equation}
 \kappa(\vect{D}) = \text{diag}(\vect{D})\vect{1}^{\text{T}} + \vect{1}\text{diag}(\vect{D})^{\text{T}} - 2\vect{D} \,,
\end{equation}
such that $\text{diag}(\vect{D})$ extracts the main diagonal of $\vect{D}$ as a column vector. This observation motivated most of the initial research on EDM completion \cite{alfakih1999solving}. Recently, an alternative formulation was proposed in \cite{mishra2011low}, which exploits the fact that every positive semidefinite matrix $\vect{D}$ with rank $r$ admits a factorization $\{ \vect{D}= \vect{VV}^\text{T} \}$, where $\vect{V}\in\mathbb{R}_{*}^{N\times r}=\{\vect{V}\in\mathbb{R}^{N\times r}:\;\det{(\vect{V}^{\text{T}}\vect{V})}\neq 0\}$. Using this factorization and assuming we know the rank of $\vect{D}$, problem (\ref{chap7:eq:edm_completion_reformulation}) can be reformulated as:
\begin{equation}
\underset{\vect{VV}^{\text{T}} \in S_+(r,N)}{\min} \norm[\text{F}]{\boldsymbol{\Omega} \circ \left[ \hat{\vect{D}} - \kappa\left(\vect{V}\vect{V}^\text{T}\right) \right]}^2 \;,
\label{chap7:eq:edm_lowrank_problem}
\end{equation}
where we have:
\begin{equation}
 S_+(r,N) = \{\vect{U}\in\mathbb{R}^{N \times N}:\; \vect{U} = \vect{U}^\text{T}\succeq 0,\; \text{rank}\left(\vect{U}\right)=r\} \;.
\end{equation}

\subsection{Privacy-preserving similarity computation}
\label{subsec:privacy}

As we stated in the Introduction, a fundamental step in the algorithm presented in this chapter is a distributed computation of similarity between two training patterns, i.e. a distributed computation of a particular entry of $\vect{D}$. If these patterns cannot be exchanged over the network, e.g. for privacy reasons, there is the need of implementing suitable protocols for privacy-preserving similarity computation. To show the applicability of the proposed approach, in our experimental simulations we make use of two state-of-the-art solutions to this problem. For completeness, we detail them here briefly.

More formally, the problem can be stated as follows. Given two training patterns $\vect{x}_i, \vect{x}_j \in \R^d$, belonging to different agents, we want to compute $\vect{x}_i^T\vect{x}_j$, without revealing the two patterns. Clearly, computing the inner product allows the computation of several other distance metrics, including the standard $L_2$ Euclidean norm. The first strategy that we investigate here is the random projection-based technique developed in \cite{liu2006random}. Suppose that both agents agree on a projection matrix $\vect{R} \in \R^{m \times d}$, with $m < d$, such that each entry $R_{ij}$ is independent and chosen from a normal distribution with mean zero and variance $\sigma^2$. We have the following lemma:
\begin{lemma}
Given two input patterns $\vect{x}_i, \vect{x}_j$, and the respective projections:
\begin{equation}
\vect{u}_i = \frac{1}{\sqrt{m}\sigma}\vect{R}\vect{x}_i \text{,  and  } \vect{u}_j = \frac{1}{\sqrt{m}\sigma}\vect{R}\vect{x}_j \,,
\end{equation}
we have that:
\begin{equation}
\mathbb{E}\left\{ \vect{u}_i^T\vect{u}_j \right\} = \vect{x}_i^T\vect{x}_j \,.
\end{equation}
\label{lemma:mult}
\end{lemma}
\begin{proof}
See \cite[Lemma 5.2]{liu2006random}.
\end{proof}
\noindent In light of Lemma \ref{lemma:mult}, exchanging the projected patterns instead of the original ones allows preserving, on average, the inner product. A thorough investigation on the privacy-preservation guarantees of this protocol can be found in \cite{liu2006random}. Additionally, we can observe that this protocol provides a reduction on the communication requirements of the application, since it effectively reduces the dimensionality of the patterns to be exchanged by a factor $m/d$.

The second protocol that we investigate in our experimental section is a more general (nonlinear) transformation introduced in \cite{bhaduri2011privacy}. It is given by:
\begin{equation}
\vect{v} = \vect{b} + \vect{Q}\tanh\left( \vect{a} + \vect{C}\vect{x} \right) \,,
\end{equation}
for a generic input pattern $\vect{x}$, where $\vect{b} \in \R^m$, $\vect{Q} \in \R^{m \times t}$, $\vect{a} \in \R^t$, $\vect{C} \in \R^{t \times d}$ are matrices whose entries are drawn from normal distributions with mean zero and possibly different variances. As in the previous method, it is possible to show that the inner product is approximately preserved, provided that the input patterns are not ``outliers'' in a specific sense. See \cite{bhaduri2011privacy} for more details and an analysis of the privacy-preservation capabilities of this scheme. Again, choosing $t$ and $m$ allows to balance between a more accurate reconstruction and a reduction on the input dimensionality.

The field of privacy-preserving similarity computation, and more in general privacy-preserving data mining, is vast and with more methods introduced each year. Although we have chosen these two protocols due to their wide diffusion and simplicity, we stress that our algorithm does not depend specifically on any of them. We refer to \cite{verykios2004state} and references therein for more general investigations on this field.

\section{Distributed Laplacian Estimation}
\label{chap7:sec:EDMdistr}

In this section, we start by formulating a problem of distributed estimation of $\vect{L}$ in Section \ref{subsec:formulation}. Then, we focus on two algorithms for its solution. The first is a modification of a state-of-the-art algorithm, described in Section \ref{subsec:block}, while the second is a fully novel protocol which is based on the ideas of `diffusion adaptation' \cite{sayed2014adaptive} introduced in Section \ref{subsec:diffusion}.

\subsection{Formulation of the problem}
\label{subsec:formulation}

In the distributed Laplacian estimation problem, we suppose that both the labeled data and the unlabeled data are distributed through a network of $L$ interconnected agents, as shown in Fig. \ref{chap7:fig:setting} and described in Appendix \ref{app:graph_theory}. Without loss of generality, we assume that data is organized as follows: the $k$th agent is provided with $N_k$ patterns, such that $N = \sum_{k=1}^L N_k$. For each agent, the first $l_k$ patterns are labeled: $S_k = \left\{(\vect{x}_{k,1},y_{k,1}),\ldots,(\vect{x}_{k,l_k},y_{k,l_k})\right\}$, while the last $u_k$ are unlabeled: $U_k = \left\{ \vect{x}_{k,l_k+1},\ldots, \vect{x}_{k,l_k+u_k} \right\}$. The local data sets are non-overlapping, so we have $S = \cup_{k=1}^L S_k$ and ${U = \cup_{k=1}^L U_k}$.

Let $\vect{\vect{L}}_k\in\mathbb{R}^{N_k\times N_k},\;k = 1\ldots L$, be the Laplacian matrices computed by each agent using its own data; we are interested in estimating in a totally decentralized fashion the Laplacian matrix $\vect{\vect{L}}$ calculated with respect to all the $N$ patterns. The local Laplacian matrices can be always expressed, rearranging the rows and the columns, as block matrices on the main diagonal of $\vect{\vect{L}}$:
\begin{equation}
\vect{\vect{L}} =
	\begin{bmatrix}
  		\vect{\vect{L}}_1 & ? & ?\\
  		 ? & \ddots & ? \\
  		?  & ? & \vect{\vect{L}}_L
	\end{bmatrix}
\label{chap7:eq:block_laplacian}
\end{equation}
The same structure of (\ref{chap7:eq:block_laplacian}) applies also to matrices $\vect{D}$ and $\vect{K}$, with $\vect{D}_k$ and $\vect{K}_k$ representing the distance matrix and kernel matrix computed over the local dataset.
This particular structure implies that the sampling set is not random, and makes non-trivial the problem of completing $\vect{\vect{L}}$ solely from the knowledge of the local matrices. At the opposite, the idea of exchanging the entire local datasets between nodes is unfeasible because of the amount of data to share. Instead of completing in a distributed manner the global Laplacian matrix, in this chapter we consider the alternative approach of computing the global EDM $\vect{D}$ first, and then using it to calculate the Laplacian. This approach has two advantages:
\begin{itemize}
\item[-] We can exploit the structure of EDMs to design efficient algorithms.
\item[-] From the global EDM we can compute, in addition to the Laplacian, the kernel matrix $\vect{K}$ for all kernel functions $\mathcal{K}$ based on Euclidean distance (e.g. the Gaussian kernel).
\end{itemize}
Based on these considerations, we propose a framework for the distributed estimation of $\vect{\vect{L}}$, which consists in five steps:
\begin{enumerate}
\item Patterns exchange: every agent exchanges a fraction $p$ of the available input data (both labeled and unlabeled) with its neighbors. This step is necessary so that the agents can increase the number of known entries in their local matrices. In order to maximize the diffusion of the data within the network, this step is iterated $n_\text{max}^{(1)}$ times; at every iteration an increasing percentage of shared data is constituted by pattern received by the neighbors in previous iterations. A simple strategy to do this consists, at the iteration $n$, to choose $\frac{n_\text{max}-n+1}{n_\text{max}}p$ patterns from the local dataset, and $\frac{n-1}{n_\text{max}}p$ patterns received in the previous $n-1$ iterations. 
In order to preserve privacy, this step can include one of the privacy-preserving strategies showed in Section \ref{subsec:privacy}.
\item Local EDM computation: each agent computes, using its original dataset and the data received from its neighbors, an incomplete approximation $\hat{\vect{D}}_k\in\mathbb{R}^{N\times N}$ of the real EDM matrix $\vect{D}$.
\item Entries exchange: the agents exchange a sample of their local EDMs $\hat{\vect{D}}_k$ with their neighbors. Again, this step is iterated $n_\text{max}^{(2)}$ times using the same rule of step 1.
\item Distributed EDM completion: the agents complete the estimate $\tilde{\vect{D}}$ of the global EDM using one of the distributed algorithms presented in the following sections.
\item Global Laplacian estimation: using $\tilde{\vect{D}}$ the agents compute the global Laplacian estimate $\tilde{\vect{\vect{L}}}$ and the kernel matrix estimate $\tilde{\vect{K}}$.
\end{enumerate}

\subsection{Decentralized block estimation}
\label{subsec:block}

As stated in the Introduction, the first algorithm that we take into account for the decentralized completion of $\vect{D}$ is a modified version of the algorithm named \textit{D-LMaFit} \cite{ling2012decentralized,lin2015decentralized}. To the best of our knowledge, this is the only existing algorithm for distributed matrix completion available in the literature.

Let $\hat{\vect{D}}$ be the incomplete global EDM matrix and denote with $\mathcal{I}$ the set of indexes corresponding to its known entries. In a centralized setting, without taking into account the structure of distance matrices, and assuming that the rank $r$ is known, $\tilde{\vect{D}}$ can be completed by solving the problem:
\begin{equation}
\begin{aligned}
&\underset{\vect{A},\vect{B},\tilde{\vect{D}}}{\min}
& & \norm[\text{F}]{\vect{AB}-\tilde{\vect{D}}}^2 \\
&\;\;\text{s. t.}
& & \tilde{D}_{ij} = \hat{D}_{ij},\;\forall(i,j)\in\mathcal{I}
\end{aligned}
\label{chap7:eq:block_matrix_completion}
\end{equation}
where $\vect{A}\in\mathbb{R}^{N\times r}$, $\vect{B}\in\mathbb{R}^{r\times N}$ represent a suitable low-rank factorization of $\tilde{\vect{D}}$.

In extending problem (\ref{chap7:eq:block_matrix_completion}) to a decentralized setting, the algorithm presented in \cite{ling2012decentralized} considers a column-wise partitioning of $\hat{\vect{D}}$ over the agents. For simplicity of notation, we suppose here that this partitioning is such that the $k$th agent stores only the columns corresponding to its local dataset. Thus, the block partitioning has the form $\hat{\vect{D}} = \left[\hat{\vect{D}}_1,\ldots,\hat{\vect{D}}_L\right]$, where $\hat{\vect{D}}_k\in\mathbb{R}^{N\times N_k}$ is the block of the matrix held by the $k$th agent, and $\mathcal{I}_k$ is the set of indexes of known entries of $\hat{\vect{D}}_k$.
The same block partition applies also to matrices $\vect{B} = \left[\vect{B}_1,\ldots,\vect{B}_L\right]$, with $\vect{B}_k\in\mathbb{R}^{r\times N_k}$, and $\tilde{\vect{D}} = \left[\tilde{\vect{D}}_1,\ldots,\tilde{\vect{D}}_L\right]$, with $\tilde{\vect{D}}_k\in\mathbb{R}^{N\times N_k}$. The matrix $\vect{A}$ cannot be partitioned, but each agent stores a local copy $\vect{A}_k$ to use in computations. The \textit{D-LMaFit} algorithm consists in an alternation of matrix factorizations and inexact average consensus, formalized in the following steps:
\begin{enumerate}
\item Initialization: For each agent, the matrices $\vect{A}_k\left[0\right]$ and $\vect{B}_k\left[0\right]$ are initialized as random matrices of appropriate dimensions. Matrix $\tilde{\vect{D}}_k\left[0\right]$ is initialized as $\tilde{\vect{D}}_k\left[0\right] = \hat{\vect{D}}_k$.
\item Update of $\vect{A}$: At time $n$, the $k$th agent updates its local copy of the matrix $\vect{A}$. If $n=0$, the updating rule is:
\begin{equation}
\vect{A}_k\left[1\right] = \sum_{i=1}^L C_{ki}\vect{A}_i\left[0\right]-\alpha\left(\vect{A}_k\left[0\right]-\tilde{\vect{D}}_k\left[0\right]\vect{B}_k^\text{T}\left[0\right]\right)\,,
\end{equation}
\noindent where $\alpha$ is a suitable positive step-size. If $n>0$, the updating rule is given by:
\begin{align}
\vect{A}_k\left[n+1\right] & = \vect{A}_k\left[n\right] -\sum_{i=1}^L \left( C_{ki}\vect{A}_i\left[n\right] - \tilde{C}_{ki}\vect{A}_i\left[n-1\right] \right)  - \nonumber \\ & \alpha\left(\vect{A}_k\left[n\right]-\vect{A}_k\left[n-1\right]-\tilde{\vect{D}}_k\left[n\right]\vect{B}_k^\text{T}\left[n\right]+\tilde{\vect{D}}_k\left[n-1\right]\vect{B}_k^\text{T}\left[n-1\right]\right) \;.
\label{chap7:eq:block_partitioning_update}
\end{align}
In Eq. \eqref{chap7:eq:block_partitioning_update}, $\tilde{\vect{C}}$ is a mixing matrix that satisfies some properties \cite{lin2015decentralized}. A suitable choice is $\tilde{\vect{C}} = (1/2)(\vect{I}+\vect{C})$.
\item Update of $\vect{B}$ and $\tilde{\vect{D}}$: At the $n$th iteration, agent $k$ updates matrices $\vect{B}_k$ and $\tilde{\vect{D}}_k$ according to:
\begin{align}
\vect{B}_k\left[n+1\right] &= \vect{A}^{\dagger}_k[n+1]\vect{A}_k^\text{T}\left[n+1\right]\tilde{\vect{D}}_k\left[n\right]\\
\tilde{\vect{D}}_k\left[n+1\right] &= \vect{A}_k\left[n+1\right]\vect{B}_k\left[n+1\right]+\nonumber\\ & \text{P}_{\mathcal{I}_k}\left(\hat{\vect{D}}_k-\vect{A}_k\left[n+1\right]\vect{B}_k\left[n+1\right]\right)
\end{align}
where $\vect{A}^{\dagger}_k[n+1]$ is the Moore-Penrose inverse of $\vect{A}_k[n+1]$, and $P_\mathcal{I}\left(\vect{M}\right):\mathbb{R}^{n \times m}\rightarrow\mathbb{R}^{n \times m}$ is a projection operator defined by:
\begin{equation}
P_\mathcal{I}\left(\vect{M}\right)_{ij} = 
	\begin{cases}
	M_{ij} & \text{ if } \left(i,j\right)\in \mathcal{I}\\
	0 & \text{ otherwise}
	\end{cases}\,.
\end{equation}
\end{enumerate}
The algorithm stops when the maximum number of iterations $n_\text{max}^\text{EDM}$ is reached.

As we stated, \textit{D-LMaFit} is not specifically designed for EDM completion. Consequently, it has some important limitations in our context. In particular, the resulting matrix $\tilde{\vect{D}}$ can have negative entries and could be non-symmetric; moreover, it is distributed across the nodes and so, if an agent wants access to the complete matrix, it has to collect the local matrices $\tilde{\vect{D}}$ through all the network. In order to at least satisfy the constraint that $\tilde{\vect{D}}$ be an appropriate EDM, we introduce the following modifications into the original algorithm:
\begin{itemize}
\item[-] The updating equation for $\tilde{\vect{D}}_k$ is modified by setting to $0$ all the negative entries. This projection operator is a standard approach in non-negative matrix factorization to enforce non-negativity constraints \cite{lin2007projected}.
\item[-] When all the agents gathered the complete matrix $\tilde{\vect{D}}$, this is symmetrized as $\tilde{\vect{D}} = \frac{\tilde{\vect{D}}+\tilde{\vect{D}}^\text{T}}{2}$.
\end{itemize}
\subsection{Diffusion gradient descent}
\label{subsec:diffusion}

The second algorithm for distributed EDM completion proposed in this chapter exploits the low-rank factorization $\vect{D} = \kappa(\vect{V}\vect{V}^\text{T})$ showed in Section \ref{subsec:EDMcomp}. In particular, we consider the general framework of DA (see Section \ref{chap3:sec:relation_other_research_fields}). To begin with, we can observe that the objective function in Eq. (\ref{chap7:eq:edm_lowrank_problem}) can be approximated locally by:
\begin{equation}
J_k(\vect{V}) = \norm[\text{F}]{\boldsymbol{\Omega}_k \circ \left[ \hat{\vect{D}}_k - \kappa\left(\vect{VV}^{\text{T}}\right)\right]}^2\; k=1,\ldots,L\;,
\end{equation}
where $\boldsymbol{\Omega}_k$ is the local auxiliary matrix associated with $\hat{\vect{D}}_k$. Hence, we can exploit a DA algorithm to minimize the joint cost function given by $\tilde{J}(\vect{V}) = \sum_{k=1}^L J_k(\vect{V})$. The DGD for the distributed completion of an EDM is defined by an alternation of updating and diffusion equations in the form of:
\begin{enumerate}
\item Initialization: All the agents initialize the local matrices $\vect{V}_k$ as random $N\times r$ matrices.
\item Update of $\vect{V}$: At time $n$, the $k$th agent updates the local matrix $\vect{V}_k$ using a gradient descent step with respect to its local cost function:
\begin{align}
\tilde{\vect{V}}_k\left[ n+1 \right] & = \vect{V}_k[n] - \eta_k[n] \nabla_{\vect{V}_k} J_k(\vect{V}) \,.
\end{align}
where $\eta_k\left[n\right]$ is a positive step-size. It is straightforward to show that the gradient of the cost function is given by:
\begin{align}
\nabla_{\vect{V}_k} J_k(\vect{V}) & = \kappa^* \Bigl\{ \boldsymbol{\Omega}_k \circ \Bigr. \nonumber\\
 & \Bigl. \circ \left( \kappa \left( \vect{V}_k\left[n\right] \vect{V}_k^{\text{T}}\left[n \right] \right)-\hat{\vect{D}}_k \right) \Bigr\} \vect{V}_k\left[ n\right] \,,
\end{align}
where $\kappa^*(\vect{A}) = 2\left[\text{diag}\left(\vect{A}1\right]-\vect{A}\right)$ is the adjoint operator of $\kappa$.
\item Diffusion: The updated matrices are combined according to the mixing weights $\vect{C}$:
\begin{equation}
\vect{V}_k\left[ n+1 \right] = \sum_{i=1}^L C_{ki}\tilde{\vect{V}}_i\left[ n+1 \right].
\end{equation}
\end{enumerate}
Compared with the state-of-the-art decentralized block algorithm presented in the previous section, the diffusion-based approach has two main advantages. First, it is able to take into account naturally the properties of EDM matrices. Secondly, at every step each node has a complete estimate of the overall matrix, instead of a single column-wise block. Thus, there is no need of gathering the overall matrix at the end of the optimization process.
 
\section{Distributed Semi-supervised Manifold Regularization}
\label{chap7:sec:algo}

In this section, we consider the more general distributed SSL setting, as illustrated in Fig. \ref{chap7:fig:setting}. We suppose that the agents in the network have performed a distributed matrix completion step, using either the algorithm in Section \ref{subsec:block} or the one in Section \ref{subsec:diffusion}, so that the estimates $\tilde{\vect{D}}$, $\tilde{\vect{\vect{L}}}$ and $\tilde{\vect{K}}$ are globally known. For the $k$th agent, we denote with $\hat{\vect{y}}_k$ the $N_k$ dimensional vector with elements:
\begin{equation}
\hat{y}_{k,i} = 
	\begin{cases} 
		y_{k,i} & \mbox{if } i \in\left\{1,\ldots,l_k\right\} \\
	 	0 & \mbox{if } i \in \left\{l_k+1,\ldots,l_k+u_k\right\}
	\end{cases}\;,
\end{equation}
and $\hat{\vect{J}}_k$ the $N_k \times N$ matrix defined by $\hat{\vect{J}}_k = \begin{bmatrix} \overline{\vect{0}}_k & \boldsymbol{\Lambda}_k & \underline{\vect{0}}_k \end{bmatrix}$, where $\boldsymbol{\Lambda}_k$ is a $N_k \times N_k$ diagonal matrix with elements:
\begin{equation}
\Lambda_{k,ii} = 
	\begin{cases} 
		1 & \mbox{if } i \in\left\{1,\ldots,l_k\right\} \\ 
		0 & \mbox{if } i \in \left\{l_k+1,\ldots,l_k+u_k\right\}
	\end{cases}\,,
\end{equation}
$\overline{\vect{0}}_k$ is a $N_k \times \sum_{j<k}N_j$ null matrix and $\underline{\vect{0}}_k$ is a $N_k \times \sum_{j>k}N_j$ null matrix. Using this notation, the optimization problem of LapKRR can be reformulated in distributed form as:
\begin{equation}
\min_{\boldalpha}\sum_{k = 1}^L\Vert \hat{\vect{y}}_k -\hat{\vect{J}}_k\tilde{\vect{K}} \boldalpha\Vert_2^2+\gamma_A\boldalpha^{\mathrm{T}} \tilde{\vect{K}}\boldalpha+\gamma_I \boldalpha^\mathrm{T}\tilde{\vect{K}}\tilde{\vect{\vect{L}}}\tilde{\vect{K}}\boldalpha\;.
\end{equation} 
Denoting with $\hat{\vect{J}}_{\text{tot}}= \sum_{k=1}^L \hat{\vect{J}}_k^\text{T}\hat{\vect{J}}_k$ and $\hat{\vect{y}}_{\text{tot}}= \sum_{k=1}^L \hat{\vect{J}}_k^\text{T}\hat{\vect{y}}_k$, we can derive the expression for the optimal weights vector $\boldalpha^*$:
\begin{equation}
\boldalpha^* = \left(\hat{\vect{J}}_{\text{tot}}\tilde{\vect{K}}+\gamma_A\vect{I}+\gamma_I\tilde{\vect{\vect{L}}}\tilde{\vect{K}}\right)^{-1}\hat{\vect{y}}_{\text{tot}}\;.
\end{equation}
The particular structure of $\boldalpha^*$ implies that the distributed solution can be decomposed as $\boldalpha^* = \sum_{k=1}^L \boldalpha^*_k$, where:
\begin{equation}
\boldalpha_k^* = \left(\hat{\vect{J}}_{\text{tot}}\tilde{\vect{K}}+\gamma_A\vect{I}+\gamma_I\tilde{\vect{\vect{L}}}\tilde{\vect{K}}\right)^{-1}\hat{\vect{J}}_k^\text{T}\hat{\vect{y}}_k\,.
\label{chap7:eq: local_solution}
\end{equation}
To compute the local solution $\boldalpha^*_k$, the $k$th agent requires only the knowledge of matrix $\hat{\vect{J}}_{\text{tot}}$, which can be computed with a distributed sum over the network using the DAC protocol. Clearly, the sum can be obtained by post-multiplying the final estimate by $L$. Overall, the distributed LapKRR algorithm can be summarized in five main steps:

\begin{enumerate}
\item Distributed Laplacian estimation: this step corresponds to the process illustrated in Sec. \ref{chap7:sec:EDMdistr}. It includes the patterns exchange (with the inclusion of a privacy-preserving strategy, if needed) and the points exchange procedures, the distributed EDM completion, and the computation of $\tilde{\vect{\vect{L}}}$ and $\tilde{\vect{K}}$.
\item Global sum of $\hat{\vect{J}}_\text{tot}$: in this step the local matrices $\hat{\vect{J}}_k^\text{T}\hat{\vect{J}}_k$ are summed up using the DAC protocol.
\item Local training: using the matrix $\hat{\vect{J}}_\text{tot}$ computed in the previous step, each agent calculates its local solution, given by:
\begin{equation}
\boldalpha_k^* = \left(\hat{\vect{J}}_{\text{tot}}\tilde{\vect{K}}+\gamma_A\vect{I}+\gamma_I\tilde{\vect{\vect{L}}}\tilde{\vect{K}}\right)^{-1}\hat{\vect{J}}_k^\text{T}\hat{\vect{y}}_k\,.
\end{equation}
\item Global sum of $\boldalpha^*$: in this step, using the DAC protocol, the local vectors $\boldalpha_k^*$ are summed up to compute the global weight vector.
\item Output estimation: when a new unlabeled pattern $\vect{x}$ is available to the network, each agent can initialize a partial output as:
\begin{equation}
f_k(\vect{x}) = \sum_{i=1}^{N_k} \mathcal{K}(\vect{x},\vect{x}_{k,i}) \beta_{k,i}^* \;,
\label{chap7:eq:partial_output}
\end{equation}
where $\vect{\beta}_k^*$ is a $N_k$-dimensional vector containing the  entries of $\boldalpha^*$ corresponding to the patterns belonging the $k$th agent. The global output is then computed as:
\begin{equation}
f(\vect{x}) = \sum_{k=1}^L f_k(\vect{x}) \;,
\end{equation}
which can be obtained efficiently with the use of the DAC protocol.
\end{enumerate}
A pseudocode of the algorithm, from the point of view of a single agent, is provided in Algorithm \ref{alg:rklms}.
\begin{AlgorithmCustomWidth}[h]
\caption{Distr-LapKRR: Pseudocode of the proposed distributed SSL algorithm ($k$th node).}
\label{alg:rklms}
\begin{algorithmic}[1]
\Require Labeled $S_k$ and unlabeled $U_k$ training data, number of nodes $L$ (global), regularization parameters $\gamma_A$, $\gamma_I$ (global)
\Ensure Optimal vector $\boldalpha_k^*$
\For{$n = 1$ to $n_\text{max}^1$}
\State Select a set of input patterns and share them with the neighbors $\mathcal{N}_k$, using a privacy-preserving transformation if needed.
\State Receive patterns from the neighbors.
\EndFor
\State Compute the incomplete EDM matrix $\hat{\vect{D}}_k$.
\For{$n = 1$ to $n_\text{max}^2$}
\State Select a set of entries from $\hat{\vect{D}}_k$ and share them with the neighbors.
\State Receive entries from the neighbors.
\State Update $\hat{\vect{D}}_k$ with the entries received.
\EndFor
\State Complete the matrix $\tilde{\vect{D}}$ using the algorithm presented in Sec. \ref{subsec:block} or in Sec. \ref{subsec:diffusion}.
\State Compute the Laplacian matrix $\tilde{\vect{\vect{L}}}$ and the kernel matrix $\tilde{\vect{K}}$ using $\tilde{\vect{D}}$.
\State Compute the sum $\hat{\vect{J}}_\text{tot}$ over the network using the DAC protocol.
\State \Return{$\boldalpha_k^*$ according to Eq. \eqref{chap7:eq: local_solution}}.
\end{algorithmic}
\end{AlgorithmCustomWidth}
\section{Experimental results}
\label{chap7:sec:results}

\subsection{Experiments setup}

We tested the performance of our proposed algorithm over five publicly available datasets. In order to get comparable results with state-of-the-art SSL algorithms, the datasets were chosen among a variety of benchmarks for SSL. A schematic overview of their characteristics is given in Tab. \ref{tab:datasets}. For further information about the datasets, we refer to \cite{belkin2006manifold} for 2Moons, to \cite{Chapelle2006} for BCI, and to \cite{melacci2011laplacian} for the rest of the datasets. The COIL dataset is used in two different versions, one with $2$ classes (COIL2) and a harder version with $20$ classes (COIL20). In all the cases, input variables are normalized between $-1$ and $1$ before the experiments.
\begin{center}
\begin{table}[h]
\footnotesize
\renewcommand{\arraystretch}{1.3}
\centering
\caption{Description of the datasets used for testing Distr-LapKRR.} 
\begin{tabular}{lcccccc}
\toprule
Name & Features & Size & N. Classes & $\vert\text{TR}\vert$ & $\vert\text{TST}\vert$ & $\vert\text{U}\vert$\\ \midrule
2Moons & 2 & 400 & 2 & $14$ & $200$ & $186$ \\
BCI & 117 & 400 & 2 & $14$ & $100$ & $286$ \\
G50C & 50 & 550 & 2 & $50$ & $186$ & $314$ \\
COIL20 & 1024 & 1440 & 20 & $40$ & $400$ & $1000$ \\
COIL2 & 1024 & 1440 & 2 & $40$ & $400$ & $1000$ \\ \bottomrule
\end{tabular}
\vspace{0.5em}
\label{tab:datasets}
\end{table}
\end{center}
\vspace{-2.5em} \noindent In our experimental setup we considered a $7$-nodes network, whose topology is kept fixed for all the experiments. The topology is generated such that each pair of agents is connected with a probability $c$. In particular, in our implementation we set $c = 0.5$, while we choose the weights matrix $\vect{C}$ using the `max-degree' strategy. This choice ensures both convergence of the DAC protocol \cite{xiao2004fast} and it satisfies the requirements of the DA framework \cite{sayed2014adaptive}. All the experiments are repeated 25 times, to average possible outliers results due to the randomness in the processes of exchange and in the initialization of the matrices in the EDM completion algorithms. At every run, data are randomly shuffled and then partitioned in a labeled training set $\text{TR}$, a test set $\text{TST}$, and an unlabeled set $\text{U}$, whose cardinalities are reported in Tab. \ref{tab:datasets}. Both the labeled and unlabeled training sets are then partitioned evenly across the nodes. All the experiments are performed using MATLAB R$2014$a on an Intel i$7$-$3820$ @$3.6$ GHz and $32$ GB of memory. 

\subsection{Distributed Laplacian estimation}

In this section we compare the performance of the two strategies for distributed EDM completion illustrated in Section \ref{chap7:sec:EDMdistr}. We analyze the matrix completion error, together with the overall computational time for the two strategies. Given an estimate $\tilde{\vect{D}}$ of $\vect{D}$, we define the matrix completion error as:
\begin{equation}
E(\tilde{\vect{D}}) = \frac{\norm[F]{\tilde{\vect{D}}-\vect{D}}}{\norm[F]{\vect{D}}}\,.
\end{equation}
The first set of experiments consists in comparing the completion error and the time required by the two algorithms, for different sizes of the sampling set of $\vect{D}$. In our context, the size of the sampling set depends only on the amount of data that are exchanged before the algorithm runs. To this end, we consider the completion error when varying the number of iterations for both the patterns exchange and the entries exchange steps, while keeping fixed the exchange fraction $p$. In particular, for all the datasets we varied the maximum number of iterations $n_\text{max}^{(1)}$ and $n_\text{max}^{(2)}$ from $0$ to $150$, by steps of $10$. Results of this experiment are presented in Fig. \ref{chap7:fig:EDM_error}. The solid red and the solid blue lines show the performance of Decentralized Block Estimation and DGD, respectively. Since the value of the completion error only depends on the input $\vect{x}$, the results for datasets COIL$20$ and COIL$2$ are reported together.
\begin{figure*}[h]
\centering
\subfloat[Dataset: 2Moons]{\includegraphics[width=0.45\columnwidth]{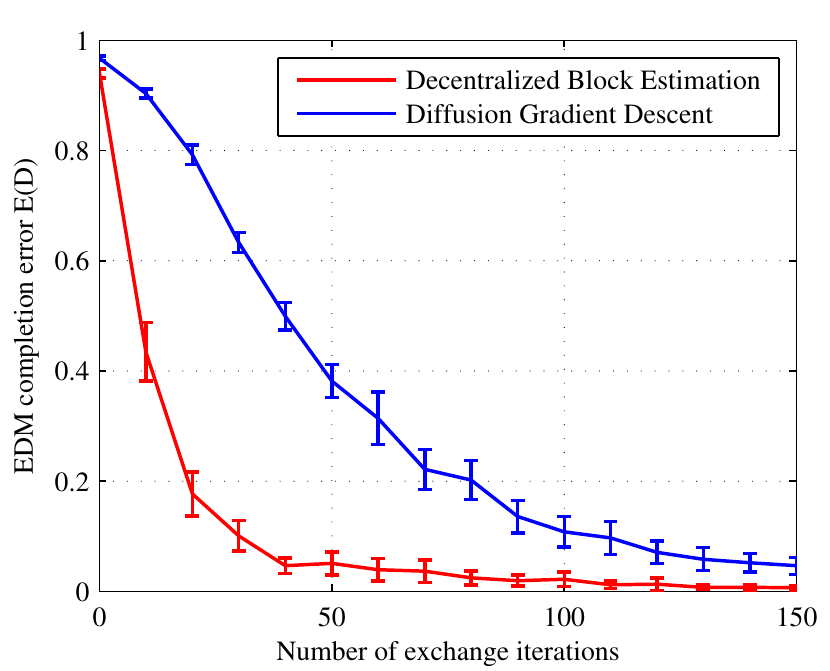}\label{chap7:fig:2Moons_D_err}} %
\subfloat[Dataset: BCI]{\includegraphics[width=0.45\columnwidth]{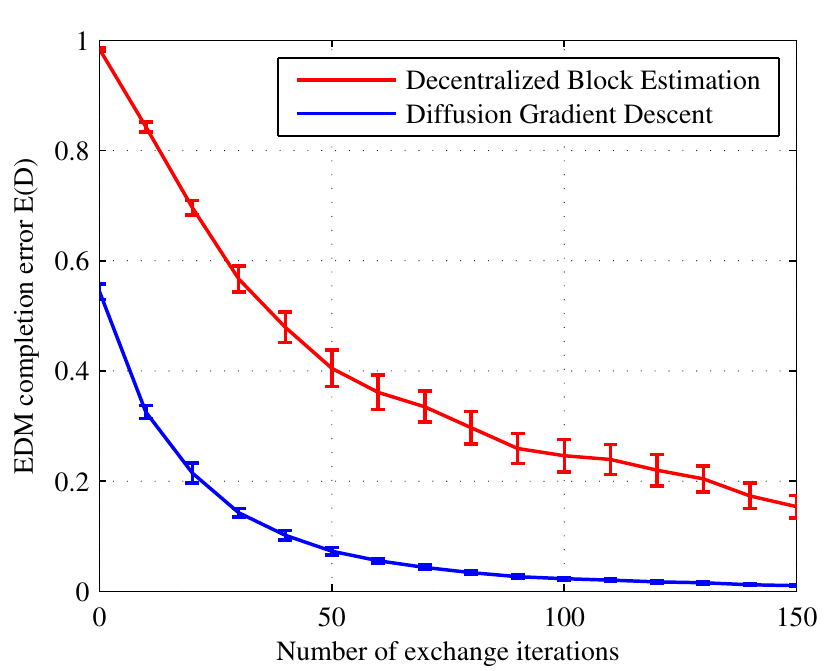} \label{chap7:fig:BCI_D_err}} %
\vfill
\subfloat[Dataset: G50C]{\includegraphics[width=0.45\columnwidth]{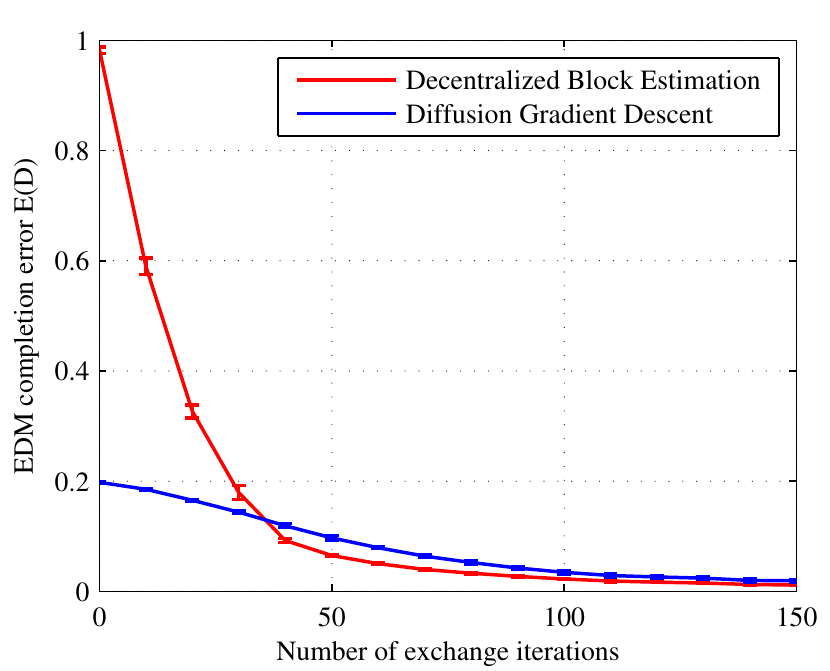} \label{chap7:fig:g50c_D_err}} %
\subfloat[Datasets: COIL$20\backslash$COIL$2$]{\includegraphics[width=0.45\columnwidth]{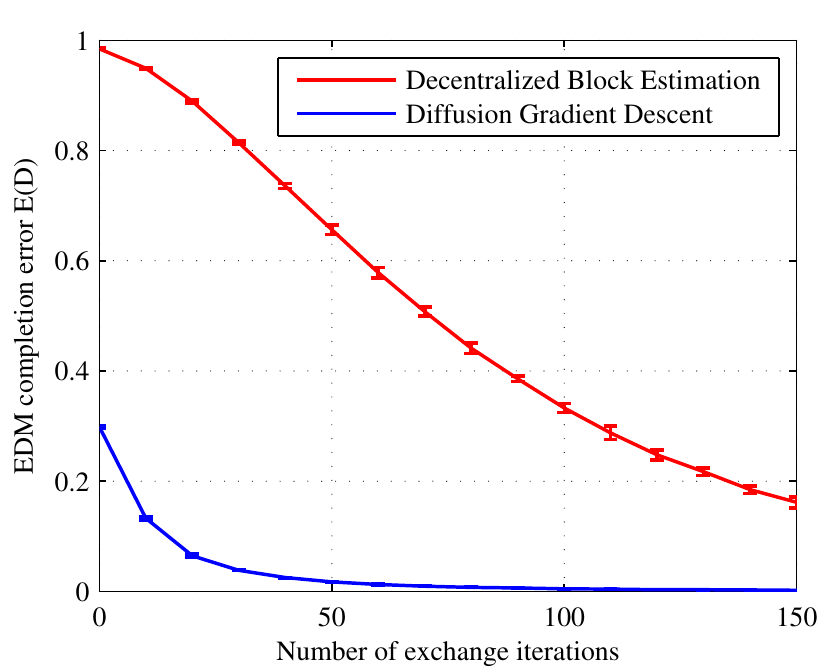} \label{chap7:fig:COIL20_D_err}}
\caption[Average EDM completion error of the two EDM completion strategies on the considered datasets.]{Average EDM completion error of the two strategies on the considered datasets, when varying the number of iterations for the patterns and entries exchange protocols. The vertical bars represent the standard deviation from the average.}
\label{chap7:fig:EDM_error}
\end{figure*}
The values for the patterns exchange fraction p$1$ and the entries exchange fraction p$2$ are chosen to balance the communication overhead and the size of the sampling set of $\vect{D}$. For both the algorithms, we set the maximum number of iterations $n_\text{max}^\text{EDM}$to $1500$, and we used a fixed step-size strategy. In particular for the Decentralized Block Estimation we set $\alpha = 0.4$, as suggested in \cite{lin2015decentralized}, while for the Diffused Gradient Descent, the optimal values for $\eta$ are chosen singularly for each dataset by searching in the interval $10^j,\, j\in\{-10,\ldots,-3\}$. These parameters, together with the values for p$1$ and p$2$, are reported in Tab. \ref{tab:datasets_params}, and are used in all the experiments.

\begin{table}
\scriptsize
\renewcommand{\arraystretch}{1.3}
\caption[Values for the parameters used in the simulations of Distr-LapKRR.]{Values for the parameters used in the simulations. The values in the first group are used in the distributed protocols and in the DGD algorithm (p$1$ and p$2$ are in percentages). Those in the second group are used to build the Laplacian and kernel matrices. In the third group are reported the parameters used in the privacy-preserving transformations.}
\centering
\begin{tabular}{l|clcll|rrrrl|rlllr}
\toprule
\multicolumn{1}{c}{Dataset} & \multicolumn{1}{|c}{p$1$} & \multicolumn{1}{c}{$n_\text{max}^{(1)}$} & \multicolumn{1}{c}{p$2$} & \multicolumn{1}{c}{$n_\text{max}^{(2)}$} & \multicolumn{1}{c}{$\eta$} & \multicolumn{1}{|c}{$\gamma_A$} & \multicolumn{1}{c}{$\gamma_I$} & \multicolumn{1}{c}{nn} & \multicolumn{1}{c}{$\sigma_K$} & \multicolumn{1}{c}{q} & \multicolumn{1}{|c}{t} & \multicolumn{1}{c}{$\sigma_a$} & \multicolumn{1}{c}{$\sigma_\text{b}$} & \multicolumn{1}{c}{$\sigma_{Q}$} & \multicolumn{1}{c}{$\sigma_\text{C}$}\\
\midrule
2Moons & $3.5$ & $100$ & $3.5$ & $100$ & $10^{-3}$ & $2^{-5}$ & $4$ & $6$ & $0.03$ & $1$ & $-$ & $-$ & $-$ & $-$ & $-$\\ 
BCI    & $2.5$ & $100$ & $2.5$ & $100$ & $10^{-6}$ & $10^{-6}$ & $1$ & $5$ & $1$ & $2$ &$10^4$ & $0$ & $0$ & $1$ & $10^{-6}$ \\ 
G50C   & $2$ & $150$ & $2.5$ & $150$ & $10^{-6}$ & $10^{-6}$ & $10^{-2}$ & $50$ & $17.5$ & $5$ & $2e^4$ & $0$ & $1$ & $1$ & $1.1e^{-6}$ \\ 
COIL   & $2$ & $150$ & $2.5$ & $150$ & $10^{-7}$ & $10^{-6}$ & $1$ & $2$ & $0.6$ & $1$ & $10^3$ & $0$ & $0$ & $1$ & $10^{-6}$\\ \bottomrule
\end{tabular}
\vspace{0.5em}
\label{tab:datasets_params}
\end{table}

We see that, with the solely exception of the $2$Moons dataset (see Fig. \ref{chap7:fig:2Moons_D_err}), the novel Diffused Gradient Descent algorithm achieves better performance when compared to the Decentralized Block Estimation, in particular when few information is exchanged before the completion process. For all the datasets, as the number of the exchange iterations increases, the diffusion strategy is able to converge rapidly to the real EDM $\vect{D}$, while the performance is poorer for the block partitioning strategy, resulting for datasets BCI and COIL in a completion error of $19\%$ even for high quantity of information exchanged (see Fig. \ref{chap7:fig:BCI_D_err} and Fig. \ref{chap7:fig:COIL20_D_err}).

When considering the time required by the two algorithms, which is shown in Fig. \ref{chap7:fig:EDM_time}, we observe that the block partition strategy requires for datasets $2$Moons and G$50$C less than half the time required by the diffused strategy, while, as the number of the features increases, the diffusion strategy tends to be less computational expensive. In fact, the time required by both strategies is nearly the same for the dataset BCI, while for COIL the diffusion strategy is $1.2$ times faster. We remark that the Decentralized Block Estimation requires an additional step for all the agents to gather the columns-wise blocks through the network, which has not been taken into account in calculating the computational time.

\begin{figure*}[h]
\centering
\subfloat[Dataset: 2Moons]{\includegraphics[width=0.45\columnwidth]{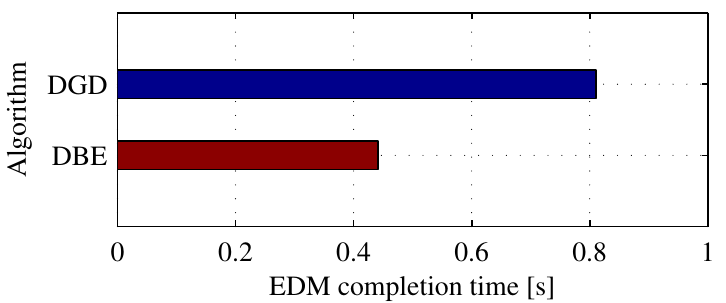}\label{chap7:fig:2Moons_D_time}} %
\subfloat[Dataset: BCI]{\includegraphics[width=0.45\columnwidth]{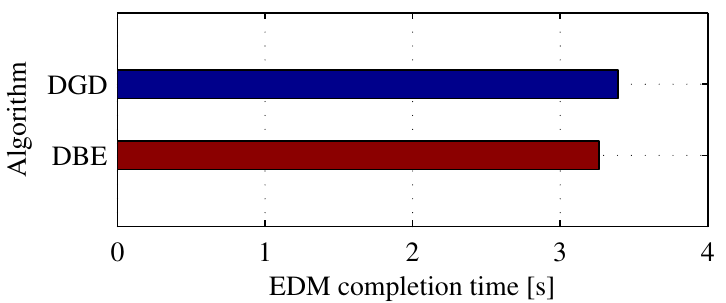} \label{chap7:fig:BCI_D_time}} %
\vfill
\subfloat[Dataset: G50C]{\includegraphics[width=0.45\columnwidth]{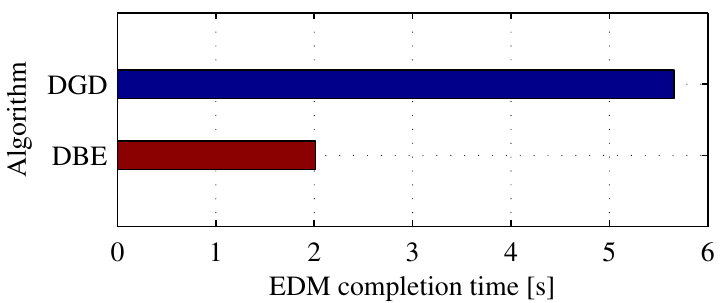} \label{chap7:fig:g50c_D_time}} %
\subfloat[Datasets: COIL$20\backslash$COIL$2$]{\includegraphics[width=0.45\columnwidth]{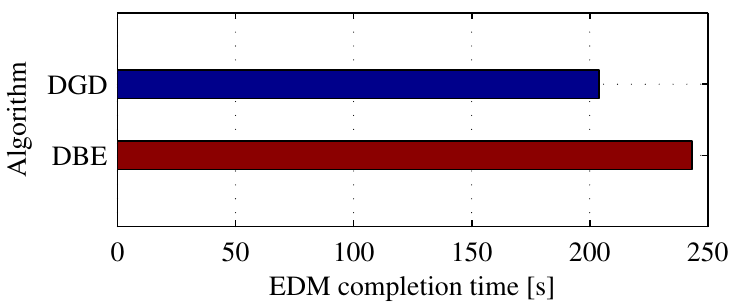} \label{chap7:fig:COIL20_D_time}}
\caption[Average EDM completion time required by the two EDM completion strategies on the considered datasets.]{Average EDM completion time required by the two strategies on the considered datasets. DGD and DBE are the acronyms for Decentralized Block Estimation and DGD respectively.}
\label{chap7:fig:EDM_time}
\end{figure*}

\subsection{Distributed semi-supervised manifold regularization}

The second experiment analyzes the performance of the distributed algorithm when compared to a centralized learning strategy and to a local learning strategy. We compare the following algorithms:
\begin{itemize}
\item[\textbf{-}] \textbf{Centr-LapKRR}: this is the algorithm depicted in Sec. \ref{subsec:ssl}. It is equivalent to a single agent collecting all the training data.

\item[\textbf{-}] \textbf{Local-LapKRR}: in the local setting, the training set is distributed across the agents and every agent trains a LapKRR on its own dataset, without any communication with other agents. The error is averaged throughout the nodes.

\item[\textbf{-}] \textbf{Distr-LapKRR}: as before, the training set is distributed within the network, but the agents converge to a centralized solution using the strategy detailed in Sec. \ref{chap7:sec:algo}. In this experiment, the EDM completion is achieved by the DGD algorithm.
\end{itemize}
For all the algorithms, we build the Laplacian and the kernel matrices according to the method detailed in \cite{melacci2011laplacian}, using the parameters reported in Tab. \ref{tab:datasets_params}. In particular the parameters for datasets G$50$C and COIL come from \cite{melacci2011laplacian}, while those for $2$Moons and BCI come from \cite{belkin2006manifold} and \cite{Chapelle2006}, respectively.  Lower values for the exchange iterations in datasets $2$Moons and BCI are chosen to balance the higher values for the exchange fractions. The classification error and the computational time for the three models over the five datasets are reported in Table \ref{tab:datasets_error}. Results of the proposed algorithm, Distr-LapKRR, are highlighted in bold. 

\begin{table}[h]
\footnotesize
\renewcommand{\arraystretch}{1.3}
\caption[Average values for classification error and computational time, together with one standard deviation, for Distr-LapKRR and comparisons.]{Average values for classification error and computational time, together with standard deviation, for the three algorithms. Results for the proposed algorithm are highlighted in bold.}
\centering
\begin{tabular}{llrr}
\toprule
\multicolumn{1}{c}{Dataset} & \multicolumn{1}{c}{Algorithm} & \multicolumn{1}{c}{Error [\%]} & \multicolumn{1}{c}{Time [s]}\\
\midrule
 & Centr-LapKRR & $0.005 \pm 0.001$ & $0.006 \pm 0.015$\\
2Moons & \textbf{Distr-LapKRR} & $\vect{0.01 \pm 0.03}$ & $\vect{0.875 \pm 0.030}$ \\
& Local-LapKRR & $0.41 \pm 0.28$ & $0.000 \pm 0.000$\\ \midrule
 & Centr-LapKRR & $0.49 \pm 0.04$ & $0.021 \pm 0.012$\\
BCI & \textbf{Distr-LapKRR} & $\vect{0.49 \pm 0.05}$ & $\vect{3.396 \pm 0.028}$\\
& Local-LapKRR & $0.54 \pm 0.14$ & $0.001 \pm 0.000$\\ \midrule
 & Centr-LapKRR & $0.07 \pm 0.02$ & $0.101 \pm 0.017$\\
G50C & \textbf{Distr-LapKRR} & $\vect{0.12 \pm 0.10}$ & $\vect{5.764 \pm 0.066}$\\
& Local-LapKRR & $0.45 \pm 0.06$ & $0.001 \pm 0.000$\\ \midrule
 & Centr-LapKRR & $0.13 \pm 0.02$ & $1.565 \pm 0.019$\\
COIL20 & \textbf{Distr-LapKRR} & $\vect{0.13 \pm 0.02}$ & $\vect{195.933 \pm 2.176}$\\
& Local-LapKRR & $0.78 \pm 0.07$ & $0.056 \pm 0.001$\\ \midrule
 & Centr-LapKRR & $0.10 \pm 0.03$ & $1.556 \pm 0.028$\\
COIL2 & \textbf{Distr-LapKRR} & $\vect{0.10 \pm 0.03}$ & $\vect{191.478 \pm 0.864}$\\
& Local-LapKRR & $0.43 \pm 0.12$ & $0.055 \pm 0.000$\\ \bottomrule
\end{tabular}
\vspace{0.5em}
\label{tab:datasets_error}
\end{table}

\noindent We can see that Distr-LapKRR is generally able to match the same performance of the Centr-LapKRR, both in mean and variance, except for a small decrease in the G$50$C dataset. Clearly, the performance of Local-LapKRR is noticeably worse than the other two algorithms, because the local models are built on considerably smaller training sets. The computational time required by the distributed algorithm is given by the sum of the time required by both the exchange protocols, the distributed Laplacian estimation, the DAC protocol, and the matrix inversion in \eqref{chap7:eq: local_solution}. When comparing the results with the values for EDM completion time obtained in the previous experiment, we notice that the order of magnitude of the time required by Distr-LapKRR is given by the time necessary to complete the distance matrix.

\subsection{Privacy preservation}

As a final experiment, we include in our algorithm the two privacy-preserving strategies presented in Sec. \ref{subsec:privacy}. In particular, we analyze the evolution of the classification error when varying the ratio $m/d$ from $0.1$ to $0.95$, i.e. when varying the dimensionality $m$ of the transformed patterns. In this experiment we do not consider the $2$Moons dataset, because of its limited number of features. Since the value of $\sigma$ in the linear random projection has no influence on the error of the transformed patterns, we set $\sigma = 1$ for all the datasets. As for the nonlinear transformation, the values for the parameters are searched inside a grid and then optimized locally. Possible values for $t$ are searched in $10^i,\, i = \{1,\ldots,5\}$, while values for the variances are searched in $10^j,\, j = \{-6,\ldots,6\}$. The optimal values for the datasets are reported in the third group of Table \ref{tab:datasets_params}. 

Results of the experiment are presented in Fig. \ref{chap7:fig:privacy_error}. The classification error for the linear random projection and nonlinear transformation are shown with solid red and dashed blue lines, respectively. In addition, the mean value for Distr-LapKRR (together with its confidence interval) is reported as a baseline, shown with a dashed black line.

\begin{figure}[h]
\centering
\subfloat[Dataset: BCI]{\includegraphics[width=0.45\columnwidth]{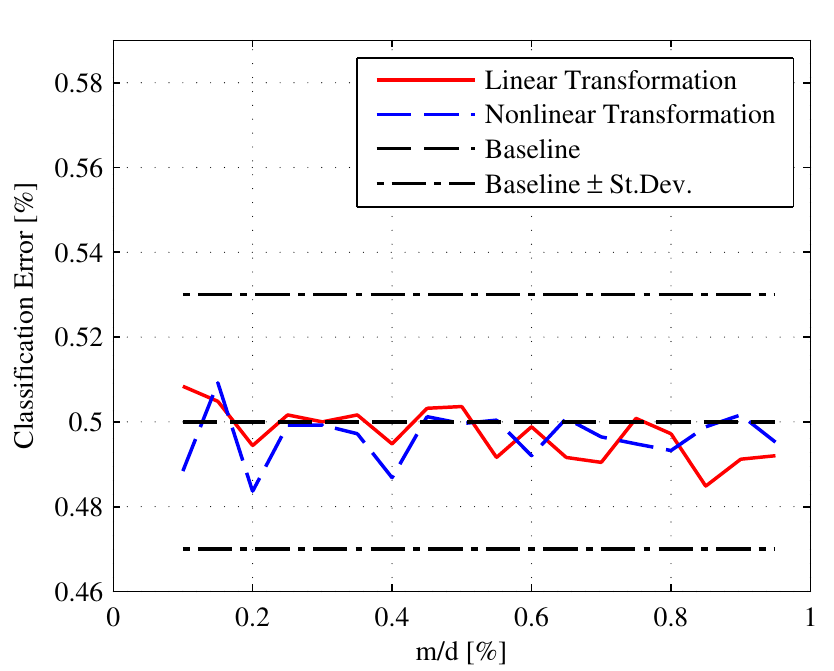} \label{chap7:fig:BCI_priv_err}} %
\subfloat[Dataset: G50C]{\includegraphics[width=0.45\columnwidth]{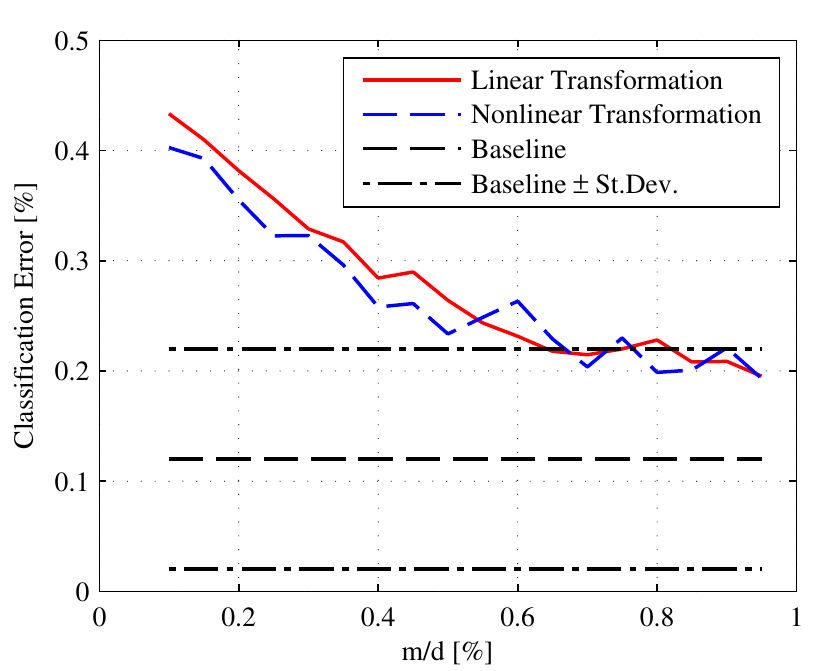} \label{chap7:fig:g50c_priv_err}} %
\vfill
\subfloat[Dataset: COIL$20$]{\includegraphics[width=0.45\columnwidth]{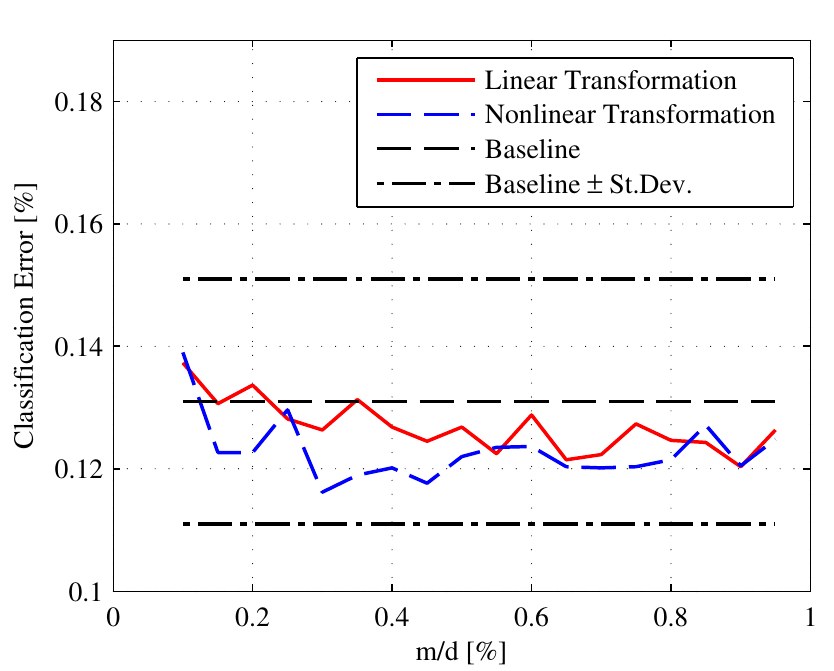} \label{chap7:fig:COIL20_priv_err}}
\subfloat[Dataset: COIL$2$]{\includegraphics[width=0.45\columnwidth]{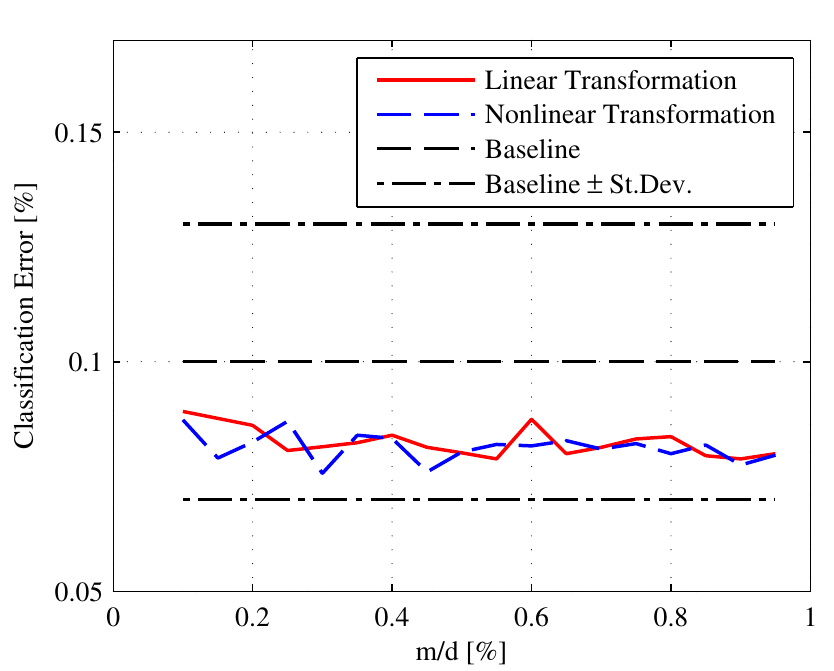} \label{chap7:fig:COIL2_priv_err}}
\caption{Average classification error of the privacy-preserving transformations on the considered datasets when varying the ratio $m/d$.}
\label{chap7:fig:privacy_error}
\end{figure}

By observing the results, we can see that when compared to Distr-LapKRR, the privacy-preserving strategies show different behaviors depending on the dataset. In particular, for dataset BCI, the error is nearly the same of Distr-LapKRR, while it is slightly lower for COIL$2$ and COIL$20$, and somewhat higher for G$50$C, where it shows a decreasing trend. 
For all the datasets, we see that the error achieved using the privacy-preserving strategies remains inside the limits of Distr-LapKRR error's confidence interval, denoting how the variability introduced does not have significant influence on the algorithm's performance. 

We notice that in most cases, we can obtain a comparable or even better performance with respect to the privacy-free algorithm, with significantly fewer features, leading to a reduction of the information exchanged and therefore of the overall computational requirements. For all the datasets, both the transformations present a non-smooth trend, caused by the heuristic nature of these methods. Moreover, the error is very similar between the strategies, suggesting that the use of a nonlinear transformation, potentially safer than a linear one, does not influence the performance.

%% file: chapters/chapter8-dist_ssl_svm.tex
\chapter[Distributed Semi-Supervised Support Vector Machines]{Distributed Semi-Supervised \\ Support Vector Machines}
\chaptermark{Distributed S3VM}
\label{chap:dist_ssl_next}

\minitoc
\vspace{15pt}

\blfootnote{The content of this chapter is adapted from the material published in \cite{scardapane2016distributedsemi}.}

\section{Introduction}

\lettrine{I}{n} the previous chapter, we have explored the problem of training a semi-supervised Laplacian KRR using a distributed computation of the underlying kernel matrix. However, despite its good performance, the resulting algorithm requires a large amount of computational and/or communication resources, which might not be available on specific devices or communication channels. To this end, in this chapter we propose two simpler algorithms for a different family of semi-supervised SVM, denoted as S$^3$VM. The S$^3$VM has attracted a large amount of attention over the last decades \citep{chapelle2008optimization}. It is based on the idea of minimizing the training error and maximizing the margin over both labeled and unlabeled data, whose labels are included as additional variables in the optimization problem. Since its first practical implementation in \citep{Joachims1999}, numerous researchers have proposed alternative solutions for solving the resulting mixed integer optimization problem, including branch and bound algorithms \citep{chapelle2006branch}, convex relaxations, convex-concave procedures \citep{chapelle2008optimization}, and others. It has been applied in a wide variety of practical problems, such as text inference \citep{Joachims1999}, and it has given birth to numerous other algorithms, including semi-supervised least-square SVMs \citep{adankon2009semisupervised}, and semi-supervised random vector functional-link networks \citep{scardapane2015semi}.

In order to simplify our derivation, in this chapter we focus on the \textit{linear} S$^3$VM formulation, whose decision boundary corresponds to an hyperplane in the input space. Due to this, the algorithms presented in this chapter can be implemented even on agents with stringent requirements in terms of power, such as sensors in a WSN. At the same time, it is known that limiting ourselves to a linear decision boundary can be reasonable, as the linear S$^3$VM can perform well in a large range of settings, due to the scarcity of labeled data \citep{chapelle2008optimization}.

Specifically, starting from the smooth approximation to the original S$^3$VM presented in \citep{chapelle2005semi}, we show that the distributed training problem can be formulated as the joint minimization of a sum of non-convex cost functions. This is a complex problem, which has been investigated only very recently in the distributed optimization literature \citep{bianchi2013convergence,di2015next}. In our case, we build on two different solutions. The first one is based on the idea of diffusion gradient descent (DGD), similarly to the previous chapter. Nevertheless, since it is a gradient-based algorithm exploiting only first order information of the objective function, it generally suffers of slow practical convergence speed, especially in the case of non-convex and large-scale optimization problems. Recently, it was showed in \citep{scutari2014decomposition,di2015next} that exploiting the structure of nonconvex functions by replacing their linearization (i.e., their gradient) with a ``better'' approximant can enhance practical convergence speed. Thus, we propose a distributed algorithm based on the recently proposed In-Network Successive Convex Approximation (NEXT) framework \citep{di2015next}.
The method hinges on successive convex approximation techniques while leveraging dynamic consensus as a mechanism to distribute the computation among the agents as well as diffuse the needed information over the network. Both algorithms are proved convergent to a stationary point of the optimization problem. Moreover, as shown in our experimental results, the NEXT exhibits a faster practical convergence speed with respect to DGD, which is paid by a larger computation cost per iteration.

The rest of the chapter is structured as follows. In Section \ref{chapter8:sec:ss_svm} we introduce the S$^3$VM model together with the approximation presented in \citep{chapelle2005semi}. In Section \ref{chapter8:sec:distributed_s3vm}, we first formulate the distributed training problem for S$^3$VMs, and subsequently we derive our two proposed solutions. Finally, Section \ref{chapter8:sec:results} details an extensive set of experimental results.

\section{Semi-Supervised Support Vector Machines}
\label{chapter8:sec:ss_svm}
Let us consider the standard SSL problem, where we are interested in learning a binary classifier starting from $L$ labeled samples $\left( \vect{x}_i, y_i \right)_{i=1}^L$ and $U$ unlabeled samples $\left( \vect{x}_i \right)_{i=1}^{U}$. As before, each input is a $d$-dimensional real vector $\vect{x}_i \in \R^d$, while each output can only take one of two possible values $y_i \in \left\{ -1, +1 \right\}$. The linear S$^3$VM optimization problem can be formulated as \citep{chapelle2005semi}:
\begin{equation}
\min_{\vect{w}, b, \hat{\vect{y}}} \; \frac{C_1}{2L}\sum_{i=1}^L l\left( y_i, f(\vect{x}_i) \right) + \frac{C_2}{2U}\sum_{i=1}^{U} l(\hat{y}_i, f\left(\vect{x}_i\right)) + \frac{1}{2} \norm{\vect{w}}^2 \,,
\label{eq:ss_svm}
\end{equation}
where $f(\vect{x}) = \vect{w}^T\vect{x} + b$, $\hat{\vect{y}} \in \left\{-1, +1\right\}^U$ is a vector of unknown labels, $l(\cdot, \cdot)$ is a proper loss function and $C_1, C_2 > 0$ are coefficients weighting the relative importance of labeled and unlabeled samples. The main difference with respect to the standard SVM formulation is the inclusion of the unknown labels $\hat{\vect{y}}$ as variables of the optimization problem. This makes Problem \eqref{eq:ss_svm} a mixed integer optimization problem, whose exact solution can be computed only for relatively small datasets, e.g. using standard branch-and-bound algorithms. We note that, for $C_2=0$, we recover the standard SVM formulation. The most common choice for the loss function is the hinge loss, given by:
\begin{equation}
l\left( y, f\left(\vect{x} \right)\right) = \max\left( 0, 1-yf(\vect{x}) \right)^p \,,
\end{equation}
where $p \in \mathbb{N}$. In this chapter, we use the choice $p=2$, which leads to a smooth and convex function. Additionally, it is standard practice to introduce an additional constraint in the optimization problem, so that the resulting vector $\hat{\vect{y}}$ has a fixed proportion $r \in \left[ 0, 1 \right]$ of positive labels:
\begin{equation}
\frac{1}{U}\sum_{i=1}^U \max\left( 0, \hat{y}_i \right) = r \,.
\label{eq:balancing_constraint}
\end{equation}
This constraint helps achieve a balanced solution, especially when the ratio $r$ reflects the true proportion of positive labels in the underlying dataset.

A common way of solving Problem \eqref{eq:ss_svm} stems from the fact that, for a fixed $\vect{w}$ and $b$, the optimal $\hat{\vect{y}}$ is given in closed form by $$\hat{y}_i = \text{sign}(\vect{w}^T\vect{x}_i + b), \; i=1,\ldots,U.$$
Exploiting this fact, it is possible to devise a \textit{continuous} approximation of the cost function in \eqref{eq:ss_svm} \citep{chapelle2008optimization}. In particular, to obtain a smooth optimization problem solvable by standard first-order methods, \citep{chapelle2005semi} propose to replace the hinge loss over the unknown labels with the approximation given by $\exp\left\{ -sf(\vect{x})^2 \right\}, s > 0$. In the following, we choose in particular $s=5$, as suggested by \citep{chapelle2008optimization}. A visual example of the approximation is illustrated in Fig. \ref{chapter8:fig:hinge_approximation}.
\begin{figure}[t]
	\centering
	\includegraphics{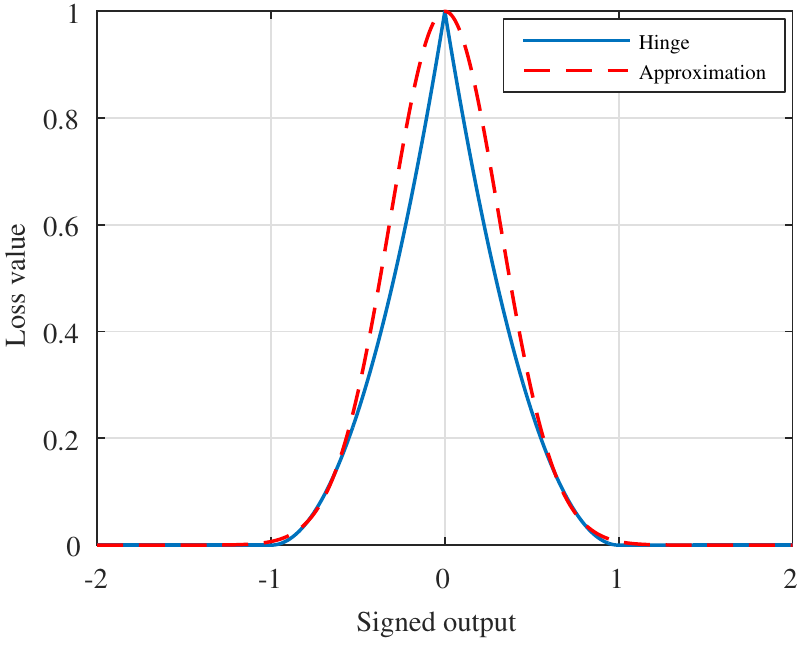}
	\caption[The hinge loss approximation is shown in blue for varying values of $f(\vect{x}_i)$, while in dashed red we show the approximation given by $\exp\left\{ -5f(\vect{x}_i)^2 \right\}$.]{For a fixed choice of $\vect{w}$ and $b$, $\max\left( 0, 1-\hat{y}_if(\vect{x}_i) \right)^2 = \max\left( 0, 1-|f(\vect{x}_i)| \right)^2$. This is shown in blue for varying values of $f(\vect{x}_i)$, while in dashed red we show the approximation given by $\exp\left\{ -5f(\vect{x}_i)^2 \right\}$.}
	\label{chapter8:fig:hinge_approximation}
\end{figure}
The resulting $\nabla$S$^3$VM optimization problem writes as:
\begin{equation}
\min_{\vect{w}, b} \; \frac{C_1}{2L}\sum_{i=1}^L l\left( y_i, f(\vect{x}_i) \right) + \frac{C_2}{2U}\sum_{i=1}^{U} \exp\left\{ -sf(\vect{x}_i)^2 \right\} + \frac{1}{2} \norm{\vect{w}}^2  \,.
\label{eq:delta_s3vm}
\end{equation}
\noindent Problem (\ref{eq:delta_s3vm}) does not incorporate the constraint in (\ref{eq:balancing_constraint}) yet. A possible way to handle the balancing constraint in \eqref{eq:balancing_constraint} is a relaxation that uses the following linear approximation \citep{chapelle2005semi}:
\begin{equation}
\frac{1}{U} \sum_{i=1}^U \vect{w}^T\vect{x}_i + b = 2r - 1 \,,
\end{equation}
which can easily be enforced for a fixed $r$ by first translating the unlabeled points so that their mean is $\vect{0}$, and then fixing the offset $b$ as $b = 2r-1$. The resulting problem can then be solved using standard first-order procedures.

\section{Distributed learning for S$^3$VM}
\label{chapter8:sec:distributed_s3vm}
In this section, we first formulate a distributed optimization problem for a $\nabla$S$^3$VM over a network of agents in Section \ref{chapter8:sec:formulation_of_the_problem}. Then, we present two alternative methods for solving the overall optimization problem in a fully decentralized fashion in Sections \ref{chapter8:sec:gradient_diffusion} and \ref{chapter8:sec:next}.
\subsection{Formulation of the problem}
\label{chapter8:sec:formulation_of_the_problem}

For the rest of this chapter, we assume that labeled and unlabeled training samples are not available on a single processor. Instead, they are distributed over a network of $N$ agents. In particular, as in the previous chapter, we assume that the $k$th node has access to $L_k$ labeled samples, and $U_k$ unlabeled ones, such that $\sum_{k=1}^N L_k = L$ and $\sum_{k=1}^N U_k = U$. Assumptions on the topology are similar to the rest of the thesis. The distributed $\nabla$S$^3$VM problem can be cast as:
\begin{equation}
\min_{\vect{w}} \;\sum_{k=1}^N l_k(\vect{w}) + \sum_{k=1}^N g_k(\vect{w}) + r(\vect{w}) \,,
\label{eq:distributed_delta_s3vm}
\end{equation}
where we have defined the following shorthands:
\begin{align}
l_k(\vect{w}) & = \frac{C_1}{2L} \sum_{i=1}^{L_k} l\left( y_{k,i}, f(\vect{x}_{k,i}) \right) \,, \\
g_k(\vect{w}) & = \frac{C_2}{2U} \sum_{i=1}^{U_k} \exp\left\{ -sf(\vect{x}_{k,i})^2 \right\} \,, \\
r(\vect{w}) & = \frac{1}{2} \norm{\vect{w}}^2 \,.
\end{align}
In the previous equations, we use the double subscript $(k,i)$ to denote the $i$th sample available at the $k$th node, and we assume that the bias $b$ has been fixed \textit{a-priori} using the strategy detailed in the previous section. In a distributed setting, this requires that each agent knows the mean of all unlabeled points given by $\frac{1}{U}\sum_{i=1}^U\vect{x}_i$. This can easily be achieved, before starting the training process, with a number of different in-network algorithms. For example, the agents can compute the average using a DAC procedure, push-sum protocols \citep{hensel2009gadget} in a P2P network, or a number of alternative techniques.

\subsection{Solution $1$: Distributed gradient descent}
\label{chapter8:sec:gradient_diffusion}

The first solution is based on the DGD procedure, which has already been used extensively in the previous chapter for the distributed EDM completion problem. The main problem is that all the previous art on DGD focused on the solution of convex versions of problem the DSO problem. In our case, the $g_k(\vect{w})$ are non-convex, and the analysis in the aforementioned papers cannot be used. However, convergence of a similar family of algorithms in the case of non-convex (smooth) cost functions has been recently studied in \citep{bianchi2013convergence}. Customizing the DGD method in (\ref{chap3:eq:gd_step}) to Problem \eqref{eq:distributed_delta_s3vm}, we obtain the following local update at each agent:
\begin{align}
&\psi_k =  \vect{w}_k[n] - \alpha_k[n] \left( \nabla l_k(\vect{w}_k[n]) + \nabla g_k(\vect{w}_k[n]) + \frac{1}{N} \nabla r(\vect{w}_k[n]) \right) \,. \label{eq:gdg_step1_real}
\end{align}
Note that we have included a factor $\frac{1}{N}$ in \eqref{eq:gdg_step1_real} in order to be consistent with the formulation in \eqref{chap3:eq:dso_cost_function}. Defining the margin $m_{k,i} = y_{k,i}f(\vect{x}_{k,i})$, we can easily show that:
\begin{align}
\nabla l_k(\vect{w}) & = %
		\displaystyle -\frac{C_1}{L} \sum_{i=1}^{L_k} \mathbb{I}(1-m_{k,i}) \cdot m_{k,i}\left( 1 - m_{k,i}\right) \,, \\
\nabla g_k(\vect{w}) & = \displaystyle - s\frac{C_2}{U} \sum_{i=1}^{U_k} \exp\left\{ -sf(\vect{x}_{k,i})^2 \right\}f(\vect{x}_{k,i})\vect{x}_{k,i} \,, \\
\nabla r(\vect{w}) & = \vect{w} \,,
\end{align}
where $\mathbb{I}(\cdot)$ is the indicator function defined for a generic scalar $o \in \R$ as:
\[
\mathbb{I}(o) = \begin{cases}
1 & \; \text{ if } o  \le 0 \\
0 & \; \text{ otherwise }
\end{cases} \,.
\]
 The overall algorithm is summarized in Algorithm \ref{algo:dgd}. Its convergence properties are illustrated in following theorem.

\begin{AlgorithmCustomWidth}[h]
\caption{Distributed $\nabla$S$^3$VM using a distributed gradient descent procedure.}
\label{algo:dgd}
\begin{algorithmic}[1]
		\Require Regularization factors $C_1, C_2$, maximum number of iterations $T$.
		\State \textbf{Initialization}:
		\State \hspace{\algorithmicindent} $\vect{w}_k[0] = \vect{0}$, $k=1, \ldots, N$.
		\For{$n$ from $0$ to $T$}
		\State \textbf{for } $k$ from $1$ to $N$ \textbf{do in parallel}
		\State \hspace{\algorithmicindent} Compute auxiliary variable $\psi_k$ using \eqref{eq:gdg_step1_real}.
		\State \hspace{\algorithmicindent} Combine estimates as $\vect{w}_k[n+1] = \sum_{t=1}^N C_{kt} \psi_t$.
		\State \textbf{end for}
		\EndFor
\end{algorithmic}
\end{AlgorithmCustomWidth}

\begin{theorem}
Let $\{\vect{w}_k[n]\}_{k=1}^N$ be the sequence generated by Algorithm 1, and let $\bar{\vect{w}}[n] = \frac{1}{N}\sum_{k=1}^N \vect{w}_k[n]$ be its average across the agents. Let us select the step-size sequence $\left\{ \alpha[n] \right\}_n$ such that i) $\alpha[n]\in (0,1]$, for all $n$, ii) $\sum_{n=0}^{\infty}\alpha[n]=\infty$; and iii) $\sum_{n=0}^{\infty}\alpha[n]^2<\infty$. Then, if the sequence $\left\{ \bar{\vect{w}}[n] \right\}_n$ is bounded \footnote{Note that this condition is not restrictive in practical implementations. Indeed, one can always limit the behavior of the algorithm using a finite (but arbitrarily large) box constraint that guarantees the boundedness of the sequence $\left\{ \bar{\vect{w}}[n] \right\}_n$, and thus the convergence of the method.}, (a) \emph{\texttt{[convergence]}:} all its limit points are stationary solutions of problem \eqref{eq:distributed_delta_s3vm}; (b) \emph{\texttt{[consensus]}}: all the sequences $\vect{w}_k[n]$ asymptotically agree, i.e.
$\lim_{n\rightarrow +\infty} \norm{\vect{w}_k[n] - \bar{\vect{w}}[n]} = 0$,  $k=1,\ldots, N$.
\end{theorem}
\begin{proof}
See \citep{bianchi2013convergence}.
\end{proof}

\subsection{Solution $2$: In-network successive convex approximation}
\label{chapter8:sec:next}

The DGD algorithm is extremely efficient to implement, however, as we discussed in the introduction, its convergence is often sub-optimal due to two main reasons. First, the update in \eqref{eq:gdg_step1_real} considers only first order information and does not take into account the fact that the local cost function has some hidden convexity (since it is composed by the sum of a convex term plus a non-convex term) that one can properly exploit. Second, each agent $k$ obtains information on the cost functions $J_t(\cdot)$, $t \neq k$, only in a very indirect way through the averaging step. In this section, we use a recent framework for in-network non-convex optimization from \citep{di2015next}, which exploits the structure of nonconvex functions by replacing their linearization (i.e., their gradient) with a ``better'' approximant, thus typically resulting in enhanced practical convergence speed. In this section we customize the NEXT algorithm from \citep{di2015next} to our case, and we refer to the original chapter for more details.

The main idea of NEXT is to parallelize the problem in \eqref{eq:distributed_delta_s3vm} such that, at each agent, the original (global) non-convex cost function is replaced with a strongly convex surrogate that preserves the first order conditions, see \citep{di2015next}. To this aim, we associate to agent $k$ the surrogate $F_k(\vect{w}; \vect{w}_k[n])$, which is obtained by: i) keeping unaltered the local convex function $l_k(\vect{w})$ and the regularization function $r(\vect{w})$; ii) linearizing the local non-convex cost $g_k(\vect{w})$
and all the other (non-convex and unknown) terms $f_l(\vect{w})$ and $g_l(\vect{w})$, $l\neq k$, around the current local iterate $\vect{w}_k[n]$. As a result, the surrogate at node $k$ takes the form:
\begin{equation}
\begin{split}
F_k(\vect{w}; \vect{w}_k[n]) & =  l_k(\vect{w}) + \tilde{g}_k(\vect{w}; \vect{w}_k[n]) + r(\vect{w}) \\
& + \boldsymbol{\pi}_k(\vect{w}_k[n])^T\left( \vect{w} - \vect{w}_k[n] \right) \,,
\end{split}
\label{eq:surrogate_cost_function}
\end{equation}
where
\begin{equation}
\tilde{g}_k(\vect{w}; \vect{w}_k[n]) = g_k(\vect{w}_k[n]) + \nabla g_k^T(\vect{w}_k[n])\left( \vect{w} - \vect{w}_k[n] \right) \,,
\label{eq:g_tilde}
\end{equation}
and $\boldsymbol{\pi}_k(\vect{w}_k[n])$ is defined as:
\begin{equation}
\boldsymbol{\pi}_k(\vect{w}_k[n]) = \sum_{t \neq k} \nabla h_k(\vect{w}_k[n]) \,,
\label{eq:pi}
\end{equation}
with $\nabla h_k(\cdot) = \nabla l_k(\cdot) + \nabla g_k(\cdot)$. Clearly, the information in \eqref{eq:pi} related to the knowledge of the other cost functions is not available at node $k$. To cope with this issue, the NEXT approach consists in replacing $\boldsymbol{\pi}_k(\vect{w}_k[n])$  in \eqref{eq:surrogate_cost_function} with a local estimate $\tilde{\boldsymbol{\pi}}_k[n]$  that asymptotically converges to $\boldsymbol{\pi}_k(\vect{w}_k[n])$, thus considering the local approximated surrogate $\tilde{F}(\vect{w}; \vect{w}_k[n], \tilde{\boldsymbol{\pi}}_k[n])$ given by:
\begin{equation}
\begin{split}
\tilde{F}_k(\vect{w}; \vect{w}_k[n], \tilde{\boldsymbol{\pi}}_k[n]) & =  l_k(\vect{w}) + \tilde{g}_k(\vect{w}; \vect{w}_k[n])+ r(\vect{w}) \\
& + \tilde{\boldsymbol{\pi}}_k[n]^T\left( \vect{w} - \vect{w}_k[n] \right) \,.
\end{split}
\label{eq:surrogate_cost_function_approximated}
\end{equation}
In the first phase of the algorithm, each agent solves a convex optimization problem involving the surrogate function in \eqref{eq:surrogate_cost_function_approximated}, thus obtaining a new estimate $\tilde{\vect{w}}_k[n]$.
Then, an auxiliary variable $\vect{z}_k[n]$ is computed as a convex combination of the current estimate $\vect{w}_k[n]$ and the new $\tilde{\vect{w}}_k[n]$, as:
\begin{equation}
\vect{z}_k[n] = \vect{w}_k[n] + \alpha[n]\left( \tilde{\vect{w}}_k[n] - \vect{w}_k[n] \right) \,.
\label{eq:z_update_next}
\end{equation}
where $\alpha[n]$ is a possibly time-varying step-size sequence. This concludes the optimization phase of NEXT. The consensus phase of NEXT consists of two main steps. First, to achieve asymptotic agreement among the estimates at different nodes, each agent updates its local estimate combining the auxiliary variables from the neighborhood, i.e., for all $k$,
\begin{equation}
\vect{w}_k[n+1] = \sum_{t=1}^N C_{kt} \vect{z}_t[n] \,.
\label{eq:first_consensus_step_next}
\end{equation}
This is similar to the diffusion step of the DGD procedure. Second, the update of the local estimate $\tilde{\boldsymbol{\pi}}_k[n]$ in \eqref{eq:surrogate_cost_function_approximated} is computed in two steps: i) an auxiliary variable $\vect{v}_k[n]$ is updated through a dynamic consensus step as:
\begin{equation}
\vect{v}_k[n+1] = \sum_{t=1}^N C_{kt}\vect{v}_t[n] + \Bigl( \nabla h_k(\vect{w}_k[n+1]) - \nabla h_k(\vect{w}_k[n])  \Bigr) \,;
\label{eq:v_update}
\end{equation}
ii) the variable $\tilde{\boldsymbol{\pi}}_k[n]$ is updated as:
\begin{equation}
\tilde{\boldsymbol{\pi}}_k[n+1] = N \vect{v}_k[n+1] - \nabla h_k(\vect{w}_k[n+1]) \,.
\label{eq:pi_update}
\end{equation}
The steps of the NEXT algorithm for Problem (\ref{eq:distributed_delta_s3vm}) are described in Algorithm \ref{algo:next}. Its convergence properties are described by a Theorem completely similar to Theorem 1, and the details on the proof can be found in \citep{di2015next}.

\begin{AlgorithmCustomWidth}[h]
\caption{Distributed $\nabla$S$^3$VM using the In-Network Convex Optimization framework.}
\label{algo:next}
\begin{algorithmic}[1]
		\Require Regularization factors $C_1, C_2$, maximum number of iterations $T$.
		\State \textbf{Initialization}:
		\State \hspace{\algorithmicindent} $\vect{w}_k[0] = \vect{0} ,\;\; k=1, \ldots, N$.
		\State \hspace{\algorithmicindent} $\vect{v}_k[0] = \nabla h_k(\vect{w}_k[0]) ,\;\; k = 1, \ldots, N$.
		\State \hspace{\algorithmicindent} $\tilde{\boldsymbol{\pi}}_k[0] = (N-1)\vect{v}_k[0] ,\;\; k = 1, \ldots, N$.
		\For{$n$ from $0$ to $T$}
		\State \textbf{for } $k$ from $1$ to $N$ \textbf{do in parallel}
		\State \hspace{\algorithmicindent} Solve the local optimization problem: \begin{center}
			$
			\tilde{\vect{w}}_k[n] = \arg\min \tilde{F}_k(\vect{w}; \vect{w}_k[n], \tilde{\boldsymbol{\pi}}_k[n]) \,.
			$
			\end{center}
		\State \hspace{\algorithmicindent} Compute $\vect{z}_k[n]$ using \eqref{eq:z_update_next}.
		\State \textbf{end for}
		\State \textbf{for } $k$ from $1$ to $N$ \textbf{do in parallel}
		\State \hspace{\algorithmicindent} Perform consensus step in \eqref{eq:first_consensus_step_next}.
		\State \hspace{\algorithmicindent} Update auxiliary variable using \eqref{eq:v_update}.
		\State \hspace{\algorithmicindent} Set $\tilde{\boldsymbol{\pi}}_k[n+1]$ as \eqref{eq:pi_update}..
		\State \textbf{end for}
		\EndFor
\end{algorithmic}
\end{AlgorithmCustomWidth}

\section{Experimental Results}
\label{chapter8:sec:results}
\subsection{Experimental Setup}
We tested the proposed distributed algorithms on three semi-supervised learning benchmarks, whose overview is given in Tab. \ref{chapter8:tab:datasets}. For more details on the datasets see \citep{Chapelle2006} and the previous chapter for the first two, and \citep{mierswa2005automatic} and Chapter \ref{chap:dist_rvfl_sequential} for GARAGEBAND. For this one, the original dataset comprises $9$ different musical genres. In order to obtain a binary classification task, we select the two most prominent ones, namely `rock' and `pop', and discard the rest of the dataset. For G$50$C and GARAGEBAND, input variables are normalized between $-1$ and $1$. The experimental results are computed over a $10$-fold cross-validation, and all the experiments are repeated $15$ times. For each repetition, the training folds are partitioned in one labeled and one unlabeled datasets, according to the proportions given in Tab. \ref{chapter8:tab:datasets}. Results are then averaged over the $150$ repetitions.

\begin{table*}
\small
\centering
\caption[Description of the datasets used for testing the distributed S$^3$VM.]{Description of the datasets. The fourth and fifth columns denote the size of the training and unlabeled datasets, respectively.}
\begin{tabular}{lllllll}
\toprule
Name & Features & Instances & $L$ & $U$ & Ref.\\
\midrule
G$50$C & $50$ & $550$ & $40$ & $455$ & \citep{Chapelle2006}\\
PCMAC & $7511$ & $1940$ & $40$ & $1700$ & \citep{Chapelle2006} \\
GARAGEBAND & $44$ & $790$ & $40$ & $670$ & \citep{mierswa2005automatic}\\
\bottomrule
\end{tabular}
\vspace{0.5em}
\label{chapter8:tab:datasets}
\end{table*}

We compare the following models:
\begin{itemize}
\item[-] \textbf{LIN-SVM}: this is a fully supervised SVM with a linear kernel, trained only on the labeled data. The model is trained using the LIBSVM library \cite{chang2011libsvm}.
\item[-] \textbf{RBF-SVM}: similar to before, but a RBF kernel is used instead. The parameter for the kernel is set according to the internal heuristic of LIBSVM.
\item[-] \textbf{C-$\nabla$S$3$VM}: this is a centralized $\nabla$S$^3$VM trained on both the labeled and the unlabeled data using a gradient descent procedure.
\item[-] \textbf{DG-$\nabla$S$3$VM}: in this case, training data (both labeled and unlabeled) is distributed evenly across the network, and the distributed model is trained using the diffusion gradient algorithm detailed in Section \ref{chapter8:sec:gradient_diffusion}.
\item[-] \textbf{NEXT-$\nabla$S$3$VM}: data is distributed over the network as before, but the model is trained through the use of the NEXT framework, as detailed in Section \ref{chapter8:sec:next}. The internal optimization problem in \eqref{eq:surrogate_cost_function_approximated} is solved using a standard gradient descent procedure. 
\end{itemize}

For C-$\nabla$S$3$VM, DG-$\nabla$S$3$VM and NEXT-$\nabla$S$3$VM we set $s = 5$ and a maximum number of iterations $T = 500$. In order to obtain a fair comparison between the algorithms, we also introduce a stopping criterion, i.e. the algorithms terminate when the norm of the gradient of the global cost function in \eqref{eq:distributed_delta_s3vm} at the current iteration is less than $10^{-5}$. Clearly, this is only for comparison purposes, and a truly distributed implementation would require a more sophisticated mechanism, which however goes outside the scope of the present chapter. The same value for the threshold is set for the gradient descent algorithm used within the NEXT framework to optimize the local surrogate function in \eqref{eq:surrogate_cost_function_approximated}. In this case, we let the gradient run for a maximum of $T = 50$ iterations. We note that, in general, we do not need to solve the internal optimization problem to optimal accuracy, as convergence of NEXT is guaranteed as long as the the problems are solved with increasing accuracy for every iteration \citep{di2015next}.

We searched the values of $C_1$ and $C_2$ by executing a $5$-fold cross-validation in the interval $\{10^{-5}, 10^{-4}, \ldots, 10^3\}$ using  C-$\nabla$S$3$VM as in \citep{chapelle2005semi}. The values of these parameters are then shared with DG-$\nabla$S$3$VM and NEXT-$\nabla$S$3$VM. For all the models, included NEXT's internal gradient descent algorithm, the step-size $\alpha$ is chosen using a decreasing strategy given by:
\begin{equation}
\alpha[n] = \frac{\alpha_0}{(n+1)^\delta} \,,
\end{equation}
where $\alpha_0, \delta > 0$ are set by the user. In particular, this strategy satisfies the convergence conditions for both the DGD algorithm and NEXT. After preliminary tests, we selected for every model the values of $\alpha_0$ and $\delta$ that guarantee the fastest convergence. The optimal values of the parameters are shown in Tab. \ref{chapter8:tab:params}.

\begin{table*}[t]
\small
\renewcommand{\arraystretch}{1.3}
\centering
\caption[Optimal values of the parameters used in the experiments for the distributed S$^3$VM.]{Optimal values of the parameters used in the experiments. In the first group are reported the values of the regularization coefficients for the three models, averaged over the $150$ repetitions. In the following groups are reported the values of the initial step-size and of the diminishing factor for C-$\nabla$S$3$VM, DG-$\nabla$S$3$VM and NEXT-$\nabla$S$3$VM respectively.}
\begin{tabular}{l|cc|cc|cc|cc}
\toprule
\multicolumn{1}{l}{Dataset} & \multicolumn{1}{|c}{$C_1$} & \multicolumn{1}{c}{$C_2$} & \multicolumn{1}{|c}{$\alpha_0^\text{C}$} & \multicolumn{1}{c}{$\delta^\text{C}$} & \multicolumn{1}{|c}{$\alpha_0^\text{DG}$} & \multicolumn{1}{c}{$\delta^\text{DG}$} & \multicolumn{1}{|c}{$\alpha_0^\text{NEXT}$} & \multicolumn{1}{c}{$\delta^\text{NEXT}$} \\
\midrule
G$50$C & $1$ & $1$ & $1$ & $0.55$ & $1$ & $0.55$ & $0.6$ & $0.8$ \\
PCMAC & $100$ & $100$ & $0.1$ & $0.55$ & $1$ & $0.9$ & $0.5$ & $0.8$ \\
GARAGEBAND & $2$ & $5$ & $0.09$ & $0.8$ & $0.1$ & $0.1$ & $0.05$ & $0.55$\\
\bottomrule
\end{tabular}
\vspace{0.5em}
\label{chapter8:tab:params}
\end{table*}
The network topologies are generated according to the `Erd\H{o}s-R\'{e}nyi model', such that every edge has a $25\%$ probability of appearing. The only constraint is that the network is connected. The topologies are generated at the beginning of the experiments and kept fixed during all the repetitions. We choose the weight matrix $\vect{C}$ using the \textit{Metropolis-Hastings} strategy as in previous chapters. This choice of the weight matrix satisfies the convergence conditions for both the distributed approaches.

\subsection{Results and discussion}

The first set of experiments consists in analyzing the performance of C-$\nabla$S$3$VM, when compared to a linear SVM and RBF SVM trained only on the labeled data. While these results are well known in the semi-supervised literature, they allow us to quantitatively evaluate the performance of C-$\nabla$S$3$VM, in order to provide a coherent benchmark for the successive comparisons. Results of this experiment are shown in Tab. \ref{chapter8:tab:centralized_results}.

\begin{table}
\small
\centering
\caption[Average value for classification error and computational time for the centralized SVMs.]{Average value for classification error and computational time for the centralized algorithms.}
\begin{tabular}{llll}
\toprule
\multicolumn{1}{l}{Dataset} & \multicolumn{1}{l}{Algorithm} & \multicolumn{1}{l}{Error [\%]} & \multicolumn{1}{l}{Time [s]}\\
\midrule
& LIN-SVM & 13.79 & 0.0008\\
G$50$C & RBF-SVM & 13.36 & 0.0005\\
& \textbf{C-$\nabla$S$3$VM} & \textbf{6.36} & \textbf{0.024}\\
\midrule
& LIN-SVM & 21.32 & 0.0035\\
PCMAC & RBF-SVM & 36.68 & 0.0032\\
& \textbf{C-$\nabla$S$3$VM} & \textbf{6.10} & \textbf{35.12}\\
\midrule
& LIN-SVM & 23.87 & 0.0010\\
GARAGEBAND & RBF-SVM & 27.92 & 0.0007\\
& \textbf{C-$\nabla$S$3$VM} & \textbf{21.50} & \textbf{0.2872}\\
\bottomrule
\end{tabular}
\vspace{0.5em}
\label{chapter8:tab:centralized_results}
\end{table}

We can see that, for all the datasets, C-$\nabla$S$3$VM outperforms standard SVMs trained only on labeled data, with a reduction of the classification error ranging from $2.37\%$ on GARAGEBAND to $15.22\%$ on PCMAC. Clearly, the training time required by C-$\nabla$S$3$VM is higher than the time required by a standard SVM, due to the larger number of training data, and to the use of the gradient descent algorithm. Another important aspect to be considered is that, with the only exception of G$50$C, the RBF-SVM fails in matching the performance of the linear model due to higher complexity of the model in relationship to the amount of training data.

Next, we investigate the convergence behavior of DG-$\nabla$S$3$VM and NEXT-$\nabla$S$3$VM, compared to the centralized implementation. In particular, we test the algorithm on randomly generated networks of $L=25$ nodes. Results are presented in Fig. \ref{chapter8:fig:obj_fcn_and_grad_norm}. Particularly, panels on the left show the evolution of the global cost function in \eqref{eq:distributed_delta_s3vm}, while panels on the right show the evolution of the squared norm of the gradient. For readability, the graphs use a logarithmic scale on the $y$-axis, while on the left we only show the first $50$ iterations of the optimization procedure. The results are similar for all three datasets, namely, NEXT-$\nabla$S$3$VM is able to converge faster (up to one/two orders of magnitude) than DG-$\nabla$S$3$VM, which can only exploit first order information on the local cost functions. Indeed, both NEXT-$\nabla$S$3$VM and the centralized implementation are able to converge to a stationary point in a relatively small number of iterations, as shown by the panels on the left. The same can be seen from the gradient norm evolution, shown on the right panels, where the fast convergence of NEXT-$\nabla$S$3$VM is even more pronounced. Similar insights can be obtained by the analysis of the box plots in Fig. \ref{chap8:fig:boxplot}, where we also compare with the results of LIN-SVM and RBF-SVM obtained previously. 

\begin{figure*}[p]
	\centering
	\subfloat[Objective function (G50C)]{\includegraphics[scale=0.76]{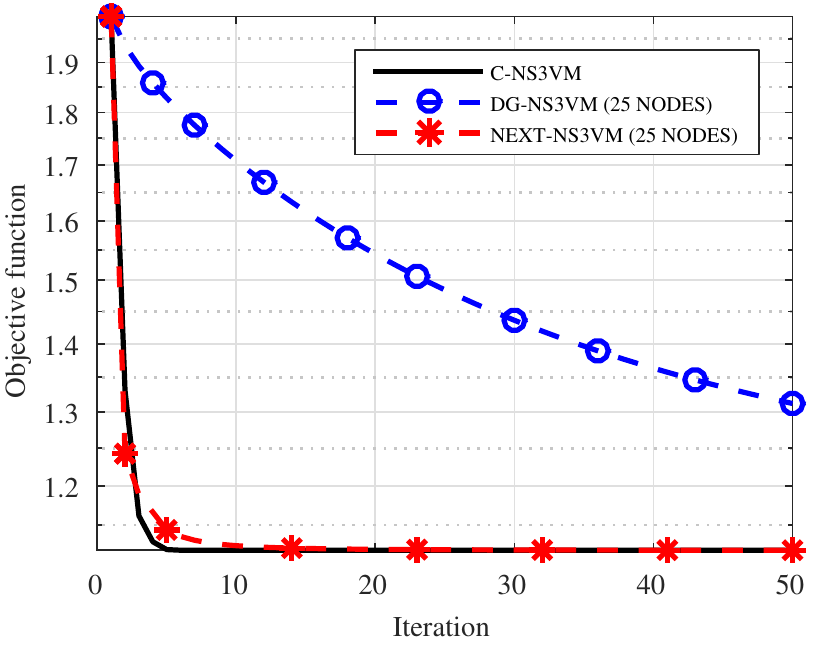}%
		\label{chapter8:fig:Objective_function_g50c}}
	\hfil
	\subfloat[Gradient norm (G50C)]{\includegraphics[scale=0.76]{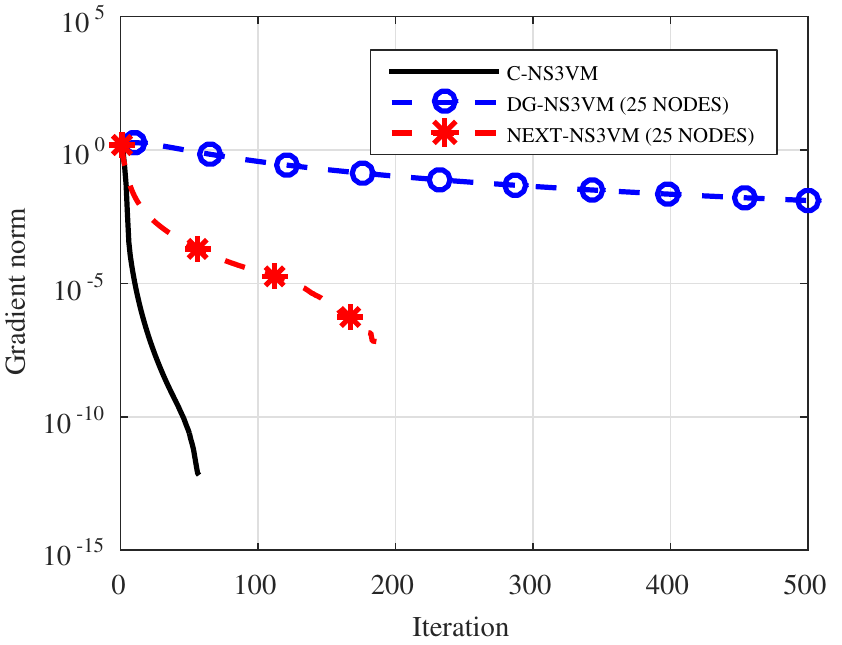} %
		\label{chapter8:fig:Gradient_norm_g50c}}
	\vfil
	\subfloat[Objective function (PCMAC)]{\includegraphics[scale=0.76]{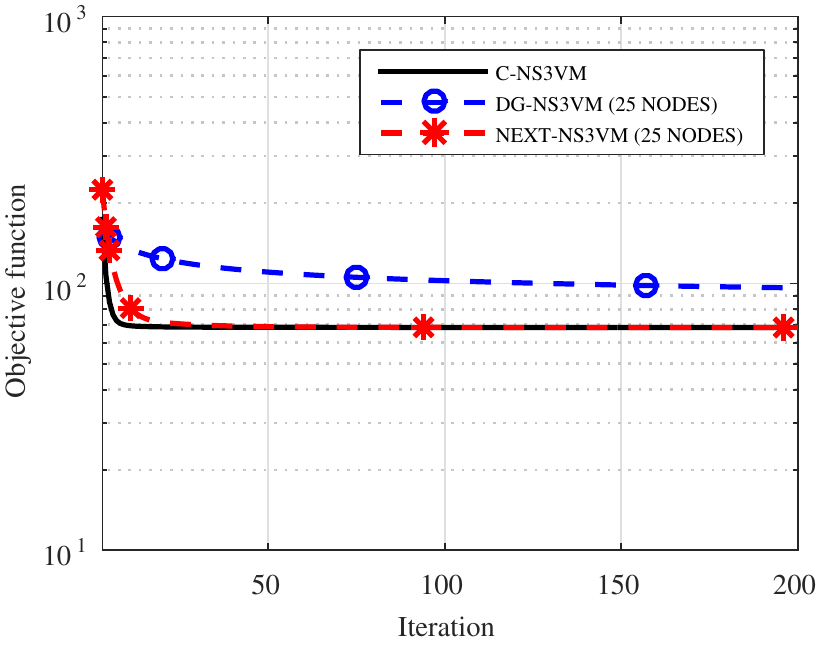}%
		\label{chapter8:fig:Objective_function_pcmac}}
	\hfil
	\subfloat[Gradient norm (PCMAC)]{\includegraphics[scale=0.76]{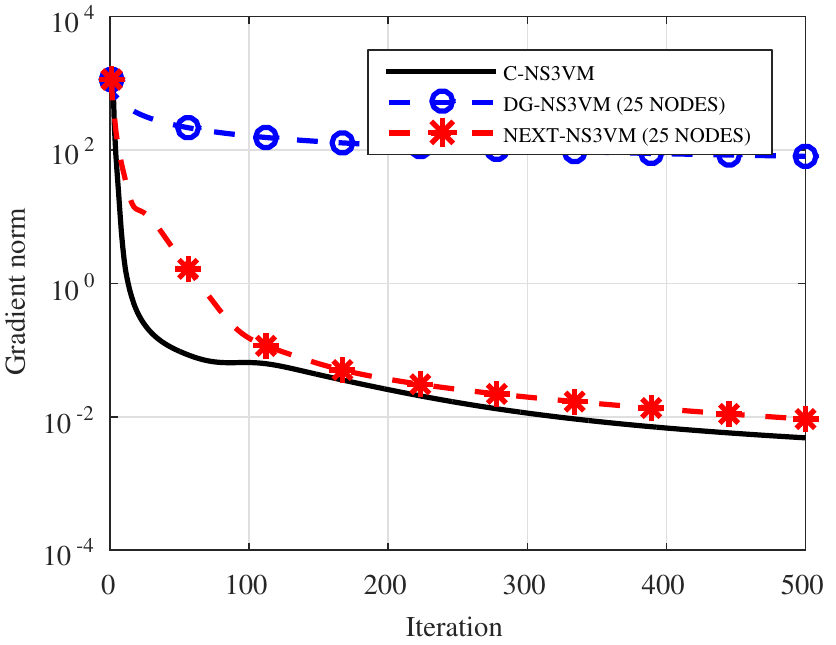} %
		\label{chapter8:fig:Gradient_norm_pcmac}}
	\vfil
	\subfloat[Objective function (GARAGEBAND)]{\includegraphics[scale=0.76]{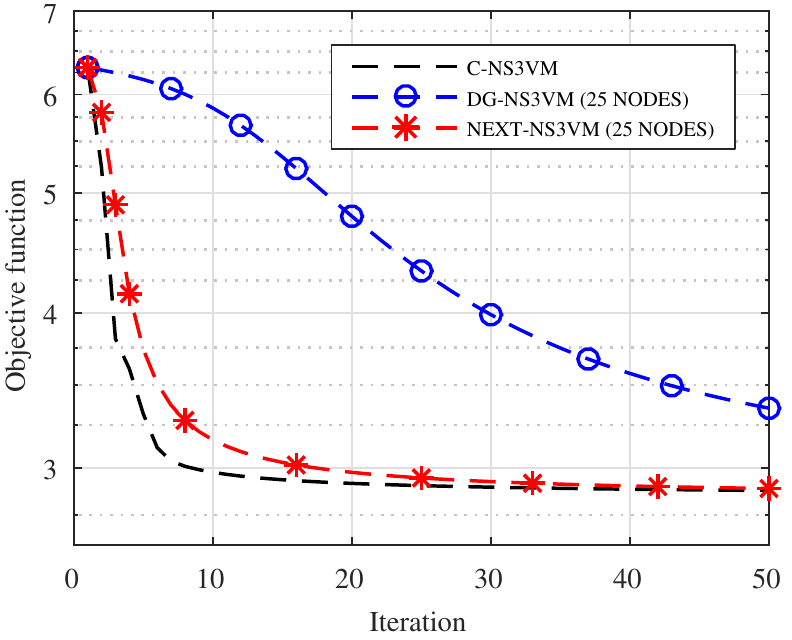}%
		\label{chapter8:fig:Objective_function_garageband}}
	\hfil
	\subfloat[Gradient norm (GARAGEBAND)]{\includegraphics[scale=0.76]{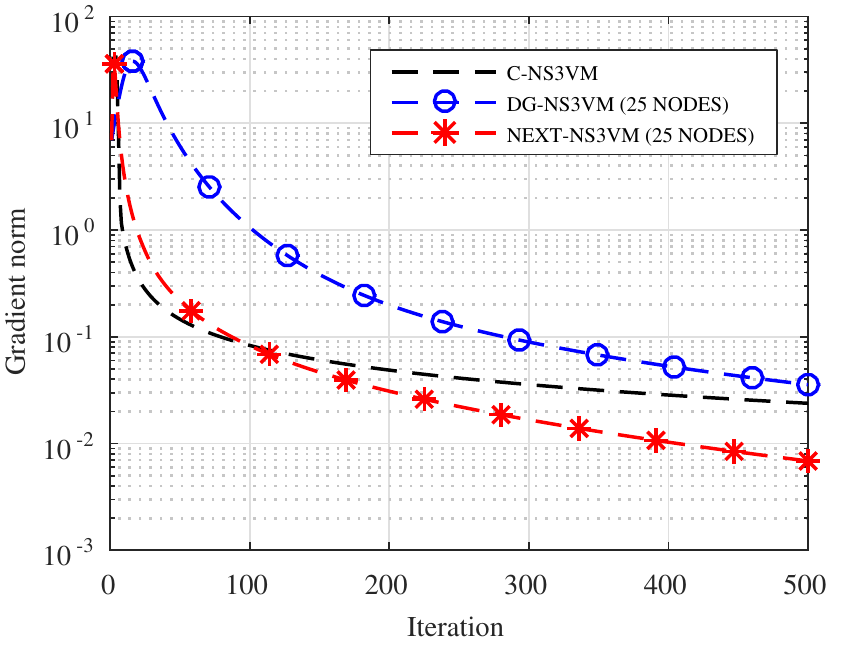} %
		\label{chapter8:fig:Gradient_norm_garageband}}
	\caption[Convergence behavior of DG-$\nabla$S$3$VM and NEXT-$\nabla$S$3$VM, compared to C-$\nabla$S$3$VM.]{Convergence behavior of DG-$\nabla$S$3$VM and NEXT-$\nabla$S$3$VM, compared to C-$\nabla$S$3$VM. The panels on the left show the evolution of the global cost function, while the panels on the right show the evolution of the squared norm of the gradient.}
	\label{chapter8:fig:obj_fcn_and_grad_norm}
\end{figure*}

\begin{figure}[t]
	\centering
	\subfloat[G50C]{\includegraphics[scale=0.7]{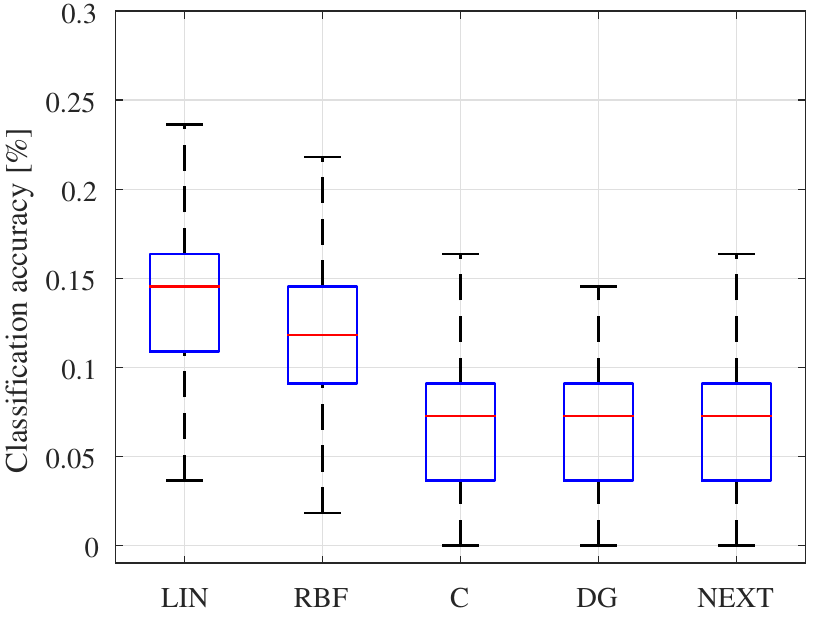}%
		\label{fig:boxplot_g50c}}
	\hfil
	\subfloat[PCMAC]{\includegraphics[scale=0.7]{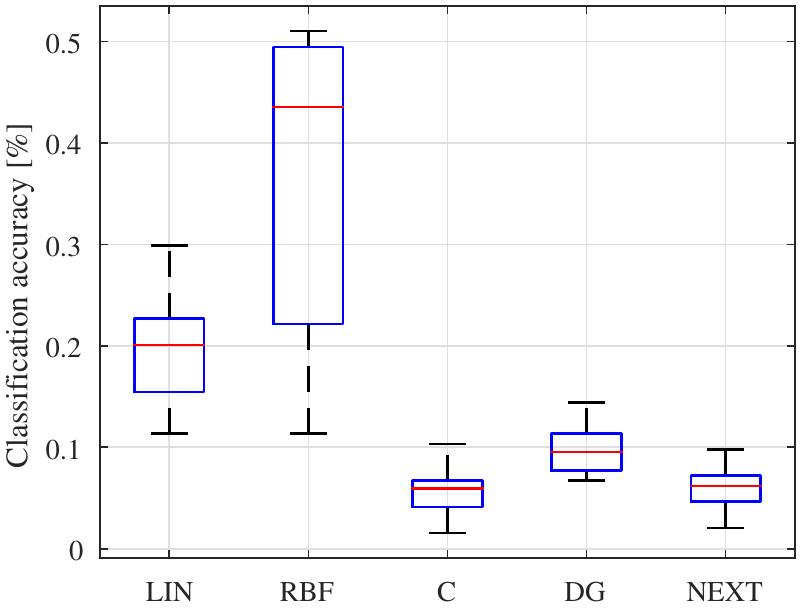}%
		\label{fig:boxplot_pcmac}}
	\hfil
	\subfloat[GARAGEBAND]{\includegraphics[scale=0.7]{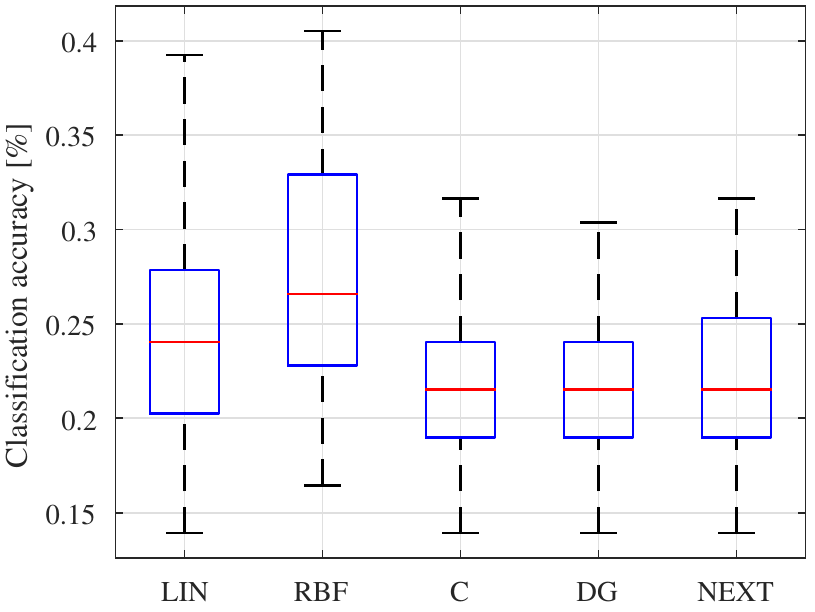}%
		\label{fig:boxplot_garageband}}
	\caption[Box plots for the classification accuracy of the centralized and distributed S$^3$VM algorithms.]{Box plots for the classification accuracy of the $5$ algorithms, in the case $N=25$. The central line is the median, the edges are the $25$th and $75$th percentiles, and the whiskers extend to the most extreme data points. For readability, the names of the algorithms have been abbreviated to LIN (LIN-SVM), RBF (RBF-SVM), C (C--$\nabla$S$3$VM), DG (DG-$\nabla$S$3$VM) and NEXT (NEXT-$\nabla$S$3$VM).}
	\label{chap8:fig:boxplot}
\end{figure}

\begin{figure*}[!p]
	\centering
	\subfloat[Classification error (G50C)]{\includegraphics[scale=0.76]{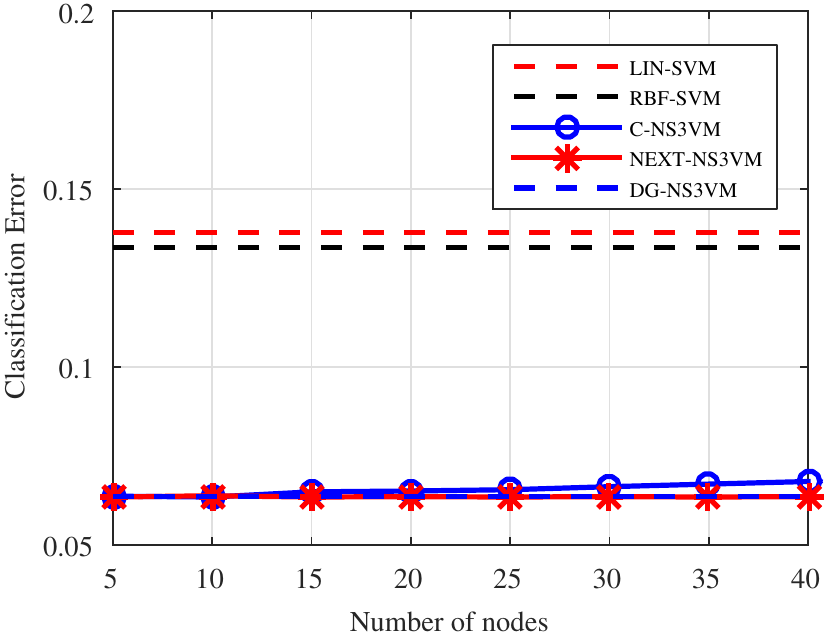}%
		\label{chapter8:fig:Classification_Error_g50c}}
	\hfil
	\subfloat[Training time (G50C)]{\includegraphics[scale=0.76]{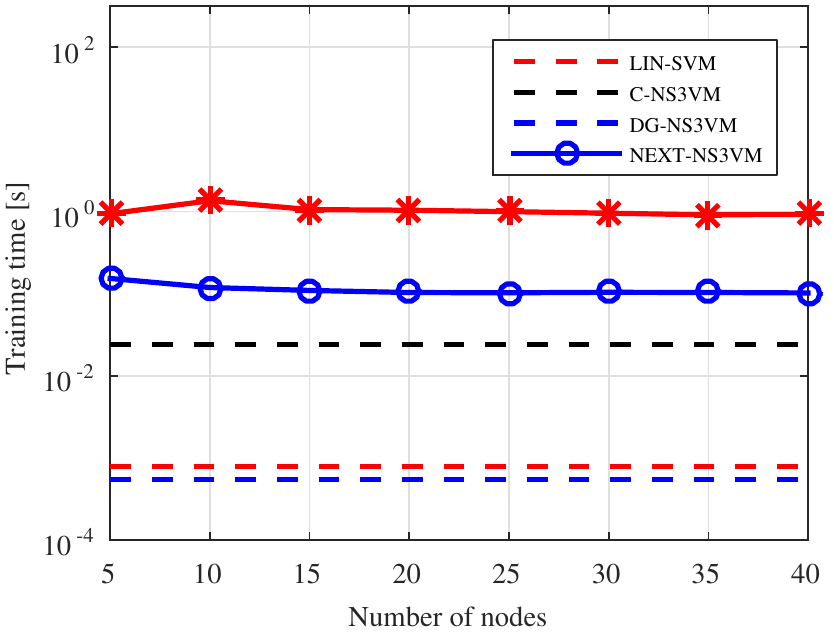} %
		\label{chapter8:fig:Training_time_g50c}}
	\vfil
	\subfloat[Classification error (PCMAC)]{\includegraphics[scale=0.76]{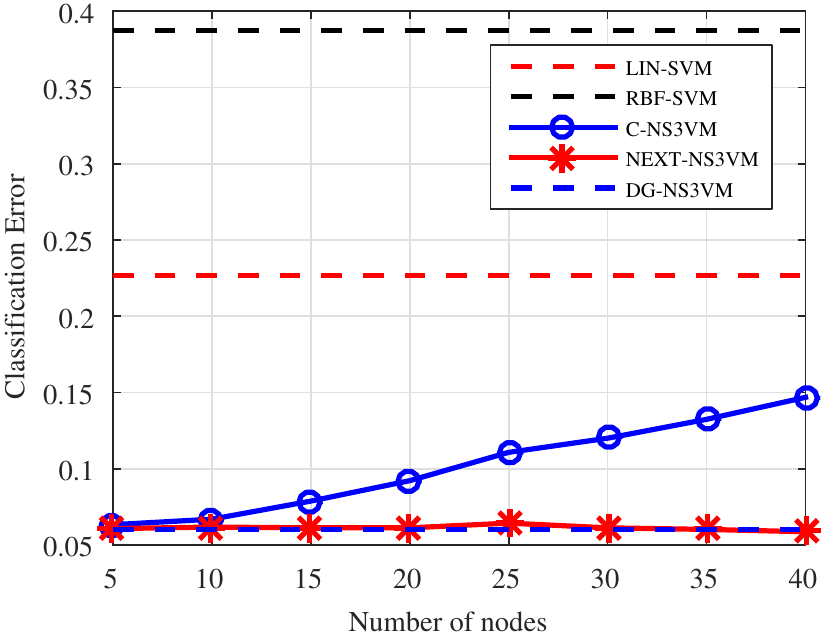}%
		\label{chapter8:fig:Classification_Error_pcmac}}
	\hfil
	\subfloat[Training time (PCMAC)]{\includegraphics[scale=0.76]{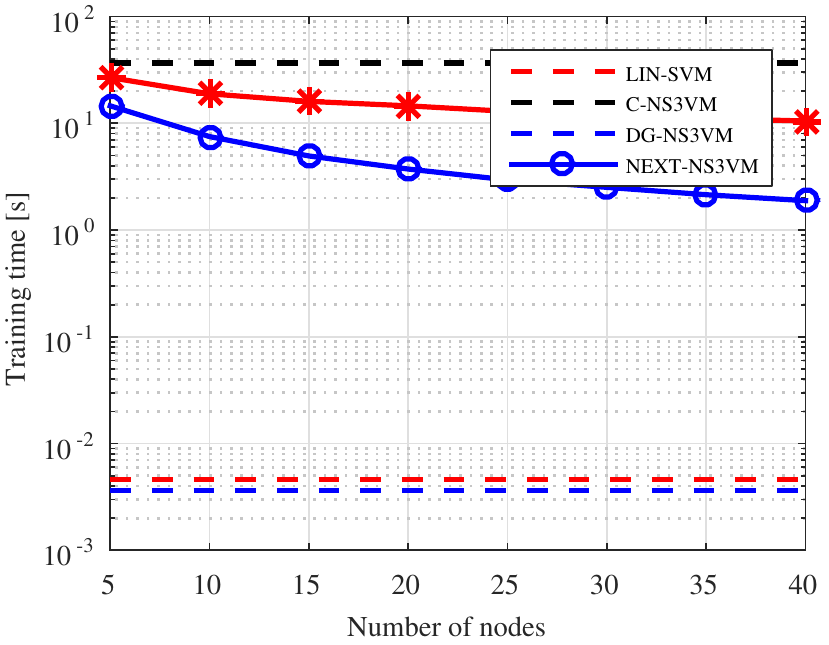} %
		\label{chapter8:fig:Training_time_pcmac}}
	\vfil
	\subfloat[Classification error (GARAGEBAND)]{\includegraphics[scale=0.76]{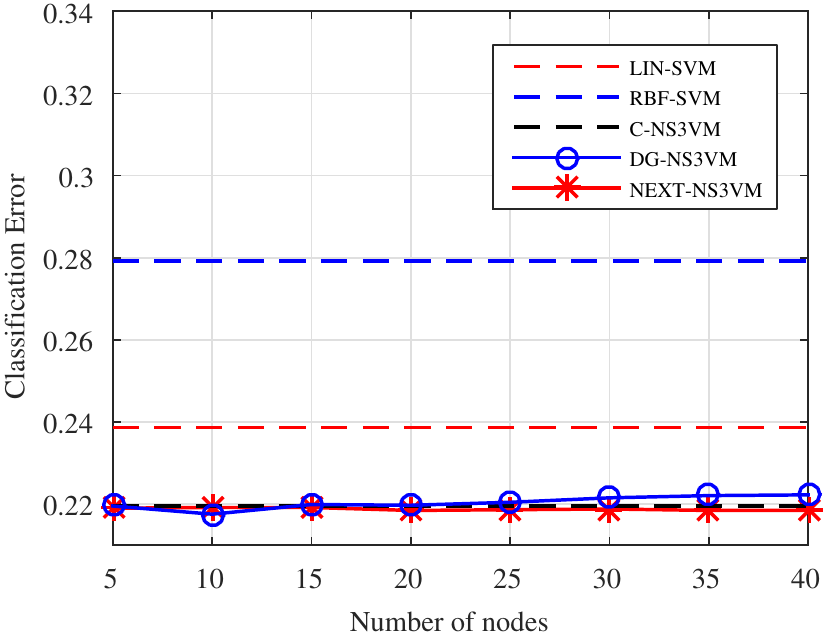}%
		\label{chapter8:fig:Classification_Error_garageband}}
	\hfil
	\subfloat[Training time (GARAGEBAND)]{\includegraphics[scale=0.76]{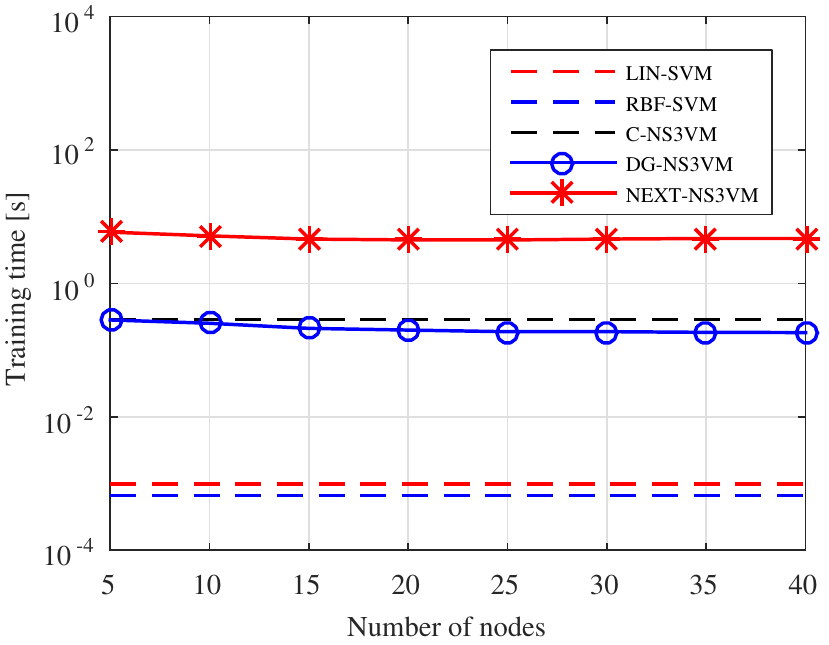} %
		\label{chapter8:fig:Training_time_garageband}}
	\vfil
	\caption[Training time and test error of GD-$\nabla$S$3$VM and NEXT-$\nabla$S$3$VM when varying the number of nodes in the network from $L=5$ to $L=40$.]{Training time and test error of GD-$\nabla$S$3$VM and NEXT-$\nabla$S$3$VM when varying the number of nodes in the network from $L=5$ to $L=40$. Results for LIN-SVM, RBF-SVM and C-$\nabla$S$3$VM are shown with dashed lines for comparison.}
	\label{chapter8:fig:classification_error_and_training_time}
\end{figure*}

As a final experiment, we investigate the scalability of the distributed algorithms, by analyzing the training time and the test error of DG-$\nabla$S$3$VM and NEXT-$\nabla$S$3$VM when varying the number of nodes in the network from $L=5$ to $L=40$ by steps of $5$. Results of this experiment are shown in Fig. \ref{chapter8:fig:classification_error_and_training_time}. The three panels on the left show the evolution of the classification error, while the three panels on the right show the evolution of the training time. Results of LIN-SVM, RBF-SVM and C-$\nabla$S$3$VM are shown with dashed lines for comparison. It is possible to see that NEXT-$\nabla$S$3$VM can track efficiently the centralized solution in all settings, regardless of the size of the network, while DG-$\nabla$S$3$VM is not able to properly converge (in the required number of iterations) for larger networks on PCMAC. With respect to training time, results are more varied. Generally speaking, NEXT-$\nabla$S$3$VM requires in average more training time than DG-$\nabla$S$3$VM. However, for large datasets (PCMAC and GARAGEBAND) both algorithms are comparable in training time with the centralized solution and, more notably, their training time generally decreases for bigger networks.

It is worth mentioning here that the results presented in this chapter strongly depend on our selection of the step-size sequences, and the specific surrogate function in \eqref{eq:surrogate_cost_function_approximated}. In the former case, it is known that the convergence speed of any gradient descent procedure can be accelerated by considering a proper adaptable step-size criterion. Along similar reasonings, the training time of NEXT-$\nabla$S$3$VM can in principle be decreased by loosening the precision to which the internal surrogate function is optimized, due to the convergence properties of NEXT already mentioned above. Finally, we can also envision a different choice of surrogate function for NEXT-$\nabla$S$3$VM, in order to achieve a different trade-off between training time and speed of convergence. As an example, we can replace the hinge loss $l_k(\vect{w})$ with its first-order linearization $\tilde{l}_k(\vect{w})$, similarly to \eqref{eq:g_tilde}. In this case, the resulting optimization problem would have a closed form solution, resulting in a faster training time per iteration (at the cost of more iterations required for convergence).

Overall, the experimental results suggest that both algorithms can be efficient tools for training a $\nabla$S$3$VM in a distributed setting, wherein NEXT-$\nabla$S$3$VM is able to converge extremely faster, at the expense of a larger training time. Thus, the choice of a specific algorithm will depend on the applicative domain, and on the amount of computational resources (and size of the training dataset) available to each agent.

%% file: chapters/chapter9-dist_esn.tex
\chapter{Distributed Training for Echo State Networks}
\chaptermark{Distributed Training for ESN}
\label{chap:dist_esn}

\minitoc
\vspace{15pt}

\blfootnote{The content of this chapter is adapted from the material published in \citep{Scardapane2015esn}, except Section \ref{sec:extension_sparse_esns}, whose content is currently under final editorial review at IEEE Computational Intelligence Magazine.}

\section{Introduction}

\lettrine[lines=2]{I}{n} the previous part of this thesis, we considered \textit{static} classification and regression tasks, where the order of presentation of the different examples does not matter. In many real world applications, however, the patterns exhibit a temporal dependence among them, as in time-series prediction. In this case, it is necessary to include some form of memory of the previously observed patterns in the ANN models. In this respect, there are two main possibilities. The first is to include an external memory, by feeding as input a buffer of the last $K$ patterns, with $K$ chosen \textit{a priori}. Differently, it is possible to consider \textit{recurrent} connections inside the ANN, which effectively create an internal memory of the previous state, making the ANN a dynamic model. This last class of ANNs are called recurrent neural networks (RNNs).

In the DL setting, the former option has been investigated extensively, particularly using linear and kernel adaptive filters (see Section \ref{chap4:sec:diffusion_filtering} and Section \ref{chap4:sec:distributed_kernel}). The latter option, however, has received considerably less attention. In fact, despite numerous recent advances (e.g. \citep{hermans2013training}), RNN training remains a daunting task even in the centralized case, mostly due to the well-known problems of the exploding and vanishing gradients \citep{pascanu2012difficulty}. A decentralized training algorithm for RNNs, however, would be an invaluable tool in multiple large-scale real world applications, including time-series prediction on WSNs \citep{predd2006survey}, and multimedia classification over P2P networks.

In this chapter we aim to bridge (partially) this gap, by proposing a distributed training algorithm for a recurrent extension of the RVFL, the ESN. ESNs were introduced by H. Jaeger \citep{jaeger2010the} and together with liquid state machines and backpropagation-decorrelation, they form the family of RNNs known as reservoir computing \citep{lukovsevivcius2009reservoir}. The main idea of ESNs, similar to RVFLs, is to separate the recurrent part of the network (the so-called `reservoir'), from the non-recurrent part (the `readout'). The reservoir is typically fixed in advance, by randomly assigning its connections, and the learning problem is reduced to a standard linear regression over the weights of the readout. Due to this, ESNs do not required complex back-propagation algorithms over the recurrent portion of the network, thus avoiding the problems of the exploding and vanishing gradients. Over the last years, ESNs have been applied successfully to a wide range of domains, including chaotic time-series prediction \citep{jaeger2004harnessing,li2012chaotic}, load prediction \cite{Bianchi2015}, grammatical inference \citep{tong2007learning}, and acoustic modeling \citep{6587732}, between others. While several researchers have investigated the possibility of spatially distributing the reservoir \citep{obst2013distributed,shutin2008echo,vandoorne2011parallel}, to the best of our knowledge, no algorithm has been proposed to train an ESN in the DL setting. 

The remaining of the chapter is formulated as follows. In Section \ref{sec:primer_esn} we introduce the basic concepts on ESNs and a least-square criterion for training them. Section \ref{sec:distributed_esn} details a distributed algorithm for ESNs, extending the ADMM-RVFL presented in Chapter \ref{chap:dist_rvfl}. After some experimental results, we also present an extension to consider ESNs with \textit{sparse} readouts in Section \ref{sec:extension_sparse_esns}.

\section{A primer on ESNs}
\label{sec:primer_esn}

An ESN is a recurrent neural network which can be partitioned in three components, as shown in Fig. \ref{chap8:fig:esn}. 

\begin{figure}[h]
	\centering
	\includegraphics[width=0.5\linewidth, keepaspectratio]{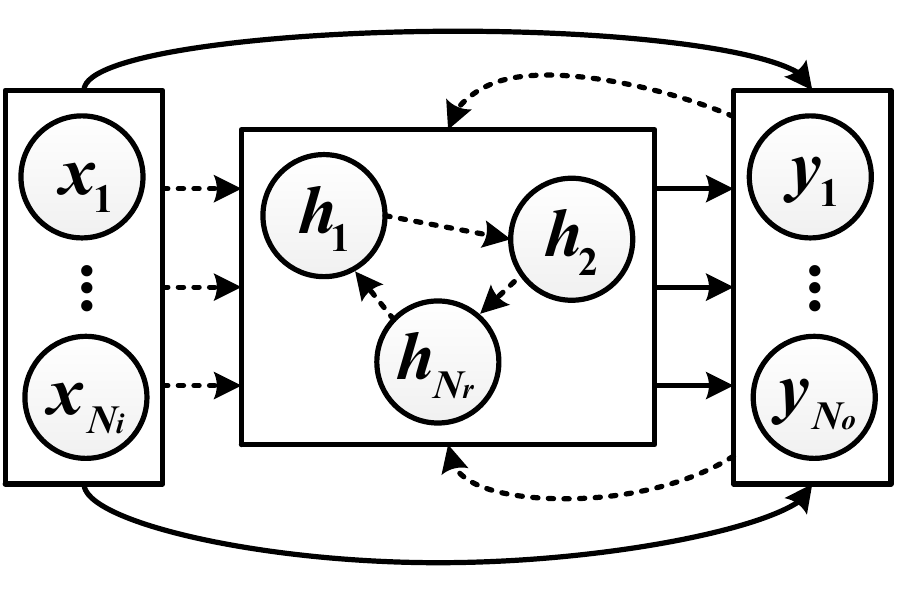}
	\caption[Schematic depiction of a ESN with multiple outputs.]{Schematic depiction of a ESN. Random connections are shown with dashed lines, while trainable connections are shown with solid lines.}
	\label{chap8:fig:esn}
\end{figure}

The $N_i$-dimensional input vector $\vect{x}[n] \in \R^{N_i}$ is fed to an $N_r$-dimensional reservoir, whose internal state $\vect{h}[n-1] \in \R^{N_r}$ is updated according to the state equation:
\begin{equation}
	\vect{h}[n] = f_{\text{res}}(\vect{W}_i^r \vect{x}[n] + \vect{W}_r^r \vect{h}[n-1] + \vect{W}_o^r\vect{y}[n-1]) \,,
	\label{chap8:eq:esn_state_update}
\end{equation}
\noindent where $\vect{W}_i^r \in \R^{N_r \times N_i}$, $\vect{W}_r^r \in \R^{N_r \times N_r}$ and $\vect{W}_o^r \in \R^{N_r \times N_o}$ are randomly generated matrices, $f_{\text{res}}(\cdot)$ is a suitably defined nonlinear function, and $\vect{y}[n-1] \in \R^{N_o}$ is the previous $N_o$-dimensional output of the network. To increase stability, it is possible to add a small uniform noise term to the state update, before computing the nonlinear transformation $f_{\text{res}}(\cdot)$ \citep{jaeger2002adaptive}. Then, the current output is computed according to:
\begin{equation}
\vect{y}[n] = f_{\text{out}}(\vect{W}_i^o \vect{x}[n] + \vect{W}_r^o \vect{h}[n]) \,,
\label{chap8:eq:esn_output_update}
\end{equation}
\noindent where $\vect{W}_i^o \in \R^{N_o \times N_i}, \vect{W}_r^o \in \R^{N_o \times N_r}$ are adapted based on the training data, and $f_{\text{out}}(\cdot)$ is an invertible nonlinear function. For simplicity, in the rest of the chapter we will consider the case of one-dimensional output, i.e. $N_o = 1$, but everything we say extends straightforwardly to the case $N_o > 1$.

To be of use in any learning application, the reservoir must satisfy the so-called `echo state property' (ESP) \citep{lukovsevivcius2009reservoir}. Informally, this means that the effect of a given input on the state of the reservoir must vanish in a finite number of time-instants. A widely used rule-of-thumb that works well in most situations is to rescale the matrix $\vect{W}_r^r$ to have $\rho(\vect{W}_r^r) < 1$, where $\rho(\cdot)$ denotes the spectral radius operator. For simplicity, we adopt this heuristic strategy in this chapter, but we refer the interested reader to \citep{yildiz2012re} for recent theoretical studies on this aspect. If the ESP is satisfied, an ESN with a suitably large $N_r$ can approximate any nonlinear filter with bounded memory to any given level of accuracy \citep{lukovsevivcius2009reservoir}.

\section{Distributed training for ESNs}
\label{sec:distributed_esn}

To train the ESN, suppose we are provided with a sequence of $Q$ desired input-outputs pairs $(\vect{x}[1], d[1]) \ldots, (\vect{x}[Q], d[Q])$. The sequence of inputs is fed to the reservoir, giving a sequence of internal states $\vect{h}[1], \ldots, \vect{h}[Q]$ (this is known as `warming'). During this phase, since the output of the ESN is not available for feedback, the desired output is used instead in Eq. \eqref{chap8:eq:esn_state_update} (so-called `teacher forcing'). Define the hidden matrix $\vect{H}$ and output vector $\vect{d}$ as:
\begin{align}
\vect{H} = & 
\left[\begin{array}{c}
\vect{x}^T[1] \,\, \vect{h}^T[1] \\
\vdots \\
\vect{x}^T[1] \,\, \vect{h}^T[Q]
\end{array}\right] \label{chap8:eq:state_matrix}\\
\vect{d} =  &
\left[\begin{array}{c}
f_{\text{out}}^{-1}(d[1]) \\
\vdots \\
f_{\text{out}}^{-1}(d[Q])
\end{array}\right] \label{chap8:eq:output_vector}
\end{align}
\noindent The optimal output weight vector is then given by solving the following regularized least-square problem:
\begin{equation}
\vect{w}^* = \argmin_{\vect{w} \in \R^{N_i+N_r}}  \frac{1}{2}\norm{\vect{H}\vect{w} - \vect{d}}^2 + \frac{\lambda}{2}\norm{\vect{w}}^2 \,,
\label{chap8:eq:esn_opt}
\end{equation}
\noindent where $\vect{w} = \left[ \vect{W}_i^o \, \vect{W}_r^o \right]^T$ and $\lambda$ is the standard regularization factor.\footnote{Since we consider one dimensional outputs, $\vect{W}_i^o$ and $\vect{W}_r^o$ are now row vectors, of dimensionality $N_i$ and $N_r$ respectively.} Solution of problem \eqref{chap8:eq:esn_opt} is a standard LRR problem as in Eq. \eqref{chap2:eq:lls}, and can be obtained in closed form as:
\begin{equation}
\vect{w}^* = \left( \vect{H}^T\vect{H} + \lambda \vect{I} \right)^{-1}\vect{H}^T \vect{d} \,.
\label{chap8:eq:ridge_regression_minimizer}
\end{equation} 
\noindent Whenever $N_r + N_i > Q$, Eq. \eqref{chap8:eq:ridge_regression_minimizer} can be computed more efficiently by rewriting it as:
\begin{equation}
\vect{w}^* = \vect{H}^T\left( \vect{H}\vect{H}^T + \lambda \vect{I} \right)^{-1}\vect{d} \,.
\label{chap8:eq:ridge_regression_minimizer_2}
\end{equation} 
\noindent More in general, we are provided with a training set $S$ of multiple desired sequences. In this case, we can simply stack the resulting hidden matrices and output vectors, and solve Eq. \eqref{chap8:eq:esn_opt}. Additionally, we note that in practice we can remove the initial $D$ elements (denoted as `wash-out' elements) from each sequence when solving the least-square problem, with $D$ specified \textit{a-priori}, due to their transient state. In the DL setting, we suppose that the $S$ sequences are distributed among the $L$ agents. Clearly, since training results in a LRR problem, at this point we can directly apply any of the algorithms presented in Chapter \ref{chap:dist_rvfl}. In particular, we choose to apply the ADMM algorithm due to its convergence properties. The resulting distributed protocol is summarized in Algorithm \ref{algo:dist_esn}.

\begin{AlgorithmCustomWidth}[h]
	\caption{ADMM-ESN: Local training algorithm for ADMM-based ESN ($k$th node).}
	\label{algo:dist_esn}
	\begin{algorithmic}[1]
		\Require Training set $S_k$ (local), size of reservoir $N_r$ (global), regularization factors $\lambda, \gamma$ (global), maximum number of iterations $T$ (global)
		\Ensure Optimal output weight vector $\vect{w}^*$
		\State Assign matrices $\vect{W}_i^r$, $\vect{W}_r^r$ and $\vect{W}_o^r$, in agreement with the other agents in the network.
		\State Gather the hidden matrix $\vect{H}_k$ and teacher signal $\vect{d}_k$ from $S_k$.
		\State Initialize $\vect{t}_k[0] = \vect{0}$, $\vect{z}[0] = \vect{0}$.
		\For{$n$ from $0$ to $T$}
		\State $\vect{w}_k[n+1] =  \left( \vect{H}_k^T \vect{H}_k + \gamma \vect{I} \right)^{-1} \left( \vect{H}^T_k \vect{d}_k - \vect{t}_k[n] + \gamma \vect{z}[n] \right)$.
		\State Compute averages $\hat{\vect{w}}$ and $\hat{\vect{t}}$ by means of the DAC procedure (see Appendix \ref{appA:sec:consensus}).
		\State $\vect{z}[n+1] =  \frac{\gamma \hat{\vect{w}} + \hat{\vect{t}}}{\lambda/L + \gamma}$.
		\State $\vect{t}_k[n+1] =  \vect{t}_k[n] + \gamma\left( \vect{w}_k[n+1] - \vect{z}[n+1] \right)$.
		\State Check termination with residuals (see Section \ref{chap4:sec:stopping_criterion}).
		\EndFor
		\State \textbf{return} $\vect{z}[n]$
	\end{algorithmic}
\end{AlgorithmCustomWidth}

\subsubsection*{Remark}
A large number of techniques have been developed to increase the generalization capability of ESNs without increasing its computational complexity \citep{lukovsevivcius2009reservoir}. Provided that the optimization problem in Eq. \eqref{chap8:eq:esn_opt} remains unchanged, and the topology of the ESN is not modified during the learning process, many of them can be applied straightforwardly to the distributed training case with the algorithm presented in this chapter. Examples of techniques that can be used in this context include lateral inhibition \citep{xue2007decoupled} and random projections \citep{butcher2013reservoir}. Conversely, techniques that cannot be straightforwardly applied include intrinsic plasticity \citep{steil2007online} and reservoir's pruning \citep{scardapane2014effective}.

\section{Experimental Setup}
\label{sec:results}
In this section we describe our experimental setup. Simulations were performed on MATLAB R2013a, on a $64$bit operative system, using an Intel\textsuperscript{\textregistered} Core\textsuperscript{\texttrademark} i5-3330 CPU with $3$ GHZ and $16$ GB of RAM.
\subsection{Description of the Datasets}

We validate the proposed ADMM-ESN on four standard artificial benchmarks applications, related to nonlinear system identification and chaotic time-series prediction. These are tasks where ESNs are known to perform at least as good as the state of the art \citep{lukovsevivcius2009reservoir}. Additionally, they are common in distributed scenarios. To simulate a large-scale analysis, we consider datasets that are approximately $1$--$2$ orders of magnitude larger than previous works. In particular, for every dataset we generate $50$ sequences of $2000$ elements each, starting from different initial conditions, summing up to $100.000$ samples for every experiment. This is roughly the limit at which a centralized solution is amenable for comparison. Below we provide a brief description of the four datasets.

The NARMA-$10$ dataset (denoted by N$10$) is a nonlinear system identification task, where the input $x[n]$ to the system is white noise in the interval $\left[0, 0.5\right]$, while the output $d[n]$ is computed from the recurrence equation \citep{jaeger2002adaptive}:
\begin{equation}
d[n] = 0.1 + 0.3d[n-1] + 0.05d[n-1]\prod_{i = 1}^{10}d[n-i] + 1.5x[n]x[n-9] \,.
\label{chap8:eq:narma-10}
\end{equation}
The output is then squashed to the interval $\left[-1, +1\right]$ by the nonlinear transformation:
\begin{equation}
d[n] = \tanh(d[n] - \hat{d}) \,,
\label{chap8:eq:shashing_output}
\end{equation}
\noindent where $\hat{d}$ is the empirical mean computed from the overall output vector.

The second dataset is the extended polynomial (denoted by EXTPOLY) introduced in \citep{butcher2013reservoir}. The input is given by white noise in the interval $\left[-1 \; +1\right]$, while the output is computed as:
\begin{equation}
d[n] = \sum_{i=0}^p\sum_{j=0}^{p-i} a_{ij}x^i[n]x^j[n-l] \,,
\label{chap8:eq:extpoly}
\end{equation}
\noindent where $p,l \in \R$ are user-defined parameters controlling the memory and nonlinearity of the polynomial, while the coefficients $a_{ij}$ are randomly assigned from the same distribution as the input data. In our experiments, we use a mild level of memory and nonlinearity by setting $p = l = 7$. The output is normalized using Eq. \eqref{chap8:eq:shashing_output}.

The third dataset is the prediction of the well-known Mackey-Glass chaotic time-series (denoted as MKG). This is defined in continuous time by the differential equation:
\begin{equation}
\dot{x}[n] = \beta x[n] + \frac{\alpha x[n-\tau]}{1 + x^\gamma[n-\tau]} \,.
\label{chap8:eq:mackeyglass}
\end{equation}
\noindent We use the common assignment $\alpha = 0.2$, $\beta = -0.1$, $\gamma = 10$, giving rise to a chaotic behavior for $\tau > 16.8$. In particular, in our experiments we set $\tau = 30$. Time-series \eqref{chap8:eq:mackeyglass} is integrated with a $4$-th order Runge-Kutta method using a time step of $0.1$, and then sampled every $10$ time-instants. The task is a $10$-step ahead prediction task, i.e.:
\begin{equation}
d[n] = x[n+10] \,.
\end{equation}

The fourth dataset is another chaotic time-series prediction task, this time on the Lorenz attractor. This is a $3$-dimensional time-series, defined in continuous time by the following set of differential equations:
\begin{equation}
\begin{cases}
\dot{x}_1[n] & = \sigma \left( x_2[n] - x_1[n] \right) \\
\dot{x}_2[n] & = x_1[n]\left(\eta - x_3[n]\right) - x_2[n] \\
\dot{x}_3[n] & = x_1[n]x_2[n] - \zeta x_3[n]
\end{cases} \,,
\label{chap8:eq:lorenz}
\end{equation}
\noindent where the standard choice for chaotic behavior is $\sigma = 10$, $\eta = 28$ and $\zeta = 8/3$. The model in Eq. \eqref{chap8:eq:lorenz} is integrated using an ODE45 solver, and sampled every second. For this task, the input to the system is given by the vector $\bigl[x_1[n] \; x_2[n] \; x_3[n]\bigr]$, while the required output is a $1$-step ahead prediction of the $x_1$ component, i.e.:
\begin{equation}
d[n] = x_1[n+1] \,.
\end{equation}
For all four datasets, we supplement the original input with an additional constant unitary input, as is standard practice in ESNs' implementations \citep{lukovsevivcius2009reservoir}.

\subsection{Description of the Algorithms}
\label{sec:desc_algos}
In our simulations we generate a network of agents, using a random topology model for the connectivity matrix, where each pair of nodes can be connected with $25\%$ probability. The only global requirement is that the overall network is connected. We experiment with a number of nodes going from $5$ to $25$, by steps of $5$. To estimate the testing error, we perform a $3$-fold cross-validation on the $50$ original sequences. For every fold, the training sequences are evenly distributed across the nodes, and the following three algorithms are compared:
\begin{description}
	\item[Centralized ESN] (C-ESN): This simulates the case where training data is collected on a centralized location, and the net is trained by directly solving problem \eqref{chap8:eq:esn_opt}.
	\item[Local ESN] (L-ESN): In this case, each node trains a local ESN starting from its data, but no communication is performed. The testing error is then averaged throughout the $L$ nodes.
	\item[ADDM-based ESN] (ADMM-ESN): This is an ESN trained with the distributed protocol introduced in the previous section. We set $\rho = 0.01$, a maximum number of $400$ iterations, and $\epsilon_{\text{abs}} = \epsilon_{\text{rel}} = 10^{-4}$.
\end{description}
\noindent All algorithms share the same ESN architecture, which is detailed in the following section. The $3$-fold cross-validation procedure is repeated $15$ times by varying the ESN initialization and the data partitioning, and the errors for every iteration and every fold are collected. To compute the error, we run the trained ESN on the test sequences, and gather the predicted outputs $\tilde{y}_1, \ldots, \tilde{y}_K$, where $K$ is the number of testing samples after removing the wash-out elements from the test sequences. Then, we compute the Normalized Root Mean-Squared Error (NRMSE), defined as:
\begin{equation}
\text{NRMSE} = \sqrt{\frac{\sum_{i = 1}^K \left[\tilde{y}_i - d_i\right]^2}{|K| \hat{\sigma}_d}} \,,
\label{chap8:eq:nrmse}
\end{equation}
\noindent where $\hat{\sigma}_d$ is an empirical estimate of the variance of the true output samples $d_1, \dots, d_K$.
\subsection{ESN Architecture}
As stated previously, all algorithms share the same ESN architecture. In this section we provide a brief overview on the selection of its parameters. First, we choose a default reservoir's size of $N_r = 300$, which was found to work well in all situations. Secondly, since the datasets are artificial and noiseless, we set a small regularization factor $\lambda = 10^{-3}$. Four other parameters are instead selected based on a grid search procedure. The validation error for the grid-search procedure is computed by performing a $3$-fold cross-validation over $9$ sequences, which are generated independently from the training and testing set. Each validation sequence has length $2000$. In particular, we select the following parameters:
\begin{itemize}
	\item The matrix $\vect{W}_i^r$, connecting the input to the reservoir, is initialized as a full matrix, with entries assigned from the uniform distribution $\left[-\alpha_{\text{i}} \; \alpha_{\text{i}}\right]$. The optimal parameter $\alpha_{\text{i}}$ is searched in the set $\left\{0.1, 0.3, \dots 0.9 \right\}$.
	\item Similarly, the matrix $\vect{W}_o^r$, connecting the output to the reservoir, is initialized as a full matrix, with entries assigned from the uniform distribution $\left[-\alpha_{\text{f}} \; \alpha_{\text{f}}\right]$. The parameter $\alpha_{\text{f}}$ is searched in the set $\left\{0, 0.1, 0.3, \dots 0.9 \right\}$. We allow $\alpha_{\text{f}} = 0$ for the case where no output feedback is needed.
	\item The internal reservoir matrix $\vect{W}_r^r$ is initialized from the uniform distribution $\left[-1 \; +1\right]$. Then, on average $75\%$ of its connections are set to $0$, to encourage sparseness. Finally, the matrix is rescaled so as to have a desired spectral radius $\rho$, which is searched in the same interval as $\alpha_{\text{i}}$.
	\item We use $\tanh(\cdot)$ nonlinearities in the reservoir, while a scaled identity $f(s) = \alpha_{\text{t}}s$ as the output function. The parameter $\alpha_{\text{t}}$ is searched in the same interval as $\alpha_{\text{i}}$.  
\end{itemize}
\noindent Additionally, we insert uniform noise in the state update of the reservoir, sampled uniformly in the interval $\left[0, 10^{-3}\right]$, and we discard $D = 100$ initial elements from each sequence.
\section{Experimental Results}
The final settings resulting from the grid-search procedure are listed in Table \ref{tab:gridsearch}. 
\begin{center}
	\begin{table}[h]
	\caption[Optimal parameters found by the grid-search procedure for testing ADMM-ESN.]{Optimal parameters found by the grid-search procedure. For a description of the parameters, see Section \ref{sec:desc_algos}.}
		{\centering\hfill{}
			\setlength{\tabcolsep}{4pt}
			\renewcommand{\arraystretch}{1.3}
			\begin{footnotesize}
				\begin{tabular}{lcccccc}   
					\toprule
					\textbf{Dataset} & $\rho$ & $\alpha_\text{i}$ & $\alpha_\text{t}$  & $\alpha_\text{f}$ & $N_r$ & $\lambda$ \\
					\midrule
					N10 & $0.9$ & $0.5$ & $0.1$ & $0.3$ & \multirow{4}{*}{$300$} &  \multirow{4}{*}{$2^{-3}$} \\
					EXTPOLY & $0.7$ & $0.5$ & $0.1$ & $0$ & & \\
					MKG & $0.9$ & $0.3$ & $0.5$ & $0$ & & \\
					LORENZ & $0.1$ & $0.9$ & $0.1$ & $0$ & & \\
					\bottomrule
				\end{tabular}
			\end{footnotesize}
		}
		\hfill{}\vspace{0.3em}
		\label{tab:gridsearch}
	\end{table}
\end{center}
\vspace{-2.5em} \noindent It can be seen that, except for the LORENZ dataset, there is a tendency towards selecting large values of $\rho$. Output feedback is needed only for the N$10$ dataset, while it is found unnecessary in the other three datasets. The optimal input scaling $\alpha_{\text{f}}$ is ranging in the interval $\left[0.5, 0.9\right]$, while the optimal teacher scaling $\alpha_{\text{t}}$ is small in the majority of cases.

The average NRMSE and training times for C-ESN are provided in Table \ref{tab:finalerrorandtime} as a reference. 
\begin{center}
	\begin{table}[h]
	\caption{Final misclassification error and training time for C-ESN, provided as a reference, together with one standard deviation.}
		{\centering\hfill{}
			\setlength{\tabcolsep}{4pt}
			\renewcommand{\arraystretch}{1.3}
			\begin{footnotesize}
				\begin{tabular}{lcc}
					\toprule
					\textbf{Dataset} & \textbf{NRMSE} & \textbf{Time} [secs]  \\
					\midrule
					N10 & $0.08 \pm 0.01$ & $9.26 \pm 0.20$ \\
					EXTPOLY & $0.39 \pm 0.01$ & $8.96 \pm 0.19$ \\
					MKG & $0.09 \pm 0.01$ & $9.26 \pm 0.20$ \\
					LORENZ & $0.67 \pm 0.01$ & $9.47 \pm 0.14$ \\
					\bottomrule
				\end{tabular}
			\end{footnotesize}
		}
		\hfill{}\vspace{0.6em}
		\label{tab:finalerrorandtime}
	\end{table}
\end{center}
\vspace{-2.5em} \noindent Clearly, these values do not depend on the size of the network, and they can be used as an upper baseline for the results of the distributed algorithms. Since we are considering the same amount of training data for each dataset, and the same reservoir's size, the training times in Table \ref{tab:finalerrorandtime} are roughly similar, except for the LORENZ dataset, which has $4$ inputs compared to the other three datasets (considering also the unitary input). As we stated earlier, performance of C-ESN are competitive with the state-of-the-art for all the four datasets. Moreover, we can see that it is extremely efficient to train, taking approximately $9$ seconds in all cases.

To study the behavior of the decentralized procedures when training data is distributed, we plot the average error for the three algorithms, when varying the number of nodes in the network, in Fig. \ref{chap8:fig:error} \subref{chap8:fig:narma10_error}-\subref{chap8:fig:lorenz_error}. The average NRMSE of C-ESN is shown as dashed black line, while the errors of L-ESN and ADMM-ESN are shown with blue squares and red circles respectively.

\begin{figure*}[h]
	\centering
	\subfloat[Dataset N10]{\includegraphics[scale=0.8]{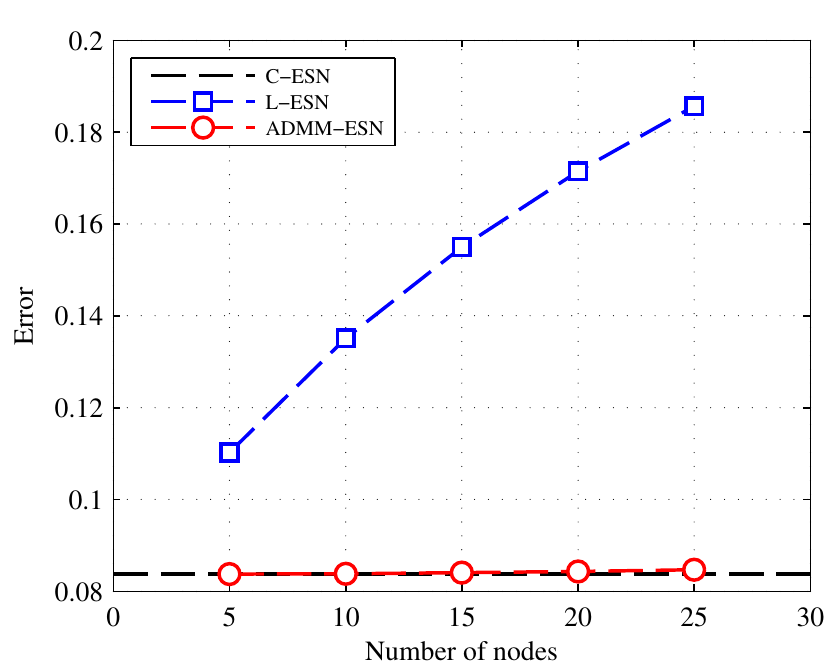}%
		\label{chap8:fig:narma10_error}}
	\hfil
	\subfloat[Dataset EXTPOLY]{\includegraphics[scale=0.8]{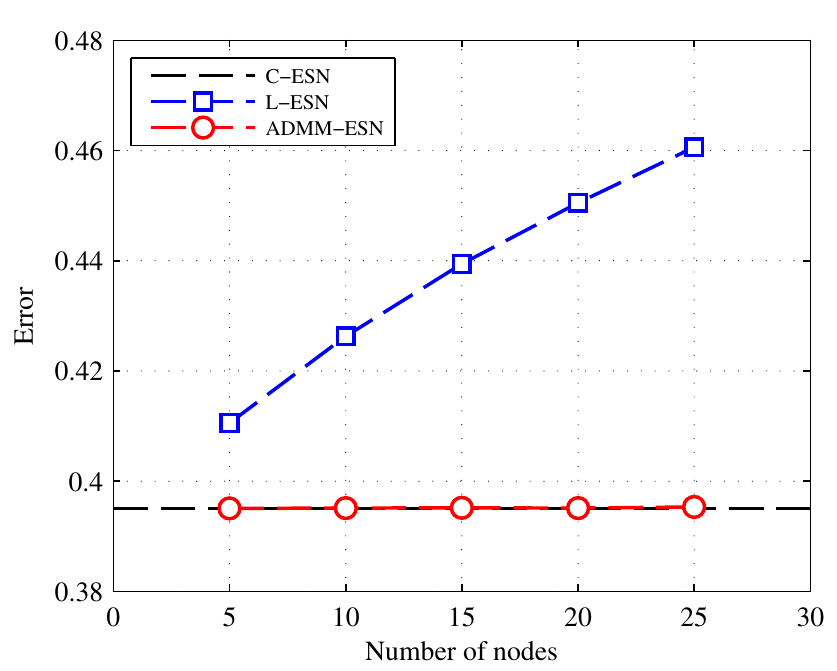}%
		\label{chap8:fig:extpoly_error}}
	\vfill
	\subfloat[Dataset MKG]{\includegraphics[scale=0.8]{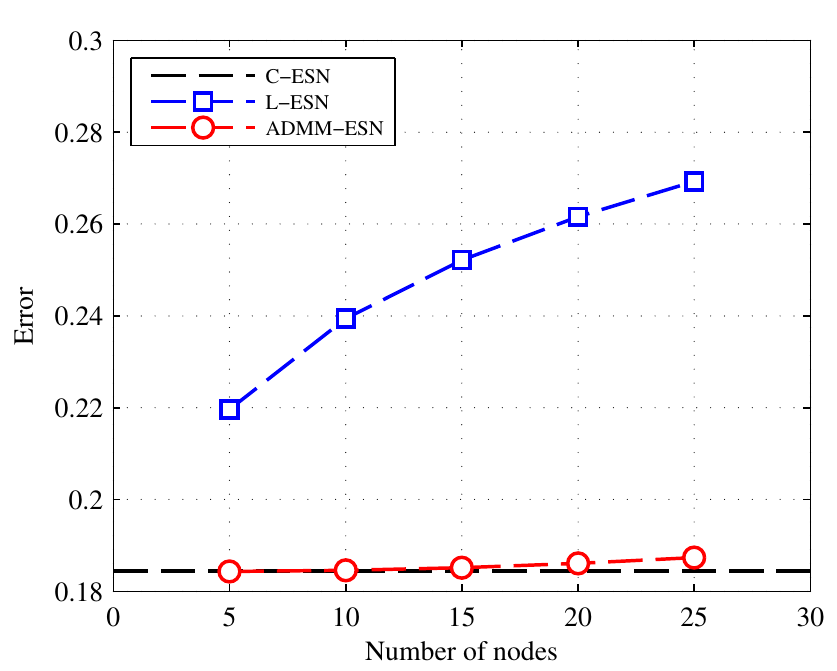}%
		\label{chap8:fig:mkg_error}}
	\hfil
	\subfloat[Dataset LORENZ]{\includegraphics[scale=0.8]{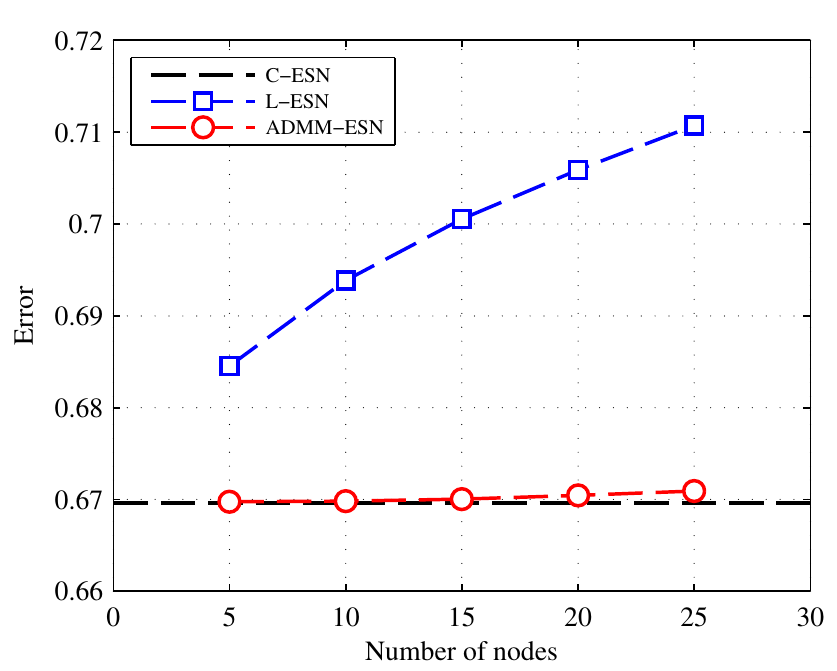}%
		\label{chap8:fig:lorenz_error}}
	\caption[Evolution of the testing error for ADMM-ESN, for networks going from $5$ agents to $25$ agents.]{Evolution of the testing error, for networks going from $5$ agents to $25$ agents. Performance of L-ESN is averaged across the nodes.}
	\label{chap8:fig:error}
\end{figure*}

Clearly, L-ESN is performing worse than C-ESN, due to its partial view on the training data. For small networks of $5$ nodes, this gap may not be particularly pronounced. This goes from a $3\%$ worse performance on the LORENZ dataset, up to a $37\%$ decrease in performance for the N10 dataset (going from an NRMSE of $0.8$ to an NRMSE of $0.11$). The gap is instead substantial for large networks of up to $25$ nodes. For example, the error of L-ESN is more than twice that of C-ESN for the N10 dataset, and its performance is $50\%$ worse in the MKG dataset. Albeit these results are expected, they are evidence of the need of a decentralized training protocol for ESNs, able to take into account all the local datasets.

As is clear from Fig. \ref{chap8:fig:error}, ADMM-ESN is able to perfectly track the performance of the centralized solution in all situations. A small gap in performance is present for the two predictions tasks when considering large networks. In particular, the performance of ADMM-ESN is roughly $1\%$ worse than C-ESN for networks of $25$ nodes in the datasets MKG and LORENZ. In theory, this gap can be reduced by considering additional iterations for the ADMM procedure, although this would be impractical in real world applications.

Training time requested by the three algorithms is shown in Fig. \ref{chap8:fig:time} \subref{chap8:fig:narma10_time}-\subref{chap8:fig:lorenz_time}. The training time for L-ESN and ADMM-ESN is averaged throughout the agents. 

\begin{figure*}[!h]
	\centering
	\subfloat[Dataset N10]{\includegraphics[scale=0.8]{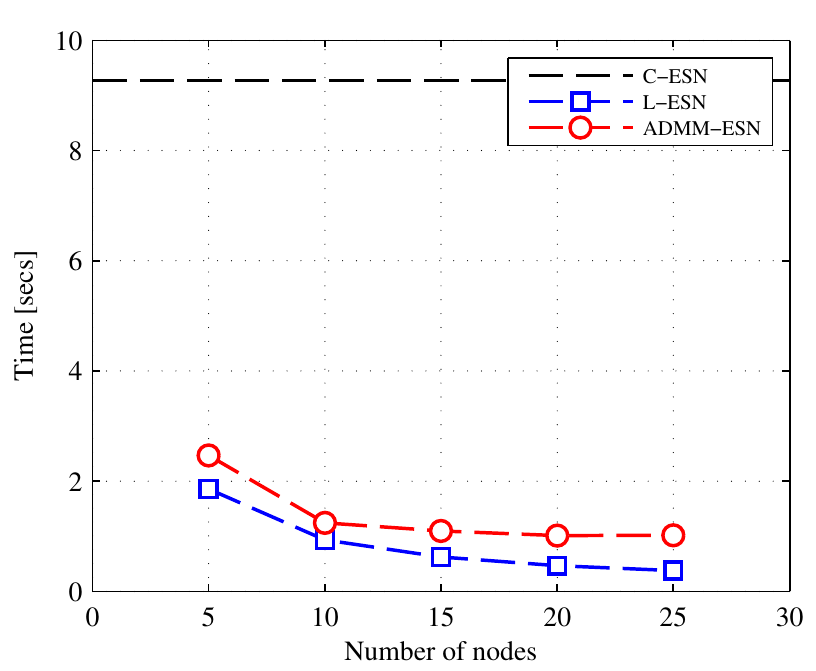}%
		\label{chap8:fig:narma10_time}}
	\hfil
	\subfloat[Dataset EXTPOLY]{\includegraphics[scale=0.8]{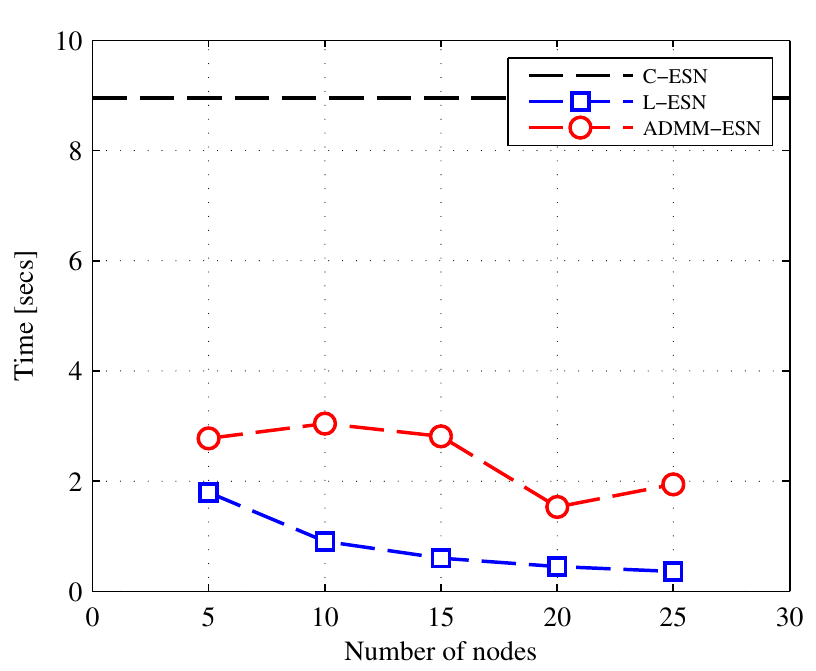}%
		\label{chap8:fig:extpoly_time}}
	\vfill
	\subfloat[Dataset MKG]{\includegraphics[scale=0.8]{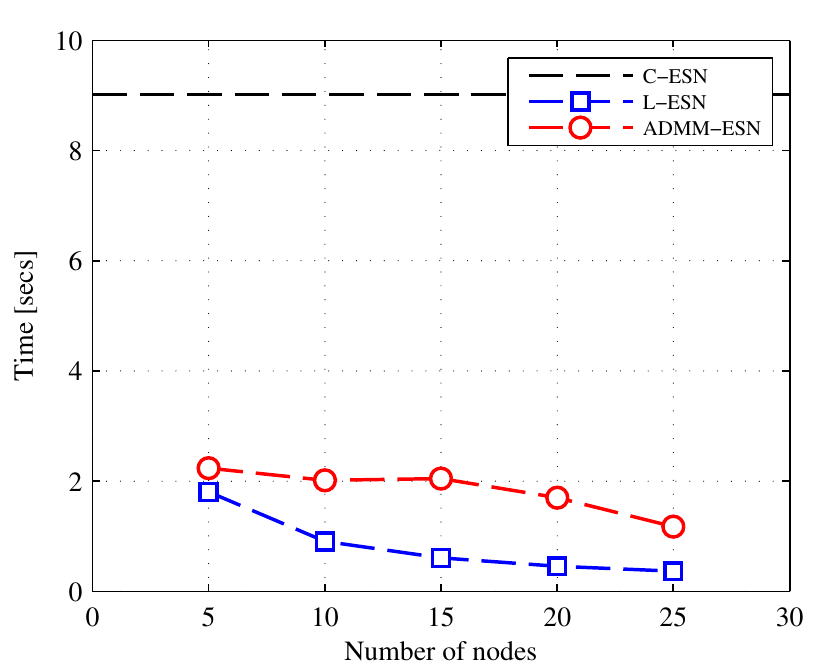}%
		\label{chap8:fig:mkg_time}}
	\hfil
	\subfloat[Dataset LORENZ]{\includegraphics[scale=0.8]{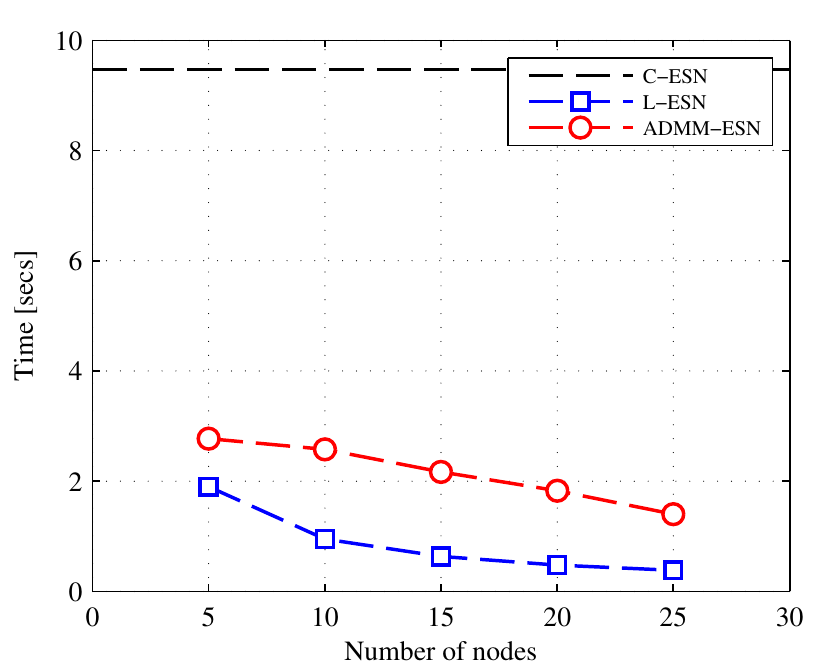}%
		\label{chap8:fig:lorenz_time}}
	\caption[Evolution of the training time for ADMM-ESN, for networks going from $5$ agents to $25$ agents.]{Evolution of the training time, for networks going from $5$ agents to $25$ agents. Time of L-ESN is averaged across the nodes.}
	\label{chap8:fig:time}
\end{figure*}

Since the computational time of training an ESN is mostly related to the matrix inversion in Eq. \eqref{chap8:eq:ridge_regression_minimizer}, training time is monotonically decreasing in L-ESN with respect to the number of nodes in the network (the higher the number of agents, the lower the amount of data at every local node). Fig. \ref{chap8:fig:time} shows that the computational overhead requested by the ADMM procedure is limited. In the best case, the N$10$ dataset with $10$ nodes, it required only $0.3$ seconds more than L-ESN, as shown from Fig. \ref{chap8:fig:time}\subref{chap8:fig:narma10_time}. In the worst setting, the EXTPOLY dataset with $15$ nodes, it required $2.2$ seconds more, as shown from Fig. \ref{chap8:fig:time}\subref{chap8:fig:extpoly_time}. In all settings, the time requested by ADMM-ESN is significantly lower compared to the training time of its centralized counterpart, showing it usefulness in large-scale applications.

\section{Extension to ESNs with Sparse Readouts}
\label{sec:extension_sparse_esns}

Up to now, the chapter has focused on training an ESN with a ridge regression routine. Still, it is known that standard ridge regression may not be the most suitable training algorithm for ESNs. Specifically, a large number of authors have been concerned with training a ESN with a sparse readout, i.e. a readout where the majority of the connections are set to zero. In the centralized case, this has been initially explored in depth in \cite{dutoit2009pruning}. The authors investigated different greedy methods to this end, including backward selection (where connections are removed one at a time based on an iterative procedure), random deletion, and others. Significant improvements are found, both in terms of generalization accuracy, and in terms of computational requirements. Moreover, having only a small amount of connections can lead to extremely efficient implementations \cite{scardapane2014effective}, particularly on low-cost devices. Thus, having the possibility of training sparse readouts for an ESN in a decentralized case can be a valuable tool. 

Since the readout is linear, sparsity can be enforced by including an additional $L_1$ regularization term to be minimized, resulting in the LASSO algorithm. For ESNs, this is derived for the first time in Ceperic and Baric \cite{7046033}. In the distributed case under consideration, the ADMM can be used for solving the LASSO problem quite efficiently, with only a minor modification with respect to the ADMM-ESN \cite{boyd2011distributed}. In particular, it is enough to replace the update for $\vect{z}[n+1]$ with:
\begin{equation}
\vect{z}[n+1] = S_{\lambda/N\rho} \left( \hat{\vect{w}}[n+1] + \hat{\vect{t}}[n] \right) \,,
\end{equation}
where the soft-thresholding operator $S_\alpha(\cdot)$ is defined for a generic vector $\vect{a}$ as:
\[
S_\alpha(\vect{a}) = \left( 1 - \frac{\alpha}{\norm{\vect{a}}} \right)_+ \vect{a} \,,
\]
and $(\cdot)_+$ is defined element-wise as $(\cdot)_+ = \max\left( 0, \cdot \right)$. In order to test the resulting sparse algorithm, we consider the MKG and N10 datasets with the same setup as before, but a lower number of elements (in total $2500$ for training and $2000$ for testing). Additionally, in order to have a slightly redundant reservoir, we select $N_r = 500$.

\subsection{Comparisons in the centralized case}

We begin our experimental evaluation by comparing the standard ESN and the ESN trained using the LASSO algorithm (denoted as L1-ESN) in the centralized case. This allows us to better investigate their behavior, and to choose an optimal regularization parameter $\lambda$. Particularly, we analyze test error, training time, and sparsity of the resulting L1-ESN when varying $\lambda$ in $10^{-j}$, with $j$ going from $-1$ to $-6$. The LASSO problems are solved using a freely available implementation of the iterated ridge regression algorithm by M. Schmidt \cite{schmidt2005least}.\footnote{\url{http://www.cs.ubc.ca/~schmidtm/Software/lasso.html}} The algorithm works by  approximating the $L_1$ term with $\norm[1]{w_i} \approx \frac{w_i^2}{\norm[1]{w_i}}$, and iteratively solving the resulting ridge regression problem. Results are presented in Fig. \ref{chap9:fig:centralized_results}, where results for MG and N10 are shown in the left and right columns, respectively.

\begin{figure*}
	\centering
	\subfloat[Test error (MG)]{\includegraphics{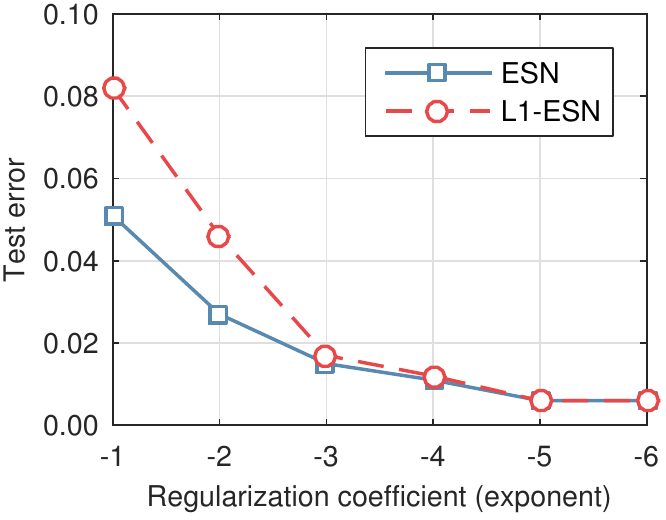}%
		\label{chap9:fig:errors_centralized_mg}} %
	\hfil
	\subfloat[Test error (N10)]{\includegraphics{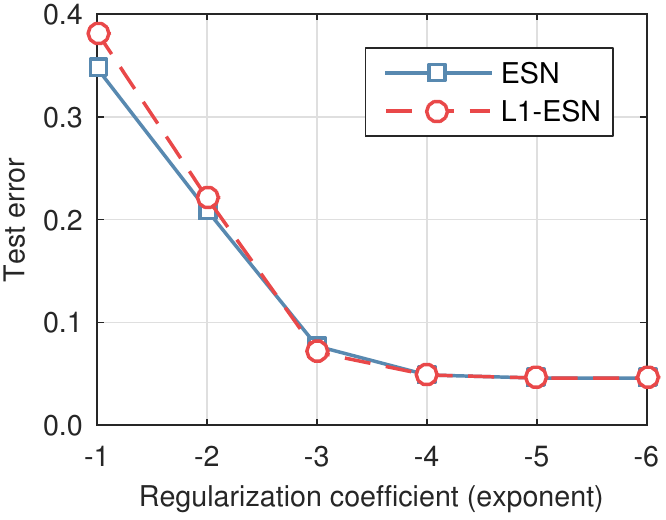}%
		\label{chap9:fig:errors_centralized_narma10}} %
	\vfil
	\subfloat[Tr. time (MG)]{\includegraphics{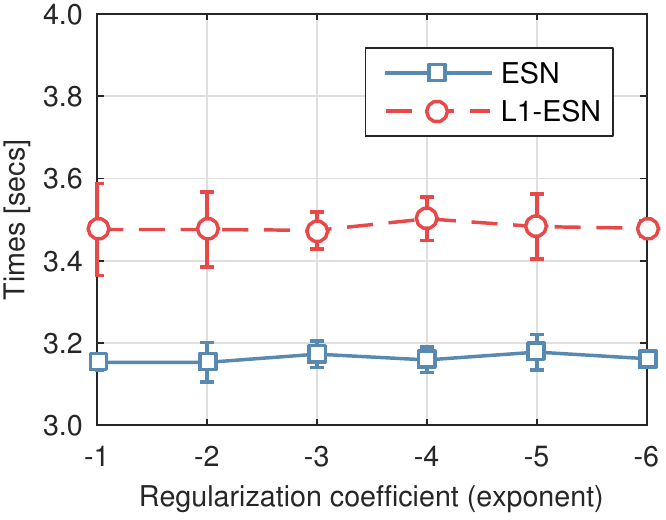}%
		\label{chap9:fig:times_centralized_mg}} %
	\hfil
	\subfloat[Tr. time (N10)]{\includegraphics{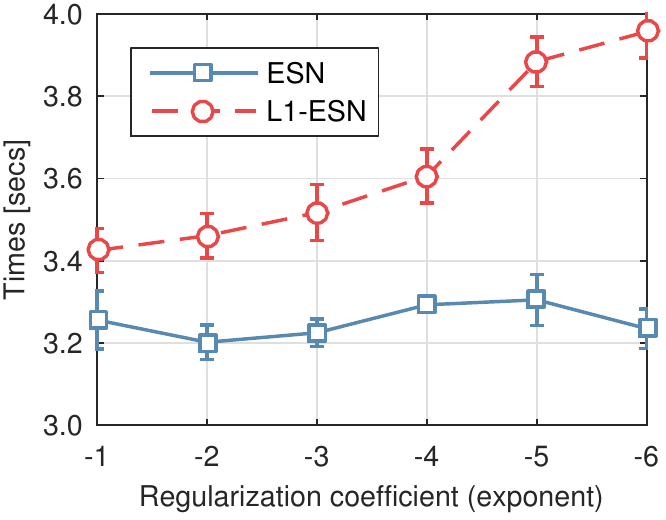}%
			\label{chap9:fig:times_centralized_narma10}} %
	\vfil
	\subfloat[Sparsity (MG)]{\includegraphics{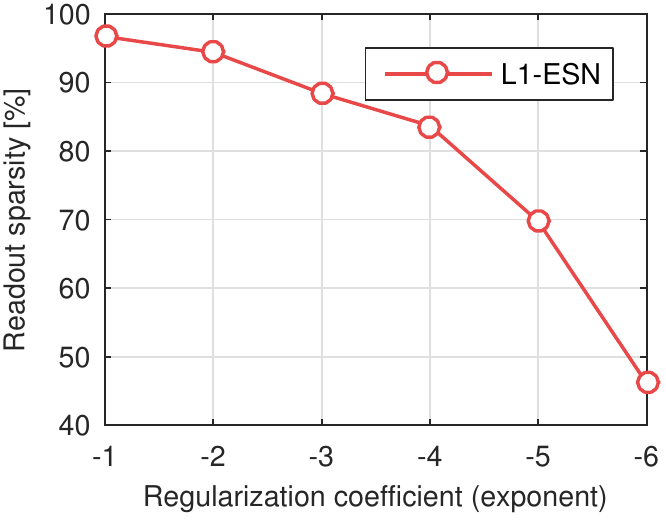}%
		\label{chap9:fig:sparsities_centralized_mg}} %
	\hfil
	 \subfloat[Sparsity (N10)]{\includegraphics{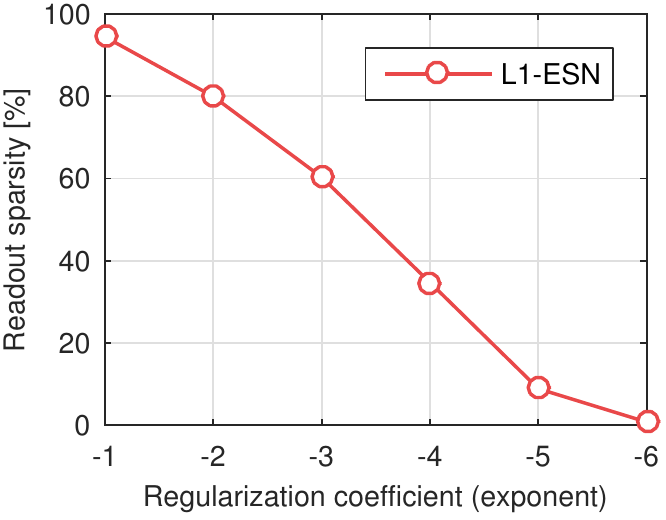}%
	 		\label{chap9:fig:sparsities_centralized_narma10}} %
	\caption[Evolution of test error, training time and sparsity when testing L1-ESN.]{Evolution of (a-b) test error, (c-d) training time and (e-f) sparsity of the output vector when varying the regularization coefficient in $10^j$, $j=-1, \ldots, -6$. Results for the MG dataset are shown on the left column, while results for the N10 dataset are shown in the right column.}
	\label{chap9:fig:centralized_results}
\end{figure*}

First of all, we can see clearly from Fig.s \ref{chap9:fig:errors_centralized_mg} and \ref{chap9:fig:errors_centralized_narma10} that the regularization effect of the two algorithms is similar, a result in line with previous works \cite{dutoit2009pruning}. Particularly, for large regularization factors, the estimates tend to provide an unsatisfactory test error, which however is relatively stable for sufficiently small coefficients. The tendency to select such a small factor is to be expected, due to the artificial nature of the datasets. A minimum in test error is reached for $j$ around $-5$ for MG, and $j$ around $-4$ for N10.

With respect to the training time, ridge regression is relatively stable to the amount of regularization, as the matrix to be inverted tends to be already well conditioned. Training time of LASSO is regular for MG, while it slightly increases for larger values of $j$ in the N10 case, as shown in Fig. \ref{chap9:fig:times_centralized_narma10}. In all cases, however, it is comparable to that of ridge regression, with a small increase of $0.5$ seconds in average.

The most important aspect, however, is evidenced in Fig.s \ref{chap9:fig:sparsities_centralized_mg} and \ref{chap9:fig:sparsities_centralized_narma10}. Clearly, sparsity of the readout goes from almost $100\%$ to $0\%$ as the regularization factor decreases. At the point of best test accuracy, the resulting readout has an average sparsity of $70\%$ for MG and $38\%$ for N10. This, combined with the simultaneous possibility of pruning the resulting reservoir \cite{dutoit2009pruning,scardapane2014effective}, can lead to an extreme saving of computational resources requested at the single sensor during the prediction phase. In order to provide a simpler comparison of the results, we also display them in tabular form in Table \ref{tab:centralized_results}.

\begin{table*}[t]
\small
\centering
\caption{The results of Fig. \ref{chap9:fig:centralized_results}, shown in tabular form, together with one standard deviation.}
\begin{tabular}{lcllll}
\toprule
Dataset & $\lambda$ & Algorithm & Test error (NRMSE) & Tr. time [secs] & Sparsity [\%]\\
\midrule
\multirow{12}{*}{MG} & \multirow{2}{*}{$10^{-1}$} & ESN & $0.051 \pm 0.010$ & $3.153 \pm 0.019$ & $0$ \\
							   &							& L1-ESN & $0.082 \pm 0.001$ & $3.476 \pm 0.112$ & $0.967 \pm 0.01$ \\ \cmidrule{2-6}
							   & \multirow{2}{*}{$10^{-2}$} & ESN & $0.027 \pm 0.006$ & $3.153 \pm 0.048$ & $0$ \\
							   &							& L1-ESN & $0.046 \pm 0.001$ & $3.476 \pm 0.091$ & $0.944 \pm  0.06$ \\ \cmidrule{2-6}
							   & \multirow{2}{*}{$10^{-3}$} & ESN & $ 0.015 \pm 0.003$ & $3.173 \pm 0.032$ & $0$ \\
							   &							& L1-ESN & $0.017 \pm 0.001$ & $3.474 \pm 0.045$ & $0.884 \pm 0.01$ \\ \cmidrule{2-6}
							   & \multirow{2}{*}{$10^{-4}$} & ESN & $0.011 \pm 0.001$ & $3.159 \pm 0.031$ & $0$ \\
							   &							& L1-ESN & $0.012 \pm 0.001$ & $3.502 \pm 0.053$ & $0.837 \pm 0.04$ \\ \cmidrule{2-6}
							   & \multirow{2}{*}{$10^{-5}$} & ESN & $0.006 \pm 0.001$ & $3.178 \pm 0.043$ & $0$ \\
							   &							& L1-ESN & $0.006 \pm 0.001$ & $3.483 \pm 0.079$ & $0.697 \pm 0.09$ \\ \cmidrule{2-6}
							   & \multirow{2}{*}{$10^{-6}$} & ESN & $0.006 \pm 0.001$ & $3.162 \pm 0.011$ & $0$ \\
							   &							& L1-ESN & $0.006 \pm 0.001$ & $3.976 \pm 0.019$ & $0.461 \pm 0.03$ \\
\midrule
\multirow{12}{*}{N10} & \multirow{2}{*}{$10^{-1}$} & ESN & $0.347 \pm 0.006$ & $3.256 \pm 0.071$ & $0$ \\
							 &							  & L1-ESN & $0.382 \pm 0.008$ & $3.425 \pm 0.053$ & $0.944 \pm 0.01$ \\ \cmidrule{2-6}
							 & \multirow{2}{*}{$10^{-2}$} & ESN & $0.209 \pm 0.004$ & $3.202 \pm 0.042$ & $0$ \\
							 &							  & L1-ESN & $0.221 \pm 0.007$ & $3.461 \pm 0.054$ & $0.799 \pm 0.04$ \\ \cmidrule{2-6}
							 & \multirow{2}{*}{$10^{-3}$} & ESN & $0.077 \pm 0.001$ & $3.225 \pm 0.034$ & $0$ \\
							 &							  & L1-ESN & $0.071 \pm 0.001$ & $3.517 \pm 0.068$ & $0.603 \pm 0.03$ \\ \cmidrule{2-6}
							 & \multirow{2}{*}{$10^{-4}$} & ESN & $0.049 \pm 0.001$ & $3.293 \pm 0.015$ & $0$ \\
							 &							  & L1-ESN & $0.049 \pm 0.001$ & $3.606 \pm 0.066$ & $0.347 \pm 0.02$ \\ \cmidrule{2-6}
							 & \multirow{2}{*}{$10^{-5}$} & ESN & $0.046 \pm 0.001$ & $3.305 \pm 0.062$ & $0$ \\
							 &							  & L1-ESN & $0.046 \pm 0.001$ & $3.884 \pm 0.060$ & $0.089 \pm 0.01$ \\ \cmidrule{2-6}
							 & \multirow{2}{*}{$10^{-6}$} & ESN & $0.046 \pm 0.001$ & $3.235 \pm 0.048$ & $0$ \\
							 &							  & L1-ESN & $0.046 \pm 0.001$ & $3.957 \pm 0.064$ & $0.008 \pm 0.01$ \\
							 
\bottomrule
\end{tabular}
\vspace{0.5em}
\label{tab:centralized_results}
\end{table*}

\subsection{Comparisons in the distributed case}

We now consider the implementation of the distributed L1-ESN over a network of agents. More in detail, training observations are uniformly subdivided among the $L$ agents in a predefined network, with $L$ varying from $5$ to $30$ by steps of $5$. For every run, the connectivity among the agents is generated randomly, such that each pair of agents has a $25\%$ probability of being connected, with the only global requirement that the overall network is connected. The following three algorithms are compared:
\begin{enumerate}
	\item \textbf{Centralized ESN} (C-L1-ESN): this simulates the case where training data is collected on a centralized location, and the net is trained by directly solving the LRR problem. This is equivalent to the ESN analyzed in the previous section, and following the results obtained there, we set $\lambda=10^{-5}$ for MG, and $\lambda=10^{-4}$ for N10.
	\item \textbf{Local ESN} (L-L1-ESN): in this case, each agent trains an L1-ESN starting from its local measurements, but no communication is performed. The testing error is averaged throughout the $L$ agents.
	\item \textbf{ADMM-based ESN} (ADMM-L1-ESN): this is trained with the algorithm introduced previously. We select $\gamma = 0.01$ and a maximum number of $400$ iterations. For the DAC protocol, we set a maximum number of $300$ iterations. DAC also stops whenever the updates (in norm) at every agent are smaller than a predefined threshold $\delta = 10^{-8}$:
	\begin{equation}
	\Bigl\lVert\boldsymbol{q}_k[n+1;j] - \boldsymbol{q}_k[n+1;j-1]\Bigr\rVert_2^2 < \delta, \; k \in \left\{ 1, 2,\dots, L \right\} \,.
	\label{eq:consensus_stopping_criterion}
	\end{equation}
\end{enumerate}
Results of this set of experiments are presented in Fig. \ref{chap9:fig:distributed_results}. Similarly to before, results from the two datasets are presented in the left and right columns, respectively. From Fig.s \ref{chap9:fig:errors_distributed_mg} and \ref{chap9:fig:errors_distributed_narma10} we see that, although L-L1-ESN achieves degrading performance for bigger networks (due to the lower number of measurements per agent), ADMM-L1-ESN is able to effectively track the performance of the centralized counterpart, except for a small deviation in MG. Indeed, it is possible to reduce this gap by increasing the number of iterations; however, the performance gain is not balanced by the increase in computational cost. 

With respect to the training time, it is possible to see from Fig.s \ref{chap9:fig:times_distributed_mg} and \ref{chap9:fig:times_distributed_narma10} that the training time is relatively steady for larger networks in ADMM-L1-ESN, showing its feasibility in the context of large sensor networks. Moreover, the computational cost requested by the distributed procedure is low and, in the worst case, it requires no more than $1$ second with respect to the cost of a centralized counterpart. Overall, we can see that our distributed protocol allows for an efficient implementation in terms of performance and training time, while at the same time guaranteeing a good level of sparsity of the resulting readout. This, in turn, is essential for many practical implementations where computational savings are necessary.

\begin{figure*}
	\centering
	\subfloat[Test error (MG)]{\includegraphics{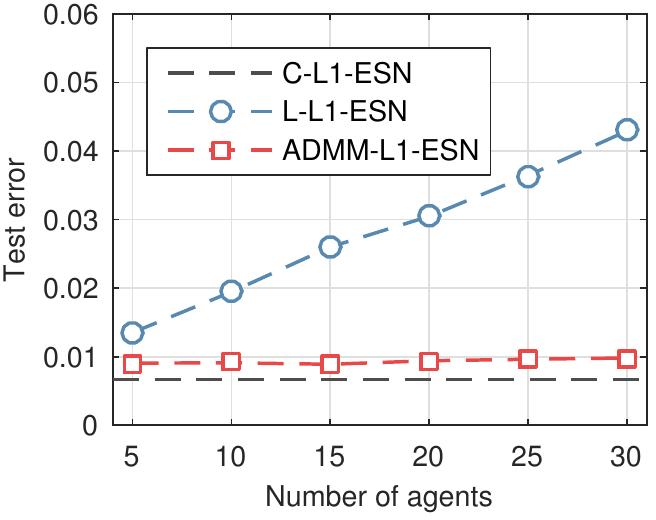}%
	\label{chap9:fig:errors_distributed_mg}}
	\hfil
	\subfloat[Test error (N10)]{\includegraphics{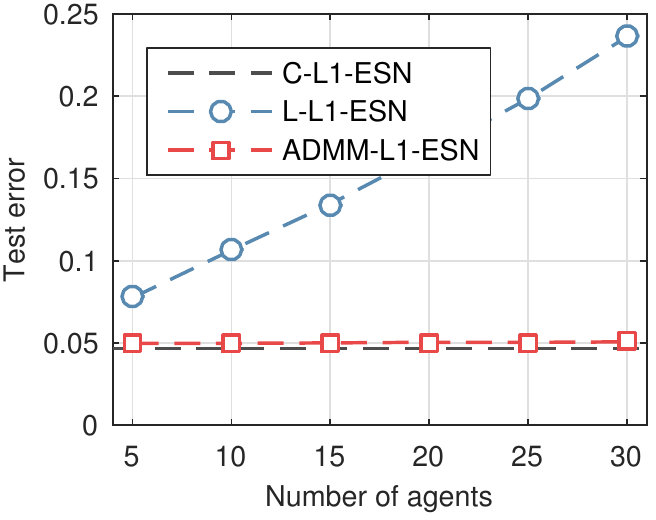}%
	\label{chap9:fig:errors_distributed_narma10}}	
	\vfil
	\subfloat[Training time (MG)]{\includegraphics{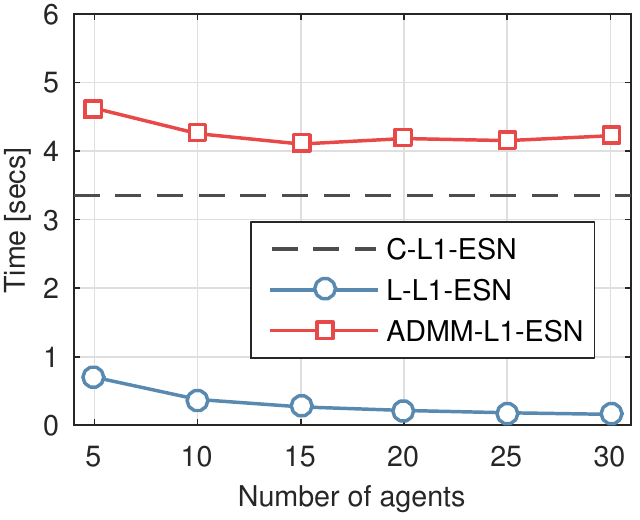}%
	\label{chap9:fig:times_distributed_mg}}
	\hfil
	\subfloat[Training time (N10)]{\includegraphics{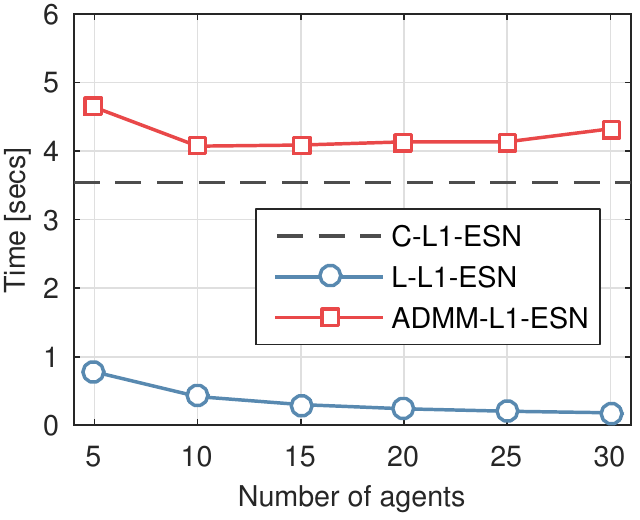}%
	\label{chap9:fig:times_distributed_narma10}}
	\vfil
	\caption[Evolution of test error, training time and sparsity when testing ADMM-L1-ESN.]{Evolution of (a-b) test error, (c-d) training time when varying the number of agents in the network from $5$ to $30$ by steps of $5$.}
	\label{chap9:fig:distributed_results}
\end{figure*}

Some additional insights into the convergence behavior of ADMM-L1-ESN can also be obtained by analyzing the evolution of the so-called (primal) residual, given by \cite{boyd2011distributed}:
\begin{equation}
r[n+1] = \frac{1}{L}\sum_{k=1}^L \Bigl\lVert \vect{w}[n+1] - \vect{z}[n+1] \Bigr\rVert_{2}  \,.
\label{eq:admm_residual}
\end{equation}
As can be seen from Fig. \ref{chap9:fig:admm_residual} (shown with a logarithmic $y$-axis), this rapidly converges towards $0$, ensuring that the algorithm is able to reach a stationary solution in a relatively small number of iterations.

\begin{figure*}
	\centering
	\subfloat[Residual (MG)]{\includegraphics{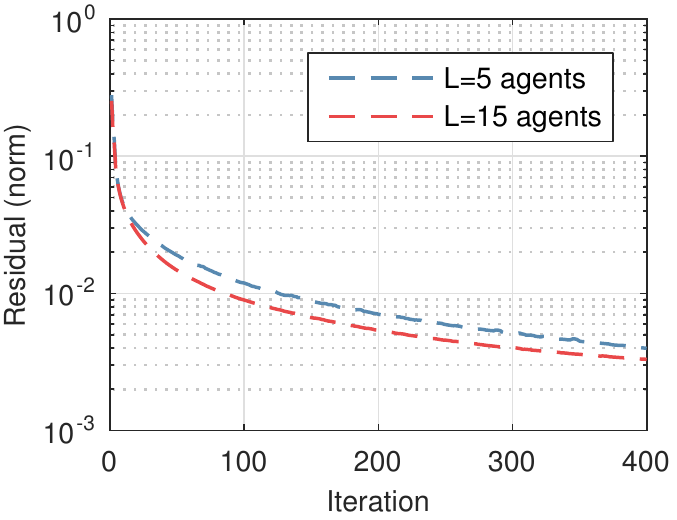}%
		\label{chap9:fig:admm_residual_mg}}
	\hfil
	\subfloat[Residual (N10)]{\includegraphics{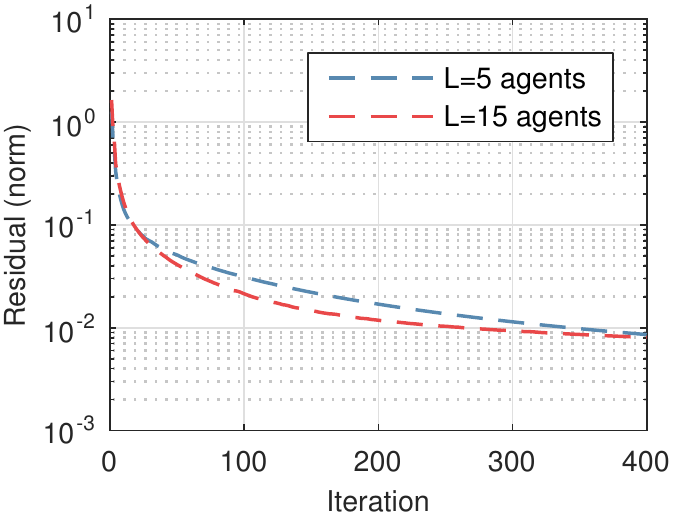}%
		\label{chap9:fig:admm_residual_narma10}}
	\vfill
	\caption{Evolution of the (primal) residual of ADMM-L1-ESN for $L=5$ and $L=15$.}
	\label{chap9:fig:admm_residual}
\end{figure*}

%% file: chapters/chapter10-dist_saf.tex
\chapter{Diffusion Spline Filtering}
\label{chap:dist_saf}

\minitoc
\vspace{15pt}

\blfootnote{A partial content of this chapter is currently under review for the 2016 European Signal Processing Conference (EUSIPCO).}

\section{Introduction}

\lettrine[lines=2]{T}{his} chapter continues the investigation on distributed training algorithms for time-varying data. Particularly, we focus on models with external memory (i.e., with a buffer of the last input elements), adequate to devices with extremely low computation resources. Available approaches in this sense include the linear diffusion filters (see Section \ref{chap4:sec:diffusion_filtering}), and kernel-based distributed filters (see Section \ref{chap4:sec:distributed_kernel}). However, the former applicability is limited to scenarios where the assumption of a linear model between the output and the observed variables is meaningful. Kernel methods, instead, are hindered by the fact that a kernel model depends by definition on the full observed dataset, as we analyzed extensively in Chapters \ref{chap:multi_agent_sl} and \ref{chap:dist_ssl}.

In this chapter, we propose a novel nonlinear distributed filtering algorithm based on the recently proposed spline adaptive filter (SAF) \cite{Scarpiniti2013}. Specifically, we focus on the Wiener SAF filter \cite{Scarpiniti2013}, where a linear filter is followed by an \textit{adaptive} nonlinear transformation, obtained with spline interpolation. They are attractive nonlinear filters for two main reasons. First, the nonlinear part is linear-in-the-parameters (LIP), allowing for the possibility of adapting both parts of the filter using standard linear filtering techniques. Secondly, while the spline can be defined by a potentially large number of parameters, only a small subset of them must be considered and adapted at each time step ($4$ in our experiments). Due to this, they allow to approximate non-trivial nonlinear functions with a small increase in computational complexity with respect to linear filters. 

Based on the general theory of DA,\footnote{Described in Chapter \ref{chap:multi_agent_sl}.} in this chapter we propose a diffused version of the SAF filter, denoted as D-SAF. In particular, we show that a cooperative behavior can be implemented by considering two subsequent diffusion operations, on the linear and non-linear components of the SAF respectively. Due to this, the D-SAF inherits the aforementioned characteristics of the centralized SAF, namely, it enables the agents to collectively converge to a non-linear function, with a small overhead with respect to a purely linear diffusion filter. In fact, D-LMS can be shown to be a special case of D-SAF, where adaptation is restricted to the linear part only. To demonstrate the merits of the proposed D-SAF, we perform an extensive set of experiments, considering medium and large-sized networks, coupled with mild and strong non-linearities. Simulations show that the D-SAF is able to efficiently learn the underlying model, and strongly outperform D-LMS and a purely non-cooperative SAF.

The rest of the chapter is organized as follows. Section \ref{sec:saf} introduces the basic framework of spline interpolation and SAFs. Section \ref{sec:diffusionsaf} formulates the D-SAF algorithm. Subsequently, we details our experimental setup and results in Section \ref{sec:exp_setup} and Section \ref{sec:exp_results} respectively.

\section{Spline Adaptive Filter}
\label{sec:saf}
Denote by $x[n]$ the input to the SAF filter at time $n$, and by $\vect{x}_n = \left[ x[n], \ldots, x[n - M + 1] \right]^T$ a buffer of the last $M$ samples. As in the previous chapters, we assume to be dealing with real inputs. Additionally, we assume that an unknown Wiener model is generating the desired response as follows:
\begin{equation}
d[n] = f_0\left( \vect{w}_0^T \vect{x}_n \right) + \nu[n] \,,
\label{eq:data_model}
\end{equation}
where $\vect{w}_0 \in \R^M$ are the linear coefficients, $f_0(\cdot)$ is a desired nonlinear function, which is supposed continuous and derivable, and $\nu[n] \sim \mathcal{N}(0, \sigma^2)$ is a Gaussian noise term. Similarly, a SAF computes the output in a two-step fashion. First, it performs a linear filtering operation given by:
\begin{equation}
s[n] = \vect{w}_n^T\vect{x}_n \,.
\label{eq:linearoutput}
\end{equation}
Then, the final output is computed via spline interpolation over $s[n]$. A spline is a flexible polynomial defined by a set of $Q$ control points (called \textit{knots}), and denoted as $\vect{Q}_i = \left[ q_{x,i} \; q_{y,i} \right]$. We suppose that the knots are uniformly distributed, i.e. $q_{x,i+1} = q_{x,i} + \Delta x$, for a fixed $\Delta x \in \R$. Without lack of generality, we also constrain the knots to be symmetrically spaced around the origin.  This pair of assumptions are at the base of the SAF family of algorithms, and dates back to earlier work on spline neurons for multilayer perceptrons \cite{guarnieri1999multilayer}. Practically, they allow for a simple derivation of the adaptation rule, while sacrificing only a small part of the flexibility of the spline interpolation framework. This is shown pictorially in Fig. \ref{fig:spline_interpolation}. 
\begin{figure}
\centering
	\includegraphics[width=0.7\columnwidth,keepaspectratio]{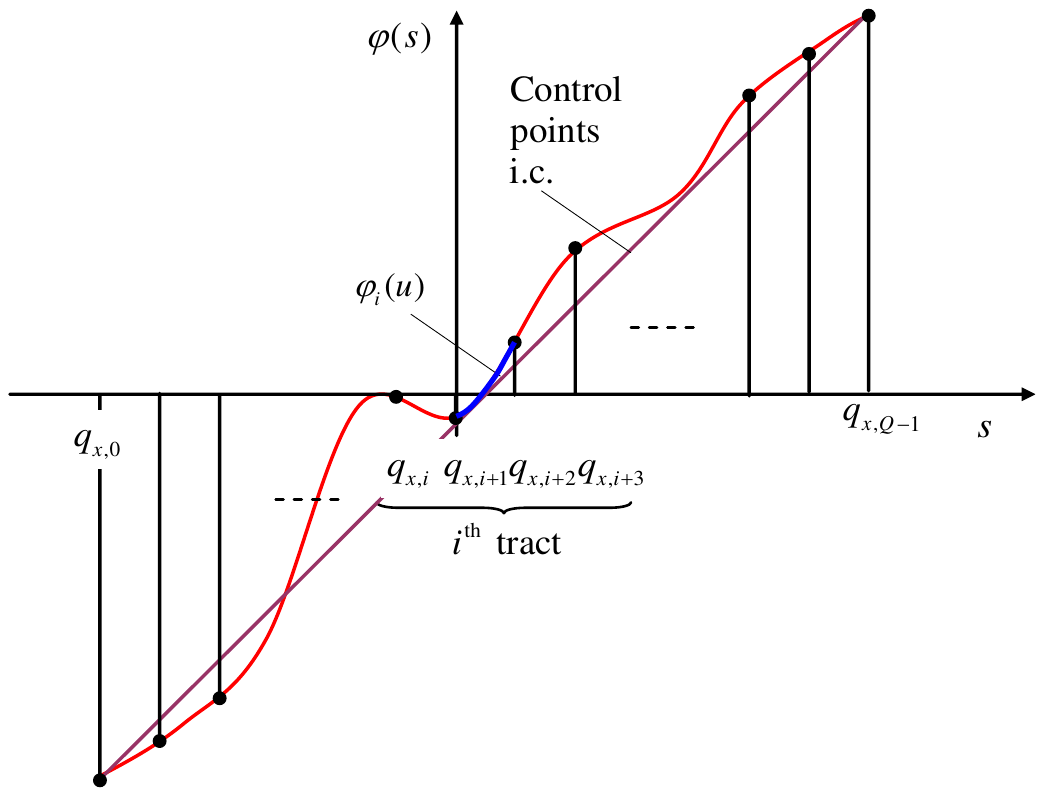}
\caption[Example of spline interpolation scheme.]{Example of spline interpolation scheme. We suppose that the control points are equispaced on the $x$-axis, and symmetrically spaced around the origin.}
\label{fig:spline_interpolation}
\end{figure}

Given the output of the linear filter $s[n]$, the spline is defined as an interpolating polynomial of order $P$, passing by the closest knot to $s[n]$ and its $P$ successive knots. In particular, due to our earlier assumptions, the index of the closest knot can be computed as:
\begin{equation}
i = \left\lfloor \frac{s[n]}{\Delta x} \right\rfloor + \frac{Q - 1}{2} \,.
\label{eq:i}
\end{equation}
Given this, we can define the normalized abscissa value between $q_{x,i}$ and $q_{x,i+1}$ as:
\begin{equation}
u = \frac{s[n]}{\Delta x} - \left\lfloor \frac{s[n]}{\Delta x} \right\rfloor \,.
\label{eq:u}
\end{equation}
From $u$ we can compute the normalized reference vector $\vect{u} = \left[ u^P \; u^{P-1} \ldots u \; 1 \right]^T$, while from $i$ we can extract the relevant control points $\vect{q}_{i,n} = \left[ q_{y,i} \; q_{y,i+1} \ldots q_{y,i+P} \right]^T$. We refer to the vector $\vect{q}_{i,n}$ as the $i$th \textit{span}. The output of the filter is then given by:
\begin{equation}
y[n] = \varphi(s[n]) = \vect{u}^T\vect{B}\vect{q}_{i,n} \,,
\label{eq:saf_output}
\end{equation}
where $\varphi(s[n])$is the adaptable nonlinearity as shown in Fig. \ref{fig:spline_interpolation}, and $\vect{B} \in \R^{\left(P+1\right) \times \left(P+1\right)}$ is called the spline basis matrix. In this chapter, we use the Catmul-Rom (CR) spline with $P=3$, given by:
\begin{equation}
\vect{B} = \frac{1}{2} 
\begin{bmatrix}
	-1 & 3 & -3 & 1 \\
	2 & -5 & 4 & -1 \\
	-1 & 0 & 1 & 0 \\
	0 & 2 & 0 & 0
\end{bmatrix} \,.
\label{eq:catmulrom_basis_matrix}
\end{equation}
Several alternative choices are available, such as the B-spline matrix \cite{Scarpiniti2013}. Different bases give rise to alternative interpolation schemes, e.g. a spline defined by a CR basis passes through all the control points, but its second derivative is not continuous, while the opposite is true for the B-spline basis. Note that both \eqref{eq:linearoutput} and \eqref{eq:saf_output} are LIP, and can be adapted with the use of any standard linear filtering technique. Applying the chain rule, it is straightforward to compute the derivative of the SAF output with respect to the linear coefficients:
\begin{align}
\frac{\partial \varphi(s[n])}{\partial \vect{w}_n} & = \frac{\partial \varphi(s[n])}{\partial u}\cdot\frac{\partial u}{\partial s[n]}\cdot\frac{\partial s[n]}{\partial \vect{w}_n} = \nonumber \\
 & = \dot{\vect{u}} \vect{B}\vect{q}_{i,n}\left(\frac{1}{\Delta x}\right)\vect{x}_n \,,
\label{eq:saf_derivative_linear}
\end{align}
where:
\begin{equation}
\dot{\vect{u}} = \frac{\partial \vect{u}}{\partial u} = \left[ Pu^{P-1} \; (P-1)u^{P-2} \ldots 1 \; 0 \right]^T \,.
\label{eq:u_dot}
\end{equation}
Similarly, for the nonlinear part we obtain:
\begin{align}
\frac{\partial \varphi(s[n])}{\partial \vect{q}_{i,n}} & = \vect{B}^T \vect{u} \,.
\label{eq:saf_derivative_nonlinear}
\end{align}
We consider a first-order adaptation for both the linear and the nonlinear part of the SAF. Defining the error $e[n] = d[n] - y[n]$, we aim at minimizing the expected mean-squared error given by:
\begin{equation}
J(\vect{w}, \vect{q}) = \mathbb{E} \left\{ e[n]^2 \right\} \,,
\label{eq:cost_function}
\end{equation}
where $\vect{q} = \left[ q_{y,1}, \ldots, q_{y,Q} \right]^T$. As is standard approach, we approximate \eqref{eq:cost_function} with the instantaneous error given by:
\begin{equation}
\hat{J}(\vect{w}, \vect{q}) = e^2[n]  \,.
\label{eq:cost_function_instantenous}
\end{equation}
Then, we apply two simultaneous steepest-descent steps to solve the overall optimization problem:
\begin{align}
\vect{w}_{n+1} & = \vect{w}_n + \mu_w e[n]\varphi'(s[n])\vect{x}_n \,, \label{eq:lmsw} \\
\vect{q}_{i, n+1} & = \vect{q}_{i, n} + \mu_q e[n]\vect{B}^T\vect{u} \,, \label{eq:lmsq}
\end{align}
where we defined $\varphi'(s[n]) = \dot{\vect{u}} \vect{B}\vect{q}_{i,n}\left(\frac{1}{\Delta x}\right)$, and we use two possibly different step-sizes $\mu_w, \mu_q > 0$. For simplicity, we consider adaptation with constant step sizes. Additionally, note that in \eqref{eq:lmsq} we adapt only the coefficients related to the $i$th span, since it can easily be shown that $\frac{\partial \hat{J}(\vect{w}_{n}, \vect{q}_{n})}{\partial \vect{q}_{n}}$ is $0$ for all the coefficients outside the span. Convergence properties of this scheme are analyzed in a number of previous works \cite{Scarpiniti2013}. The overall algorithm is summarized in Algorithm \ref{algo:saf}. A standard way to initialize the coefficients of the spline is to consider:
\begin{equation}
q_{x,i} = q_{y,i}, \;\; i = 1, \ldots, Q \,,
\label{eq:spline_init}
\end{equation} 
such that $\varphi(s[n]) = s[n]$. Using this initialization criterion, the LMS filter can be considered as a special case of the SAF, where adaptation is restricted to the linear part, i.e. $\mu_q = 0$.

\begin{AlgorithmCustomWidth}[h]
\caption{SAF: Summary of the SAF algorithm with first-order updates.}
\label{algo:saf}
\begin{algorithmic}[1]
 \State Initialize $\vect{w}_{-1} = \delta[n], \vect{q}_0$
 \For{$n=0,1,\ldots$}
 	\State $s[n] = \vect{w}_n^T \vect{x}_n$
 	\State $u = s[n]/\Delta x - \lfloor s[n]/\Delta x \rfloor$
 	\State $i = \lfloor s[n]/\Delta x \rfloor + (Q-1)/2$
 	\State $y[n] = \vect{u}^T\vect{B}\vect{q}_{i,n}$
 	\State $e[n] = d[n] - y[n]$
 	\State $\vect{w}_{n+1} = \vect{w}_n + \mu_w e[n]\varphi'(s[n])\vect{x}_n$
 	\State $\vect{q}_{i, n+1} = \vect{q}_{i,n} + \mu_q e[n]\vect{B}^T\vect{u}$
\EndFor
\end{algorithmic}
\end{AlgorithmCustomWidth}

\section{Diffusion SAF}
\label{sec:diffusionsaf}

Consider a network model as in the previous chapters. At a generic time instant $n$, each agent receives some input/output data denoted by $\left( \vect{x}_n^{(k)}, d^{(k)}[n] \right)$, where we introduce an additional superscript $(k)$ for explicating the node dependence. We assume that streaming data at the local level is generated similarly to \eqref{eq:data_model}, according to:
\begin{equation}
d^{(k)}[n] = f_0\left( \vect{w}_0^T \vect{x}^{(k)}_n \right) + \nu^{(k)}[n] \,.
\label{eq:data_model_distributed}
\end{equation}
More in particular, we assume that $\vect{w}_0$ and $f_0(\cdot)$ are shared over the network, which is a reasonable assumption in many situations \cite{cattivelli2010diffusion,sayed2014adaptive}. Each node, however, receives input data with possibly different autocorrelation $\vect{R}^{(k)}_u = \mathrm{E}\left\{ \vect{x}^{(k)T}\vect{x}^{(k)} \right\}$, and different additive noise terms $\nu^{(k)}[n] \sim \mathcal{N}(0, \sigma_k^2)$. Additionally, we assume that the nodes have agreed beforehand on a specific spline basis matrix $\vect{B}$, and on a set of initial control points $\vect{q}_0$. Both quantities are common throughout the network. This is shown schematically in Fig. \ref{fig:diffused_saf_depiction}.

\begin{figure}
\centering
\includegraphics[width=0.7\columnwidth,keepaspectratio]{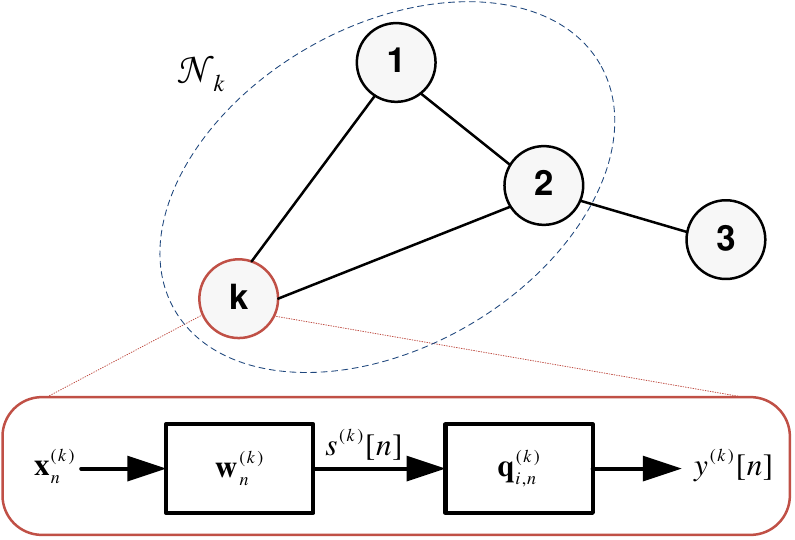}
\caption[Schematic depiction of SAF interpolation performed over a network of agents.]{Schematic depiction of SAF interpolation performed over a network of agents. Each agent is connected to a neighborhood of other agents, and at every time instant it updates a local estimate of the optimal SAF model.}
\label{fig:diffused_saf_depiction}
\end{figure}

Given these assumptions, the network objective is to find the optimal SAF parameters $(\vect{w}, \vect{q})$ such that the following global cost function is minimized:
\begin{equation}
J_{\text{glob}}(\vect{w}, \vect{q}) = \sum_{k=1}^L J_{\text{loc}}^{(k)}(\vect{w}, \vect{q}) = \sum_{k=1}^L \mathrm{E}\left\{ e^{(k)}[n]^2 \right\} \,,
\label{eq:global_cost_function}
\end{equation}
where each expectation is defined with respect to the local input statistics. Remember that the main idea of DA techniques is to interleave parallel adaptation steps with diffusion steps, where information on the current estimates are locally combined based on the mixing matrix $\vect{C}$ (see for example \cite[Section V-B]{sayed2014adaptive}). Denote by $\left( \vect{w}_{n}^{(k)}, \vect{q}_n^{(k)} \right)$ the SAF estimate of node $k$ at time-instant $n$. In the diffusion SAF (D-SAF), each node starts by diffusing its own estimate of the linear part of the SAF filter:
\begin{align}
\boldsymbol{\psi}_{n}^{(k)} & = \sum_{l \in \mathcal{N}_{k}} C_{kl} \vect{w}_{n}^{(k)} \,. \\[-\baselineskip] \tag*{w-diffusion}
\label{eq:w_diffusion}
\end{align}
Next, we can use the new weights $\boldsymbol{\psi}_{n}^{(k)}$ to compute the linear output of the filter as $s^{(k)}[n] = \boldsymbol{\psi}_{n}^{(k)T} \vect{x}_n^{(k)}$. From this, each node can identify its current span index $i$ with \eqref{eq:i} and \eqref{eq:u}. In the second phase, the nodes performs a second diffusion step over their span:
\begin{align}
\boldsymbol{\xi}^{(k)}_{i,n} = \sum_{l \in \mathcal{N}_{k}} C_{kl} \vect{q}^{(k)}_{i,n} \,. \\[-\baselineskip] \tag*{q-diffusion}
\label{eq:q_diffusion}
\end{align}
Note that the q-diffusion step requires combination of the coefficients in the span $\vect{q}^{(k)}_{i,n}$, hence its complexity is independent on the number of control points in the spline, being defined only by the spline order $P$.

Once the nodes have diffused their information, they can proceed to a standard adaptation step as in the single-agent case. In particular, the spline output given the new span is obtained as:
\begin{equation}
y^{(k)}[n] = \varphi_k(s^{(k)}[n]) = \vect{u}^T\vect{B}\boldsymbol{\xi}^{(k)}_{i,n} \,.
\end{equation}
From this, the local error is given as $e^{(k)}[n] = d^{(k)}[n] - y^{(k)}[n]$. The two gradient descent steps are then:
\begin{align}
\vect{w}_{n+1}^{(k)} & = \boldsymbol{\psi}_{n}^{(k)} + \mu_{w}^{(k)} e^{(k)}[n] \varphi'(s^{(k)}[n])\vect{x}^{(k)}_n \,, \\[-0.3\baselineskip] \label{eq:w_k_adapt} \tag*{w-adapt} \\
\vect{q}^{(k)}_{i,n+1} & = \boldsymbol{\xi}^{(k)}_{i,n} + \mu_{k}^{(k)} e^{(k)}[n]\vect{B}^T\vect{u} \,. \\[-0.3\baselineskip] \label{eq:q_k_adapt} \tag*{q-adapt}
\end{align}
where the two step sizes $\mu_{w}^{(k)}, \mu_{q}^{(k)}$ are possibly different across different agents. The overall algorithm is summarized in Algorithm \ref{algo:dsaf}. Note that in this chapter we consider a diffusion step prior to the adaptation step. In the DA literature, this is known as a combine-then-adapt (CTA) strategy \cite{takahashi2010diffusion}. This is true even if the two diffusion steps are not consecutive in Algorithm \ref{algo:dsaf}. In fact, Algorithm \ref{algo:dsaf} is equivalent to the case where the full vector $\vect{q}_{n}^{(k)}$ is exchanged before selecting the proper span. Following similar reasonings, we can easily obtain an adapt-then-combine (ATC) strategy by inverting the two steps. Additionally, similarly to what we remarked in Section \ref{sec:saf}, we note that D-LMS \cite{lopes2008diffusion} is a special case of the D-SAF, where each node initialize its nonlinearity with \eqref{eq:spline_init}, and $\mu_q^{(k)} = 0, k = 1, \ldots, L$.

\begin{AlgorithmCustomWidth}[h]
\caption{D-SAF: Summary of the D-SAF algorithm (CTA version).}
\label{algo:dsaf}
\begin{algorithmic}[1]
 \State Initialize $\vect{w}_{-1}^{(k)} = \delta[n], \vect{q}_0^{(k)}$, for $k=1,\ldots,L$
 \For{$n=0,1,\ldots$}
 	\For{$k=1,\ldots, L$}
 		\State $\boldsymbol{\psi}_{n}^{(k)} = \sum_{l \in \mathcal{N}_{k}} C_{kl} \vect{w}_{n}^{(k)}$
	 	\State $s^{(k)}[n] = \boldsymbol{\psi}_{n}^{(k)T} \vect{x}_n^{(k)}$
	 	\State $u = s^{(k)}[n]/\Delta x - \lfloor s^{(k)}[n]/\Delta x \rfloor$
	 	\State $i = \lfloor s^{(k)}[n]/\Delta x \rfloor + (Q-1)/2$
	 	\State $\boldsymbol{\xi}^{(k)}_{i,n} = \sum_{l \in \mathcal{N}_{k}} C_{kl} \vect{q}^{(k)}_{i,n}$
	 	\State $y^{(k)}[n] = \vect{u}^T\vect{B}\boldsymbol{\xi}^{(k)}_{i,n}$
	 	\State $e^{(k)}[n] = d^{(k)}[n] - y^{(k)}[n]$
	 	\State $\vect{w}_{n+1}^{(k)} = \boldsymbol{\psi}_{n}^{(k)} + \mu_{w}^{(k)} e^{(k)}[n] \varphi'(s^{(k)}[n])\vect{x}^{(k)}_n$
	 	\State $\vect{q}^{(k)}_{i,n+1} = \boldsymbol{\xi}^{(k)}_{i,n} + \mu_{k}^{(k)} e^{(k)}[n]\vect{B}^T\vect{u}$
 	\EndFor
 \EndFor
\end{algorithmic}
\end{AlgorithmCustomWidth}

\section{Experimental Setup}
\label{sec:exp_setup}

To test the proposed D-SAF, we consider network topologies with $L$ agents, whose connectivity is generated randomly, such that every pair of nodes has a $60\%$ probability of being connected. To provide sufficient diversity, we experiment with a small network with $L=10$ and a larger network with $L=30$. Data is generated according to the Wiener model in \eqref{eq:data_model_distributed}, where the optimal weights $\vect{w}_0$ are extracted randomly from a normal distribution, while the nonlinearity $f_{0}(\cdot)$ for the initial experiments is depicted in Fig. \ref{fig:nonlinearity}. This represents a mild nonlinearity. 

\begin{figure}[h]
\centering
\includegraphics[scale=0.8]{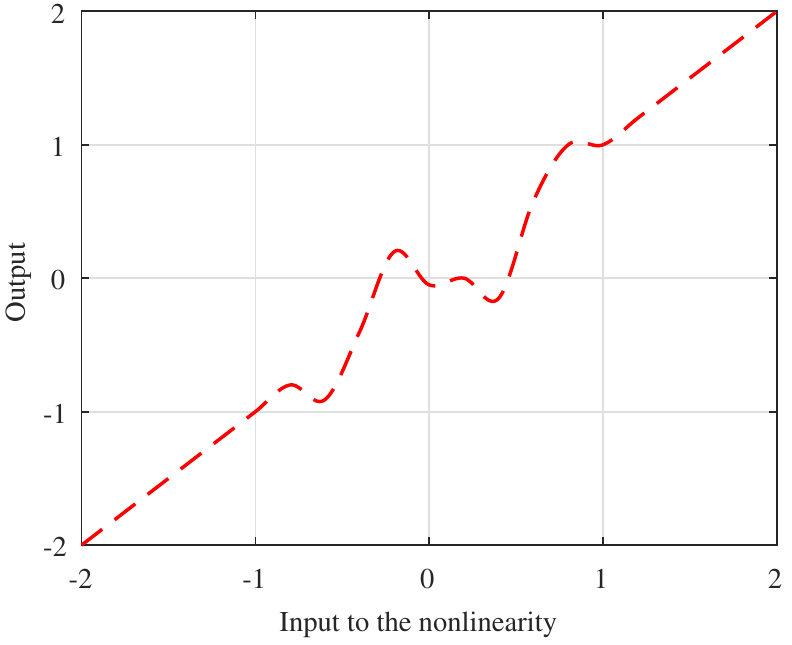}
\caption[Nonlinear distortion applied to the output signal in experiments $1$ and $2$ for testing D-SAF.]{Nonlinear distortion applied to the output signal in experiments $1$ and $2$.}
\label{fig:nonlinearity}
\end{figure}

The input signal at each node is generated following the experiments in \cite{Scarpiniti2013}, and it consists of $25000$ samples generated according to:
\begin{equation}
x_k[n] = a_k x_k[n-1] + \sqrt{1-a_k^2} \xi[n] \,,
\label{eq:input}
\end{equation}
where the correlation coefficients $a_k$ are assigned randomly at every node from an uniform probability distribution in $\left[0, 0.8\right]$, while $\xi[n]$ is a white Gaussian noise term with zero mean and unitary variance. The desired signal is then given by \eqref{eq:data_model_distributed}, where the noise variances $\sigma^{(k)}[n]$ at every node are assigned randomly in $\left[-10, -25\right]$ dB. The mixing coefficients are chosen according to the `metropolis' strategy as in previous chapters. In all experiments, knots are equispaced in $\left[-2, +2\right]$ with $\Delta x = 0.2$.

We compare D-SAF with a non-cooperative SAF (denoted as NC-SAF), which corresponds in choosing a diagonal mixing matrix $\vect{C} = \vect{I}$. Similarly, we compare with the standard D-LMS \cite{lopes2008diffusion}, and a non-cooperative LMS, denoted as NC-LMS. To average out statistical effects, experiments are repeated $15$ times, by keeping fixed the topology of the network and the optimal parameters of the system. Results are then averaged throughout the nodes.

\section{Experimental Results}
\label{sec:exp_results}

\subsection{Experiment 1 - Small Network ($L = 10$)}

\begin{figure*}[t]
\centering
\subfloat[Correlation coefficients]{\includegraphics[scale=0.7]{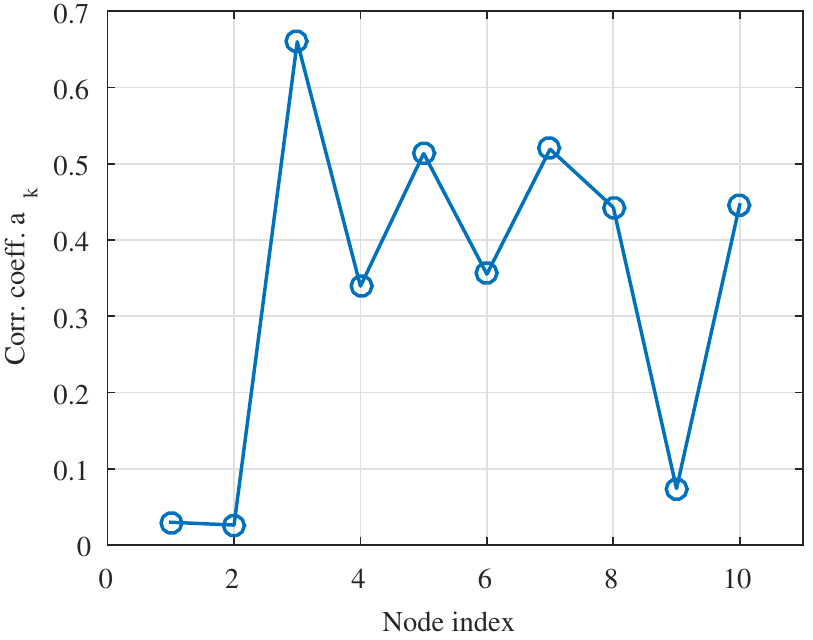}%
\label{fig:corrcoeffs_exp1}}
\hfil
\subfloat[Noise variances]{\includegraphics[scale=0.7]{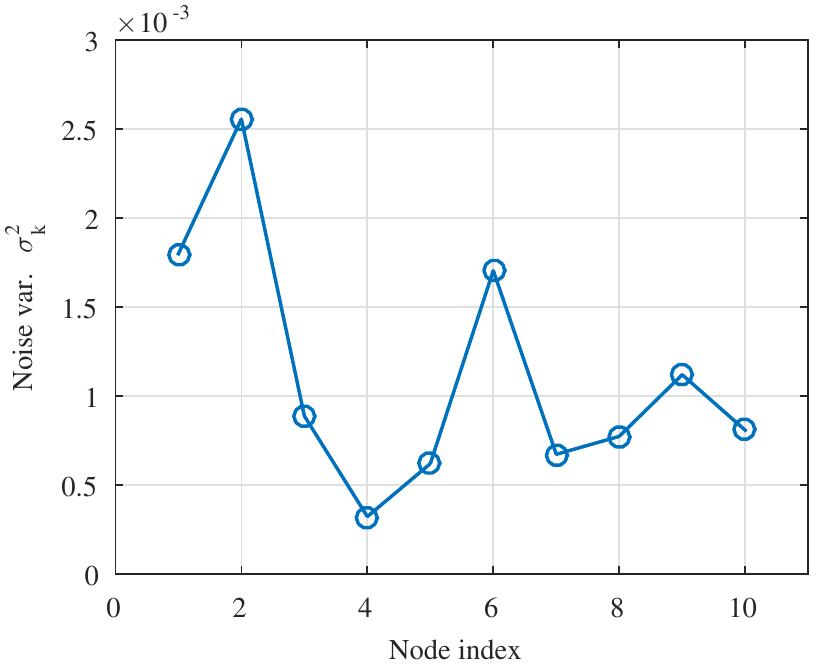}%
\label{fig:sigma_exp1}}
\hfil
\subfloat[Step sizes]{\includegraphics[scale=0.7]{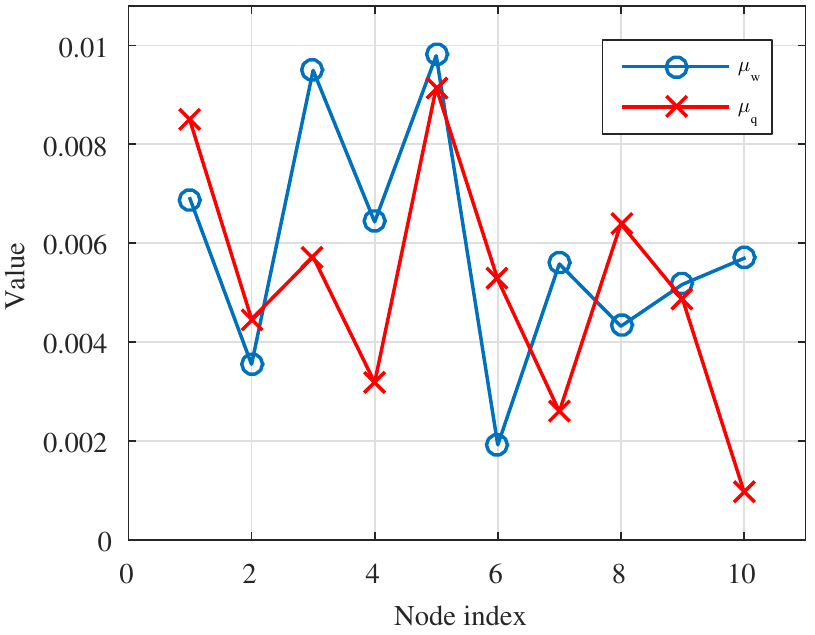}%
\label{fig:mu_exp1}}
\caption[Dataset setup for the first experiment of D-SAF.]{Dataset setup for experiment 1. (a) Correlation coefficients in \eqref{eq:input}; (b) Noise variances in \eqref{eq:data_model_distributed}; (c) Step sizes.}
\label{fig:setup_exp1}
\end{figure*}

In the first experiment, we consider a network with $L=10$. Details on the signal generation are provided in Fig. \ref{fig:setup_exp1}. In particular, the local correlation coefficients are shown in Fig. \ref{fig:corrcoeffs_exp1}, and the amount of noise variance in Fig. \ref{fig:sigma_exp1}. The step-sizes are instead given in Fig. \ref{fig:mu_exp1}. These settings allow a certain amount of variety on the network. As an example, input values at node $3$ are highly correlated, while node $2$ has the strongest amount of noise. Similarly, speed of adaptation (and consequently steady-state convergence) covers a large range of settings, as depicted in Fig. \ref{fig:mu_exp1}. The first measure of error that we consider is the mean-squared error (MSE), defined in dB as:
\begin{equation}
\text{MSE}_k[n] = 10 \log \left( d_k[n] - y_k[n] \right)^2 \,.
\label{eq:mse}
\end{equation}
Results in term of MSE are given in Fig. \ref{fig:mse_exp1}, where the proposed algorithm is shown with a solid violet line. Here and in the following figures, the MSE is computed by averaging \eqref{eq:mse} over the different nodes. 

\begin{figure}[h]
\centering
\includegraphics[scale=1]{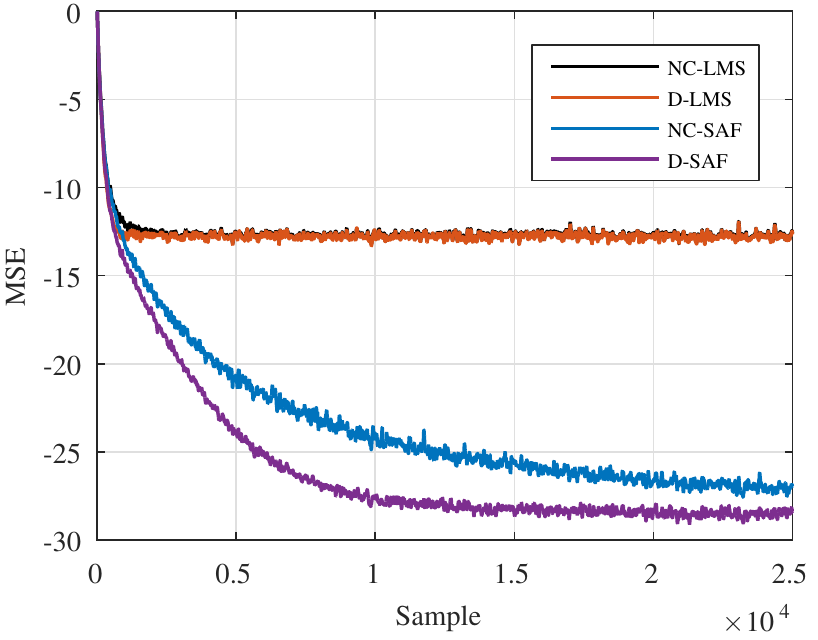}
\caption[Average MSE evolution for experiment 1 of D-SAF.]{MSE evolution for experiment 1, averaged across the nodes.}
\label{fig:mse_exp1}
\end{figure}

As expected, due to the nonlinear distortion, LMS achieves a generally poor performance, with a steady-state MSE of $-12$ dB. Additionally, there is almost no improvement when considering D-LMS compared to NC-LMS. The SAF filters are instead able to approximate extremely well the desired system. The diffusion strategy, however, provides a significant improvement in convergence time with respect to the non-cooperative version, as is evident from Fig. \ref{fig:mse_exp1}. Further clarifications on the two algorithms can be obtained by considering the linear mean-squared deviation (MSD) given by:

\begin{equation}
\text{MSD}_k^{\text{l}} = 10\log\left( \norm{\vect{w}_0 - \vect{w}^{(k)}_n} \right) \,,
\label{eq:msd_linear}
\end{equation}
and the nonlinear MSD given by:
\begin{equation}
\text{MSD}_k^{\text{nl}} = 10\log\left( \norm{\vect{q}_0 - \vect{q}^{(k)}_n} \right) \,.
\label{eq:msd_nonlinear}
\end{equation}
\noindent The overall behavior of the MSD is shown in Fig. \ref{fig:msd_exp1}. In particular, we show in Fig. \ref{fig:msd_lin_exp1} and Fig. \ref{fig:msd_nonlin_exp1} the global MSD of the network, which is obtained by averaging the local MSDs at every node. 

\begin{figure*}[h]
\centering
\subfloat[Linear MSD (global)]{\includegraphics[scale=0.8]{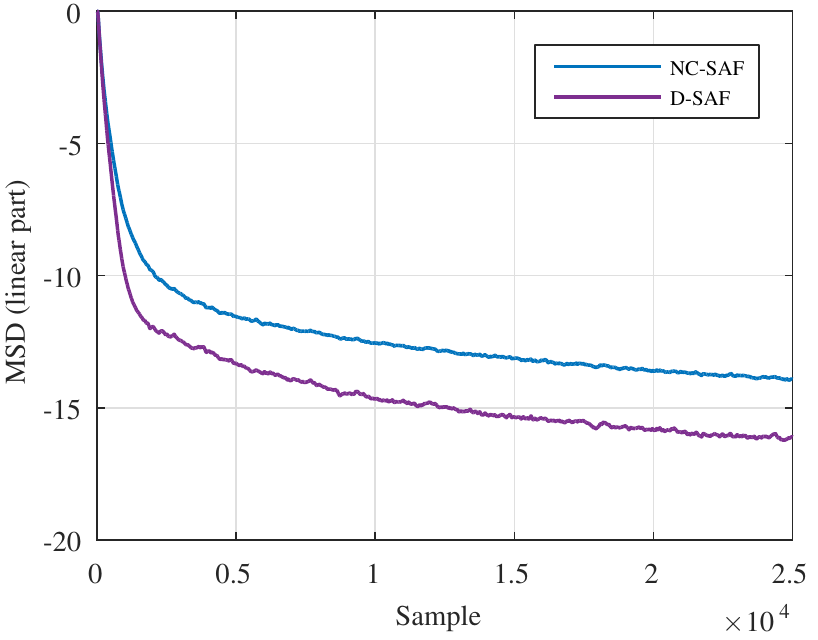}%
\label{fig:msd_lin_exp1}}
\hfil
\subfloat[nonlinear MSD (global)]{\includegraphics[scale=0.8]{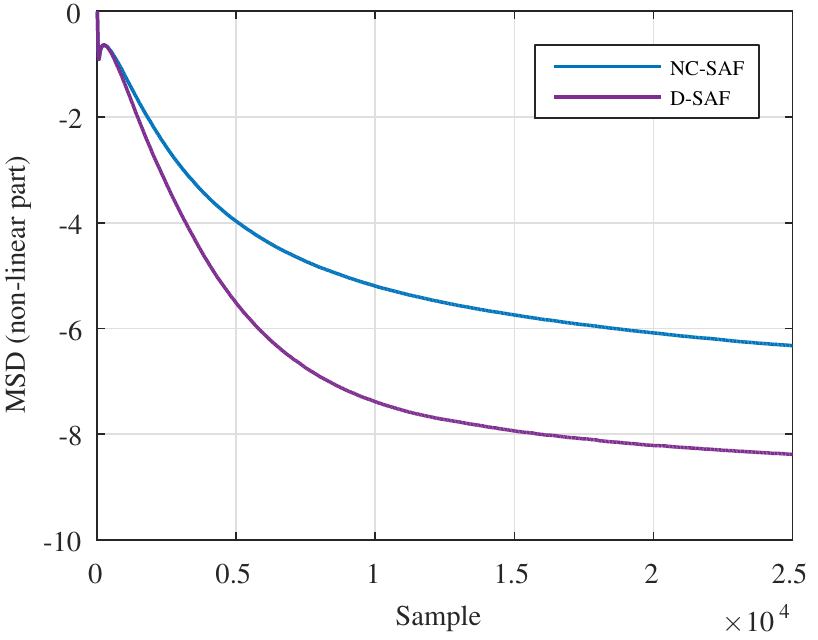}%
\label{fig:msd_nonlin_exp1}}
\vfill
\subfloat[Linear MSD (local)]{\includegraphics[scale=0.8]{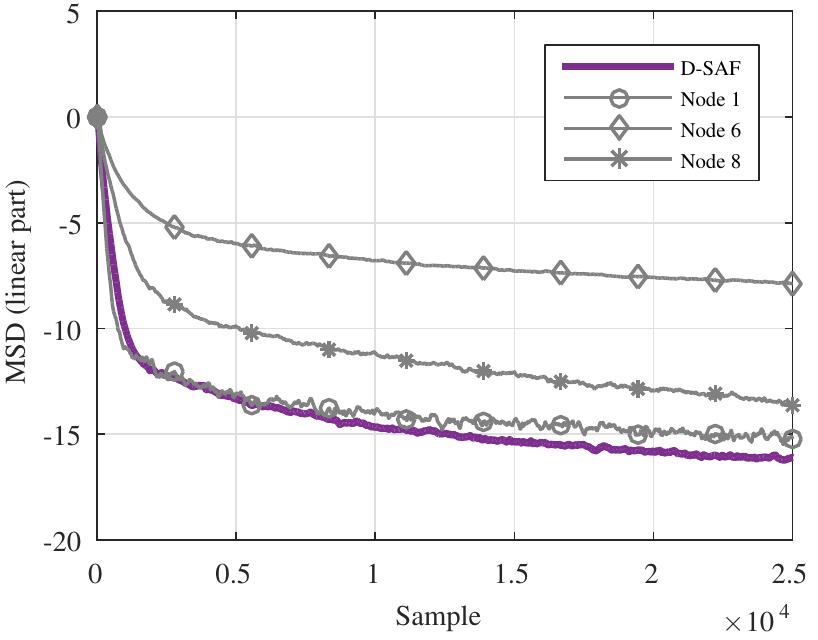}%
\label{fig:msd_lin_local}}
\hfil
\subfloat[nonlinear MSD (local)]{\includegraphics[scale=0.8]{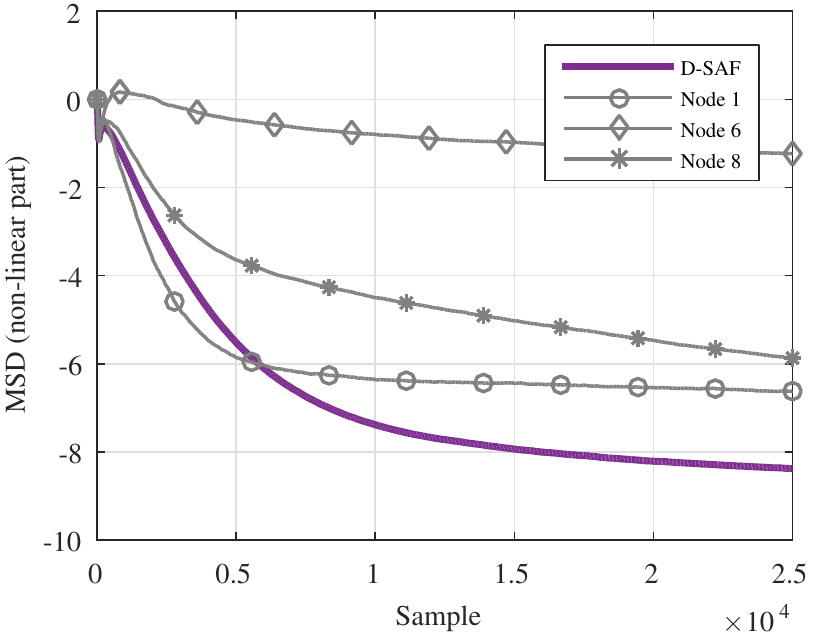}%
\label{fig:msd_nonlin_local}}
\vfill
\caption[MSD evolution for experiment $1$ of D-SAF.]{MSD evolution for experiment $1$. (a-b) Global MSD evolution. (c-d) MSD evolution for D-SAF and $3$ representative nodes running NC-SAF.}
\label{fig:msd_exp1}
\end{figure*}

It can be seen that the MSD achieved with a diffusion algorithm strongly outperform the average MSD obtained with a non-cooperative solution. Additionally, the gap in the linear and nonlinear case is similar, with a steady-state difference of $4$ dB. The reason of this difference is shown in Fig. \ref{fig:msd_lin_local} and Fig. \ref{fig:msd_nonlin_local}, where we plot the MSD evolution for D-SAF and for three representative agents running NC-SAF. It can be seen that, due to the differences in configuration, some nodes have a much slower convergence than other, such as node $6$ compared to node $1$. However, these statistical variations are successfully averaged out by the diffusion algorithm, which is able to outperform even the fastest node in the network. This is shown visually in Fig. \ref{fig:spline_exp1}, where we show the resulting nonlinear models for three representative nodes running NC-SAF, and for the nodes running D-SAF.

\begin{figure*}[h]
\centering
\subfloat{\includegraphics[scale=0.8]{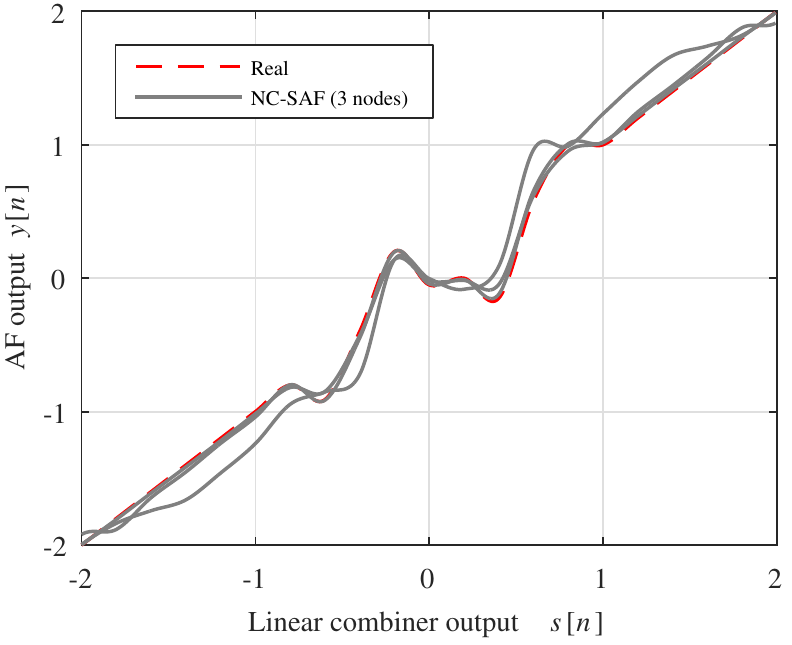}%
\label{fig:spline_nc_exp1}}
\hfil
\subfloat{\includegraphics[scale=0.8]{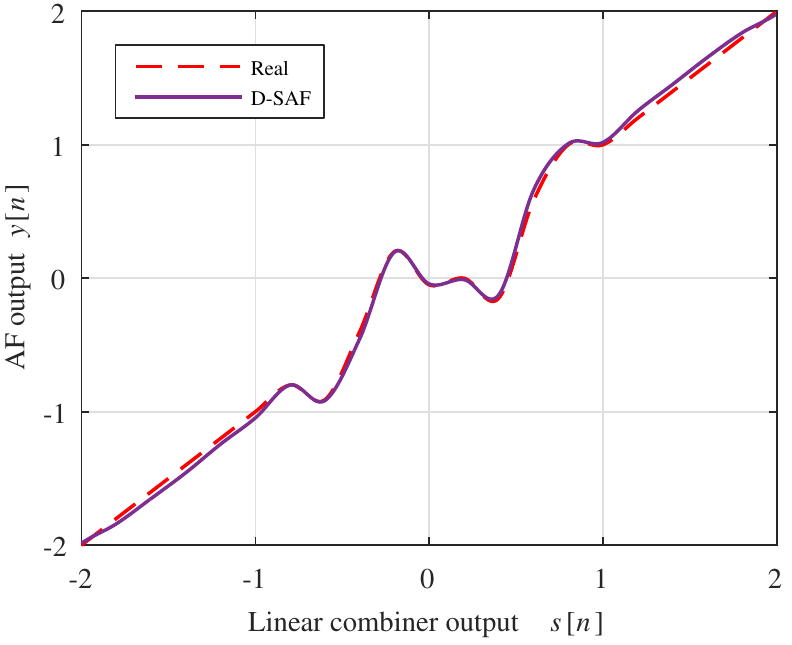}%
\label{fig:spline_diff_exp1}}
\vfill
\caption[Final estimation of the nonlinear model in experiment 1 of D-SAF.]{Final estimation of the nonlinear model in experiment 1. (a) Three representative nodes running NC-SAF. (b) Final spline of the nodes running D-SAF.}
\label{fig:spline_exp1}
\end{figure*}
\subsection{Experiment 2 - Large Network ($L=30$)}
For the second set of experiments, we consider a larger network with $L=30$ agents. Everything else is kept fixed as before, in particular, each pair of nodes in the network has a $60\%$ probability of being connected, with the only requirement that the global network is connected. The local correlation coefficients in \eqref{eq:input}, noise variances in \eqref{eq:data_model_distributed}, and local step-sizes are extracted randomly at every node using the same settings as the previous experiment. In this case, this setting provides a larger range of configurations for the different nodes. The results in term of MSE evolution are shown in Fig. \ref{fig:mse_exp2}, where the proposed D-SAF is again shown with a violet line. 
\begin{figure}[h]
\centering
\includegraphics[scale=1]{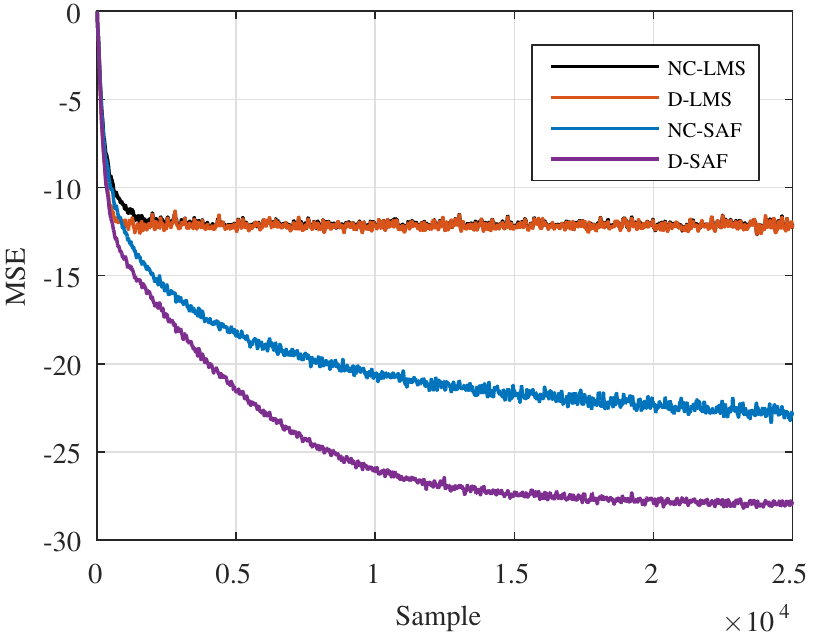}
\caption[Average MSE evolution for experiment 2 of D-SAF.]{MSE evolution for experiment 2, averaged across the nodes.}
\label{fig:mse_exp2}
\end{figure}

\noindent While the performance of NC-LMS and D-LMS are similar to those exhibited in the previous experiment, it is interesting to observe that, by increasing the amount of nodes in the network, the convergence of NC-SAF is slower in this case, to the point that the algorithm is not able to converge efficiently in the provided number of samples. D-SAF, instead, is robust to this increase in network's size, and it is able to reach almost complete convergence in less than $15000$ samples. Clearly, this is expected from the behavior of the algorithm. The larger the network, the higher the amounts of neighbors each single agent has. Thus, the diffusion steps are able to fuse more information, providing a faster convergence, as also evidenced by previous literature in DA \cite{sayed2014adaptive}. Due to this, the algorithm is able to average out the performance of isolated nodes, where convergence is not achieved. This can be seen from Fig. \ref{fig:spline_exp2}, where we plot the splines obtained from $3$ representative nodes running NC-SAF in Fig. \ref{fig:spline_nc_exp2}, and the spline resulting from D-SAF in Fig. \ref{fig:spline_diff_exp2}. In Fig. \ref{fig:spline_nc_exp2}, it is possible to see that some nodes achieve perfect convergence, while others would require a larger amount of samples. Even worse, some nodes are actually diverging from the optimal solution, due to their peculiar configuration. Despite this, D-SAF is converging globally to an optimal solution, as shown by the black line in Fig. \ref{fig:spline_diff_exp2}.
\begin{figure*}[h]
\centering
\subfloat{\includegraphics[scale=0.8]{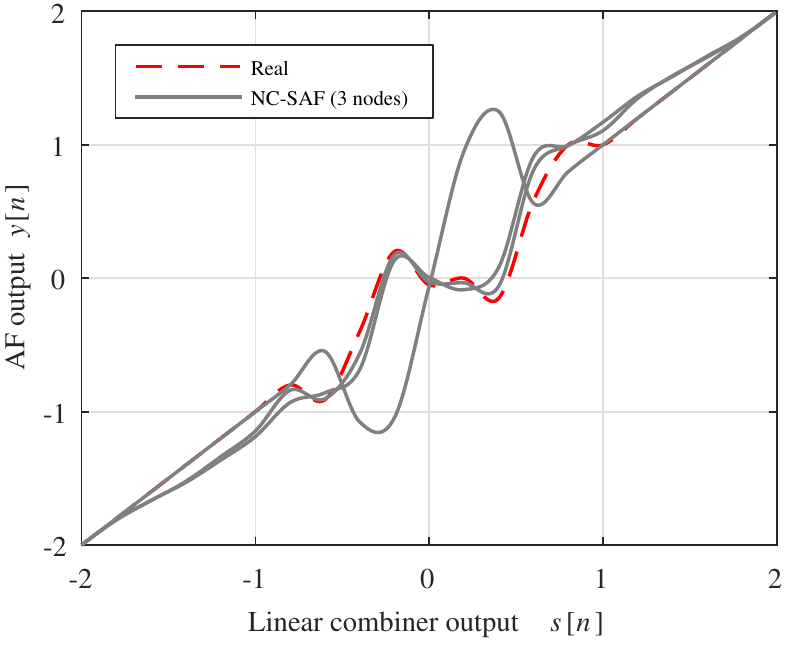}%
\label{fig:spline_nc_exp2}}
\hfil
\subfloat{\includegraphics[scale=0.8]{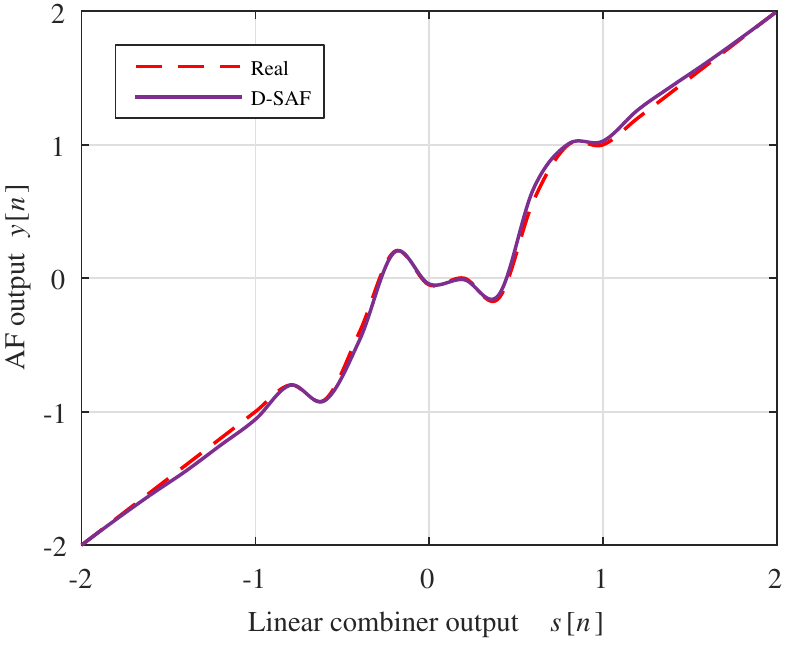}%
\label{fig:spline_diff_exp2}}
\vfill
\caption[Final estimation of the nonlinear model in experiment 2 of D-SAF.]{Final estimation of the nonlinear model in experiment 2. (a) Three representative nodes running NC-SAF. (b) Final spline of the nodes running D-SAF.}
\label{fig:spline_exp2}
\end{figure*}
\subsection{Experiment 3 - Strong nonlinearity ($L=15$)}
As a final validation, we consider an intermediate network with $L=15$, but we change the output nonlinearity in \eqref{eq:input} with the one showed in Fig. \ref{fig:spline_exp3}. This is a stronger nonlinearity, with two larger peaks. As before, the correlation coefficients in \eqref{eq:input}, the variances of the noise, and the local step-sizes are assigned randomly at every node. Results of the experiment are shown in Fig. \ref{fig:mse_exp3}. 
\begin{figure}[h]
\centering
\includegraphics[scale=1]{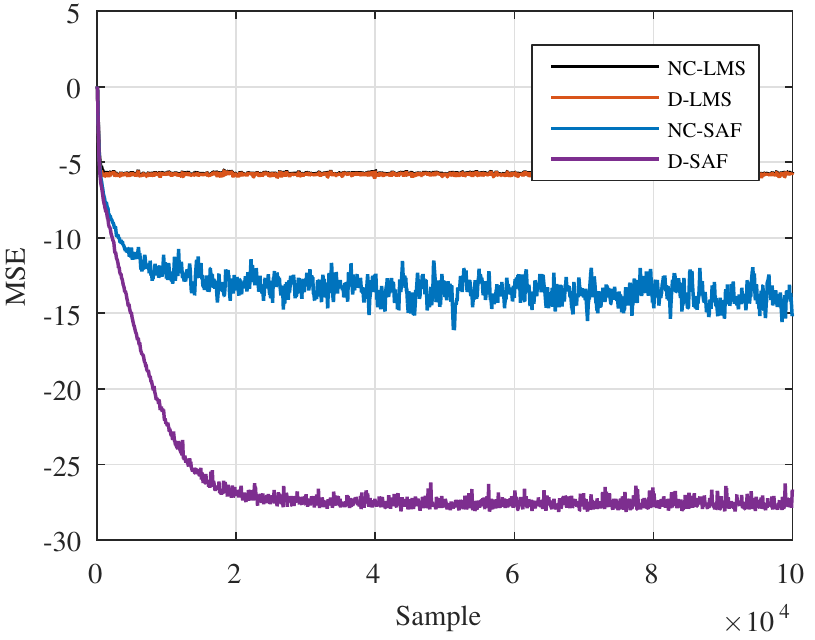}
\caption[Average MSE evolution for experiment 3 of D-SAF.]{MSE evolution for experiment 3, averaged across the nodes.}
\label{fig:mse_exp3}
\end{figure}
\begin{figure*}[h]
\centering
\subfloat{\includegraphics[scale=0.8]{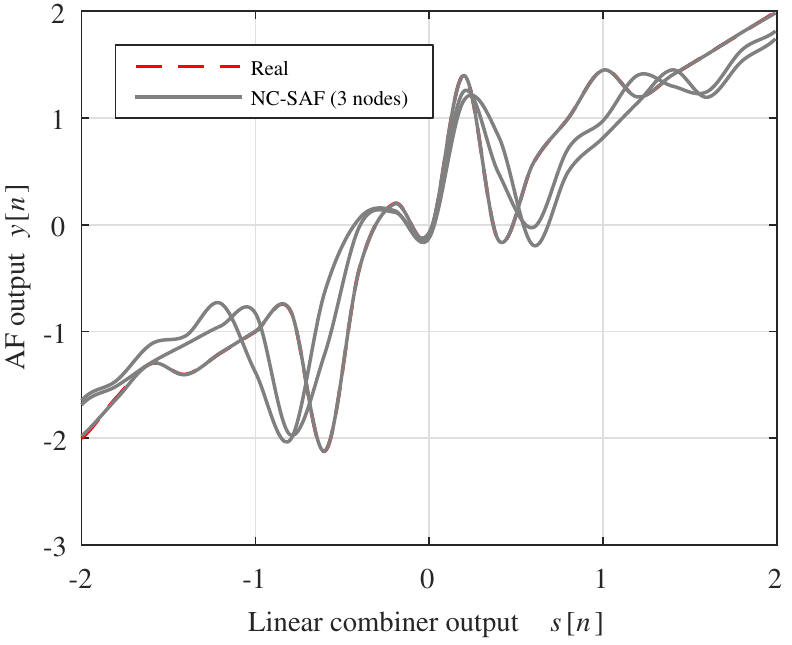}%
\label{fig:spline_nc_exp3}}
\hfil
\subfloat{\includegraphics[scale=0.8]{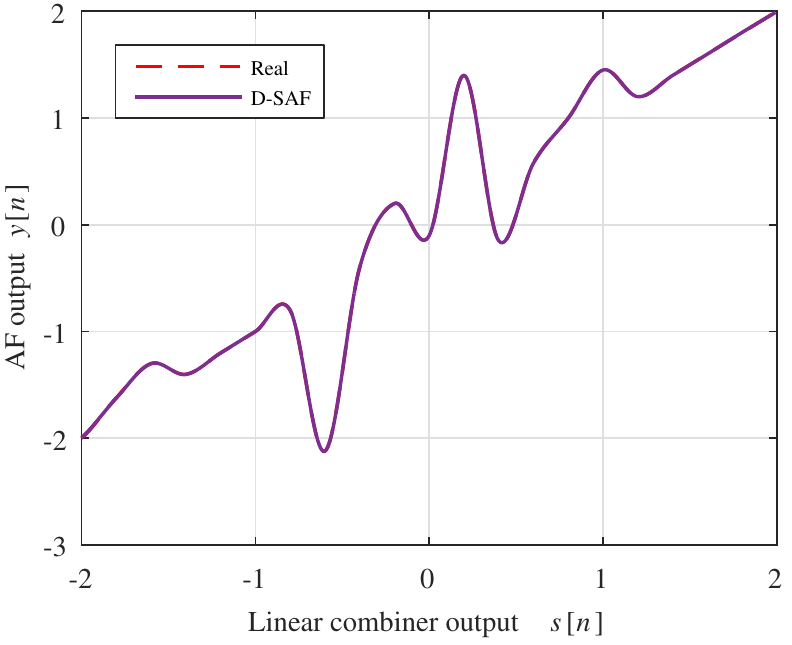}%
\label{fig:spline_diff_exp3}}
\vfill
\caption[Final estimation of the nonlinear model in experiment 3 of D-SAF.]{Final estimation of the nonlinear model in experiment 3. (a) Three representative nodes running NC-SAF. (b) Final spline of the nodes running D-SAF.}
\label{fig:spline_exp3}
\end{figure*}
Due to the increased nonlinearity, the two linear filters are performing poorly, with a steady-state MSE of $-5$ dB. Convergence is also slowed for NC-SAF, while there is now a gap of almost $10$ dB between the final MSE of the cooperative and non-cooperative versions of SAF. It is particularly interesting to observe the final nonlinearities at every node. These are shown for three representative nodes running NC-SAF in Fig. \ref{fig:spline_nc_exp3}, and for D-SAF in Fig. \ref{fig:spline_diff_exp3}. Overall, a small portion of nodes running NC-SAF is achieving a satisfactory convergence, while several of them are only achieving a moderate convergence, or no convergence at all. Despite this, and the smaller size of the network with respect to the second experiment, D-SAF is still obtaining an almost complete convergence at a global level. Overall, this shows that the algorithm is robust to variations in the network's size, local configurations, and amount of nonlinearity.

%% file: chapters/chapter11-conclusions.tex
\chapter{Conclusions and Future Works}
\label{chap:conclusions}

\lettrine[lines=2]{D}{istributed} learning has received considerable attention over the past years due to its broad real-world applications. It is
common nowadays that data must be collected, stored locally and data exchange is not allowed for specific reasons, such as technological bottlenecks or privacy concerns. In such a circumstance, it is necessary and useful to build in a decentralized fashion an ANN model. Motivated by this, throughout this thesis we have put forth multiple algorithms to such end. 

Initially, we have explored extensions to the DL setting of the well-known RVFL network (Chapters \ref{chap:dist_rvfl}-\ref{chap:dist_rvfl_vertical_partitioning}). As in the centralized case, distributed RVFL networks are able to provide strong nonlinear modeling capabilities, while at the same time allowing for a fast and simple set of training algorithms, which are fundamentally framed in the linear regression literature. Thus, they provide a good compromise between a linear model and more complex nonlinear ANNs, such as distributed SVMs \cite{forero2010consensus}.

The successive chapters have considered the more complex problem of distributed training in the presence of labeled and unlabeled data, thus extending the theory of SSL \cite{Chapelle2006}. This is a relatively new problem in the literature, with a large set of possible real world applications. In this sense, the two distributed models explored in Chapters \ref{chap:dist_ssl} and \ref{chap:dist_ssl_next} are only initial explorations of a field which can potentially reveal much promise.

Finally, in the last part of the thesis we have been concerned with learning from time-varying data. Although this is a well-known setting, both of the algorithms that we presented are relatively novel. Indeed, Chapter \ref{chap:dist_esn} has introduced one of the first available algorithms for training recurrent networks, while the diffusion SAF in Chapter \ref{chap:dist_saf} can be seen as a general nonlinear extension of the much celebrated D-LMS \cite{sayed2014adaptive}.

Below we provide a set of possible future lines of research, which refer to specific portions of the thesis, along with the main content of each chapter.
\begin{itemize}
\item In \textbf{Chapters \ref{chap:dist_rvfl}-\ref{chap:dist_rvfl_vertical_partitioning}} we have detailed distributed algorithms for learning a RVFL network, in the case of batch and online learning, both for HP and VP partitioned data. In Chapter \ref{chap:dist_rvfl_sequential}, in particular, we have focused on the application to multiple distributed music classification tasks, including genre and artist recognition. These problems arise frequently in real-world scenarios, including P2P and mobile networks. Our experimental results show that the proposed algorithms can be efficiently applied in these situations, and compares favorably with a centralized solution in terms of accuracy and speed. Clearly, the algorithms can be successfully applied to distributed learning problems laying outside this specific applicative domain, particularly in real-world big data scenarios. Moreover, although in Chapter \ref{chap:dist_rvfl_sequential} we have focused on local updates based on the BRLS algorithm, nothing prevents the framework from being used with different rules, including efficient stochastic gradient descent updates. Similar considerations also apply for Chapter \ref{chap:dist_rvfl} and Chapter \ref{chap:dist_rvfl_vertical_partitioning}.
\item In \textbf{Chapter \ref{chap:dist_ssl}} we have proposed a totally decentralized algorithm for SSL in the framework of
MR. The core of our proposal is constituted by a distributed protocol designed to compute the Laplacian matrix. Our experimental results show that, also in this case, the proposed algorithm is able to match efficiently the performance of a centralized model built on the overall training set. Although we have focused on a particular algorithm belonging to MR, namely LapKRR, the framework is easily applicable to additional algorithms, including the laplacian SVM (LapSVM) \cite{belkin2006manifold}, and others. Moreover, extensions beyond MR are possible, i.e. to all the methods that encode information in the form of a matrix of pairwise distances, such as spectral dimensionality reduction, spectral clustering, and so on. In the case of kernels that directly depend on the dot product between patterns (e.g. the polynomial one), particular care must be taken in designing appropriate privacy-preserving protocols for distributed margin computation \cite{shi2012margin}, an aspect which is left to future investigations. Currently, the main limit of our algorithm is the computation time required by the distributed algorithm for completing the Laplacian matrix. This is due to a basic implementation of the two optimization algorithms. In this
sense, in future works we intend to improve the distributed algorithm to achieve better computational performance. Examples of possible modifications include adaptive strategies for the choice of the step-size, as well as early stopping protocols.
\item Next, in \textbf{Chapter \ref{chap:dist_ssl_next}} we have solved the problem of distributed SSL via another type of semi-supervised SVM, framed in the transductive literature. Particularly, we have leveraged over recent advances on distributed non-convex optimization, in order to provide two flexible mechanisms with a different balance in computational requirements and speed of convergence. A natural extension would be to consider different
semi-supervised techniques to be extended to the distributed setting, particularly among those developed for the S$^3$VM \cite{chapelle2008optimization}.
\item In \textbf{Chapter \ref{chap:dist_esn}} we have introduced a decentralized algorithm for training an ESN.  Experimental results on multiple benchmarks, related to non-linear system identification and chaotic time-series prediction, demonstrated that it is able to efficiently track a purely centralized solution, while at the same time imposing a small computational overhead in terms of vector$-$matrix operations requested to the single node. This represents a first step towards the development of data-distributed strategies for general RNNs, which would provide invaluable tools in real world applications. Future lines of research involve considering different optimization procedures with respect to ADMM, or more flexible DAC procedures. More in general, it is possible to consider other distributed training criteria beyond ridge regression and LASSO (such as training via a support vector algorithm) to be implemented in a distributed fashion. Finally, ESN are known to perform worse for problems that require a long memory. In this case, it is necessary to devise distributed strategies for other classes of recurrent networks, such as LSTM architectures \cite{hochreiter1997long,monner2012generalized}.
\item Finally, in \textbf{Chapter \ref{chap:dist_saf}} we have investigated a distributed algorithm for adapting a particular class of nonlinear filters, called SAF, using the general framework of DA. The algorithm inherits the properties of SAFs in the centralized case, namely, it allows for a flexible nonlinear estimation of the underlying function, with a relatively small increase in computational complexity. In particular, the algorithm can be implemented with two diffusion steps, and two gradient descent steps, thus requiring in average only twice as much computations as the standard D-LMS. Our experimental results show that D-SAF is able to efficiently learn hard nonlinearities, with a definite increase in convergence time with respect to a non-cooperative implementation. In the respective chapter, we have focused on a first-order adaptation algorithm, with CTA combiners. In future works, we plan to extend the D-SAF algorithm to the case of second-order adaptation with Hessian information, ATC combiners, and asynchronous networks. Additionally, we plan to investigate diffusion protocols for more general architectures, including Hammerstein and IIR spline filters.
\end{itemize}
A few general considerations on the thesis are also worth mentioning here:
\begin{itemize}
\item \textbf{Fixed topology}: for simplicity, in this thesis we have supposed that the network of agents is fixed, and connectivity is known at the agent level. This is not a necessary condition (indeed, many practical applications might require time-varying connectivities), and work along this sense is planned in the near future. Indeed, many of the tools employed throughout the thesis, e.g. ADMM, already possess extensions to this scenario, which can in principle be applied to the problems considered here.
\item \textbf{Synchronization}: similar considerations apply for the problem of having synchronized updates, requiring a general mechanism for coordinating the agents. Investigations along this line can start by considering asynchronous versions of SGD \cite{ormandi2011asynchronous}, or DA \cite{zhao2015asynchronous1}.
\item \textbf{Specific ML fields}: another important aspect is that, similarly to SSL, many subfields of ML remain to be extended to the distributed setting. As an example, there is limited literature for distributed implementation of \textit{active learning} strategies \cite{tong2002support}, where agents are allowed to request a set of labels on items that they assume to be interesting. This can potentially reduce drastically the training time and the amount of communication overhead.
\item \textbf{Multilayer networks}: we have investigated distributed methods only for ANN having at most one hidden layer of nonlinearities, which are known as `shallow' in the current terminology \cite{Schmidhuber2014}. Indeed, we saw in Section \ref{sec:distributed_mlp} that investigations on distributed deep neural networks have been limited. This is due to the large number of parameters to be exchanged, and to the resulting non-convex optimization problem. Both these problems require additional investigations in order to be properly addressed.
\item \textbf{Additional distributed techniques}: finally, we expect that techniques originally developed for distributed signal processing and distributed AI might be applied to the problem of DL, resulting in beneficial effects in term of in-network communication and/or computational requirements. This is the case, for example, of message censoring \cite{tay2007asymptotic}, a set of techniques allowing each individual agent to decide whether to take a specific measurement and propagate it over the network.
\end{itemize}

%% file: chapters/appendixA-graph_theory.tex
\chapter{Elements of Graph Theory}
\chaptermark{Algebraic Graph Theory}
\label{app:graph_theory}

\section{Algebraic graph theory}

Consider a graph composed by $L$ nodes, whose connectivity is fixed and known in advance. Mathematically, this graph can be represented by the so-called adjacency matrix $\vect{A} \in \left\{0,1\right\}^{L \times L}$, defined as:
\begin{equation}
A_{ij} = %
\begin{cases}
	1 & \text{ if node } i \text{ is connected to node } j \\
	0 & \text{ otherwise}
\end{cases} \;.
\label{appA:eq:adjacency_matric}
\end{equation}
The symbol $\mathcal{N}_k$ denotes the exclusive neighborhood of node $k$, i.e. the set of nodes connected to $k$, with the exclusion of $k$ itself. In this thesis, we are concerned with graphs which are undirected, meaning that $\vect{A}$ is symmetric, and connected, meaning that each node can be reached from every other node in a finite sequence of steps. Additionally, we suppose that there are no self-loops. We can define the degree $d_k$ of node $k$ as the number of nodes which are connected to it:
\begin{equation}
d_k = |\mathcal{N}_k| = \sum_{l=1}^L A_{kl} \;.
\label{appA:eq:degree}
\end{equation}
The degree $d$ of the network is defined as the maximum degree of its composing nodes:
\begin{equation}
d = \max_{k = 1, \ldots, L} d_k \;.
\label{appA:eq:max_degree}
\end{equation}
The degree matrix $\vect{D} \in \mathbb{N}^{L \times L}$ is then defined as:
\begin{equation}
\vect{D} = \text{diag}\left\{ d_1, \ldots, d_L \right\} \;,
\label{eq:degree_matrix}
\end{equation}
where $\text{diag}\left\{\cdot\right\}$ constructs a diagonal matrix from its arguments. Lastly, the Laplacian matrix $\vect{L} \in \mathbb{Z}^{L \times L}$ is defined as:
\begin{equation}
\vect{L} = \vect{D} - \vect{A} \;.
\label{appA:eq:laplacian_matrix}
\end{equation}
From the previous definitions, we obtain:
\begin{equation}
L_{ij} = %
\begin{cases}
	d_i & \text{ if } i = j \\
	-1  & \text{ if } i \in \mathcal{N}_j \\
	0   & \text{ otherwise }
\end{cases} \;.
\label{appA:laplacian_matrix_entry}
\end{equation}
It is known that an analysis of the Laplacian matrix allows to derive multiple important properties of the underlying graph. As an example, it can be shown that $\lambda_0(\vect{L}) = 0$, while the second-smallest eigenvalue is directly related to the connectivity of the graph itself \citep{newman2010networks}. In Chapter \ref{chap:dist_ssl} we make use of a variant of $\vect{L}$, called the normalized Laplacian matrix and defined as:
\begin{equation}
\vect{L}' = \vect{D}^{-\frac{1}{2}}\vect{L}\vect{D}^{-\frac{1}{2}} \;.
\label{appA:eq:normalized_laplacian_matrix}
\end{equation}
If follows straightforwardly that:
\begin{equation}
L'_{ij} = %
\begin{cases}
	1 & \text{ if } i = j \\
	\displaystyle - \frac{1}{\sqrt{d_id_j}} & \text{ if } i \in \mathcal{N}_j \\
	0 & \text{ otherwise }
\end{cases} \;.
\label{appA:eq:normalized_laplacian_matrix_enty}
\end{equation}

\section{Decentralized average consensus}
\label{appA:sec:consensus}

Suppose now that the nodes in the graph represent agents in a physical network. Additionally, each of them has access to a measurement vector $\vect{m}_k \in \R^S$. The task is for each of them to compute the global average given by:
\begin{equation}
\vect{\hat{m}} = \frac{1}{L} \sum_{k=1}^L \vect{m}_k \,.
\label{appA:eq:average}
\end{equation}

\noindent For generality, however, we allow every node to communicate only with its direct neighbors. With respect to the categorization of Section \ref{sec:categorization_dl_algorithms}, this is denoted as one-hop communication. DAC is an iterative network protocol to compute the global average (or, equivalently, sum) with respect to the local measurement vectors, requiring only local communications between them \citep{barbarossa2013distributed,olfati2007consensus,Xiao2007}. Its simplicity makes it suitable for implementation even in the most basic networks, such as robot swarms. Each agent initializes its estimate of the global average as $\vect{m}_k[0] = \vect{m}_k$. Then, at a generic iteration $n$, the local DAC update is given by:
\begin{equation}
\vect{m}_k[n] = \sum_{j = 1}^L C_{kj} \vect{m}_j[n-1] \,,
\label{eq:consensus}
\end{equation}
where the weight $C_{kj}$ is a real-valued scalar denoting the confidence that the $k$th node has with respect to the information coming from the $j$th node. By reorganizing these weights in a $L \times L$ connectivity matrix $\vect{C}$, and defining:
\begin{equation}
\vect{M}[n] = \left[ \vect{m}_1[n] \ldots \vect{m}_L[n] \right] \,,
\end{equation}
Eq. \eqref{eq:consensus} can be rewritten more compactly as:
\begin{equation}
\vect{M}[n] = \vect{C}\vect{M}[n-1] \,.
\label{eq:consensus_global_update}
\end{equation}

\noindent If the weights of the connectivity matrix $\vect{C}$ are chosen appropriately, this recursive procedure converges to the global average given by Eq. \eqref{appA:eq:average} \citep{olfati2007consensus}:
\begin{equation}
\lim_{n \rightarrow +\infty} \vect{m}_k[n] = \frac{1}{L} \sum_{k=1}^L \vect{m}_k[0] = \hat{\vect{m}}, \, k =  1, 2,\dots, L \,.
\label{eq:consensus_convergence}
\end{equation}
\noindent Practically, the procedure can be stopped after a certain predefined number of iterations is reached, or when the norm of the update is smaller than a certain user-defined threshold $\delta$:
\begin{equation}
\Bigl\lVert\vect{m}_k[n] - \vect{m}_k[n-1]\Bigr\rVert_2^2 < \delta, \; k =  1, 2,\dots, L \,.
\label{eq:consensus_stopping_criterion}
\end{equation}
\noindent In the case of undirected, connected networks, convergence is guaranteed provided that the connectivity matrix $\vect{C}$ respects the following properties:
\begin{equation}
\vect{C}\cdot\vect{1} = \vect{1} \,,
\end{equation}
\begin{equation}
\rho\left(\vect{C} - \frac{\vect{1}\cdot\vect{1}^T}{L}\right)  < 1 \,.
\end{equation}
\noindent A simple way of ensuring this is given by choosing the so-called `max-degree' weights \citep{olfati2007consensus}:
\begin{equation}
C_{kj} = \begin{cases}
\frac{1}{d + 1} & \text{ if $k$ is connected to $j$ } \\
1 - \frac{d_k}{d + 1} & \text{ if } k = j \\
0 & \text{ otherwise}
\end{cases}
\label{eq:max_degree_weights} \,,
\end{equation}
\noindent In practice, many variations on this standard procedure can be implemented to increase the convergence rate, such as the `definite consensus' \citep{georgopoulos2014distributed}, or the strategy introduced in \citep{sardellitti2010fast}. In this thesis, Eq. \eqref{eq:max_degree_weights} is used for choosing $\vect{C}$ unless otherwise specified. Other strategies are explored in Section \ref{chap5:sec:dac_comparison}.

%% file: chapters/appendixB-software.tex
\chapter{Software Libraries}
\chaptermark{Software Libraries}
\label{app:software}

In this appendix, we present the open-source software libraries developed during the course of the PhD. The libraries can be used to replicate most of the experiments and simulations presented in the previous chapters. All of them are implemented in the MATLAB environment.

\section{Lynx MATLAB Toolbox (Chapters \ref{chap:dist_rvfl}-\ref{chap:dist_rvfl_vertical_partitioning})}

Lynx is a research-oriented MATLAB toolbox, to provide a simple environment for performing large-scale comparisons of SL algorithms.\footnote{\url{https://github.com/ispamm/Lynx-Toolbox}} Basically, the details of a comparison, in terms of algorithms, datasets, etc., can be specified in a human understandable configuration file, which is loaded at runtime by the toolbox. In this way, it is possible to abstract the elements of the simulation from the actual code, and to repeat easily previously defined experiments. An example of configuration file is provided below, where two different algorithms (an SVM and a RVFL network) are compared on a well-known UCI benchmark.
\begin{lstlisting}[caption=Demo code for the Lynx Toolbox]
% Add the SVM to the simulation
add_model('SVM', 'Support Vector Machine', @SupportVectorMachine);

% Add the RVFL network
add_model('RVFL', 'Random Vector Network', @RandomVectorFunctionalLink);

% Add a dataset
add_dataset('Y', 'Yacht', 'uci_yacht');
\end{lstlisting}
Inside the toolbox, we implemented a set of utilities in order to simplify the development of distributed algorithms. Specifically, the toolbox allows to develop additional `features', which are objects that perform specific actions during the course of a simulation. We implemented a \verb|InitializeTopology()| feature, which takes care of partitioning the training data evenly across a network of agents, and provides the algorithms with a set of network specific functions, such as DAC protocols. Below is an example of enabling this feature in a configuration file:
\begin{lstlisting}[caption=Example of data partitioning]
% Define a network topology (in this case, a randomly generated one with 20 agents)
r = RandomTopology(20, 0.2);

% Initialize the feature (with a set of user-specified flags)
feat = InitializeTopology(r, 'disable_parallel', 'distribute_data', 'disable_plot');

% Add the feature to the simulation
add_feature(feat);
\end{lstlisting}
Due to the way in which the toolbox is structured, it is possible to combine distributed and non-distributed algorithms in the same simulation, leaving to the software the task of choosing whether to partition or not the data, and to collect the results from the different agents in the former case. The configuration files for Chapter \ref{chap:dist_rvfl} and Chapter \ref{chap:dist_rvfl_vertical_partitioning} are available on the author's website,\footnote{\url{http://ispac.diet.uniroma1.it/scardapane/software/code/}.} together with a set of additional examples of usage.

\section{Additional software implementations}

\subsection{Distributed LapKRR (Chapter \ref{chap:dist_ssl})}

The code for this chapter is available on BitBucket.\footnote{\url{https://bitbucket.org/robertofierimonte/distributed-semisupervised-code/}} The network utilities (e.g. random graph generation) are adapted from the Lynx toolbox (see previous section). Each set of experiments can be repeated by running the corresponding script in the `Scripts' folder. Specifically, there are three scripts, which can be used to replicate the experiments on EDM completion, distributed SSL, and privacy preservation, respectively.

\subsection{Distributed S$^3$VM (Chapter \ref{chap:dist_ssl_next})}

Similarly to the previous chapter, the code has been released on BitBucket,\footnote{\url{https://bitbucket.org/robertofierimonte/code-distributed-s3vm}} with some adaptations from the Lynx toolbox in terms of network utilities. With respect to the library for distributed LapKRR, the code has been designed in a more flexible fashion, as it allows to define a variable number of centralized and distributed algorithms to be compared in the \verb|test_script.m| file:

\clearpage

\begin{lstlisting}[caption=Definition of semi-supervised distributed algorithms]
% Define a centralized RBF and a centralized Nabla-S3VM
centralized_algorithms = {...
    CentralizedSVM('RBF-SVM', 1, '2', ''), ...
    NablaS3VM('NS3VM (GD)', 1, 1, 5)
};

% Define a distributed Nabla-S3VM
distributed_algorithms = {...
    DiffusionNablaS3VM('D-NS3VM', 1, 1, 5)
};
\end{lstlisting}
New algorithms can be defined by extending the abstract class \verb|LearningAlgorithm|.

\subsection{Distributed ESN (Chapter \ref{chap:dist_esn})}

This code is released on a different package on BitBucket,\footnote{\url{https://bitbucket.org/ispamm/distributed-esn}}, following the general ideas detailed above. Specifically, configuration and execution are divided in two different scripts, which can be easily customized. The ESN is implemented using a set of functions adapted from the Simple ESN toolbox developed by the research group of H. Jaeger.\footnote{\url{http://organic.elis.ugent.be/node/129}}

\subsection{Diffusion Spline Filtering (Chapter \ref{chap:dist_saf})}

This package has been developed in order to provide an effective testing ground for distributed filtering applications.\footnote{\url{https://bitbucket.org/ispamm/diffusion-spline-filtering}} As for the previous libraries, it is possible to declare dynamically new distributed filters to be tested, as long as they derive correctly from the base abstract class \verb|DiffusionFilter|. 

%% file: chapters/acknowledgments.tex
\chapter*{Acknowledgments}
\addcontentsline{toc}{chapter}{\bf Acknowledgments}
\markboth{Acknowledgments}{Acknowledgments}

\setlength{\parskip}{1em}

Throughout these years, I have had the pleasure of working with a large number of people, all of whom have taught me something. I am indebted to all of them, and this thesis, in its smallness, is dedicated to them. Needless to say, this thesis is also dedicated to those that have been the closest to me: family and loved ones.

\hangpara{1em}{1} %
To start with, I would like to express my gratitude to my supervisor, Prof. Aurelio Uncini, for its never-ending support. The same gratitude also extends to current and past members of the ISPAMM research group, including (in strict alphabetical order): Andrea Alemanno, Danilo Comminiello, Francesca Ortolani, Prof. Raffaele Parisi, and Michele Scarpiniti. It has been a real pleasure working with all of you so far.

\hangpara{1em}{1} %
I am indebted to Prof. Massimo Panella for giving me the opportunity of spending a few months in this strange land which is Australia. My most sincere gratitude goes to my Australian supervisor, Prof. Dianhui Wang, for his warmth and inspiration. Another thanks goes to anyone, and particularly his students, who has welcomed me and cheered me during my stay there.

\hangpara{1em}{1} %
During my PhD program, I have had the possibility of collaborating with many researchers from my department. Thus, I would like to thank, in no particular order, Prof. Antonello Rizzi, Filippo Bianchi, Marta Bucciarelli, Prof. Fabiola Colone, Andrea Proietti, Paolo Di Lorenzo, Luca Liparulo and Rosa Altilio. Most of all, I thank Prof. Sergio Barbarossa for introducing me with enthusiasm to many topics on distributed optimization, that are used extensively here.

\hangpara{1em}{1} %
Thanks to the students that I have had the pleasure of partially supervising, including Roberto Fierimonte, Valentina Ciccarelli, Marco Biagi and Gabriele Medeot.

\hangpara{1em}{1} %
Thanks to Prof. Giandomenico Boffi and Prof. Carlo Cirotto for having organized the wonderful SEFIR schools in Perugia, and to Prof. Giovanni Iacovitti and Prof. Giulio Iannello for allowing me to participate there twice. Of course, I thank all those that I have met there, including in particular Prof. Flavio Keller.

\hangpara{1em}{1} %
Before starting my PhD program, I had a fantastic yearlong work experience. Many of the colleagues I met there have remained my friends, and among them, I would like to particularly thank Fernando Nigro, Alessandra Piccolo and Ilaria Piccolo.

\hangpara{1em}{1} %
I thank Prof. Asim Roy from Arizona State University and Prof. Plamen Angelov from Lancaster University for giving me a great opportunity to see close-hand how an international conference is organized. Even if I could not participate in the end, I have had a great experience, and I thank my `colleagues' Bill Howell, Teng Teck-Hou, and José Iglesias.

\hangpara{1em}{1} %
Thanks are also in order to Prof. Amir Hussain of Stirling University. As an honorary fellow of his research group, I hope in a long and fruitful collaboration in the future.

\hangpara{1em}{1} %
Thanks to Antonella Blasetti for giving me space to talk about machine learning in her fantastic events as head of the Lazio-Abruzzo Google Developer Group (LAB-GDG). A general thanks goes to all the members of the LAB-GDG.

\hangpara{1em}{1} %
A few final thanks to Prof. Stefano Squartini, for his insightful discussions, to Dr. Steven van Vaerenbergh for his advice and help on the Lynx MATLAB toolbox, and to the awesome staff at the $2014$ IEEE International Conference on Acoustics, Speech, and Signal Processing.

\setlength{\parskip}{0em}